\documentclass[twoside,11pt]{article}

    \usepackage[accepted]{ewrl_2022}

\usepackage{lipsum}  

\ewrlheading{15}{2022}{September 2022, Milan, Italy}{Orin Levy and Yishay Mansour}

\usepackage{algorithm}
\usepackage[noend]{algpseudocode}
\usepackage{float}

\usepackage{microtype}
\usepackage{graphicx}
\usepackage{subfigure}
\usepackage{booktabs} 
\usepackage{hyperref}

\usepackage{amsmath}
\usepackage{amssymb}
\usepackage{mathtools}
\usepackage{wrapfig}

\usepackage[capitalize,noabbrev]{cleveref}

\ShortHeadings{Learning Efficiently Function Approximation for Contextual MDP}{Orin Levy and  Yishay Mansour}
\firstpageno{1}

\begin{document}

\title{Learning Efficiently Function Approximation for Contextual MDP}

 \author{\name Orin Levy \email orinlevy@mail.tau.ac.il\\
       Tel Aviv University\\
       Israel
       \AND
       \name Yishay Mansour \email mansour.yishay@gmail.com\\
        Tel Aviv University and Google Research\\
       Israel}

\maketitle

\begin{abstract}
    We study learning contextual MDPs using a function approximation for both the rewards and the dynamics.
    We consider both the case that the dynamics dependent or independent of the context. For both models we derive polynomial sample and time complexity (assuming an efficient ERM oracle). Our methodology gives a general reduction from learning contextual MDP to supervised learning.
\end{abstract}

\begin{keywords}
  Reinforcement Leaning, Sample Complexity, Contextual MDP, Function Approximation 
\end{keywords}

\section{Introduction}
Markov decision processes (MDPs) are commonly used to describe dynamic environments. MDPs characterize many real-life tasks in a variety of applications including: advertising, healthcare, games, robotics and more. In those applications, at each episode an agent arrives and interacts with the environment with the goal of maximizing her return. (See, e.g., \citet{Sutton2018}.)

In many applications, in each episode, there are additional exogenous factors that affect the environment, which we refer to as the \emph{context}. One can extend the state space to include the context, but this has the disadvantage of greatly increasing the state space, and hence the complexity of learning and even the representation of a policy. An alternative approach is to keep a small state space, and regard the context as an additional side-information. Contextual Markov Decision Process (CMDP) describes such a model, where for each context there is a potential different optimal policy.

CMDPs are useful to model many user-driven applications, where the context is a user-related information which influences the optimal decision making.
One natural application is in healthcare. We can model the interaction with a given patient using an MDP. For a given medical treatment, the expected outcome of a patient is highly dependent on his medical history and other personal parameters, which we model as her context. For example, the success probability of a given treatment might heavily depend on the patient's age and weight.

We abstract the patient's medical history and any other relevant information as the context.
The benefit of using a CMDP is the fact that most patients behave similarly, although the context space may be large, and there might be unforeseen connection between the context and the outcomes. CMDPs allow to share information and behavior between different contexts in a natural way.

\noindent\textbf{Our contributions.}
We present efficient learning algorithms 
for CMDP, given an access to an ERM oracle.
We consider a finite horizon CMDP, where the rewards are an arbitrary function of the context and the state-action. The dynamics may be either \emph{context-free}, where the context does not influence the dynamics, or \emph{context-dependent}, where different contexts induce different dynamics.
Clearly, the most challenging model is the unknown context-dependent dynamics.
Our method induces an efficient reduction from learning contextual MDP (a Reinforcement Learning problem) to supervised learning.

The learning process
outputs an explicit function approximation of the rewards and the dynamics. 
Following the learning phase, 
our learner receives the current context, builds a related MDP for that context, computes an optimal policy for that MDP, and later runs that policy. Both the construction of the MDP and computing the optimal policy are done in polynomial time in the number of states, actions, and horizon.

For \emph{context free} dynamics, we give an efficient algorithm that creates an unbiased sample of the context-reward and context-next state  pairs, for each significantly-reachable state and action.
We use the unbiased sample to approximate 
both the context-free dynamics and the
state-action rewards.

The most challenging case is unknown \emph{context dependent} dynamics. Here, we are unable to define an unbiased sample at the state-action  level,
since we do not know the probability of a state-action pair for given a context and policy. However, we give an efficient algorithm that constructs for each layer two unbiased data sets. 
Both function approximations
are done once per an entire layer of the MDP. 

Table~\ref{tbl: summary} contains a summary of our sample complexity results, up to poly-logarithmic factors. In all cases we have a polynomial dependence in all our parameters.
Specifically, $d$, the function approximation class pseudo-dimension or fat-shattering dimension (see Appendix~\ref{Appendix:function-approx-dim} for more information regarding the dimensions), the number of states $|S|$, actions $|A|$, horizon $H$, inverse accuracy $1/\epsilon$ and poly-logarithmic dependence in the inverse confidence parameter $\log(1/\delta)$.
The ERM oracle complexity in all cases is $|S||A|$ except for unknown context dependent case where it is only $H$. We remark that our definition bounds by $\epsilon^2$ the squared error rather than $\epsilon$. This essentially introduces additional $\epsilon^2$ factors which do not exist using the standard definition. We use our definition mainly for convenience.

\begin{table*}[ht]
\caption{\label{tbl: summary}
summary of our results. }
 \begin{center}
        \begin{tabular}[c]{| c | c | c |}
            \hline
            Dynamics & Absolute Loss & Square Loss
            \\
            \hline
            Known, context-free 
            &  
            $
            d \epsilon^{-2} H^2 |S|^4 |A|^3 \log (1/\delta)
            $ 
            & 
            $
            d \epsilon^{-4} H^4 |S|^6 |A|^5 \log(1/\delta)
            $
            \\
            \hline
            Unknown, context-free 
            &
            $
            d \epsilon^{-3} H^5 |S|^5 |A|^3 \log(1/\delta) 
            $
            &  
            $
            d\epsilon^{-4} H^4 |S|^6 |A|^5 \log(1/\delta) 
            $
            \\
            \hline
            Known, context-dependent 
            &
            $
            d\epsilon^{-6} H^5 |S|^5 |A|^3 \log(1/\delta) 
            $
            &  
            $
            d \epsilon^{-8}  H^7 |S|^5 |A|^3 \log(1/\delta) 
            $
            \\
            \hline
            Unknown, context-dependent 
            &  
            $
            d \epsilon^{-6} H^9 |S|^{11} |A|^2 \log(1/\delta)
            $
            &
            $
            d \epsilon^{-8} H^{13} |S|^{15} |A|^2 \log(1/\delta)
            $
            \\
            \hline
    \end{tabular}
\end{center}
\end{table*}
\section{Related Work}\label{sec: related-work}
\noindent\textbf{Contextual Reinforcement Leaning.}
CMDP was introduce by \citet{hallak2015contextual}.
\citet{modi2018markov} gives a general framework for deriving generalization bounds as a function of the covering number for smooth CMDPs and contextual linear combination of MDPs.
For smooth CMDPs they obtain sample complexity upper bound of $\tilde{O}\left({NH^2|S|A|}\epsilon^{-3}(|S|+ \ln\frac{N|S||A|}{\delta}\ln\frac{N}{\delta})\right)$, and a lower bound of $\Omega \left(\frac{N|S||A|}{\epsilon^2}\right)$ where $N$ is the covering number of the context space, which can be exponential in the dimension of it.
For the contextual linear combination of MDPs, they obtain a sample complexity bound of
$O\left( {\epsilon^{-2} m^2 H^4 |S| |A|}\log\frac{1}{\delta}\max\{m^2, |S|^2 \log^2\left({m |S||A|}/{\delta}\right)\} \right)$ where $m$ is the number of combined MDPs.
In contrast, our bounds depend on the complexity dimension (VC,Pseudo etc.) which can be logarithmic in the covering number of the context space (see Subsection 27.2 in~\citet{shalev2014understanding}) or independent of it.
For example, the $\gamma$-fat shattering dimension of of linear functions is $1/\gamma^2$.
However, our results do not contradicts the above lower bound, as the pseudo and fat-shattering dimensions are known to be tightly upper bounded by the covering number of the function class input space (i.e., the domain). For smooth CMDP, the function classes used to approximate the rewards and dynamics are $L_r$ and $L_p$-Lipschitz (respectively) and it is known that 
the $\gamma$-fat shattering dimension of $L$-Lipshcitz function class is (approximatly) linear in the covering number of the domain. 
Our work generalize the work of \citet{modi2018markov} since we have no assumption regarding the CMDP or the function classes.

\citet{modi2020no}
give a regret analysis for Generalized Linear Models (GLMs). Our function approximation framework is much more general than GLM.

\citet{foster2021statistical} 
present a statistical complexity measure for interactive decision making and present an application of it to contextual RL. They assume an access to an online estimation oracle with regret guarantees.
Using it, they obtain $\tilde{O}(\sqrt{T })$ regret.
However, this oracle is very strong and might be computationally inefficient. 
It is also unclear whether their algorithmic approach can be extended to offline oracles for estimation.
In contrast, we use a standard ERM oracle.

\citet{jiang2017contextual} present  OLIVE
which is sample efficient for  Contextual Decision Processes with a small Bellman rank. We do not make any assumptions on the Bellman rank.

\noindent\textbf{Reward-Free exploration.}
The setting of unknown and context-free dynamics is closely related to Reward-free RL \citet{jin2020rewardFree,zhang2021Reward-free,menard2021fastRewardFree,chen2022statistical,DBLP:conf/icml/QiuYWY21}.
Our main motivation for developing the context-free algorithms is to extend them later to the context-dependent case.
%
%

\noindent\textbf{Contextual Bandits.} Contextual bandits (CMAB) are a natural extension of the Multi-Arm Bandit (MAB), augmented by a context which influences the rewards \cite{Slivkins-book,MAB-book}. \citet{agarwal2014taming} use efficiently an optimization oracle  to derive an optimal regret bound. 
Regression based approaches appear in \cite{agarwal2012contextual,foster2018practical,foster2020beyond,simchi2021bypassing,xu2020upper}.
We differ from CMAB, since our main challenge is the dynamics, and the need to optimize future rewards, which is the case in most RL settings.




\section{Preliminaries and Notations}


\noindent\textbf{Markov Decision Process (MDP).}
We consider an episodic MDP with a finite horizon $H$, and assume w.l.o.g  it  is layered, loop free and has a unique start state.
   A \emph{Markov Decision Process (MDP)} is a tuple $(S,A,P,r,s_0, H)$, where (1) $S$ is a finite state space decomposed into $H+1$ disjoint subsets (layers) $S_0, S_1, \ldots, S_H$ such that transitions are only possible between consecutive layers (i.e., loop-free),  (2) $A$ is a finite action space, (3) $s_0\in S$ is the unique start state, 
   (4) $P(\cdot|s,a)$ defines the transition probability function, i.e.,  $P(s' | s,a)$ is the probability that we reach state $s'$ given that we are in state $s$ and perform action $a$,
    (5) $R(s,a)\in[0,1]$ is a  random variable  for the reward of performing action $a$ in state $s$, and $r(s,a)$ is its expectation, i.e., $r(s,a) = \mathbb{E}[R(s,a)] $, and 
    (6) $H$ is the finite horizon.
    

\noindent\textbf{Policy.}
    A \emph{stochastic policy} $\pi$ is a mapping from states to distribution over actions, i.e., $\pi : S \to \Delta(A)$.  
    A \emph{deterministic policy} $\pi$ is a mapping from states to actions, i.e., $\pi : S \to A$.

\noindent\textbf{Occupancy measure.} 
Let $q_h(s | \pi, P)$ denote the probability of reaching state $s\in S_h$  at time $h \in [H]$ of an episode generated using  policy $\pi$ and dynamics $P$. %
%
%
%

\noindent\textbf{Episode and trajectory.}
At the start of each episode we select a policy $\pi$,
run it, and observe
%
a trajectory\\
${\tau = (s_0, a_0, r_0, s_1, \ldots, s_{H-1}, a_{H-1}, r_{H-1}, s_H)}$, where for all $h \in [H-1]$, $a_h\sim\pi(\cdot|s_h)$,  $r_h \sim R(s_h, a_h)$ and  $s_{h+1} \sim P(\cdot| s_h, a_h)$
\footnote{W.l.o.g. we assume that $r(s_H,a_H) = 0$ for any $s_H\in S_H$ and $a_H\in A$ so we can omit it.}.

\noindent\textbf{Value function.}
Given a policy $\pi$ and a MDP 
    $
        {M 
        = 
        (S,A,P,r,s_0, H)}
    $, 
the
$h \in [H-1] $ stage value function of a state $s \in S_h$ is defined as
    $
        {V^{\pi}_{M,h} (s)
        = 
        \mathbb{E}_{\pi, M} 
        \Big[
        \sum_{k=h}^{H-1} r(s_k, \pi(s_k))|s_h = s \Big]}
    $.\\
For brevity, when $h = 0$ we denote $V^{\pi}_{M,0}(s_0) := V^{\pi}_M (s_0)$.

\noindent\textbf{Optimal policy and Bellman equations.}
    A (deterministic) optimal policy $\pi^\star_M$ for MDP $M$ satisfies, for every stage $h \in [H-1]$ and a state $s\in S_h$, 
    $
        \pi^\star_{M,h} (s) \in
        \arg \max_{\pi:S \to A}\{V^{\pi}_{M,h}(s)\}.
    $

\noindent\textbf{Planning.}
Given an MDP $M= (S, A, P, r, s_0, H)$ the procedure $\texttt{Planning}(M)$ returns an optimal policy $\pi^\star_M$ and its value $V^\star_M(s_0)$ and runs in time $O(|S|^2\; |A|\; H)$.

\noindent\textbf{Contextual MDP (CMDP) }
%
is a tuple $(\mathcal{C},S, A, \mathcal{M})$ where $\mathcal{C}\subseteq \mathbb{R}^{d'}$ is the context space, $S$ is the state space and $A$ is the action space. The mapping $\mathcal{M}$  maps a context $c\in \mathcal{C}$ to a MDP
    $
        \mathcal{M}(c) 
        =
        (S, A, P^c, r^c,s_0, H)
    $.
There is an unknown distribution  $\mathcal{D}$ over the context space $\mathcal{C}$, and for each episode a context $c$ is sampled i.i.d. from $\mathcal{D}$.
For mathematical convenience, we assume the context space is finite (but potentially huge). Our results naturally extend to infinite contexts space.

\noindent\textbf{Context-free dynamics vs. context-dependent dynamics.}
A CMDP has \emph{context-free} dynamics when the context effects only the rewards function, while the dynamics are identical for all contexts. i.e., there exits a dynamics $P$ such that for all $c \in \mathcal{C} $, $ P^c = P$. 
A \emph{context-dependent} dynamics has a potentially different dynamics $P^c$ for each context $c$.
We consider both settings.
    


%

\noindent\textbf{Context-dependent policy.}
    A context-dependent policy $\pi = \left( \pi_c: S \to \Delta(A) \right)_{c \in \mathcal{C}}$ maps a context $c \in \mathcal{C}$ to a policy $\pi_c : S \to \Delta(A)$. We similarly define a deterministic context-dependent policy.
    
\noindent\textbf{Optimal context-dependent policy} 
    is a policy $\pi^\star = (\pi^\star_c)_{c \in \mathcal{C}}$ that satisfies, for every context $c \in \mathcal{C}$, 
    $
        \pi^\star_c \in
        \arg \max_{\pi:S \to \Delta(A)} V^{\pi}_{\mathcal{M}(c)}(s_0)
    $.


\noindent\textbf{Losses.}
The square loss is $\ell_2(z,y)=(z-y)^2$ and absolute loss is $\ell_1(z,y)=|z-y|$.

\noindent\textbf{Function approximation}    
using functions class $\mathcal{F}$.
The \emph{squared error} (\emph{absolute error}, respectively) of a function $f\in \mathcal{F}$ is 
    $
        sqerr(f)= \mathbb{E}_{x}
        [\ell_2(f(x),f^\star(x))]
    $ 
    ($
        abserr(f)= \mathbb{E}_{x}
        [\ell_1(f(x) , f^\star(x))]
    $), where $f^*(x)$ is the target function (and we might have $f^*\not\in \mathcal{F}$).
    The \emph{squared approximation error} (\emph{absolute approximation error}) of $\mathcal{F}$ is $\alpha^2_2(\mathcal{F})=\inf_{f\in\mathcal{F}} sqerr(f)$ ($\alpha_1(\mathcal{F})=\inf_{f\in\mathcal{F}} abserr(f)$). Note that for square loss we square the approximation error, while this is not standard, it is mainly for mathematical convenience. When clear from the context, we use $\alpha$ instead of  $\alpha_1$ or $\alpha_2$.

\noindent\textbf{ERM oracle.}
    Let $\mathcal{X}$ be some domain, and let $\mathcal{F}$ be a function class that maps $\mathcal{X}$ to $[0,1]$. An Empirical Risk Minimization (ERM) oracle for $\mathcal{F}$ with respect to a loss function $\ell$ takes as input a data set $D = \{(x_i, y_i)\}_{i=1}^n$ with $x_i \in \mathcal{X}$, $y_i \in [0,1]$ and computes $\widehat{f} = \arg\min_{f \in \mathcal{F}} \sum_{(x,y) \in D} \ell (f(x), y)$.

\noindent\textbf{Function class complexity measures.} Our sample complexity bounds are stated in the terms of the pseudo and fat-shattering dimension of the function class (see~\cite{Bartlett1999NeuralNetsBook}), which are complexity measures for learning real-valued function classes.
It is known that if the pseudo/fat-shattering dimension of the function class $\mathcal{F}$ is finite, then $\mathcal{F}$ has a uniform convergence property. Hence $\mathcal{F}$ is learnable using an ERM algorithm up to an $\epsilon$ error, with probability at least $1-\delta$. $m(\epsilon, \delta)$ is the required sample complexity.
%
For more information regarding the dimensions definitions and the sample complexity requires for learning, please see Appendix~\ref{Appendix:function-approx-dim}.

\noindent\textbf{Reward function approximation.}\label{par: reward function approx gurantees for all s,a}
For every state $s \in S$ and action $a \in A$ we have a function class $\mathcal{F}_{s,a}^R=\{f_{s,a}: \mathcal{C} \to [0,1]\}$, which maps context $c$ to (approximate) reward.
The function $N_R(\mathcal{F}, \epsilon, \delta)$ maps a function class $\mathcal{F}$, required accuracy $\epsilon \in (0,1)$ and confidence $\delta \in (0,1)$ to the number of required samples for the ERM oracle to guarantee, with probability $1-\delta$, that $\mathbb{E}[\ell(\widehat{f}(x),f^\star(x))]\leq \epsilon+\alpha$, where $\alpha$ is the approximation error.

\noindent\textbf{Dynamics function approximation.}
For the unknown context-free case we simply use a tabular approximation (see Section \ref{sec:UCFD}). For the unknown context-dependent case, we use a function approximation per layer, as we define in Section \ref{sec:UCDD}.

\noindent\textbf{Reachability.} The reachability of a state is the maximum probability of reaching it, by any policy.
    A state $s_h \in S_h$ is \emph{$\beta$-reachable for dynamics $P$} if there exists a policy $\pi$ such that $q_h(s_h | \pi, P) \geq \beta$. 
%
    For a dynamics $P$ and $s_h \in S_h$ let $\pi_{s_h}$  denote the policy with the highest probability to visit $s_h$.
    Hence, a state $s_h$ is $\beta$-reachable for dynamics $P$ iff $\pi_{s_h}$ satisfies that $q_h(s_h | \pi_{s_h}, P) \geq \beta$.

\noindent\textbf{Learning objective. }
A mapping $\widehat{\pi}^\star$ from contexts $c$ to a policy $\widehat{\pi}^\star_c$ is $\epsilon$-optimal if
$
\mathbb{E}_{c \sim \mathcal{D}}
        [
            V^{\pi^\star_c}_{\mathcal{M}(c)}(s_0)
            -
            V^{\widehat{\pi}^\star_c}_{\mathcal{M}(c)}(s_0)
        ]
        \leq 
        \epsilon + O(\alpha H)
$,
%
where $\alpha$ is an agnostic approximation error. 
The goal of the learning algorithm is to compute a mapping from contexts to policies which are $\epsilon$-optimal.
%
%
The learning algorithm is  sample efficient if it uses $T=poly(|S|,|A|,H,\frac{1}{\epsilon}, \log \frac{1}{\delta})$ samples, and is computationally efficient if it is sample efficient, its running time is  $poly(T,|S|,|A|,H)$, and the number of oracle queries is $poly(|S|,|A|,|H|)$.
%

    


%

\noindent\textbf{Mathematical notations.}
We denote expectation by $\mathbb{E}[\cdot]$ and probabilities by $\mathbb{P}[\cdot]$. The indicator function is $\mathbb{I}[G]$ returns $1$ if event $G$ holds and $0$ otherwise.








\section{Context-Free Dynamics}\label{sec:UCFD}
This section addresses the case of an unknown dynamics which do not depend on the context (i.e., context-free dynamics). The main goal of this section is to provide intuition for our approach, and develop algorithmic tools that we will later use to solve the unknown context-dependent dynamics case, which is the main contribution of the paper.

\noindent\textbf{Our approach.}
Our goal is to collect ``sufficient" i.i.d examples $(c,r)$ of contexts and rewards for each state-action pair $(s,a)$, to learn the context-dependent rewards function using ERM.
This goal is not trivial even without the context, due to inner dependencies in a trajectory $\tau = (c, s_0, a_0, r_0, s_1, a_1, r_1, \ldots , s_H)$ generated by a policy $\pi$ and the dynamics $P$.
%
For simplicity, we  sample for each state-action pair independently. 
To collect i.i.d samples efficiently for each state, we need to compute an exploration policy which (approximately) maximizes the probability to visit the target state.
%
However, special care needs to be taken for states which are ``hard" to reach, i.e., states which are not $\beta$-reachable, for some parameter $\beta$.
Conceptually, we transition states which are not $\beta$-reachable to a sink state and avoid the need to approximate their dynamics. 
Rather than using a fixed parameter $\beta$, we (later) introduce a more gradual transition which improves our dependency on $\epsilon$ in the resulted sample complexity bound.

\noindent\textbf{Algorithm overview.}
%
We approximate the true dynamics $P$ by $\widehat{P}$, one layer at a time. 
after we learned the first $h-1$ layers, we use $\widehat{P}$ to compute a policy $\widehat{\pi}_s$ to reach every state $s$ in layer $h$ (i.e., $s\in S_h$), and the probability of reaching it.
%
Intuitively, if $\widehat{P}\approx P$ for the first $h-1$ layers, then $\widehat{\pi}_s$ will approximately maximizing the probability of reaching $s$, and would allow to sample it efficiently. Once we reach state $s$ we use the various actions in a round-robin manner.

Algorithm EXPLORE-UCFD (Algorithms~\ref{alg: EXPLORE-UCFD-main} and~\ref{alg: EXPLORE-UCFD}) works in phases, where in phase $h $ we approximate the dynamics and rewards of layer $h$. 
We define the approximated dynamics $\widehat{P}$ as the empirical dynamics, when in addition we transition not $\beta$-reachable states to the sink state $s_{sink}$ (a new state which we add to the approximated model).

We first collect samples for each (significantly reachable) state in layer $h$ and then use them to approximate the dynamics, using simple tabular estimation. Using the same sample we also estimate the rewards using ERM oracle. The required accuracy for each state-action pair $(s_h,a_h)$ is determined by the accuracy-per-state function $\epsilon_\star(\cdot)$ which gets $\widehat{p}_{s_h}:= q_h(s_h| \widehat{\pi}_{s_h}, \widehat{P})$ as an input.  
After collecting sufficient number of samples for every (significantly reachable) state in layers up to $ h - 1$, we have a good approximation of the dynamics up to layer $h-1$. This yields a good approximation of the occupancy measure of layer $h$ for any policy $\pi$.

Given a state $s_h\in S_h$ and the approximated dynamics $\widehat{P}$ we compute $\widehat{\pi}_{s_h} = \arg \max_{\pi} q_h(s_h| \pi,\widehat{P})$ using a planning algorithm. We run $\widehat{\pi}_{s_h}$ to generate the sample of $s_h$.
In order to control the number of sampled episodes we define significantly reachable states as $\beta$-reachable for $\widehat{P}$, i.e., they have $q_h(s_h| \widehat{\pi}_{s_h},\widehat{P})\geq \beta$. 

At the end of the sampling we have for each $\beta$-reachable  for $\widehat{P}$ state $s_h\in S_h$ for $\widehat{P}$, and every action $a_h$ a data set which contains tuples $(s_h, a_h, s_{h+1})$.
Then, we use a tabular approximation to learn the context-free dynamics. 
For the rewards, we use the collected examples of tuples $((c,s_h, a_h), r_h)$ and run the ERM on that data set to compute a function approximation for the rewards
\footnote{When collecting samples for state $s$ and action $a$, we update their sample only, to guarantee it is i.i.d.}.
%
Note that it is important that we first fix the approximation of layers up to $h$, which guarantee that we use the same $\widehat{\pi}_s$ and $\widehat{p}_s$ in each sampling state $s\in S_h$.

\noindent\textbf{The approximated dynamics $\widehat{P}$.} Let $n(s'|s,a)$ denote the number of times the triplet $(s,a,s')$ was observed and $n(s,a)$ denote the number of times the pair $(s,a)$ was observed.
We have a threshold $N_P(\gamma,\delta_1)=O(\gamma^{-2}(|S|+\log(1/\delta))$.
If $s$ is not $\beta$-reachable (given our learned dynamics) or $(s,a)$ is sampled less than  $N_P(\gamma,\delta_1)$ times it transition to the sink state $s_{sink} $. When $(s,a)$ is sampled at least  $N_P(\gamma,\delta_1)$ times, we use the empirical next state distribution, i.e., 
$\widehat{P}(s' | s,a) = \frac{n(s' | s,a)}{n(s,a)}$, to approximate the transition probability distribution of $(s,a)$.

\noindent\textbf{Accuracy per state function.}
We set a refined desired accuracy per state function $\epsilon_\star$,
and saves a  $1/\epsilon$ factors in the sample complexity. 
We do not approximate states which are very hard to reach. States which are very easy to reach, we want maximum accuracy. For intermediate levels we have a gradual accuracy dependency. This is captured in our definition of the accuracy-per-state function $\epsilon_\star$, which depends the probability to visit state $s$, i.e., 
$\widehat{P}_s := q_h(s_h | \widehat{\pi}_{s}, \widehat{P})$. 
We define it as follows:
where $B>0$ is a constant we determine later.
\begin{align*}
    \epsilon_\star(\widehat{P}_s) =
    \begin{cases}
        1 &, \text{ if } \widehat{P}_{s} < \frac{\epsilon}{B|S|}\\
        \frac{\epsilon}{B \widehat{P}_{s}  |S| |A| } &, \text{ if } p_{s} \in [ \frac{\epsilon}{B |S|} , \frac{1}{|S| } ]\\
        \frac{\epsilon}{B H |S| |A| } &, \text{ if }\widehat{P}_{s} > \frac{\epsilon}{B|S|}
    \end{cases}
\end{align*}

For the $\ell_2$ loss we use $\epsilon^2_\star(\widehat{P}_s)$.
We also denote $m_{s,a}(\widehat{P}_s)=\max{\{N_R (\mathcal{F}^R_{s,a} ,\epsilon_\star(\widehat{p}_{s}), \delta_1), N_P(\gamma, \delta_1)}\}$.

    \begin{algorithm}[ht]
    \caption{EXPLORE Unknown Context Free Dynamics (sketch for the $\ell_1$ loss)}
    \label{alg: EXPLORE-UCFD-main}
    \begin{algorithmic}[1]
        \State { initialize counters
        $n(s,a) =0 , n(s'|s,a) = 0 $ for all $(s,a,s') \in S \times A \times S$.
        }
        \For{$h \in [H-1] $}
            \State
            {   compute the approximated dynamics $\widehat{P}$ up  to layer $h-1$
            }
            \For{$s\in S_h$}
                \State{
                compute $\widehat{p}_{s}$, the highest probability to visit $s$ in $\widehat{P}$, and a policy $\widehat{\pi}_{s}$ that reaches it}
        
                \If{$\widehat{p}_{s_h}  \geq \beta$}
                    \For{ $a \in A$}
                        \State{initialize $Sample(s,a) = \emptyset$}
                        \State{set $\widehat{\pi}_{s}(s) \gets a$}
                        \For{${t = 1, 2, \ldots,                            
                        \lceil 
                            \frac{2}{\widehat{p}_{s} - \gamma h}(\ln(\frac{1}
                            {\delta_1})
                            +
                            m_{s,a}(\widehat{P}_s)\;) 
                        \rceil}$
                        }
                            \State{observe context $c$, run $\widehat{\pi}_{s}$}
                            \State{observe trajectory, update sample and counters}
                        \EndFor{}
            
                        \If{$|Sample(s,a)|\geq m_{s,a}(\widehat{P}_s)$ }
                        \State{
                        $
                            {f_{s,a} = \texttt{ERM}(\mathcal{F}^R_{s,a}, Sample(s,a), \ell_1)} 
                        $}
                        \Else
                        \State{\textbf{return } \texttt{FAIL}}
                        \EndIf{}
                \EndFor{}
            \Else
                \State{ set  for all $a\in A:\;  f_{s,a} = 0$}
            \EndIf{}{}
        \EndFor{}
    \EndFor{}
    
    \State{ \textbf{return }
        $
            F =
            \{f_{s,a} : \; \forall s\in S, a\in A\}, \widehat{P}
        $}
    \end{algorithmic}
\end{algorithm}

\noindent\textbf{Approximate optimal policy.} For a context $c$, let the true MDP be $\mathcal{M}(c) = (S, A, P, r^c, s_0, H)$ and the approximated MDP be $\widehat{\mathcal{M}}(c) = (S\cup \{s_{sink}\}, A, \widehat{P}, \widehat{r}^c, s_0, H)$.
Let $\pi^\star_c$ and $\widehat{\pi}^\star_c$ be an optimal policy for $\mathcal{M}(c)$ and $\widehat{\mathcal{M}}(c)$ respectively.
For both $\ell_1$ and $\ell_2$ loss function, we obtain the following.
\begin{theorem}\label{thm: resoult case 2}
    With probability $1-\delta$ it holds that  
    $
        {\mathbb{E}_{c \sim \mathcal{D}}[V^{\pi^\star_c}_{\mathcal{M}(c)}(s_0) - V^{\widehat{\pi}^\star_c}_{\mathcal{M}(c)}(s_0)] \leq 
         \epsilon + 2\alpha H} 
    $, 
    after collecting\\
    $\Tilde{O}\Big( d \epsilon^{-3} H^5 |S|^5 |A|^3 \log \frac{H|S||A|}{\delta}\Big)$ trajectories for  $\ell_1$ loss, and $\Tilde{O}\Big( d \epsilon^{-4} H^4 |S|^6 |A|^5 \log \frac{H |S||A|}{\delta}\Big)$ for the $\ell_2$ loss. $\alpha$ and $d$ are the maximal approximation error and fat-shattering / pseudo dimension over all states and actions, respectively. 
\end{theorem}

\noindent\textbf{Analysis outline.}
%
For every layer $h \in [H-1]$ we define the following good events.
Event $G^h_1$ states that for every state-action pair $(s_h ,a_h ) \in S_h \times A$, if $s_h$ is $\beta$-reachable for $\widehat{P}$, then sufficient number of samples were collected for the pair $(s_h,a_h)$. 
%
Event $G^h_2$ states that for every state-action pair $(s_h ,a_h) \in S_h \times A $ such that $s_h$ is $\beta$-reachable, the learned dynamics $\widehat{P}$ approximate the true dynamics $P$ up to a small error of $\gamma$, i.e., $\|\widehat{P}(\cdot|s_h, a_h) -P(\cdot|s_h, a_h)\|_1\leq \gamma$.

%
Event $G^h_3$ state that for every state-action pair $(s_h ,a_h ) \in S_h \times A$ such that $s_h$ is $\beta$-reachable for $\widehat{P}$,
the ERM oracle returns a function $f_{s_h,a_h}(c)$ with low generalization error.
Let $G_i = \cap_{h \in [H-1]}G^h_i$ for all $i\in\{1,2,3\}$.
We analyse the value error caused by both the dynamics and rewards approximation under the good events. 
We also show that the event $G_1 \cap G_2 \cap G_3$ holds with high probability.

For the analysis, we define an intermediate MDP  $\widetilde{\mathcal{M}}(c) = (S\cup\{s_{sink}\}, A, \widehat{P}, r^c, s_0, H)$, which differ from $\mathcal{M}(c)$ only in the dynamics and from $\widehat{\mathcal{M}}(c)$ only in the rewards function. 
Let $\alpha$ denote the maximal approximation error.
For $\beta = \frac{\epsilon}{B |S| H}$, $B = 24$ and $\gamma = \frac{\epsilon}{48|S|H^2}$
we obtain the following bound on the value difference caused by the dynamics approximation.
\begin{lemma}\label{lemma: UCFD - dynamics error}
    If event $G_1 \cap G_2$ holds, then 
    for every context $c \in \mathcal{C}$ 
    and context-dependent policy $\pi = (\pi_c)_{c \in \mathcal{C}}$ it holds that 
    $
        {|V^{\pi_c}_{\mathcal{M}(c)}(s_0) - V^{\pi_c}_{\widetilde{M}(c)}(s_0)|
        \leq \epsilon /16}
    $.
\end{lemma}
\begin{proof}[sketch]
Under the event $G_1 \cap G_2$, for every $\beta$-reachable state $s$ for $\widehat{P}$ and an action $a$ it holds that\\
$\|P(\cdot |s,a) - \widehat{P}(\cdot|s,a)\| \leq \gamma$.
We show in Lemma~\ref{lemma: occ-measure-dist-UCFD} that for any policy $\pi = (\pi_c)_{c \in \mathcal{C}}$ it holds that\\ $\forall c\in \mathcal{C} \; \forall h \in [H]$,  $\sum_{s_{h} \in S_h}|q_h(s_h | \pi_c, P) - q_h(s_h | \pi_c, \widehat{P})| \leq \gamma h + \beta \sum_{k = 0}^{h-1}|S_k| $.
For our choice of $\beta$ and $\gamma$, the latter yields that
$ |V^{\pi_c}_{\mathcal{M}(c)}(s_0) - V^{\pi_c}_{\widetilde{M}(c)}(s_0)| \leq \sum_{h=0}^{H-1} \sum_{s_{h} \in S_h}|q_h(s_h | \pi, P) - q_h(s_h | \pi, \widehat{P})| \leq  \epsilon /16 $.
\end{proof}

The following lemma bounds the expected value difference caused by the rewards approximation.
\begin{lemma}\label{lemma: UCFD - rewards error.}
    If event $G_3$ holds,
    then for every policy ${\pi = (\pi_c)_{c \in \mathcal{C}}}$, it holds that  
    $$
        {\mathbb{E}_{c \sim \mathcal{D}}[|V^{\pi_c}_{\widehat{\mathcal{M}}(c)}(s_0) - V^{\pi_c}_{\widetilde{M}(c)}(s_0)|]
        \leq
        \epsilon /8 + \alpha H} 
    .
    $$
\end{lemma}

By combining  Lemmas~\ref{lemma: UCFD - dynamics error} and~\ref{lemma: UCFD - rewards error.} we obtain an expected value difference bound for any policy.
\begin{lemma}\label{lemma: new val-diff-2}
If events $G_1$ ,$G_2$ and $G_3$ hold, then for every policy $\pi = (\pi_c)_{c \in \mathcal{C}}$ it holds that 
$$
   \mathbb{E}_{c \sim \mathcal{D}}[|V^{\pi_c}_{\mathcal{M}(c)}(s_0) - V^{\pi_c}_{\widehat{\mathcal{M}}(c)}(s_0)|] \leq 
    3\epsilon /16 + \alpha H 
.
$$
\end{lemma}

The above lemma establishes Theorem~\ref{thm: resoult case 2}. 
For detailed analysis, see Appendix~\ref{Appendix:UCFD}.

\noindent\textbf{Known Dynamics.}
When the context free dynamics is known, we can achieve better sample complexity, as the following theorem states. (For more details, see Appendix~\ref{Appendix:KCFD}.)
\begin{theorem}\label{thm:known-dynamics-context-free}
    With probability $1-\delta$  it holds that 
    $
        \mathbb{E}_{c \sim \mathcal{D}}[V^{\pi^\star_c}_{\mathcal{M}(c)}(s_0) - V^{\widehat{\pi}^\star_c}_{\mathcal{M}(c)}(s_0)] \leq 
        \epsilon + 2\alpha H ,
    $
after collecting\\ 
$\Tilde{O}\Big( d \epsilon^{-2} H^2 |S|^4 |A|^3 \log \frac{|S||A|}{\delta}\Big)$ 
trajectories for the $\ell_1$ loss, and 
$\Tilde{O}\Big( d \epsilon^{-4} H^4 |S|^6 |A|^5 \log \frac{|S||A|}{\delta}\Big)$ for the $\ell_2$ loss.     
\end{theorem}

\section{Context Dependent Dynamics}\label{sec:UCDD}
In this section we address the challenging model of context dependent dynamics, where each context induces a potentially different dynamics. Clearly, this implies that for any policy $\pi$, the occupancy measure is determined by the context.
Hence, a state $s\in S$ that is highly-reachable for some context $c_1 \in \mathcal{C}$ might be poorly-reachable for a different context $c_2 \in \mathcal{C}$. 

For the \emph{unknown context-dependent dynamics} we move the approximation from being per state-action pair to being per layer. (While this is a slight modification of the assumption, it is still very reasonable.) Conceptually, we move from collecting samples per state-action, to collecting samples per layer, and those samples are index by context-state-action tuples $(c,s,a)$. We construct an unbiased data set with respect to those tuples. However, 
we guarantee that the collected samples have the ``right'' marginal distribution over the entire layer. This requires a much more involved algorithm and analysis.
Thus, we extend the definition of reachability.

\noindent\textbf{Good contexts of a state.}
    For a state $s \in S$ we define the set of \emph{$\beta$-good contexts} with respect to $P$ as\\ 
    $
        {\mathcal{C}^\beta(s|P) 
        := 
        \{c \in \mathcal{C} : s \text{ is } \beta\text{-reachable for $P^c$}\}}
    $.
Given the approximated dynamics $\widehat{P}^c$, we define $
        \widehat{\mathcal{C}}^{\beta}(s):=\mathcal{C}^\beta(s|\widehat{P})  $.
Note that there might be no context $c$ which is good for all states (unlike in the context-independent dynamics). 
The following defines the modification of the $\beta$-reachability.

\noindent\textbf{$(\gamma, \beta)$-good states.}
    Let $\gamma, \beta \in (0,1]$. For each layer $h \in [H]$ we define the set of \emph{$(\gamma, \beta)$-good states with respect to $P$} as  
    $
        {S^{\gamma, \beta}_{h,P}
        :=
        \{s_h \in S_h :
        \mathbb{P}_{c \sim \mathcal{D}}[c \in \mathcal{C}^\beta(s_h|P)] \geq \gamma \}}
    $.
    Given the approximated dynamics $\widehat{P}$, we define $ \widehat{S}^{\gamma, \beta}_h =  S^{\gamma, \beta}_{h,\widehat{P}}$.
    We define the \emph{target domain} 
    $
        {\mathcal{X}^{\gamma, \beta}_h = 
        \{
            (c, s, a) :
            s \in \widehat{S}^{\gamma, \beta}_h , c \in \widehat{\mathcal{C}}^{\beta}(s), 
            a \in A
        \}}
    $ of collected examples.

\noindent\textbf{Function approximation for each layer.} 
A major hurdle caused by the context-dependent dynamics is that for each $(s_h,a_h)$ the probability of sampling $((c, s_h, a_h), r_h)$ is highly dependent on the context $c$ through the dynamics $P^c$.
Our goal is to create an unbiased sample, which we will perform for an entire level, but this seems very challenging to achieve at the individual state-action level.
To exemplify that, assume we observe the context $c$ and run some policy $\pi$ to generate a trajectory $\tau = (c, s_0,a_0,r_0, s_1, \ldots,s_{H})$. For every layer $h \in [H-1]$ the distribution of the example $((c,s_h, a_h), r_h)$ is $\mathcal{D}(c) \cdot q_h(s_h,a_h| P^c, \pi) \cdot \mathbb{P}[R^c(s_h,a_h)= r_h | c, s_h, a_h]$. 
If we aim to collect samples for each state-action pair separately, the appropriate distribution is $\mathcal{D}(c)\cdot \mathbb{P}[R^c(s_h,a_h)= r_h | c, s_h,a_h]$.
Hence, we need to guarantee that the contexts are sampled in an unbiased way, i.e., the marginal context distribution for any state-action pair is $\mathcal{D}$. 
If the dynamics were known, we would overcome this using Importance Sampling. When the dynamics are unknown, we side step this issue, and create an unbiased sample at the layer level.
The advantage of tuples $((c,s,a),r)$ is that we can sample them for the entire layer and obtain good estimates on average, and then claim that for the ``important" states we have a good approximation. At the layer level, the occupancy measure determines the joint distribution over $(c,s_h, a_h) \in \mathcal{X}^{\gamma, \beta}_h$, which is the desired distribution. Hence, we approximate the rewards as a function of context, state and action.  
Similarly for the dynamics.

\noindent\textbf{Dynamics and rewards function approximation.}\label{par: dynamics function approx gurantees for all layer h}
We slightly modify our assumption for the function approximation class,
which works per layer and not per state-action.

For each layer $h \in [H-1]$ we have a function class for the dynamics ${\mathcal{F}^P_{h}=\{f^P_{h}: \mathcal{C} \times S_h \times A \times S_{h+1} \to [0,1]\}}$ and for the rewards $\mathcal{F}^R_{h}=\{f^R_{h}: \mathcal{C} \times S_h \times A \to [0,1]\}$. Intuitively, given that we are in state $s$, perform action $a$ and the context is $c$, the function $f^P_{h}\in \mathcal{F}^P_{h}$ and $f^R_{h}\in \mathcal{F}^R_{h}$, approximates the transition probability to state $s'$, i.e., $P^c(s'|s,a)$, and the expected reward, $r^c(s,a)$, respectively.
For the dynamics approximation, we also assume reliability for every layer $h$, i.e., $\alpha_1(\mathcal{F}^P_h) = \alpha^2_2(\mathcal{F}^P_h) =0 $.

The functions $N_P(\mathcal{F}^P_{h}, \epsilon, \delta)$ and $N_R(\mathcal{F}^R_{h},\epsilon, \delta)$ map a function class, required accuracy $\epsilon $ and confidence $\delta $ to the 
required number of examples for the ERM oracle to have the desired guarantees. For the dynamics the ERM guarantee is that with probability $1-\delta$, $\mathbb{E}[\ell(f^P_h(x),y)]\leq \epsilon$. For the rewards the guarantee is that $\mathbb{E}[\ell(f^R_h(x),y)]\leq \epsilon+\alpha$, where $\alpha$ is the approximation error.

\noindent\textbf{Layer dynamics realizability assumption.}
We assume that for each layer $h\in [H-1]$ we have a function $f^P_{h}\in \mathcal{F}^P_{h}$ such that $f^P_{h}(c,s,a,s')=P^c(s'|s,a)$. 
We relay on this assumption when approximating the dynamics, in order to properly estimate whether a state is $(\gamma,\beta)$-good.

\noindent\textbf{Sample collection.}
In algorithm EXPLORE-UCDD (see Algorithm~\ref{alg: EXPLORE-UCDD-main} or~\ref{alg: EXPLORE-UCDD}), we learn the dynamics and rewards for each layer, given the approximated dynamics of previous layers.
%
When learning the dynamics associated with layer $h$, we collect examples of the form $(c,s_h,a_h,s_{h+1})$ from each trajectory $\tau = (c, s_0, a_0, r_0, \ldots, s_H)$ that contains $(s_h, a_h)$. We add to our data set a sample $((c,s_h,a_h,s_{h+1}),1)$ and samples $((c,s_h,a_h,s'),0)$ for each $s'\in S_{h+1}\setminus\{s_{h+1}\}$.
This reduces the learning dynamics to a 
 regression problem.
%
When learning the rewards associated with layer $h$,  as before, we collect samples of the form $((c, s_h, a_h),r_h)$, and use them to approximate the rewards function for the layer.


\noindent\textbf{Algorithm overview.}
%
Algorithm EXPLORE-UCDD (see Algorithms~\ref{alg: EXPLORE-UCDD-main} and~\ref{alg: EXPLORE-UCDD}) runs in $H$ phases, one per layer. In phase $h\in [H-1]$ we maintain an approximate dynamics for all previous layers $k\leq h-1$, which we already learned.
In phase $h$ we run multiple iterations, 
in each iteration, (1) we select at random a (approximately) $(\gamma,\beta)$-good state $s_h\in \widetilde{S}^{\gamma,\beta}_h$ and an action $a_h\in A$. (2) Given a context $c$ and a state $s_h$ we compute a policy $\widehat{\pi}^c_{s_h}$ which maximizes the probability of reaching state $s_h$ under the approximated dynamics $\widehat{P}^c$.
(3) We run $\widehat{\pi}^c_{s_h}$. If it reaches $s_h$ we play $a_h$, get a reward $r_h$ and transits to $s_{h+1}$, we add: (a) to the dynamics data set $Sample^P(h)$: $((c,s_h,a_h,s'),\mathbb{I}[s_{h+1}=s'])$ for each $s'\in S_{h+1}$, (b) to the reward data set $Sample^R(h)$: $((c,s_h,a_h),r_h)$.
(4) After collecting sufficient number of samples, we use the ERM oracle to  (a) approximate the transition probabilities of layer $h$, i.e., $f^P_h = \texttt{ERM}(\mathcal{F}^P_h, Sample^P(h),\ell)$. (b) approximate the rewards function of layer $h$, i.e., $f^R_h = \texttt{ERM}(\mathcal{F}^R_h, Sample^R(h),\ell)$.
Consider the following algorithm sketch for the $\ell_1$ loss and the parameters set $\delta_1 = \frac{\delta}{8H}$, $\delta_2 = \frac{\delta}{8|S|}$, $\epsilon_2 = \gamma/4$, $\beta = \gamma =\frac{\epsilon}{20 |S|H}$, $\rho = \frac{\beta}{16 |S|H}$, $\epsilon_P = \frac{ \epsilon^2}{10\cdot 16 \cdot 20  |A| |S|^4 H^3}$ and $\epsilon_R = \frac{\epsilon^2}{20^2 |S||A| H^2}$, where ${m_h=2\max\{N_P(\mathcal{F}^P_h ,\epsilon_P, \delta_1/2)), N_R(\mathcal{F}^R_h ,\epsilon_R, \delta_1/2))\}}$.

\begin{algorithm}[ht]
    \caption{EXPLORE Unknown Context Dependent Dynamics (sketch for the $\ell_1$ loss)}
    \label{alg: EXPLORE-UCDD-main}
    \begin{algorithmic}[1]
        \For{$h \in [H-1]$ }
            \State{
            compute $\widehat{P}^c$, the approximated context-dependent dynamics, up to layer $h-1$}
            \State{let $\widetilde{S}^{\gamma, \beta}_h$ denote the approximation of the set of $(\gamma,\beta)$-good states w.r.t $\widehat{P}^c$}
            \For{ $t= 1,2, \ldots , \left\lceil 
                    \frac{8 |S|}{\beta \cdot \gamma}(\ln(\frac{1}{\delta_1})
                    +
                    m_h)
                \right\rceil$ }
                \State{observe context $c_t$ and choose $(s_h, a_h) \in \widetilde{S}^{\gamma, \beta}_h \times A$ uniformly at random}
                \State{
                compute $\widehat{p}^{c_t}_{s}$, the highest probability to visit $s$ in $\widehat{P}^{c_t}$, and a policy $\widehat{\pi}^{c_t}_{s}$ that reaches it }
                \State{set $\widehat{\pi}^{c_t}_{s_h}(s_h)\gets a_h$}
                    \If{$ \widehat{p}^{c_t}_{s_h} \geq \beta$}
                        \State{ run $\widehat{\pi}^{c_t}_{s_h}$ and generate trajectory $\tau$}
                        \If{$(s_h, a_h, r_h, s_{h+1})$ is in $\tau$}
                            \State{add $((c_t,s_h,a_h), r_h)$ to the rewards sample }
                            \State{ add $\{((c_t,s_h,a_h, s'_{h+1}),\mathbb{I}[s_{h+1} = s'_{h+1}]) : s'_{h+1} \in S_{h+1}\}$ to the dynamics sample
                            }
                        \EndIf{}
                    \EndIf{}{}
            \EndFor{}
            \If{$|Sample^P(h)| \geq m_h $}
            \State{$f^P_h = \texttt{ERM}(\mathcal{F}^P_h , Sample^P(h),\ell_1)$}
            \State{$f^R_h = \texttt{ERM}(\mathcal{F}^R_h, Sample^R(h),\ell_1)$}
            \Else
            \State{\textbf{return }\texttt{FAIL}}   
            \EndIf{}{}
        \EndFor{}
    \State{ \textbf{return }$\{f^R_h , f^P_h, \widehat{S}^{\gamma, \beta}_h : \forall h \in [H-1]\}$ }
    \end{algorithmic}
\end{algorithm}

\noindent\textbf{Approximate optimal policy.} 
Given a context $c$, we define an MDP with the learned rewards
and dynamics 
and compute its optimal policy. 
We define the approximated CMDP as $(\mathcal{C}, S\cup \{s_{sink}\}, A, \widehat{\mathcal{M}})$ where $\widehat{\mathcal{M}}(c) = (S\cup \{s_{sink}\},A,\widehat{P}^c, \widehat{r}^c, s_0, H)$. We define $\widehat{r}^c$ as
${\widehat{r}^c (s_h, a_h) =  f^R_h(c, s_h, a_h) \cdot \mathbb{I}[s_h \in \widetilde{S}^{\gamma, \beta}_h, c \in \widehat{\mathcal{C}}^\beta(s_h)}]$ and $\widehat{r}^c(s_{sink}, a_h) = 0$.
The dynamics $\widehat{P}^c$  uses the dynamics approximation functions normalized, and states which are not $(\gamma,\beta)$-good transition to the sink.
 For any other state-action  we define a transition to the sink w.p. $1$.   
 For a context $c \in \mathcal{C}$, let $\widehat{\pi}^\star_c$ and $\pi^\star_c$ denote an optimal policy for $\widehat{\mathcal{M}}(c)$ and $\mathcal{M}(c)$, respectively.
Let $\alpha$ denote the maximal approximation error over all layers $h$ w.r.t. $\ell \in \{\ell_1, \ell_2\}$.
We obtain the following.
\begin{theorem}\label{thm: opt policy UCDD}
    With probability at least $1-(\delta + \frac{\epsilon}{5})$ it holds that 
    $
        \mathop{\mathbb{E}}_{c \sim \mathcal{D}}[V^{\pi^\star_c}_{\mathcal{M}(c)}(s_0) - V^{\widehat{\pi}^\star_c}_{\mathcal{M}(c)}(s_0)] \leq 
        \epsilon + 2\alpha H 
    $, 
    after collecting $\Tilde{O}\Big( d \epsilon^{-6} H^9 |S|^{11} |A|^2 \log \frac{H}{\delta} \Big)$ trajectories for  the  $\ell_1$ loss, and $\Tilde{O}\Big( d \epsilon^{-8} H^{13} |S|^{15} |A|^2 \log \frac{H}{\delta}\Big)$ 
    for the $\ell_2$ loss.
\end{theorem}

\noindent\textbf{Analysis outline.}
In the analysis, we show that the following good events hold, with high probability, for every layer $h \in [H-1]$:
(1) Every state $s_h \in \widehat{S}^{\gamma, \beta}_h$ we identify correctly.
(2) We collect sufficient number of samples of $(c, s_h, a_h) \in \mathcal{X}^{\gamma, \beta}_h$.
(3) Our approximation of the dynamics has low generalization error.
(4) Our approximation of the rewards has low generalization error.
The above form the good events $G_1$, $G_2$, $G_3$ and $G_4$, and there is a choice of parameters such that they all hold with high probability.

Our analysis (see Appendix~\ref{Appendix: UCDD})
shows that under these good events, our approximation of the dynamics and rewards for every layer $h$ and $(c, s, a) \in \mathcal{X}^{\gamma, \beta}_h$ is accurate, with high probability. We also show that any $(c, s, a) \notin \mathcal{X}^{\gamma, \beta}_h$ adds only small error to our estimations. Hence, in expectation over $c \in \mathcal{C}$ we have small errors for both rewards and dynamics. 
%
\begin{lemma}\label{lemma: total-var-p}
    Under the good events $G_1$, $G_2$ and $G_3$, for all $(c, s_h, a_h) \in  \mathcal{X}^{\gamma, \beta}_h$, there exist parameters choice such that
    $   
        \mathbb{P}\Big[\|\widehat{P}^c(\cdot| s_h, a_h) - P^c(\cdot| s_h, a_h)\|_1\leq \epsilon/40 |S| H^2 \Big|$
    $ (c,s_h,a_h) \in \mathcal{X}^{\gamma,\beta}_h \Big]
        \geq 1- \epsilon/(10 |S||A|H)
    $, 
    where $\widehat{P}^c$ is the approximated dynamics and we ignore $s_{sink}$.
\end{lemma}
\begin{proof}[sketch for the $\ell_1$ loss]
    By the good event $G_3$ and Markov's inequality the following holds  
    $$
        \mathbb{P}
        \Big[ |f^P_h(c, s_h, a_h, s_{h+1}) - P^c(s_{h+1}| s_h, a_h)| \geq \rho \Big| (c,s_h,a_h)
        \in \mathcal{X}^{\gamma,\beta}_h \Big]
        \leq 
        \epsilon_P/\rho
    .
    $$
    Since $\sum_{s_{h+1} \in S_{h+1}} P^c(s_{h+1}|s_h, a_h) = 1$, by union bound over $s_{h+1} \in S_{h+1}$ and $\widehat{P}^c$ definition we obtain
    \begingroup
    \allowdisplaybreaks
    \begin{align*}
    &\mathop{\mathbb{P}}_{(c,s_h,a_h)}\Big[  
            \forall s_{h+1} \in S_{h+1}.\;
            \frac{P^c(s_{h+1}|s_h, a_h) - \rho}{1 + \rho |S|}  \leq
            \widehat{P}^c(s_{h+1}|s_h, a_h)
            \leq
            \frac{P^c(s_{h+1}|s_h, a_h) + \rho}{1 - \rho |S|}
            \;\Big|
            (c, s_h, a_h) \in  \mathcal{X}^{\gamma, \beta}_h   \Big]
            \geq
            \\
            &1 - \frac{\epsilon_P }{\rho}|S_{h+1}|
    .
    \end{align*}
    \endgroup
    Now, using simple calculation, we derive the lemma for our choice of $\beta$, $\rho$,$\epsilon_P$.
\end{proof}

In the analysis, we define an intermediate MDP  
$\widetilde{\mathcal{M}}(c) = (S \cup \{s_{sink}\}, A, \widehat{P}^c, r^c, s_0, H)$, where $\widehat{P}^c$ is the approximated dynamics and $r^c$ is the true rewards function extended to $s_{sink}$ by defining $\forall c \in \mathcal{C}, \; a \in A: r^c(s_{sink},a):=0$. We use it to estimate the influence of the error on the rewards separately from the error of the dynamics.

The following lemma states the expected value-difference caused by the dynamics approximation. 
\begin{lemma}\label{lemma: UCDD-dynamics-error}
    Under the good events $G_1$, $G_2$ and $G_3$, there exist a parameters choice such that for every policy $\pi = (\pi_c)_{c \in \mathcal{C}}$,
    it holds that
    $
        \mathbb{E}_{c \sim \mathcal{D}}
        [|V^{\pi_c}_{\mathcal{M}(c)}(s_0) - V^{\pi_c}_{\widetilde{\mathcal{M}}(c)}(s_0)|]
        \leq 
        0.225 \epsilon
    $.
\end{lemma}
\begin{proof}[sketch for the $\ell_1$ loss]
    For a context $c$, let $G(c)$ denote the following event
    $$
        G(c) ={ \{ \forall h \in [H].\;\;\;
		\|q_h (\cdot | \pi_c, P^c) - q_h(\cdot| \pi_c, \widehat{P}^c)\|_1
		\leq 
		3\epsilon/(40 H)\}}.
	$$	
    We show in Lemma~\ref{lemma: UCDD l_1 occ measure diff} that
    $$
    \mathbb{P}_c[G(c)]=
    \mathbb{P}_c\left[
		\forall h \in [H].\;\;\;
		\|q_h (\cdot | \pi_c, P^c) - q_h(\cdot| \pi_c, \widehat{P}^c)\|_1
		\leq 
		3\epsilon/(40 H) \right]
		\geq
		 1 
        - 
        3\epsilon/(20H),
    $$ 
    yielding the lemma since,
    $$
        {\mathbb{E}_{c \sim \mathcal{D}}
        [|V^{\pi_c}_{\mathcal{M}(c)}(s_0) - V^{\pi_c}_{\widetilde{\mathcal{M}}(c)}(s_0)|]}
        \leq
        {\mathbb{E}_{c \sim \mathcal{D}}\left[\sum_{h=1}^{H-1} \|q_h (\cdot | \pi_c, P^c) - q_h(\cdot| \pi_c, \widehat{P}^c)\|_1 \Big| G(c)\right] + 3\epsilon/20}
    = 0.225 \epsilon
    .
    $$
\end{proof}

The next lemma states the expected value-difference caused by the rewards approximation. 
\begin{lemma}
    Under the good events $G_1$, $G_2$ and $G_3$, there exist a parameters choice such that for every policy $\pi = (\pi_c)_{c \in \mathcal{C}}$ it holds that
    $
        \mathbb{E}_{c \sim \mathcal{D}}
        [|V^{\pi_c}_{\widetilde{M}(c)}(s_0) - V^{\pi_c}_{\widehat{M}(c)}(s_0)|]
        \leq
        0.2 \;\epsilon
        +
        \alpha H 
       .
    $
\end{lemma}
By combining Lemmas~\ref{lemma: UCDD-dynamics-error} and~\ref{lemma: UCDD-rewards-error} we obtain the following lemma, which establishes Theorem~\ref{thm: opt policy UCDD}.
\begin{lemma}\label{lemma: UCDD-rewards-error}
    Under the good events $G_1$, $G_2$,$G_3$ and $G_4$, for every policy $\pi = (\pi_c)_{c \in \mathcal{C}}$ it holds that 
    $$
        \mathbb{E}_{c \sim \mathcal{D}}
        [| V^{\pi_c}_{\mathcal{M}(c)}(s_0) -  V^{\pi_c}_{\widehat{\mathcal{M}}(c)}(s_0)|]
        \leq
        0.5 \epsilon + \alpha H .
    $$
\end{lemma}

\noindent\textbf{Known context-dependent dynamics.} For this case we continue in the approach of collecting examples for each $(\gamma,\beta)$-good state and action. We obtain the following result. For more details, see Appendix~\ref{Appendix:KCDD}.
\begin{theorem}\label{thm: opt policy CDD}
    With probability $1-\delta$ it holds that 
    $
        \mathbb{E}_{c \sim \mathcal{D}}[V^{\pi^\star_c}_{\mathcal{M}(c)}(s_0) - V^{\widehat{\pi}^\star_c}_{\mathcal{M}(c)}(s_0)] \leq 
        \epsilon + 2\alpha H 
    $, 
    after collecting $\Tilde{O}\Big(d \epsilon^{-6} H^5 |S|^5 |A|^3 \log \frac{|S||A|}{\delta}\Big)$
    trajectories for the $\ell_1$ loss, and 
    $\Tilde{O}\Big(d \epsilon^{-8} H^7 |S|^5 |A|^3 \log \frac{|S||A|}{\delta}\Big)$ trajectories for the  $\ell_2$.
\end{theorem}

\section{Discussion}\label{sec:disscution}
To the best of our knowledge, our work is the first to drive sample complexity bounds for CMDP, without assuming any additional assumptions regarding it.
Our sample complexity bounds do not depend on the size of the context space, which allows it to be huge.

An interesting future research direction is to drive lower bounds. Clearly, the sample complexity is lower bounded by classical PAC lower bounds, of  $\Omega(d \epsilon^{-2} \log\frac{1}{\delta} )$, where $d$ is the complexity dimension (i.e., VC, Natrajan, Fat-shattering, Pseudo dimension). Also, non-contextual MDP sample complexity lower bounds apply to our case and give $\Omega\left( {\epsilon^{-2}|S||A| \log(|S|/\delta)} \right)$.
Deriving stronger lower bounds that are based on the special structure of the CMDP is an important open problem.


\section*{Acknowledgements}
This project has received funding from the European Research Council (ERC) under the European Union’s Horizon 2020 research and innovation program (grant agreement No. 882396), by the Israel Science Foundation(grant number 993/17), Tel Aviv University Center for AI and Data Science (TAD), and the Yandex Initiative for Machine Learning at Tel Aviv University.

OL thanks Idan Attias for helpful discussions and patient explanations about ERM. OL thanks Aviv Rosenberg for helpful advices and his comments on a former version of the paper.

\vskip 0.2in
\bibliography{ref}


\newpage
\renewcommand{\theHsection}{A\arabic{section}}

\appendix

\section{Real-Valued Function Class Dimensions}\label{Appendix:function-approx-dim}

Our sample complexity bounds are stated in the terms of the Pseudo and $\gamma$-fat-shattering dimension of the function class, which are complexity measured for learning real-valued function classes.

In the following, we define the Pseudo and $\gamma$-fat-shattering dimension of a function class $\mathcal{F}$. For more information, see~\cite{Bartlett1999NeuralNetsBook}.

\subsection{Pseudo Dimension Definition}

\begin{definition}[pseudo-shattering, Definition 11.1 in~\cite{Bartlett1999NeuralNetsBook}]
    Let $\mathcal{F}$ be a set of function from a domain $\mathcal{X}$ to $\mathbb{R}$ and suppose that $\mathcal{S} = \{x_1,x_2,\ldots,x_m\} \subseteq \mathcal{X}$. Then $\mathcal{S}$ is pseudo-shattered by $\mathcal{F}$ if there are real numbers $r_1, r_2, \ldots,r_m$ such that for each $b \in \{0,1\}^m$ there is a function $f_b$ in $\mathcal{F}$ with $sign(f_b(x_1) - r_i) = b_i$ for $1 \leq i \leq m$. We say that $r=(r_1,r_2,\ldots,r_m)$ witnesses the shattering.
\end{definition}

\begin{definition}[pseudo-dimension, Definition 11.2 in~\cite{Bartlett1999NeuralNetsBook}]
    Suppose that $\mathcal{F}$ is a set of functions from a domain $\mathcal{X}$ to $\mathbb{R}$. Then $\mathcal{F}$ has pseudo-dimension $d$ if $d$ is the maximum cardinality of a subset $\mathcal{S}$ of $\mathcal{X}$ that is pseudo-shattered by $\mathcal{F}$. If no such maximum exists, we say that $\mathcal{F}$ has infinite pseudo-dimension. The pseudo-dimension of $\mathcal{F}$ is denoted $Pdim(\mathcal{F})$. 
\end{definition}

\subsection{Fat-Shattering Dimension Definition}

\begin{definition}[$\gamma$-shattering, Definition 11.10 in~\cite{Bartlett1999NeuralNetsBook}]
    Let $\mathcal{F}$ be a set of functions mapping from a domain $\mathcal{X}$ to $\mathbb{R}$ and suppose that $\mathcal{S} = \{x_1,x_2,\ldots,x_m\} \subseteq \mathcal{X}$. Suppose also that $\gamma$ is a positive real number. Then $\mathcal{S}$ is $\gamma$-shattered by $\mathcal{F}$ if there are real numbers $r_1, r_2, \ldots,r_m$ such that for each $b \in \{0,1\}^m$ there is a function $f_b$ in $\mathcal{F}$ with $f_b(x_i) \geq r_i + \gamma$ if $b_i =1$, and $f_b(x_i) \leq r_i - \gamma$ if $b_i =0$, for $1 \leq i \leq m$.
    We say that $r=(r_1,r_2,\ldots,r_m)$ witnesses the shattering.
    
    Thus, $\mathcal{S}$ is $\gamma$-shattered if it is shattered with a 'width of shattering' of at least $\gamma$.
    This notion of shattering leads to the following dimension.

\end{definition}

\begin{definition}[Fat shattering dimension, Definition 11.11 in~\cite{Bartlett1999NeuralNetsBook}]
    Suppose that $\mathcal{F}$ is a set of functions from a domain $\mathcal{X}$ to $\mathbb{R}$ and that $\gamma>0$. Then $\mathcal{F}$ has $\gamma$-dimension $d$ if $d$  is the maximum cardinality of a subset $\mathcal{S}$ of $\mathcal{X}$ that is $\gamma$-shattered by $\mathcal{F}$. If no such maximum exists, we say that $\mathcal{F}$ has infinite $\gamma$-dimension. The $\gamma$-dimension of $\mathcal{F}$ is denoted $fat_\mathcal{F}(\gamma)$. This defined a function $fat_\mathcal{F}: \mathbb{R}^+ \to \mathbb{N} \cap \{0, \infty\}$, which we call the fat-shattering dimension of $\mathcal{F}$. We say that $\mathcal{F}$ has finite fat-shattering dimension whenever it is the case that for all $\gamma > 0$, $fat_\mathcal{F}(\gamma)$ is finite.
\end{definition}

\begin{remark}
    For every function class $\mathcal{F}$ and $\gamma > 0$ it holds that $fat_\mathcal{F}(\gamma) \leq Pdim(\mathcal{F})$.
\end{remark}

\subsection{Sample Complexity Results}

The following theorems state that if the Pseudo/fat-shattering dimension of the function class $\mathcal{F}$ is finite, then $\mathcal{F}$ has a uniform convergence property. Hence $\mathcal{F}$ is learnable using an ERM algorithm up to an $\epsilon$ error, with probability at least $1-\delta$. $m(\epsilon, \delta)$ is the required sample complexity for the learning task.

\begin{theorem}[Adaption of Theorem 19.2 in~\cite{Bartlett1999NeuralNetsBook}]
    Let $\mathcal{F}$ be a hypothesis space of real valued functions with a finite pseudo dimension, denoted $Pdim(\mathcal{F}) < \infty$. Then, $\mathcal{F}$ has a uniform convergence with 
    \[
        m(\epsilon, \delta) = O \Big( \frac{1}{\epsilon^2}( Pdim(\mathcal{F}) \ln \frac{1}{\epsilon} + \ln \frac{1}{\delta})\Big).
    \]    
\end{theorem}

\begin{theorem}[Adaption of Theorem 19.1 in~\cite{Bartlett1999NeuralNetsBook}]
    Let $\mathcal{F}$ be a hypothesis space of real valued functions with a finite fat-shattering dimension, denoted $fat_{\mathcal{F}}(\gamma)$. Then, $\mathcal{F}$ has a uniform convergence with 
    \[
        m(\epsilon, \delta) = O \Big( \frac{1}{\epsilon^2}( fat_{\mathcal{F}}(\epsilon/256)  \ln^2 \frac{1}{\epsilon} + \ln \frac{1}{\delta})\Big).
    \]    
\end{theorem}

\section{Known and Context-Free Dynamics}\label{Appendix:KCFD}


In this section we assume a known context-independent transition probability function, i.e.,
$        \forall c\in \mathcal{C} : P^c = P$
and $P$ is known to the learner.

\subsection{Algorithm}

Let use first outline the main ideas of our algorithm EXPLORE-KCFD (Algorithm~\ref{alg: EXPLOR-KCFD}). Since the dynamics are context-free we have that for any policy $\pi$ and state $s_h \in S_h$, the probability to visit $s_h$ using $\pi : S \to A$ that is identical for any context $c$, i.e., for any context $c \in \mathcal{C}$, we have $q_h(s_h | \pi, P^c) = q_h(s_h | \pi, P)$.
Given the dynamics $P$ for each state $s_h \in S_h$ Algorithm EXPLORE-KCFD computes, using the planning oracle, a policy $\pi_{s_h} := \arg \max_{\pi: S \to A} q_h(s_h | \pi, P)$ and the probability $p_{s_h}$ that $\pi_{s_h}$ reaches $s_h$. 
Then, the probability that $ \pi_{s_h}$ reaches $s_h$ is used to check whether the state is $\beta$-reachable with respect to the known dynamics $P$ for $\beta$ that will be determined later.
We use $\pi_{s_h,a_h}$ to generate trajectory $\tau$. If $\tau$ contains state $s_h$ we add $((c, s_h, a_h),r_h)$ to our reward sample of $(s_h,a_h)$. Clearly the collected samples are i.i.d. After collecting ``sufficient'' number of contexts-rewards we use the ERM oracle to compute a function $f_{s_h,a_h}(c)$ that approximates $r^c(s_h, a_h)$.

We set a refined desired accuracy per state, which depends on its maximum probability, and saves a  $1/\epsilon$ factors in the sample complexity. 
States which are very hard to reach, we do not approximate. States which are very easy to reach, we want maximum accuracy. For intermediate levels we have a gradual accuracy dependency. This is captured in our definition of the accuracy-per-state function $\epsilon_\star$, which depends the probability to visit state $s$, i.e., 
$p_s := q_h(s_h | \pi_{s_h}, P)$. 
\[
   \epsilon_\star(p_s) =
    \begin{cases}
        1 &, \text{ if } p_{s} < \frac{\epsilon}{B|S|}\\
        \frac{\epsilon}{B H |S| |A|} &, \text{ if } p_{s}  > \frac{1}{|S| }\\
        \frac{\epsilon}{B p_{s}  |S| |A| }&, \text{ if } p_{s}  \in [ \frac{\epsilon}{B |S|} , \frac{1}{|S| } ]\\
    \end{cases}
\]
where $B>0$ is a constant.

Since we sample only the $\beta$-reachable states for every action $a \in A$,
Algorithm EXPLORE-KCFD (Algorithm~\ref{alg: EXPLOR-KCFD}) learns an approximation of the context-dependent reward function given $P$ (i.e., the context-free dynamics is known to the learner) and $N_R(\mathcal{F}, \epsilon, \delta)$ (the sample complexity function of the ERM oracle) efficiently.

In more details:

Let $\texttt{Planning}(M)$  denote a planning algorithm
which gets as input a MDP $M=(S, A, P, r, s_0, H)$ The planning algorithm returns an optimal policy for the $H$ finite horizon return and the appropriate value function. It runs in time $O(|S||A|H)$.\footnote{
For example, policy iteration is such a planning algorithm. It is finds an optimal policy (and its value) with respect to the finite horizon return, and can be computed in time polynomial in the MDPs parameters.} 

Algorithm PaP (Algorithm~\ref{alg: PaP}) returns for each state $s\in S$ a policy $\pi_{s}$ that maximizes the probability to visit it, denoted $p_s$. 

Algorithm EXPLORE-KCFD (Algorithm~\ref{alg: EXPLOR-KCFD}),  uses $\pi_{s}$ to sample each $\beta$-reachable state $s$ for each action $a \in A$ for sufficiently many times to create a large enough sample $Sample(s,a)$ containing the tuples $((c,s,a), R^c(s,a))$. Then, we feed the ERM with that sample and output an approximation of the reward function $r^c(s,a)$ using $f_{s,a} = \texttt{ERM}(\mathcal{F}^R_{s,a}, Sample(s,a), \ell)$.
For not-$\beta$-reachable state $s$ we set $f_{s,a}=0 \;\; \forall a \in A$.  
The algorithm returns $F = \{f_{s,a} ,\;\; \forall (s,a) \in S \times A\}$, or \textit{Fail} if insufficient number of samples have been collected for any $\beta$-reachable state.

To improve our overall sample complexity, we define the accuracy-per-state function, for both $\ell_1$ and $\ell_2$ :

For $\ell_1$ we define it as
\[
    \epsilon^1_\star(p_s) :=  \epsilon_\star(p_s) =
        \begin{cases}
            1 &, \text{ if } p_{s} < \frac{\epsilon}{6|S| }\\
            \frac{\epsilon}{6H|S||A|} &, \text{ if } p_{s}  > \frac{1}{|S| }\\
            \frac{\epsilon}{6p_{s}|S||A| }&, \text{ if } p_{s}  \in [ \frac{\epsilon}{6|S|} , \frac{1}{|S| } ]\\
        \end{cases}
\]

For $\ell_2$ we define it as $\epsilon^2_\star(p_s) := (\epsilon_\star(p_s))^2$, or equivalently,
\[
    \epsilon^2_\star(p_s) =
        \begin{cases}
            1 &, \text{ if } p_{s} < \frac{\epsilon}{6|S| }\\
            \frac{\epsilon^2}{36H^2|S|^2|A|^2} &, \text{ if } p_{s}  > \frac{1}{|S| }\\
            \frac{\epsilon^2}{36(p_{s} )^2 |S|^2 |A|^2}&, \text{ if } p_{s}  \in [ \frac{\epsilon}{6|S|} , \frac{1}{|S| } ]\\
        \end{cases}.
\]
In both functions, where the required accuracy for a state $s$ is $1$, we do not sample it. 

Algorithm EXPLOIT-KCFD (Algorithm~\ref{alg: EXPLOIT-KCFD}) get as inputs the MDP parameters and the functions approximate the rewards (that computed using EXPLORE-KCFD algorithm). Given a context $c$ it computed the approximated MDP $\widehat{\mathcal{M}}(c)$ and use it to compute a near optimal policy $\pi^\star_c$. Then, it run $\pi^\star_c$ to generate trajectory.
Recall that $\widehat{\mathcal{M}}(c)=(S,A,P,s_0,\widehat{r}^{c},H)$ where we define $\forall s\in S, a\in A: \widehat{r}^{c}(s,a) = f_{s,a}(c)$.

\begin{algorithm}
    \caption{Find Fast Policy (FFP)}
    \label{alg: FFP}
    \begin{algorithmic}[1]
       \State{\textbf{inputs: }
        \begin{itemize}
            \item MDP parameters:$S$  - the states space,$A$ - a finite actions space,$P$ - transition probabilities matrix,$s_0$ - the unique start state,$H$ - the horizon length.
            \item $s$ - the target state.  
        \end{itemize}}
       \State {let $r(s',a)=\mathbb{I}[s'=s]$}
       \State{$(p_s, \pi_s)\gets \texttt{Planning} (M=(S,A,P,r, s_0,H))$}
       \State \textbf{return: }
        $(p_s, \pi_s)$
    \end{algorithmic}
\end{algorithm}

\begin{algorithm}
    \caption{Policies and Probabilities(PaP)}
    \label{alg: PaP}
    \begin{algorithmic}[1]
       \State{\textbf{inputs: }
            MDP parameters:$S = \{S_0, S_1, \ldots, S_H\} $ - a layered states space,$A$ - a finite actions space,$P$ - transition probabilities matrix,$s_0$ - the unique start state, $H$ - the horizon length.}

        \For{$h \in [H-1]$}
            \For{$s\in S_h$}
               \State {$(p_s, \pi_s)\gets \texttt{FFP} (S, A, P, s_0, H, s)$}
            \EndFor{}
        \EndFor{}
       \State \textbf{return :}
        $\{(p_s, \pi_s) \;\;\forall s\in S\}$
    \end{algorithmic}
\end{algorithm}

\begin{remark}
    Since the reward function defined in algorithm FFP has a reward of $1$ for state $s$ and $0$ for any other state (regardless of the action), the value function of any policy computed using this rewards function is exactly the probability the policy visits state $s$.
\end{remark}

\begin{algorithm}
    \caption{Explore Rewards Known Context-Free Dynamics (EXPLOR-KCFD)}
    \label{alg: EXPLOR-KCFD}
    \begin{algorithmic}[1]
       \State{ \textbf{inputs: }
        \begin{itemize}
            \item MDP parameters: $S = \{S_0, S_1, \ldots, S_H\} $ - a layered states space,$A$ - a finite actions space, $P$ - transition probabilities matrix, $s_0$ - the unique start state, $H$ - the horizon length.
            \item Accuracy and confidence parameters: $\epsilon$,$\delta$.
            \item $\ell$ - the loss function ( $\ell \in \{\ell_1, \ell_2 \}$).
            \item $\forall s \in S , a \in A : \;\; \mathcal{F}^R_{s,a}$ - the function classes use to approximate the rewards function.
            \item $N_R(\mathcal{F}, \epsilon, \delta)$ - sample complexity function for approximating the reward with respect to $\ell$.
            \item $\epsilon^i_\star(\cdot)$ - the accuracy-per-state function, (assumed to be $\epsilon^1_\star$ or $\epsilon^2_\star$, with accordance to $\ell$).
        \end{itemize}}
       \State {set $\delta_1= \frac{\delta}{4|S||A|}$, $\beta = \frac{\epsilon}{6 |S|}$}
       \State{$\{(p_s, \pi_s)\} \gets \texttt{PaP}(S,A,P,s_0,H)$}
        \For{$h \in [H-1]$}
            \For{$s\in S_h$}
                \If{$p_{s}  \geq \beta$}
                    \For{ $a \in A$}
                        \State{set $\pi_s(s) \gets a$}
                       \State{compute the required number of episodes: 
                        \[
                            T_{s,a} =
                            \Big\lceil 
                            \frac{2}{p_s}(\ln(\frac{1}
                            {\delta_1})
                            +
                            N_R(\mathcal{F}^R_{s,a} ,\epsilon_\star^i(p_s), \delta_1)) 
                            \Big\rceil
                        \]
                        } 
                       \State{initialize $Sample(s,a) = \emptyset$}
                       \State{$\pi_s(s)\gets a$}
                        \For{$t = 1, 2, \ldots, T_{s,a}$}
                           \State{observe context $c$}
                           \State{run $\pi_{s}$ to generate trajectory $\tau_t$}
                           \If{$(s, a, r) \in \tau_t$, for a reward $r \in [0,1]$}
                                \State{update $Sample(s,a) = Sample(s,a) + \{((c,s,a), r)\}$
                                }
                            \EndIf{}
                        \EndFor{}
                        \If{$|Sample (s,a)| \geq  N_R(\mathcal{F}^R_{s,a} ,\epsilon^i_{s}(p_s), \delta_1)$}
                            \State{call Oracle: 
                            $
                                f_{s,a} = \texttt{ERM}(\mathcal{F}^R_{s,a}, Sample(s,a), \ell) 
                            $}
                            
                        \Else
                            \State{\textbf{return } \texttt{FAIL}}
                        \EndIf{}
                \EndFor{}
            \Else
            \State{set 
            $
                \forall a\in A:  f_{s,a} = 0
            $}
            \EndIf{}
        \EndFor{}
    \EndFor{}
    \State{
        \textbf{return } 
        $
            F =
            \{f_{s,a} : \; \forall s\in S, a\in A\}
        $}
    \end{algorithmic}
\end{algorithm}

\begin{algorithm}
    \caption{ Exploit-CMDP for Known and Context-Free-Dynamics (EXPLOIT-KCFD)}
    \label{alg: EXPLOIT-KCFD}
    \begin{algorithmic}[1]
        \State{ \textbf{inputs: }
        \begin{itemize}
            \item The MDP parameters: $S$, $A$, $P$, $s_0$, $H$.
            \item The functions approximate the rewards for each state-action pair: $\{f_{s,a}| \forall (s,a) \in S \times A\}$.
        \end{itemize}}
        \For{$t=1,2,...$}
            \State{observe context $c_t$}
            \State {define the approximated reward function
            $\forall s\in S, a\in A: \widehat{r}^{c_t}(s,a) = f_{s,a}(c_t)$}
            \State{define the approximated CMDP $\widehat{\mathcal{M}}(c_t)=(S,A,P,s_0,\widehat{r}^{c_t},H)$}
            \State{compute an optimal policy of the approximated model $(\pi_t , V_t) \gets \texttt{Planning}(\widehat{\mathcal{M}}(c_t))$}
            \State{run $\pi_t$ in episode $t$.}
        \EndFor{}
\end{algorithmic}
\end{algorithm}

\subsection{Analysis}

\subsubsection{Analysis Outline}

In the following analysis, our goal is to bound the expected value difference between the true model $\mathcal{M}(c)$ and $\widehat{\mathcal{M}}(c)$, for any context-dependent policy $\pi = (\pi_c)_{c \in \mathcal{C}}$, with high probability. (See Lemmas~\ref{lemma: vall-diff l_2 KCFD} and~\ref{lemma: vall-diff l_1 KCFD}, for the $\ell_2$ and $\ell_1$ losses, respectively).

Using that bound, we derive a bound on the expected value difference between the true optimal context-dependent policy $\pi^\star = (\pi^\star_c)_{c \in \mathcal{C}}$ and our approximated optimal policy $\widehat{\pi}^\star = (\hat{\pi}^\star_c)_{c \in \mathcal{C}}$, which holds with high probability. (See Theorems~\ref{thm: PAC optimal policy for Known CFD with l_2} and~\ref{thm: PAC optimal policy for Known CFD with l_1} for the $\ell_2$ and $\ell_1$ losses, respectively).

We present analysis for both $\ell_2$ (see Sub-subsection~\ref{subsubsec:KCFD-l-2}) and $\ell_1$ (see Sub-subsection~\ref{subsubsec:KCFD-l-1} losses in the agnostic case.

Lastly, we derive sample complexity bound using known uniform convergence sample complexity bounds for the Pseudo dimension (See Theorem~\ref{thm: pseudo dim}) and the fat-shattering dimension (See Theorem~\ref{thm: fat dim}). For the sample complexity analysis, see Sub-subsection~\ref{subsubsec:sample-complexity-l-2-KCFD} for the $\ell_2$ loss, and~\ref{subsubsec:sample-complexity-l-1-KCFD} for the $\ell_1$ loss.

\begin{remark}
    Throughout the analysis, we strongly use that we collect samples only for $\beta$-reachable states, where $\beta = \frac{\epsilon}{6 |S|}$.
\end{remark}

\subsubsection{Good Events}

We analyse algorithm EXPLORE-KCFD (Algorithm~\ref{alg: EXPLOR-KCFD}) under the following good events:

\paragraph{Event $G_1$.}
Let $G_1$ be the event that for every $\frac{\epsilon}{6|S|}$-reachable state $s$ and each action $a \in A$ we have $|Sample(s,a)|\geq N_R(\mathcal{F}^R_{s,a}, \epsilon^i_\star(p_s), \delta_1)$ ( where $i \in \{1, 2\}$, in accordance to the used loss function).

\begin{lemma}\label{lemma: G_1 KCFD}
    It holds that $\mathbb{P}[G_1] \geq 1 - \frac{1}{4}\delta$.
\end{lemma}

\begin{proof}
    Fix a pair $(s,a)$ of a $\frac{\epsilon}{6|S|}$-reachable state $s \in S$ and an action $a\in A$.
    
    Assume we run $\pi_s$ for $T$ episodes, and in state $s$ the agent always plays action $a$.
    
    Let $\mathbb{I}_t[(s,a)]$ be an indicator which indicates whether $(s,a)$ was sampled in the $t$'th episode. Then, $\mathbb{E}[\mathbb{I}_t[(s,a)]] = p_s \geq \frac{\epsilon}{|S|}$.
    
    We wish to collect $m_{s,a} := N_R(\mathcal{F}^R_{s,a} ,\epsilon^i_\star(p_s), \delta_1)$ samples. For $T$ such that $Tp_s \geq m_{s,a}$ we would like to lower bound the number of episodes $T$ needed to collect at least $m_{s,a}$ samples with probability at least $1 - \delta_1$. For that mission, we use multiplicative Chernoff bound. Thus, we need to find $\beta \in [0,1]$ such that $(1 - \beta)Tp_s = m_{s,a}$.  $\beta = \frac{Tp_s - m_{s,a}}{Tp_s}$ is satisfying the requirement. 
    
    Hence,
    \begin{align*}
        \mathbb{P}[\sum_{t = 1}^T \mathbb{I}_t[(s,a)] 
        \leq  
        m_{s,a}]
        &=
        \mathbb{P}[\frac{1}{T}\sum_{t = 1}^T \mathbb{I}_t[(s,a)] 
        \leq  
        (1 - \beta)T p_s]
        \\
        &\leq
        \exp{(-\frac{\beta^2 Tp_s}{2})}
        \\
        &=
        \exp{(-\frac{(Tp_s - m_{s,a})^2 }{2 Tp_s})}
        \\
        &\leq
        \delta_1 \iff T \geq \frac{2}{p_s}\left(\ln\frac{1}{\delta_1} + m_{s,a}\right).
    \end{align*}
    \

    Thus, for any $\frac{\epsilon}{6|S|}$-reachable state $s \in S$ and an action $a\in A$, if we run 
    \[
        T_{s,a} =
        \left\lceil 
         \frac{2}{p_s}(\ln(\frac{1}{\delta_1}) +N_R(\mathcal{F}^R_{s,a} ,\epsilon^i_\star(p_s), \delta_1)) 
        \right\rceil
    \]
    iterations, we collect at least $N_R(\mathcal{F}^R_{s,a} ,\epsilon^i_\star(p_s), \delta_1)$ examples. Since $\delta_1 = \frac{\delta}{4 |S| |A|}$, the lemma follows using union bound.
\end{proof}

\paragraph{Event $G_2$.}
Let $G_2$ be the event where for every
for any pair $(s,a)$ of a $\frac{\epsilon}{6|S|}$-reachable state $s$ and an action $a$,  
we have    
    \begin{equation*}\label{Good event: G_2, L_2, KCFD}
        \mathbb{E}_{c \sim \mathcal{D}}
        [(f_{s,a}(c) - r^c(s,a))^2] 
        \leq
        \epsilon^2_\star(p_s)
        +
        \alpha^2_2(\mathcal{F}^R_{s,a}).
    \end{equation*}
where $f_{s,a} = \texttt{ERM}(\mathcal{F}_{s,a}, Sample(s,a), \ell_2)$.

We similarly define the event $G_2$ for the $\ell_1$ loss where 
    \begin{equation*}\label{Good event: G_2, L_1, KCFD}
        \mathbb{E}_{c \sim \mathcal{D}}
        [|f_{s,a}(c) - r^c(s,a)|] 
        \leq
        \epsilon^1_\star(p_s)
        +
        \alpha_1(\mathcal{F}^R_{s,a}),
    \end{equation*}
and $f_{s,a} = \texttt{ERM}(\mathcal{F}_{s,a}, Sample(s,a), \ell_1)$.

\begin{lemma}\label{lemma: G_2 KCFD}
 It holds that $\mathbb{P}[G_2|G_1] \geq 1-\frac{1}{4}\delta$.
\end{lemma} 

\begin{proof}
    Follows immediately form ERM guarantees~\ref{par: reward function approx gurantees for all s,a} for every pair $(s,a)$ of a $\frac{\epsilon}{6 |S|}$-reachable state $s \in S$ and an action $a \in A$, when combined using union bound over each pair $(s,a)$. 
\end{proof}

\begin{lemma}\label{lemma: good events probs KCFD}
    It holds that $\mathbb{P}[G_1 \cap G_2] \geq 1-  \frac{\delta}{2}$.
\end{lemma}

\begin{proof}
    By the results of Lemmas~\ref{lemma: G_1 KCFD} and~\ref{lemma: G_2 KCFD} when combined using a union bound.
\end{proof}

Bellow we present analysis for both $\ell_1$ and $\ell_2$ losses.

\subsubsection{Analysis for \texorpdfstring{$\ell_2$}{Lg} loss}\label{subsubsec:KCFD-l-2}
Let $\alpha^2_2 := \max_{(s,a) \in S \times A} \alpha^2_2(\mathcal{F}^R_{s,a})$. 

The following lemma shows that under the good events $G_1$ and $G_2$, the value of any context-dependent policy with respect to the approximated model $\widehat{\mathcal{M}}$ is similar to that with respect to the true model $\mathcal{M}$, in expectation over the context.

\begin{lemma}\label{lemma: vall-diff l_2 KCFD}
Assume the events $G_1$ and $G_2$ hold.
Then for any context-dependent policy ${\pi =(\pi_c: S \to \Delta(A))_{c\in \mathcal{C}}}$ it holds that
\begin{equation*}
    \mathbb{E}_{c\sim \mathcal{D}}
    [|V^{\pi_c}_{\mathcal{M}(c)}(s_0) - V^{\pi_c}_{\widehat{\mathcal{M}}(c)}(s_0)|] \leq
    \frac{1}{2}\epsilon + \alpha_2 H .
\end{equation*}
\end{lemma}

\begin{proof}
Recall that for every context $c \in \mathcal{C}$, state $ s\in S$ and action $ a\in A$, the expected reward is $r^c(s,a) \in [0,1]$.
By construction of $\widehat{\mathcal{M}}(c)$ , for any state $s\in S$ which are not $\frac{\epsilon}{6|S|}$-reachable, we set $f_{s,a}(c)=0$ for any action $a$. Hence,
\begin{equation*}
    |r^c(s,a) - f_{s,a}(c)| \leq 1.
\end{equation*}

Since the good event $G_2$ holds, for every state-action pair $(s,a)$, such that state $s$ is $\frac{\epsilon}{6|S|}$-reachable, it holds that
    \begin{equation*}
        \epsilon^2_\star(p_s) 
        +
        \alpha^2_2(\mathcal{F}^R_{s,a})
        \underbrace{\geq}_{G_2}
        \mathbb{E}_{c \sim \mathcal{D}}
        [(f_{s,a}(c) - r^c(s,a))^2]
        \underbrace{\geq}_{\text{Jensen's inequality}}
        \mathbb{E}^2_{c \sim \mathcal{D}}
        [|f_{s,a}(c) - r^c(s,a)|].
    \end{equation*}
Using that for all $a, b \in [0, \infty)$ it holds that $\sqrt{a} + \sqrt{b} \geq \sqrt{a + b}$, we obtain
\begin{equation}
        \sqrt{\epsilon^2_\star(p_s)} 
        +
        \alpha_2
        \geq  \sqrt{\epsilon^2_\star(p_s)} 
        +
        \alpha_2(\mathcal{F}^R_{s,a}) \geq \sqrt{\epsilon_\star(p_s) +
        \alpha^2_2(\mathcal{F}^R_{s,a})}
        \geq
        \mathbb{E}_{c \sim \mathcal{D}}
        [|f_{s,a}(c) - r^c(s,a)|].
\end{equation}
The above in particular implies that
\begin{equation}\label{ineq: case 1 G2}
        \sqrt{\epsilon^2_\star(p_s)} 
        \geq
        \mathbb{E}_{c \sim \mathcal{D}}
        [|f_{s,a}(c) - r^c(s,a)|-\alpha_2].
\end{equation}

Fix any context-dependent policy $\pi =(\pi_c: S \to \Delta(A))_{c\in \mathcal{C}}$.  
%
By definition of the value function we have \textbf{for any context $c$}:
    \begin{align*}
        V^{\pi_c}_{\mathcal{M}(c)}(s_0) 
        =
        \mathbb{E}_{\pi_c, \mathcal{M}(c)} 
        \Big[
        \sum_{h=0}^{H-1} r^c(s_h, a_h)|s_0 \Big]
        = 
        \sum_{h=0}^{H-1}\sum_{s\in S_h}q_h(s| \pi_c, P) 
        \sum_{a\in A}\pi_c(a|s)r^c(s,a)
    \end{align*}
and    
    \begin{align*}
        V^{\pi_c}_{\widehat{\mathcal{M}}(c)}(s_0)
        = 
        \mathbb{E}_{\pi_c, \widehat{\mathcal{M}}(c)} 
        \Big[
            \sum_{h=0}^{H-1} \widehat{r}^c(s_h, a_h)|s_0 \Big]
         = 
         \sum_{h=0}^{H-1}\sum_{s\in S_h} q_h(s| \pi_c, P)
         \sum_{a\in A}\pi_c(a|s) f_{s,a}(c).
    \end{align*}


 By inequality~\ref{ineq: case 1 G2}, linearity of expectation and triangle inequality we obtain the following derivation:
 \begingroup
    \allowdisplaybreaks
    \begin{align*}
        &\mathbb{E}_{c \sim \mathcal{D}}
        \left[
            \left|
                V^{\pi_c}_{\mathcal{M}(c)}(s_0) 
                - 
                V^{\pi_c}_{\widehat{\mathcal{M}}(c)}(s_0)
            \right|
        \right] 
        \\
        = &
        \mathbb{E}_{c \sim \mathcal{D}}
        \left[
            \left|
                \sum_{h=0}^{H-1}\sum_{s\in S_h} q_h(s| \pi_c, P) 
                \sum_{a\in A}\pi_c(a|s) r^c(s,a)
                - 
                \sum_{h=0}^{H-1} \sum_{s\in S_h} q_h(s| \pi_c, P) 
                \sum_{a\in A}\pi_c(a|s) f_{s,a}(c)
            \right|
        \right] \\
        \leq &
        \mathbb{E}_{c \sim \mathcal{D}}
        \left[   
            \sum_{h=0}^{H-1}
            \sum_{s\in S_h} q_h(s| \pi_c, P)
            \sum_{a\in A} \pi_c(a|s) |r^c(s,a)-f_{s,a}(c)|
        \right] \\
        = &
        \mathbb{E}_{c \sim \mathcal{D}}
        \left[   
            \sum_{h=0}^{H-1}
            \sum_{s\in S_h} q_h(s| \pi_c, P)
            \sum_{a\in A} \pi_c(a|s) (|r^c(s,a)-f_{s,a}(c)|- \alpha_2 + \alpha_2)
        \right] \\
        \\
        = &
        \alpha_2 H
        +
        \mathbb{E}_{c \sim \mathcal{D}}
        \left[   
            \sum_{h=0}^{H-1}
            \sum_{s\in S_h} q_h(s| \pi_c, P)
            \sum_{a\in A} \pi_c(a|s) (|r^c(s,a)-f_{s,a}(c)|- \alpha_2)
        \right] \\
        = &
        \alpha_2 H
        +
        \mathbb{E}_{c \sim \mathcal{D}}
        \left[
            \sum_{h=0}^{H-1}\sum_{a\in A}
            \sum_{s \in S_h : p_{s} < \frac{\epsilon}{6|S|}}
            q_h(s| \pi_c, P) \pi_c(a|s) 
       (|r^c(s,a)-f_{s,a}(c)|- \alpha_2)
        \right]\\
        & + 
        \mathbb{E}_{c \sim \mathcal{D}}
        \left[
            \sum_{h=0}^{H-1} \sum_{a\in A} 
            \sum_{s \in S_h : p_{s} > \frac{1}{|S|}}
            q_h(s| \pi_c, P) \pi_c(a|s) 
        (|r^c(s,a)-f_{s,a}(c)|- \alpha_2)
        \right]\\
        & +
        \mathbb{E}_{c \sim \mathcal{D}}
        \left[
            \sum_{h=0}^{H-1} \sum_{a\in A} 
            \sum_{s \in S_h : p_{s} \in [\frac{\epsilon}{6|S|}, \frac{1}{|S|}]}
            q_h(s|\pi_c, P) \pi_c(a|s)
            (|r^c(s,a)-f_{s,a}(c)|- \alpha_2)
        \right]\\
        \leq &
        \alpha_2 H
        +
        \mathbb{E}_{c \sim \mathcal{D}}
        \left[
            \sum_{h=0}^{H-1}\sum_{a\in A}
            \pi_c(a|s) 
            \sum_{s \in S_h : p_{s} < \frac{\epsilon}{6|S|}}
            q_h(s| \pi_c, P) 
        \cdot 1
        \right]\\
        & + 
        \mathbb{E}_{c \sim \mathcal{D}}
        \left[
            \sum_{h=0}^{H-1} \sum_{a\in A} 
            \pi_c(a|s)
            \sum_{s \in S_h : p_{s} > \frac{1}{|S|}}
            p_s
            (|r^c(s,a)-f_{s,a}(c)|- \alpha_2)
        \right]\\
        & +
        \mathbb{E}_{c \sim \mathcal{D}}
        \left[
            \sum_{h=0}^{H-1} \sum_{a\in A} 
            \pi_c(a|s)
            \sum_{s \in S_h : p_{s} \in [\frac{\epsilon}{6|S|}, \frac{1}{|S|}]}
            p_s
            (|r^c(s,a)-f_{s,a}(c)|- \alpha_2)
        \right]\\
        & \leq 
        \alpha_2 H
        +        
        \frac{\epsilon}{6}
        \\
        & + 
        \mathbb{E}_{c \sim \mathcal{D}}
        \left[
            \sum_{h=0}^{H-1} 
            \sum_{a\in A} 
            \sum_{s \in S_h : p_{s} > \frac{1}{|S|}}
            p_s
            (|r^c(s,a)-f_{s,a}(c)|- \alpha_2)
        \right]\\
        & +
        \mathbb{E}_{c \sim \mathcal{D}}
        \left[
            \sum_{h=0}^{H-1} \sum_{a\in A} 
            \sum_{s \in S_h : p_{s} \in [\frac{\epsilon}{6|S|}, \frac{1}{|S|}]}
            p_s
            (|r^c(s,a)-f_{s,a}(c)|- \alpha_2)
        \right]\\
        & =
        \alpha_2 H
        +        
        \frac{\epsilon}{6}
        \\
        & + 
            \sum_{h=0}^{H-1} 
            \sum_{a\in A} 
            \sum_{s \in S_h : p_{s} > \frac{1}{|S|}}
            p_s
            \mathbb{E}_{c \sim \mathcal{D}}
            \left[    
                (|r^c(s,a)-f_{s,a}(c)|- \alpha_2)
            \right]\\
        & +
            \sum_{h=0}^{H-1} \sum_{a\in A} 
            \sum_{s \in S_h : p_{s} \in [\frac{\epsilon}{6|S|}, \frac{1}{|S|}]}
            p_s
            \mathbb{E}_{c \sim \mathcal{D}}
            \left[    
                (|r^c(s,a)-f_{s,a}(c)|- \alpha_2)
            \right]\\
        &\underbrace{\leq}_{\text{By ineq~\ref{ineq: case 1 G2}}}
        \alpha_2 H
        +
        \frac{\epsilon}{6}
        + 
            \sum_{h=0}^{H-1} 
            \sum_{a\in A} 
            \sum_{s \in S_h : p_{s} > \frac{1}{|S|}}
            p_s
            \sqrt{\epsilon^2_\star(p_s)}
        +
            \sum_{h=0}^{H-1} \sum_{a\in A} 
            \sum_{s \in S_h : p_{s} \in [\frac{\epsilon}{6|S|}, \frac{1}{|S|}]}
            p_s
            \sqrt{\epsilon^2_\star(p_s)}\\  
        & =
        \alpha_2 H
        +
        \frac{\epsilon}{6}
        + 
            \sum_{h=0}^{H-1} 
            \sum_{a\in A} 
            \sum_{s \in S_h : p_{s} > \frac{1}{|S|}}
            p_s
            \frac{\epsilon}{6H|S||A|}
        +
            \sum_{h=0}^{H-1} \sum_{a\in A} 
            \sum_{s \in S_h : p_{s} \in [\frac{\epsilon}{6|S|}, \frac{1}{|S|}]}
            p_s
            \frac{\epsilon}{6p_s|S||A|} 
        \\
        & =        
        \alpha_2 H
        +
        \frac{\epsilon}{2}.
    \end{align*}
    \endgroup







\end{proof}

\begin{theorem}\label{thm: PAC optimal policy for Known CFD with l_2}
With probability at least $1-\delta$ it holds that
\[
        \mathbb{E}_{c \sim \mathcal{D}}[V^{\pi^\star_c}_{\mathcal{M}(c)}(s_0) - V^{\widehat{\pi}^\star_c}_{\mathcal{M}(c)}(s_0)] \leq 
         \epsilon + 2\alpha_2 H,
\]
where
$\pi^\star =(\pi^\star_c)_{c \in \mathcal{C}}$ is the optimal policy for the true model $\mathcal{M}(c)$, and
 $\widehat{\pi}^\star = (\widehat{\pi}^\star_c)_{c\in \mathcal{C}}$ is the optimal policy for $\widehat{\mathcal{M}}(c)$.
\end{theorem}

\begin{proof}
Assume the good events $G_1$ and $G_2$ hold. 
By Lemma~\ref{lemma: vall-diff l_2 KCFD} we have for $\pi^\star =(\pi^\star_c)_{c \in \mathcal{C}}$ that
\begin{equation*}
    \left|\mathbb{E}_{c \sim \mathcal{D}}
    [
        V^{\pi^\star_c}_{\mathcal{M}(c)}(s_0) 
        -  V^{\pi^\star_c}_{\widehat{\mathcal{M}}(c)}(s_0)
    ]\right|
    \leq
    \mathbb{E}_{c \sim \mathcal{D}}[|V^{\pi^\star_c}_{\mathcal{M}(c)}(s_0) -  V^{\pi^\star_c}_{\widehat{\mathcal{M}}(c)}(s_0)|]
    \leq
     \frac{1}{2}\epsilon + \alpha_2 H,
\end{equation*}
yielding,
\begin{equation*}
    \mathbb{E}_{c \sim \mathcal{D}}
    [
        V^{\pi^\star_c}_{\mathcal{M}(c)}(s_0) 
    ]
    -
    \mathbb{E}_{c \sim \mathcal{D}}
    [
    V^{\pi^\star_c}_{\widehat{\mathcal{M}}(c)}(s_0)
    ]
    \leq
    \frac{1}{2}\epsilon + \alpha_2 H.
\end{equation*}

Similarly, we have for $\widehat{\pi}^\star = (\widehat{\pi}^\star_c)_{c\in \mathcal{C}}$ that
\begin{equation*}
    \mathbb{E}_{c \sim \mathcal{D}}[ V^{\widehat{\pi}^\star_c}_{\widehat{\mathcal{M}}(c)}(s_0)]
    - 
    \mathbb{E}_{c \sim \mathcal{D}}[
    V^{\widehat{\pi}^\star_c}_{\mathcal{M}(c)}(s_0) ] \leq \frac{1}{2}\epsilon + \alpha_2 H.
\end{equation*}
Also, for all $c \in \mathcal{C}$, since $\widehat{\pi}^\star_c$ is the optimal policy for $\widehat{\mathcal{M}}(c)$ we have $V^{\widehat{\pi}^\star_c}_{\widehat{\mathcal{M}}(c)}(s_0) \geq V^{\pi^\star_c}_{\widehat{\mathcal{M}}(c)}(s_0)$, which implies that
\begin{equation*}
    \mathbb{E}_{c \sim \mathcal{D}}[V^{\pi^\star_c}_{\widehat{\mathcal{M}}(c)}(s_0)]
    -
    \mathbb{E}_{c \sim \mathcal{D}}[V^{\widehat{\pi}^\star_c}_{\widehat{\mathcal{M}}(c)}(s_0)] 
    \leq
    0
 .
\end{equation*}
Since by Lemma~\ref{lemma: good events probs KCFD} we have that $G_1$ and $G_2$ hold with probability at least $1-{\delta}/{2}$,
the theorem follows by summing the above three inequalities.
\end{proof}

\subsubsection{Analysis for \texorpdfstring{$\ell_1$}{Lg} loss}\label{subsubsec:KCFD-l-1}

Let $\alpha_1 := \max_{(s,a) \in S \times A} \alpha_1(\mathcal{F}^R_{s,a})$. 
The following lemma shows that under the good events $G_1$ and $G_2$, the value of any context-dependent policy with respect to the approximated model $\widehat{\mathcal{M}}$ is similar to that with respect to the true model $\mathcal{M}$, in expectation over the context.

\begin{lemma}\label{lemma: vall-diff l_1 KCFD}
Assume the events $G_1$ and $G_2$ hold.
Then for any context-dependent policy ${\pi =(\pi_c: S \to \Delta(A))_{c\in \mathcal{C}}}$
it holds that 
\begin{equation*}
    \mathbb{E}_{c\sim \mathcal{D}}
    [|V^{\pi_c}_{\mathcal{M}(c)}(s_0) - V^{\pi_c}_{\widehat{\mathcal{M}}(c)}(s_0)|] \leq 
    \frac{1}{2}\epsilon + \alpha_1 H.
\end{equation*}

\end{lemma}

\begin{proof}
Recall that for every context $c \in \mathcal{C}$, state $ s\in S$ and action $ a\in A$, the expected reward is $r^c(s,a) \in  [0,1]$.
By construction of $\widehat{\mathcal{M}}(c)$ , for any $s\in S$ which are not $\frac{\epsilon}{6|S|}$-reachable, we set $f_{s,a}(c)=0$ for any action $a$. Hence,
\begin{equation*}
    |r^c(s,a) - f_{s,a}(c)| \leq 1.
\end{equation*}

Since the good event $G_2$ holds, for every state-action pair $(s,a)$ such that $s$ is $\frac{\epsilon}{6|S|}$-reachable it holds that
    \begin{equation}
        \epsilon^1_\star(p_s) 
        +
        \alpha_1
        \geq
        \epsilon^1_\star(p_s) 
        +
        \alpha_1(\mathcal{F}^R_{s,a})
        \geq
        \mathbb{E}_{c \sim \mathcal{D}}
        [|f_{s,a}(c) - r^c(s,a)|].
    \end{equation}
The above implies that   
    \begin{equation}\label{ineq: case 1 G2 l_1}
        \epsilon^1_\star(p_s) 
        \geq
        \mathbb{E}_{c \sim \mathcal{D}}
        [|f_{s,a}(c) - r^c(s,a)| -\alpha_1].
    \end{equation}
    
Fix any context-dependent policy $\pi =(\pi_c: S \to \Delta(A))_{c\in \mathcal{C}}$. 
%
By definition of the value function we have for any fixed context $c$:
    \begin{align*}
        V^{\pi_c}_{\mathcal{M}(c)}(s_0)
        =
        \mathbb{E}_{\pi_c, \mathcal{M}(c)} 
        \Big[
        \sum_{h=0}^{H-1} r^c(s_h, a_h)|s_0 = s_0 \Big]
        = 
        \sum_{h=0}^{H-1}\sum_{s\in S_h}q_h(s| \pi_c, P) 
        \sum_{a\in A}\pi_c(a|s)r^c(s,a),
    \end{align*}
    and
    \begin{align*}
        V^{\pi_c}_{\widehat{\mathcal{M}}(c)}(s_0)
        = 
        \mathbb{E}_{\pi_c, \widehat{\mathcal{M}}(c)} 
        \Big[
            \sum_{h=0}^{H-1} \widehat{r}^c(s_h, a_h)|s_0 \Big]
        = 
         \sum_{h=0}^{H-1}\sum_{s\in S_h} q_h(s| \pi_c, P)
         \sum_{a\in A}\pi(a|s) f_{s,a}(c).
    \end{align*}


 By inequality~\ref{ineq: case 1 G2 l_1}, linearity of expectation and triangle inequality we derive the following.
 \begingroup
 \allowdisplaybreaks
     \begin{align*}
        &\mathbb{E}_{c \sim \mathcal{D}}
        \left[
            \left|
                V^{\pi_c}_{\mathcal{M}(c)}(s_0) 
                - 
                V^{\pi_c}_{\widehat{\mathcal{M}}(c)}(s_0)
            \right|
        \right]
        \\
        = &
        \mathbb{E}_{c \sim \mathcal{D}}
        \left[
            \left|
                \sum_{h=0}^{H-1}\sum_{s\in S_h} q_h(s| \pi_c, P) 
                \sum_{a\in A}\pi_c(a|s) r^c(s,a)
                - 
                \sum_{h=0}^{H-1} \sum_{s\in S_h} q_h(s| \pi_c, P) 
                \sum_{a\in A}\pi_c(a|s) f_{s,a}(c)
            \right|
        \right] \\
        \leq &
        \mathbb{E}_{c \sim \mathcal{D}}
        \left[   
            \sum_{h=0}^{H-1}
            \sum_{s\in S_h} q_h(s| \pi_c, P)
            \sum_{a\in A} \pi_c(a|s) |r^c(s,a)-f_{s,a}(c)|
        \right] \\
        = &
        \mathbb{E}_{c \sim \mathcal{D}}
        \left[   
            \sum_{h=0}^{H-1}
            \sum_{s\in S_h} q_h(s| \pi_c, P)
            \sum_{a\in A} \pi_c(a|s) (|r^c(s,a)-f_{s,a}(c)|- \alpha_1 + \alpha_1)
        \right] \\
        \\
        = &
        \alpha_1 H
        +
        \mathbb{E}_{c \sim \mathcal{D}}
        \left[   
            \sum_{h=0}^{H-1}
            \sum_{s\in S_h} q_h(s| \pi_c, P)
            \sum_{a\in A} \pi_c(a|s) (|r^c(s,a)-f_{s,a}(c)|- \alpha_1)
        \right] \\
        = &
        \alpha_1 H
        +
        \mathbb{E}_{c \sim \mathcal{D}}
        \left[
            \sum_{h=0}^{H-1}\sum_{a\in A}
            \sum_{s \in S_h : p_{s} < \frac{\epsilon}{6|S|}}
            q_h(s| \pi_c, P) \pi_c(a|s) 
       (|r^c(s,a)-f_{s,a}(c)|- \alpha_1)
        \right]\\
        & + 
        \mathbb{E}_{c \sim \mathcal{D}}
        \left[
            \sum_{h=0}^{H-1} \sum_{a\in A} 
            \sum_{s \in S_h : p_{s} > \frac{1}{|S|}}
            q_h(s| \pi_c, P) \pi_c(a|s) 
        (|r^c(s,a)-f_{s,a}(c)|- \alpha_1)
        \right]\\
        & +
        \mathbb{E}_{c \sim \mathcal{D}}
        \left[
            \sum_{h=0}^{H-1} \sum_{a\in A} 
            \sum_{s \in S_h : p_{s} \in [\frac{\epsilon}{6|S|}, \frac{1}{|S|}]}
            q_h(s|\pi_c, P) \pi_c(a|s)
            (|r^c(s,a)-f_{s,a}(c)|- \alpha_1)
        \right]\\
        \leq &
        \alpha_1 H
        +
        \mathbb{E}_{c \sim \mathcal{D}}
        \left[
            \sum_{h=0}^{H-1}\sum_{a\in A}
            \pi_c(a|s) 
            \sum_{s \in S_h : p_{s} < \frac{\epsilon}{6|S|}}
            q_h(s| \pi_c, P) 
        \cdot 1
        \right]\\
        & + 
        \mathbb{E}_{c \sim \mathcal{D}}
        \left[
            \sum_{h=0}^{H-1} \sum_{a\in A} 
            \pi_c(a|s)
            \sum_{s \in S_h : p_{s} > \frac{1}{|S|}}
            p_s
            (|r^c(s,a)-f_{s,a}(c)|- \alpha_1)
        \right]\\
        & +
        \mathbb{E}_{c \sim \mathcal{D}}
        \left[
            \sum_{h=0}^{H-1} \sum_{a\in A} 
            \pi_c(a|s)
            \sum_{s \in S_h : p_{s} \in [\frac{\epsilon}{6|S|}, \frac{1}{|S|}]}
            p_s
            (|r^c(s,a)-f_{s,a}(c)|- \alpha_1)
        \right]\\
        \leq & 
        \alpha_1 H
        +        
        \frac{\epsilon}{6}
        \\
        & + 
        \mathbb{E}_{c \sim \mathcal{D}}
        \left[
            \sum_{h=0}^{H-1} 
            \sum_{a\in A} 
            \sum_{s \in S_h : p_{s} > \frac{1}{|S|}}
            p_s
            (|r^c(s,a)-f_{s,a}(c)|- \alpha_1)
        \right]\\
        & +
        \mathbb{E}_{c \sim \mathcal{D}}
        \left[
            \sum_{h=0}^{H-1} \sum_{a\in A} 
            \sum_{s \in S_h : p_{s} \in [\frac{\epsilon}{6|S|}, \frac{1}{|S|}]}
            p_s
            (|r^c(s,a)-f_{s,a}(c)|- \alpha_1)
        \right]\\
        = &
        \alpha_1 H
        +        
        \frac{\epsilon}{6}
        \\
        & + 
            \sum_{h=0}^{H-1} 
            \sum_{a\in A} 
            \sum_{s \in S_h : p_{s} > \frac{1}{|S|}}
            p_s
            \mathbb{E}_{c \sim \mathcal{D}}
            \left[    
                (|r^c(s,a)-f_{s,a}(c)|- \alpha_1)
            \right]\\
        & +
            \sum_{h=0}^{H-1} \sum_{a\in A} 
            \sum_{s \in S_h : p_{s} \in [\frac{\epsilon}{6|S|}, \frac{1}{|S|}]}
            p_s
            \mathbb{E}_{c \sim \mathcal{D}}
            \left[    
                (|r^c(s,a)-f_{s,a}(c)|- \alpha_1)
            \right]\\
        \underbrace{\leq}_{\text{By ineq~\ref{ineq: case 1 G2 l_1}}}&
        \alpha_1 H
        +
        \frac{\epsilon}{6}
        + 
            \sum_{h=0}^{H-1} 
            \sum_{a\in A} 
            \sum_{s \in S_h : p_{s} > \frac{1}{|S|}}
            p_s
            \epsilon^1_\star(p_s)
        +
            \sum_{h=0}^{H-1} \sum_{a\in A} 
            \sum_{s \in S_h : p_{s} \in [\frac{\epsilon}{6|S|}, \frac{1}{|S|}]}
            p_s
            \epsilon^1_\star(p_s)
        \\   
        = &
        \alpha_1 H
        +
        \frac{\epsilon}{6}
        + 
            \sum_{h=0}^{H-1} 
            \sum_{a\in A} 
            \sum_{s \in S_h : p_{s} > \frac{1}{|S|}}
            p_s
            \frac{\epsilon}{6H|S||A|}
        +
            \sum_{h=0}^{H-1} \sum_{a\in A} 
            \sum_{s \in S_h : p_{s} \in [\frac{\epsilon}{6|S|}, \frac{1}{|S|}]}
            p_s
            \frac{\epsilon}{6p_s|S||A|} 
        \\
        = &        
        \alpha_1 H
        +
        \frac{\epsilon}{2}. 
    \end{align*}
    \endgroup

\end{proof}

\begin{theorem}\label{thm: PAC optimal policy for Known CFD with l_1}
With probability at least $1-\delta$ it holds that
\[
        \mathbb{E}_{c \sim \mathcal{D}}[V^{\pi^\star_c}_{\mathcal{M}(c)}(s_0) - V^{\widehat{\pi}^\star_c}_{\mathcal{M}(c)}(s_0)] \leq 
        \epsilon +2\alpha_1 H,
\]
where
$\pi^\star =(\pi^\star_c)_{c \in \mathcal{C}}$ is the optimal policy for the true model $\mathcal{M}(c)$, and
 $\widehat{\pi}^\star = (\widehat{\pi}^\star_c)_{c\in \mathcal{C}}$ is the optimal policy for $\widehat{\mathcal{M}}(c)$.
\end{theorem}

\begin{proof}
Assume the good events $G_1$ and $G_2$ hold. 
By Lemma~\ref{lemma: vall-diff l_1 KCFD} we have for $\pi^\star =(\pi^\star_c)_{c \in \mathcal{C}}$,that
\begin{equation*}
    \left|\mathbb{E}_{c \sim \mathcal{D}}
    [
        V^{\pi^\star_c}_{\mathcal{M}(c)}(s_0) 
        -  V^{\pi^\star_c}_{\widehat{\mathcal{M}}(c)}(s_0)
    ]\right|
    \leq
    \mathbb{E}_{c \sim \mathcal{D}}[|V^{\pi^\star_c}_{\mathcal{M}(c)}(s_0) -  V^{\pi^\star_c}_{\widehat{\mathcal{M}}(c)}(s_0)|]
    \leq
     \frac{1}{2}\epsilon + \alpha_1 H,
\end{equation*}
yielding,
\begin{equation*}
    \mathbb{E}_{c \sim \mathcal{D}}
    [
        V^{\pi^\star_c}_{\mathcal{M}(c)}(s_0) 
    ]
    -
    \mathbb{E}_{c \sim \mathcal{D}}
    [
    V^{\pi^\star_c}_{\widehat{\mathcal{M}}(c)}(s_0)
    ]
    \leq
    \frac{1}{2}\epsilon + \alpha_1 H.
\end{equation*}

Similarly, we have for $\widehat{\pi}^\star = (\widehat{\pi}^\star_c)_{c\in \mathcal{C}}$ that
\begin{equation*}
    \mathbb{E}_{c \sim \mathcal{D}}[ V^{\widehat{\pi}^\star_c}_{\widehat{\mathcal{M}}(c)}(s_0)]
    - 
    \mathbb{E}_{c \sim \mathcal{D}}[
    V^{\widehat{\pi}^\star_c}_{\mathcal{M}(c)}(s_0) ] \leq \frac{1}{2}\epsilon + \alpha_1 H.
\end{equation*}
Also, for all $c \in \mathcal{C}$, since $\widehat{\pi}^\star_c$ is the optimal policy for $\widehat{\mathcal{M}}(c)$ we have $V^{\widehat{\pi}^\star_c}_{\widehat{\mathcal{M}}(c)}(s_0) \geq V^{\pi^\star_c}_{\widehat{\mathcal{M}}(c)}(s_0)$, which implies that
\begin{equation*}
    \mathbb{E}_{c \sim \mathcal{D}}[V^{\pi^\star_c}_{\widehat{\mathcal{M}}(c)}(s_0)]
    -
    \mathbb{E}_{c \sim \mathcal{D}}[V^{\widehat{\pi}^\star_c}_{\widehat{\mathcal{M}}(c)}(s_0)] 
    \leq
    0
 .
\end{equation*}


Since by Lemma~\ref{lemma: good events probs KCFD} we have that $G_1$ and $G_2$ hold with probability at least $1-\delta/2$,
the theorem follows by summing the above three inequalities.
\end{proof}

\subsection{Sample Complexity Bounds}\label{sec: sampel complexity KCFD}

Given standard sample complexity bounds for learning a function class using ERM, we can bound the required sample complexity of our Algorithm EXPLORE-KCFD. The following theorems state the sample complexity bounds.

\begin{theorem}[Adaption of Theorem 19.2 in~\cite{Bartlett1999NeuralNetsBook}]\label{thm: pseudo dim}
    Let $\mathcal{F}$ be a hypothesis space of real valued functions with a finite pseudo dimension, denoted $Pdim(\mathcal{F}) < \infty$. Then, $\mathcal{F}$ has a uniform convergence with 
    \[
        m(\epsilon, \delta) = O \Big( \frac{1}{\epsilon^2}( Pdim(\mathcal{F}) \ln \frac{1}{\epsilon} + \ln \frac{1}{\delta})\Big).
    \]    
\end{theorem}

\begin{theorem}[Adaption of Theorem 19.1 in~\cite{Bartlett1999NeuralNetsBook}]\label{thm: fat dim}
    Let $\mathcal{F}$ be a hypothesis space of real valued functions with a finite fat-shattering dimension, denoted $fat_{\mathcal{F}}$. Then, $\mathcal{F}$ has a uniform convergence with 
    \[
        m(\epsilon, \delta) = O \Big( \frac{1}{\epsilon^2}( fat_{\mathcal{F}}(\epsilon/256)  \ln^2 \frac{1}{\epsilon} + \ln \frac{1}{\delta})\Big).
    \]    
\end{theorem}

\subsubsection{Sample bounds for the \texorpdfstring{$\ell_2$}{Lg} loss}\label{subsubsec:sample-complexity-l-2-KCFD}

We prove sample complexity bound of our algorithm for function classes with finite Pseudo dimension when using $\ell_2$ loss.
\begin{corollary}\label{corl: sample complexity KCFD l_2 pseudo}
    Assume that for every $(s,a) \in S \times A$ we have that $Pdim(\mathcal{F}^R_{s,a}) < \infty$.
    Let $Pdim = \max_{(s,a) \in S \times A} Pdim(\mathcal{F}^R_{s,a})$.
    Then, after collecting 
    \[
        O \Big( \frac{|S|^2|A|}{\epsilon} \ln{\frac{|S||A|}{\delta}} +
        \frac{H^4 |S|^6|A|^5 \; (Pdim \ln{\frac{H^2 |S|^2 |A|^2}{\epsilon^2}} + \ln {\frac{|S||A|}{\delta}})}{\epsilon^4}  \Big)
    \] 
    trajectories, with probability at least $1-\delta$ it holds that
    \[
        \mathbb{E}_{c \sim \mathcal{D}}[V^{\pi^\star_c}_{\mathcal{M}(c)}(s_0) - V^{\widehat{\pi}^\star_c}_{\mathcal{M}(c)}(s_0)] \leq 
        \epsilon +2\alpha_2 H.
    \]
\end{corollary}

\begin{proof}
    Recall that for each state-action pair $(s,a)$ such that $s$ is $\frac{\epsilon}{6 |S|}$-reachable, we run for $T_{s,a} =   \lceil \frac{2}{p_s}(\ln(\frac{1}{\delta_1})+N_R(\mathcal{F}^R_{s,a} ,\epsilon_\star^2(p_s), \delta_1)) \rceil$ episodes.
    By Theorem~\ref{thm: PAC optimal policy for Known CFD with l_2}, for
    $
        \sum_{h =0}^{H-1} \sum_{s \in S_h : p_s \geq \epsilon/6 |S|} \sum_{a \in A} T_{s,a}
    $    
    samples we have with probability at least $1-\delta$ that
    \[
        \mathbb{E}_{c \sim \mathcal{D}}[V^{\pi^\star_c}_{\mathcal{M}(c)}(s_0) - V^{\widehat{\pi}^\star_c}_{\mathcal{M}(c)}(s_0)] \leq 
       \epsilon +2\alpha_2 H.
    \]
    Since for every $(s,a) \in S \times A$ we have that $Pdim(\mathcal{F}^R_{s,a}) < \infty$, and $Pdim = \max_{(s,a) \in S \times A} Pdim(\mathcal{F}^R_{s,a})$, by Theorem~\ref{thm: pseudo dim}, for every $(s,a) \in S \times A$  we have 
    \[
        N_R(\mathcal{F}^R_{s,a}, \epsilon^2_\star(p_s), \delta_1)
        =
        O \Big( \frac{ Pdim \ln \frac{1}{\epsilon^2_\star(p_s)} + \ln \frac{1}{\delta_1}}{\epsilon_\star^4(p_s)}\Big).
    \]
    Using the accuracy-per-state function, we derive the overall sample complexity bound in the following computation.
    \begingroup
    \allowdisplaybreaks
    \begin{align*}
        \sum_{h=0}^{H-1} \sum_{s \in S_h: p_s \geq \epsilon/6|S|}
        \sum_{a \in A}
        T_{s,a}
        =
        \sum_{h=0}^{H-1} \sum_{s \in S_h: p_s \geq \epsilon/6|S|}
        \sum_{a \in A}
        O\Big( \frac{1}{p_s}(\ln(\frac{1}{\delta_1})+ N_R(\mathcal{F}^R_{s,a} ,\epsilon_\star^2(p_s), \delta_1))  \Big)
    \end{align*}     
    \begin{align*}     
        =&
        \sum_{h=0}^{H-1} \sum_{s \in S_h: p_s \geq \epsilon/6|S|}
        \sum_{a \in A}
        O\Big( \frac{1}{p_s}(\ln(\frac{1}{\delta_1})
        +
        \frac{ Pdim \ln \frac{1}{\epsilon^2_\star (p_s)} + \ln \frac{1}{\delta_1}}{\epsilon_\star^4(p_s)})  \Big)
        \\ 
        =&
        \sum_{h=0}^{H-1} \sum_{s \in S_h: p_s \in [\epsilon/6|S|,1/|S|]}
        \sum_{a \in A}
        O\Big( \frac{1}{p_s}(\ln(\frac{1}{\delta_1})
        +
        \frac{ Pdim \ln \frac{1}{\epsilon^2/36 p_s^2 |S|^2 |A|^2} 
        + \ln \frac{1}{\delta_1}}{\epsilon^4/36^2 p_s^4 |S|^4 |A|^4})  \Big)
        \\ 
        &+
        \sum_{h=0}^{H-1} \sum_{s \in S_h: p_s > 1/|S|}
        \sum_{a \in A}
        O\Big( \frac{1}{p_s}(\ln(\frac{1}{\delta_1})
        +
        \frac{ Pdim \ln \frac{1}{\epsilon^2 / 36 H^2 |S|^2 |A|^2} + \ln \frac{1}{\delta_1}}{\epsilon^4 / 36^2 H^4 |S|^4 |A|^4})  \Big)
        \\
        =&
        \sum_{h=0}^{H-1} \sum_{s \in S_h: p_s \in [\epsilon/6|S|,1/|S|]}
        \sum_{a \in A}
        O\Big( \frac{1}{p_s}(\ln(\frac{1}{\delta_1})
        +
       \frac{  p_s^4 |S|^4 |A|^4(Pdim \ln \frac{p_s^2 |S|^2 |A|^2}{\epsilon^2} 
        + \ln \frac{1}{\delta_1})}{\epsilon^4} ) \Big)
        \\ 
        &+
        \sum_{h=0}^{H-1} \sum_{s \in S_h: p_s > 1/|S|}
        \sum_{a \in A}
        O\Big( \frac{1}{p_s}(\ln(\frac{1}{\delta_1})
        +
        \frac{H^4 |S|^4 |A|^4 (Pdim \ln \frac{H^2 |S|^2 |A|^2}{\epsilon^2 } + \ln \frac{1}{\delta_1})}{\epsilon^4} ) \Big)
        \\
        \leq&
        \sum_{h=0}^{H-1} \sum_{s \in S_h: p_s \in [\epsilon/6|S|,1/|S|]}
        \sum_{a \in A}
        O\Big( \frac{|S|}{\epsilon}\ln{\frac{|S||A|}{\delta}}
        +
        \frac{ p_s^3 |S|^4 |A|^4 (Pdim \ln{\frac{|S|^2 |A|^2}{\epsilon^2}} 
        + \ln {\frac{|S||A|}{\delta}})}{\epsilon^4}  \Big)
        \\ 
        &+
        \sum_{h=0}^{H-1} \sum_{s \in S_h: p_s > 1/|S|}
        \sum_{a \in A}
        O\Big( |S| \ln{\frac{|S||A|}{\delta }}
        +
        \frac{ H^4 |S|^5 |A|^4 (Pdim \ln \frac{H^2 |S|^2 |A|^2}{\epsilon^2 } + \ln \frac{|S||A|}{\delta})}{\epsilon^4}  \Big)
        \\
        \underbrace{\leq}_{(\star)}& 
        \sum_{h=0}^{H-1} \sum_{s \in S_h: p_s \in [\epsilon/6|S|,1/|S|]}
        \sum_{a \in A}
        O\Big( \frac{|S|}{\epsilon}\ln{\frac{|S||A|}{\delta}}
        +
        \frac{|S| |A|^4 (Pdim \ln{\frac{|S|^2 |A|^2}{\epsilon^2}} 
        + \ln {\frac{|S||A|}{\delta}})}{\epsilon^4}  \Big)
        \\ 
        &+
        \sum_{h=0}^{H-1} \sum_{s \in S_h: p_s > 1/|S|}
        \sum_{a \in A}
        O\Big( |S| \ln{\frac{|S||A|}{\delta }}
        +
        \frac{H^4 |S|^5 |A|^4 (Pdim \ln \frac{H^2 |S|^2 |A|^2}{\epsilon^2 } + \ln \frac{|S||A|}{\delta})}{\epsilon^4}  \Big)
        \\ 
        =&
        O \Big( \frac{|S|^2|A|}{\epsilon} \ln{\frac{|S||A|}{\delta}} +
        \frac{H^4 |S|^6|A|^5 \; (Pdim \ln{\frac{H^2 |S|^2 |A|^2}{\epsilon^2}} + \ln {\frac{|S||A|}{\delta}})}{\epsilon^4}  \Big)\\
    \end{align*}
    \endgroup
    Where $(\star)$ is since in that regime we have $p_s \in [\epsilon/6 |S|, 1/|S|]$, hence $p^3_s \leq 1/ |S|^3$.
\end{proof}
We also show similar sample complexity for function classes with finite fat-shattering dimension when using $\ell_2$ loss.
\begin{remark}
    The sample complexity for function classes with finite fat-shattering dimension with $\ell_2$ loss, where in $Fdim$ below we also maximizes over $\epsilon^2_\star(p_s) $ and the maximum is bounded and independent of $p_s$.
\end{remark}

\begin{corollary}\label{corl: sample complexity KCFD l_2 fat}
    Assume that for every $(s,a) \in S \times A$ we have that $\mathcal{F}^R_{s,a}$ has finite fat-shattering dimention.
    Let $Fdim = \max_{(s,a) \in S \times A} fat_{\mathcal{F}^R_{s,a}}(\epsilon^2_\star(p_s)/256)$.
    Then, after collecting 
    \[
        O \Big( \frac{|S|^2|A|}{\epsilon} \ln{\frac{|S||A|}{\delta}} +
        \frac{H^4 |S|^6|A|^5 \; (Fdim \ln^2{\frac{H^2 |S|^2 |A|^2}{\epsilon^2}} + \ln {\frac{|S||A|}{\delta}})}{\epsilon^4}  \Big)
    \] 
    trajectories, with probability at least $1-\delta$ it holds that
    \[
        \mathbb{E}_{c \sim \mathcal{D}}[V^{\pi^\star_c}_{\mathcal{M}(c)}(s_0) - V^{\widehat{\pi}^\star_c}_{\mathcal{M}(c)}(s_0)] \leq 
        \epsilon +2\alpha_2 H.
    \]
\end{corollary}

\begin{proof}
    Recall that for each state-action pair $(s,a)$ such that $s$ is $\frac{\epsilon}{6 |S|}$-reachable, we run for $T_{s,a} =   \lceil \frac{2}{p_s}(\ln(\frac{1}{\delta_1})+N_R(\mathcal{F}^R_{s,a} ,\epsilon_\star^2(p_s), \delta_1)) \rceil$ episodes.
    By Theorem~\ref{thm: PAC optimal policy for Known CFD with l_2}, for $\sum_{h =0}^{H-1} \sum_{s \in S_h : p_s \geq \epsilon/6 |S|} \sum_{a \in A} T_{s,a}$ samples we have with probability at least $1-\delta$ that
    \[
        \mathbb{E}_{c \sim \mathcal{D}}[V^{\pi^\star_c}_{\mathcal{M}(c)}(s_0) - V^{\widehat{\pi}^\star_c}_{\mathcal{M}(c)}(s_0)] \leq 
        \epsilon +2\alpha_2 H.
    \]
    Since for every $(s,a) \in S \times A$ we have that $\mathcal{F}^R_{s,a}$ has finite fat-shattering dimension, and $Fdim = \max_{(s,a) \in S \times A} fat_{\mathcal{F}^R_{s,a}}(\epsilon^2_\star(p_s))$, by Theorem~\ref{thm: fat dim}, for every $(s,a) \in S \times A$  we have 
    \[
        N_R(\mathcal{F}^R_{s,a}, \epsilon^2_\star(p_s), \delta_1)
        =
        O \Big( \frac{ Fdim \ln^2 \frac{1}{\epsilon^2_\star(p_s)} + \ln \frac{1}{\delta_1}}{\epsilon_\star^4(p_s)}\Big).
    \]
    Using the accuracy-per-state function, we derive the overall sample complexity bound in the following computation.
    \begingroup
    \allowdisplaybreaks
    \begin{align*}
        &\sum_{h=0}^{H-1} \sum_{s \in S_h: p_s \geq \epsilon/6|S|}
        \sum_{a \in A}
        T_{s,a}
        \\
        = &
        \sum_{h=0}^{H-1} \sum_{s \in S_h: p_s \geq \epsilon/6|S|}
        \sum_{a \in A}
        O\Big( \frac{1}{p_s}(\ln(\frac{1}{\delta_1})+ N_R(\mathcal{F}^R_{s,a} ,\epsilon_\star^2(p_s), \delta_1))  \Big)
        \\
        = &
        \sum_{h=0}^{H-1} \sum_{s \in S_h: p_s \geq \epsilon/6|S|}
        \sum_{a \in A}
        O\Big( \frac{1}{p_s}(\ln(\frac{1}{\delta_1})
        +
        \frac{ Fdim \ln^2 \frac{1}{\epsilon^2_\star (p_s)} + \ln \frac{1}{\delta_1}}{\epsilon_\star^4(p_s)})  \Big)
        \\ 
        = &
        \sum_{h=0}^{H-1} \sum_{s \in S_h: p_s \in [\epsilon/6|S|,1/|S|]}
        \sum_{a \in A}
        O\Big( \frac{1}{p_s}(\ln(\frac{1}{\delta_1})
        +
        \frac{ Fdim \ln^2 \frac{1}{\epsilon^2/36 p_s^2 |S|^2 |A|^2} 
        + \ln \frac{1}{\delta_1}}{\epsilon^4/36^2 p_s^4 |S|^4 |A|^4})  \Big)
        \\ 
        & +
        \sum_{h=0}^{H-1} \sum_{s \in S_h: p_s > 1/|S|}
        \sum_{a \in A}
        O\Big( \frac{1}{p_s}(\ln(\frac{1}{\delta_1})
        +
        \frac{ Fdim \ln^2 \frac{1}{\epsilon^2 / 36 H^2 |S|^2 |A|^2} + \ln \frac{1}{\delta_1}}{\epsilon^4 / 36^2 H^4 |S|^4 |A|^4})  \Big)
        \\
        = &
        \sum_{h=0}^{H-1} \sum_{s \in S_h: p_s \in [\epsilon/6|S|,1/|S|]}
        \sum_{a \in A}
        O\Big( \frac{1}{p_s}(\ln(\frac{1}{\delta_1})
        +
       \frac{  p_s^4 |S|^4 |A|^4(Fdim \ln^2 \frac{p_s^2 |S|^2 |A|^2}{\epsilon^2} 
        + \ln \frac{1}{\delta_1})}{\epsilon^4} ) \Big)
        \\ 
        & +
        \sum_{h=0}^{H-1} \sum_{s \in S_h: p_s > 1/|S|}
        \sum_{a \in A}
        O\Big( \frac{1}{p_s}(\ln(\frac{1}{\delta_1})
        +
        \frac{H^4 |S|^4 |A|^4 (Fdim \ln^2 \frac{H^2 |S|^2 |A|^2}{\epsilon^2 } + \ln \frac{1}{\delta_1})}{\epsilon^4} ) \Big)
        \\
        \leq &
        \sum_{h=0}^{H-1} \sum_{s \in S_h: p_s \in [\epsilon/6|S|,1/|S|]}
        \sum_{a \in A}
        O\Big( \frac{|S|}{\epsilon}\ln{\frac{|S||A|}{\delta}}
        +
        \frac{ p_s^3 |S|^4 |A|^4 (Fdim \ln^2{\frac{|S|^2 |A|^2}{\epsilon^2}} 
        + \ln {\frac{|S||A|}{\delta}})}{\epsilon^4}  \Big)
        \\ 
        & +
        \sum_{h=0}^{H-1} \sum_{s \in S_h: p_s > 1/|S|}
        \sum_{a \in A}
        O\Big( |S| \ln{\frac{|S||A|}{\delta }}
        +
        \frac{ H^4 |S|^5 |A|^4 (Fdim \ln^2 \frac{H^2 |S|^2 |A|^2}{\epsilon^2 } + \ln \frac{|S||A|}{\delta})}{\epsilon^4}  \Big)
        \\
        \underbrace{\leq}_{(\star)}& 
        \sum_{h=0}^{H-1} \sum_{s \in S_h: p_s \in [\epsilon/6|S|,1/|S|]}
        \sum_{a \in A}
        O\Big( \frac{|S|}{\epsilon}\ln{\frac{|S||A|}{\delta}}
        +
        \frac{|S| |A|^4 (Fdim \ln^2{\frac{|S|^2 |A|^2}{\epsilon^2}} 
        + \ln {\frac{|S||A|}{\delta}})}{\epsilon^4}  \Big)
        \\ 
        & +
        \sum_{h=0}^{H-1} \sum_{s \in S_h: p_s > 1/|S|}
        \sum_{a \in A}
        O\Big( |S| \ln{\frac{|S||A|}{\delta }}
        +
        \frac{H^4 |S|^5 |A|^4 (Fdim \ln^2 \frac{H^2 |S|^2 |A|^2}{\epsilon^2 } + \ln \frac{|S||A|}{\delta})}{\epsilon^4}  \Big)
        \\ 
        = &
        O \Big( \frac{|S|^2|A|}{\epsilon} \ln{\frac{|S||A|}{\delta}} +
        \frac{H^4 |S|^6|A|^5 \; (Fdim \ln^2{\frac{H^2 |S|^2 |A|^2}{\epsilon^2}} + \ln {\frac{|S||A|}{\delta}})}{\epsilon^4}  \Big)\\
    \end{align*}
    \endgroup

where $(\star)$ is since in that regime we have $p_s \in [\epsilon/6 |S|, 1/|S|]$, hence $p^3_s \leq 1/ |S|^3$.
\end{proof}

\subsubsection{Sample bounds for the \texorpdfstring{$\ell_1$}{Lg} loss}\label{subsubsec:sample-complexity-l-1-KCFD}
We bound the sample complexity for function classes with finite Pseudo dimension with $\ell_1$ loss.
\begin{corollary}\label{corl: sample complexity KCFD l_1 pseudo}
    Assume that for every $(s,a) \in S \times A$ we have that $Pdim(\mathcal{F}^R_{s,a}) < \infty$.
    Let $Pdim = \max_{(s,a) \in S \times A} Pdim(\mathcal{F}^R_{s,a})$.
    Then, after collecting 
    \[
        O \Big( \frac{|S|^2|A|}{\epsilon} \ln{\frac{|S||A|}{\delta}} 
        +
        \frac{H^2 |S|^4 |A|^3 \;(Pdim \ln{\frac{H |S| |A|}{\epsilon}} + \ln {\frac{|S||A|}{\delta}})}{\epsilon^2}  \Big)
    \] 
    trajectories, with probability at least $1-\delta$ it holds that
    \[
        \mathbb{E}_{c \sim \mathcal{D}}[V^{\pi^\star_c}_{\mathcal{M}(c)}(s_0) - V^{\widehat{\pi}^\star_c}_{\mathcal{M}(c)}(s_0)] \leq 
        \epsilon + 2\alpha_1 H .
    \]
\end{corollary}

\begin{proof}
    Recall that for each state-action pair $(s,a)$ such that $s$ is $\frac{\epsilon}{6 |S|}$-reachable, we run for $T_{s,a} =   \lceil \frac{2}{p_s}(\ln(\frac{1}{\delta_1})+N_R(\mathcal{F}^R_{s,a} ,\epsilon_\star^1(p_s), \delta_1)) \rceil$ episodes.
    By Theorem~\ref{thm: PAC optimal policy for Known CFD with l_1}, for
    $
        \sum_{h =0}^{H-1} \sum_{s \in S_h : p_s \geq \epsilon/6 |S|} \sum_{a \in A} T_{s,a}
    $    
    samples we have with probability at least $1-\delta$ that
    \[
        \mathbb{E}_{c \sim \mathcal{D}}[V^{\pi^\star_c}_{\mathcal{M}(c)}(s_0) - V^{\widehat{\pi}^\star_c}_{\mathcal{M}(c)}(s_0)] \leq 
        \epsilon + 2\alpha_1 H .
    \]
    Since for every $(s,a) \in S \times A$ we have that $Pdim(\mathcal{F}^R_{s,a}) < \infty$, and $Pdim = \max_{(s,a) \in S \times A} Pdim(\mathcal{F}^R_{s,a})$, by Theorem~\ref{thm: pseudo dim}, for every $(s,a) \in S \times A$  we have 
    \[
        N_R(\mathcal{F}^R_{s,a}, \epsilon^1_\star(p_s), \delta_1)
        =
        O \Big( \frac{ Pdim \ln \frac{1}{\epsilon^1_\star(p_s)} + \ln \frac{1}{\delta_1}}{\epsilon_\star^2(p_s)}\Big).
    \]
    Using the accuracy-per-state function, we derive the overall sample complexity bound in the following computation.
    \begingroup
    \allowdisplaybreaks
    \begin{align*}
        &\sum_{h=0}^{H-1} \sum_{s \in S_h: p_s \geq \epsilon/6|S|}
        \sum_{a \in A}
        T_{s,a}
        \\
        = &
        \sum_{h=0}^{H-1} \sum_{s \in S_h: p_s \geq \epsilon/6|S|}
        \sum_{a \in A}
        O\Big( \frac{1}{p_s}(\ln(\frac{1}{\delta_1})+ N_R(\mathcal{F}^R_{s,a} ,\epsilon_\star^1(p_s), \delta_1))  \Big)
        \\
        = &
        \sum_{h=0}^{H-1} \sum_{s \in S_h: p_s \geq \epsilon/6|S|}
        \sum_{a \in A}
        O\Big( \frac{1}{p_s}(\ln(\frac{1}{\delta_1})
        +
        \frac{ Pdim \ln \frac{1}{\epsilon^1_\star (p_s)} + \ln \frac{1}{\delta_1}}{\epsilon_\star^2(p_s)})  \Big)
        \\ 
        = &
        \sum_{h=0}^{H-1} \sum_{s \in S_h: p_s \in [\epsilon/6|S|,1/|S|]}
        \sum_{a \in A}
        O\Big( \frac{1}{p_s}(\ln(\frac{1}{\delta_1})
        +
        \frac{ Pdim \ln \frac{1}{\epsilon/6 p_s |S||A|} 
        + \ln \frac{1}{\delta_1}}{\epsilon^2/36 p_s^2 |S|^2 |A|^2})  \Big)
        \\ 
        & +
        \sum_{h=0}^{H-1} \sum_{s \in S_h: p_s > 1/|S|}
        \sum_{a \in A}
        O\Big( \frac{1}{p_s}( \ln(\frac{1}{\delta_1})
        +
        \frac{ Pdim \ln \frac{1}{\epsilon / 6 H |S| |A|} + \ln \frac{1}{\delta_1}}{\epsilon^2 / 36 H^2 |S|^2 |A|^2})  \Big)
        \\
        = &
        \sum_{h=0}^{H-1} \sum_{s \in S_h: p_s \in [\epsilon/6|S|,1/|S|]}
        \sum_{a \in A}
        O\Big( \frac{1}{p_s}( \ln(\frac{1}{\delta_1})
        +
       \frac{  p_s^2 |S|^2 |A|^2 (Pdim \ln \frac{p_s |S| |A|}{\epsilon} 
        + \ln \frac{1}{\delta_1})}{\epsilon^2})  \Big)
        \\ 
        & +
        \sum_{h=0}^{H-1} \sum_{s \in S_h: p_s > 1/|S|}
        \sum_{a \in A}
        O\Big( \frac{1}{p_s}(\ln(\frac{1}{\delta_1})
        +
        \frac{H^2 |S|^2 |A|^2 (Pdim \ln \frac{H |S| |A|}{\epsilon } + \ln \frac{1}{\delta_1})}{\epsilon^2} ) \Big)
        \\
        \leq &
        \sum_{h=0}^{H-1} \sum_{s \in S_h: p_s \in [\epsilon/6|S|,1/|S|]}
        \sum_{a \in A}
        O\Big( \frac{|S|}{\epsilon}\ln{\frac{|S||A|}{\delta}}
        +
        \frac{ p_s |S|^2 |A|^2 (Pdim \ln{\frac{|S| |A|}{\epsilon}} 
        + \ln {\frac{|S||A|}{\delta}})}{\epsilon^2}  \Big)
        \\ 
        & +
        \sum_{h=0}^{H-1} \sum_{s \in S_h: p_s > 1/|S|}
        \sum_{a \in A}
        O\Big( |S| \ln{\frac{|S||A|}{\delta }}
        +
        \frac{ H^2 |S|^3 |A|^2 (Pdim \ln \frac{H |S| |A|}{\epsilon} + \ln \frac{|S||A|}{\delta})}{\epsilon^2}  \Big)
        \\
        \leq & 
        \sum_{h=0}^{H-1} \sum_{s \in S_h: p_s \in [\epsilon/6|S|,1/|S|]}
        \sum_{a \in A}
        O\Big( \frac{|S|}{\epsilon}\ln{\frac{|S||A|}{\delta}}
        +
        \frac{|S| |A|^2 (Pdim \ln{\frac{|S| |A|}{\epsilon}} 
        + \ln {\frac{|S||A|}{\delta}})}{\epsilon^2}  \Big)
        \\ 
        & +
        \sum_{h=0}^{H-1} \sum_{s \in S_h: p_s > 1/|S|}
        \sum_{a \in A}
        O\Big( |S| \ln{\frac{|S||A|}{\delta }}
        +
        \frac{H^2 |S|^3 |A|^2 (Pdim \ln \frac{H |S| |A|}{\epsilon} + \ln \frac{|S||A|}{\delta})}{\epsilon^2}  \Big)
        \\ 
        = &
        O \Big( \frac{|S|^2|A|}{\epsilon} \ln{\frac{|S||A|}{\delta}} 
        +
        \frac{H^2 |S|^4 |A|^3 \;(Pdim \ln{\frac{H |S| |A|}{\epsilon}} + \ln {\frac{|S||A|}{\delta}})}{\epsilon^2}  \Big) %
    \end{align*}
    \endgroup
    where in the first inequality we used the upper bound on $1/p_s$ and in the second inequality we upper bounded $p_s$.
\end{proof}


We also show similar sample complexity for function classes with finite fat-shattering dimension when using $\ell_1$ loss.
\begin{remark}
    The sample complexity for function classes with finite fat-shattering dimension with $\ell_1$ loss, where in $Fdim$ below we also maximizes over $\epsilon_\star(p_s) $ and the maximum is bounded and independent of $p_s$.
\end{remark}

\begin{corollary}\label{corl: sample complexity KCFD l_1 fat}
    Assume that for every $(s,a) \in S \times A$ we have that $\mathcal{F}^R_{s,a}$ has finite fat-shattering dimention.
    Let $Fdim = \max_{(s,a) \in S \times A} fat_{\mathcal{F}^R_{s,a}}(\epsilon_\star(p_s)/256)$.
    Then, for 
    \[
        O \Big( 
        \frac{|S|^2|A| H^2}{\epsilon^2}  \;(Fdim \ln^2{\frac{\max\{|S|, H\}}{\epsilon}} + \ln {\frac{|S||A|}{\delta}}) \Big)
    \] 
    samples, with probability at least $1-\delta$ it holds that
    \[
        \mathbb{E}_{c \sim \mathcal{D}}[V^{\pi^\star_c}_{\mathcal{M}(c)}(s_0) - V^{\widehat{\pi}^\star_c}_{\mathcal{M}(c)}(s_0)] \leq 
        \epsilon + 2\alpha_1 H .
    \]
\end{corollary}

\begin{proof}
    Recall that for each state-action pair $(s,a)$ such that $s$ is $\frac{\epsilon}{6 |S|}$ reachable, we run for $T_{s,a} =   \lceil \frac{2}{p_s}(\ln(\frac{1}{\delta_1})+N_R(\mathcal{F}^R_{s,a} ,\epsilon_\star^1(p_s), \delta_1)) \rceil$ episodes.
    By Theorem~\ref{thm: PAC optimal policy for Known CFD with l_2}, for $\sum_{h =0}^{H-1} \sum_{s \in S_h : p_s \geq \epsilon/6 |S|} \sum_{a \in A} T_{s,a}$ samples we have with probability at least $1-\delta$ that
    \[
        \mathbb{E}_{c \sim \mathcal{D}}[V^{\pi^\star_c}_{\mathcal{M}(c)} - V^{\widehat{\pi}^\star_c}_{\mathcal{M}(c)}(s_0)] \leq 
         \epsilon + 2\alpha_1 H .
    \]
    Since for every $(s,a) \in S \times A$ we have that $\mathcal{F}^R_{s,a}$ has finite fat-shattering dimension, and $Fdim = \max_{(s,a) \in S \times A} fat_{\mathcal{F}^R_{s,a}}(\epsilon_\star(p_s))$, by Theorem~\ref{thm: fat dim}, for every $(s,a) \in S \times A$  we have 
    \[
        N_R(\mathcal{F}^R_{s,a}, \epsilon^2_\star(p_s), \delta_1)
        =
        O \Big( \frac{ Fdim \ln^2 \frac{1}{\epsilon^1_\star(p_s)} + \ln \frac{1}{\delta_1}}{\epsilon_\star^2(p_s)}\Big).
    \]
    Using the accuracy-per-state function, we derive the overall sample complexity bound in the following computation.
    \begingroup
    \allowdisplaybreaks
    \begin{align*}
        &\sum_{h=0}^{H-1} \sum_{s \in S_h: p_s \geq \epsilon/6|S|}
        \sum_{a \in A}
        T_{s,a}
        =
        \sum_{h=0}^{H-1} \sum_{s \in S_h: p_s \geq \epsilon/6|S|}
        \sum_{a \in A}
        O\Big( \frac{1}{p_s}(\ln(\frac{1}{\delta_1})+ N_R(\mathcal{F}^R_{s,a} ,\epsilon_\star^1(p_s), \delta_1))  \Big)
        \\
        =&
        \sum_{h=0}^{H-1} \sum_{s \in S_h: p_s \geq \epsilon/6|S|}
        \sum_{a \in A}
        O\Big( \frac{1}{p_s}(\ln(\frac{1}{\delta_1})
        +
        \frac{ Fdim \ln^2 \frac{1}{\epsilon^1_\star (p_s)} + \ln \frac{1}{\delta_1}}{\epsilon_\star^2(p_s)})  \Big)
        \\ 
        =&
        \sum_{h=0}^{H-1} \sum_{s \in S_h: p_s \in [\epsilon/6|S|,1/|S|]}
        \sum_{a \in A}
        O\Big( \frac{1}{p_s}(\ln(\frac{1}{\delta_1})
        +
        \frac{ Fdim \ln^2 \frac{1}{\epsilon/6 p_s |S||A|} 
        + \ln \frac{1}{\delta_1}}{\epsilon^2/36 p_s^2 |S|^2 |A|^2})  \Big)
        \\ 
        &+
        \sum_{h=0}^{H-1} \sum_{s \in S_h: p_s > 1/|S|}
        \sum_{a \in A}
        O\Big( \frac{1}{p_s}( \ln(\frac{1}{\delta_1})
        +
        \frac{ Fdim \ln^2 \frac{1}{\epsilon / 6 H |S| |A|} + \ln \frac{1}{\delta_1}}{\epsilon^2 / 36 H^2 |S|^2 |A|^2})  \Big)
        \\
        =&
        \sum_{h=0}^{H-1} \sum_{s \in S_h: p_s \in [\epsilon/6|S|,1/|S|]}
        \sum_{a \in A}
        O\Big( \frac{1}{p_s}( \ln(\frac{1}{\delta_1})
        +
       \frac{  p_s^2 |S|^2 |A|^2 (Fdim \ln^2 \frac{p_s |S| |A|}{\epsilon} 
        + \ln \frac{1}{\delta_1})}{\epsilon^2})  \Big)
        \\ 
        &+
        \sum_{h=0}^{H-1} \sum_{s \in S_h: p_s > 1/|S|}
        \sum_{a \in A}
        O\Big( \frac{1}{p_s}(\ln(\frac{1}{\delta_1})
        +
        \frac{H^2 |S|^2 |A|^2 (Fdim \ln^2 \frac{H |S| |A|}{\epsilon } + \ln \frac{1}{\delta_1})}{\epsilon^2} ) \Big)
        \\
        \leq&
        \sum_{h=0}^{H-1} \sum_{s \in S_h: p_s \in [\epsilon/6|S|,1/|S|]}
        \sum_{a \in A}
        O\Big( \frac{|S|}{\epsilon}\ln{\frac{|S||A|}{\delta}}
        +
        \frac{ p_s |S|^2 |A|^2 (Fdim \ln^2{\frac{|S| |A|}{\epsilon}} 
        + \ln {\frac{|S||A|}{\delta}})}{\epsilon^2}  \Big)
        \\ 
        &+
        \sum_{h=0}^{H-1} \sum_{s \in S_h: p_s > 1/|S|}
        \sum_{a \in A}
        O\Big( |S| \ln{\frac{|S||A|}{\delta }}
        +
        \frac{ H^2 |S|^3 |A|^2 (Fdim \ln^2 \frac{H |S| |A|}{\epsilon} + \ln \frac{|S||A|}{\delta})}{\epsilon^2}  \Big)
        \\
        \leq& 
        \sum_{h=0}^{H-1} \sum_{s \in S_h: p_s \in [\epsilon/6|S|,1/|S|]}
        \sum_{a \in A}
        O\Big( \frac{|S|}{\epsilon}\ln{\frac{|S||A|}{\delta}}
        +
        \frac{|S| |A|^2 (Fdim \ln^2{\frac{|S| |A|}{\epsilon}} 
        + \ln {\frac{|S||A|}{\delta}})}{\epsilon^2}  \Big)
        \\ 
        &+
        \sum_{h=0}^{H-1} \sum_{s \in S_h: p_s > 1/|S|}
        \sum_{a \in A}
        O\Big( |S| \ln{\frac{|S||A|}{\delta }}
        +
        \frac{H^2 |S|^3 |A|^2 (Fdim \ln^2 \frac{H |S| |A|}{\epsilon} + \ln \frac{|S||A|}{\delta})}{\epsilon^2}  \Big)
        \\ 
        =&
        O \Big( \frac{|S|^2|A|}{\epsilon} \ln{\frac{|S||A|}{\delta}} 
        +
        \frac{H^2 |S|^4 |A|^3 \;(Fdim \ln^2{\frac{H |S| |A|}{\epsilon}} + \ln {\frac{|S||A|}{\delta}})}{\epsilon^2}  \Big)
    \end{align*}
    \endgroup
       where in the first inequality we used the upper bound on $1/p_s$ and in the second inequality we upper bounded $p_s$.
\end{proof}


\section{Unknown and Context Free Dynamics.}\label{Appendix:UCFD}

When the dynamics is unknown, we have an additional hurdle which is the need to approximate it. 
To collect i.i.d samples efficiently for each state, we still need to find an exploration policy which (approximately) maximizes the probability to visit the target state.

We can compute approximate the true dynamics $P$ by $\widehat{P}$, and use $\widehat{P}$ to compute a policy $\widehat{\pi}_s$ to reach state $s$. If  $\widehat{P}\approx P$ this will result in a similar sample size to approximate the rewards, as in the known dynamics case. We will have a worse sample complexity due to the need to approximate the dynamics well.

\subsection{Basic Lemmas}

In this subsection, we present basic concentration-bounds based lemmas that used to compute the required sample complexity to obtain good dynamics approximation, with high probability.

    


\begin{lemma}\label{lemma: chernoff case UCFD}
    Let $\delta' \in (0, 1)$, layer $h \in [H]$ and a state $s_h \in S_h$.
    Let $\pi: S \to \Delta(A)$ be a policy that satisfies $q_h(s_h | \pi, P) \geq \beta$, for $\beta \in (0,1]$.
    Let $m$ be the desired number of visits in $s_h$.
    Then, if running $\pi$ for  
    $T \geq \frac{2}{\beta}(\ln \frac{1}{\delta'} + m)$ episodes, the agent will visit state $s_h$ at least $m$ times, with probability at least $1 - \delta'$.
\end{lemma} 

\begin{proof}
    Follows form multiplicative Chernoff bound. 
\end{proof}

\begin{lemma}\label{lemma: UCFD-sample complexity-by-Huver-Carol}
    Let $\gamma, \delta_1 \in (0,1)$, $h \in [H-1]$ and $(s_h,a_h) \in S_h \times A$.
    For every $s_{h+1} \in S_{h+1}$, denote by $n(s_{h+1} | s_h, a_h)$ the number of times the agent observed a trajectory contains the triplet $(s_h, a_h, s_{h+1})$, out of $m$ trajectories that contain the pair $(s_h, a_h)$.
    
    Define for every $s_{h+1} \in S_{h+1}$: $\widehat{P}(s_{h+1} | s_h, a_h) = \frac{n(s_{h+1} | s_h, a_h)}{N_P(\gamma, \delta_1)}$.
    
    Then, for $N_P(\gamma, \delta_1) \geq \frac{2}{\gamma^2}\Big(\ln \Big(\frac{1}{\delta_1} + (|S|+1)\ln{2} \Big) \Big)$ we have with probability at least $1 - \delta_1$
    \[
        \|P(\cdot  | s_h, a_h) - \widehat{P}(\cdot  | s_h, a_h) \|_1
        \leq \gamma .
    \]
\end{lemma}

\begin{proof}
    By Bretagnolle Huber-Carol inequality we have
    \begin{align*}
        \mathbb{P}[ \|P(\cdot  | s_h, a_h) - \widehat{P}(\cdot  | s_h, a_h) \|_1
        \geq \gamma]
        &=
        \mathbb{P}[\sum_{s_{h+1} \in S_{h+1}}|\frac{n(s_{h+1} | s_h, a_h)}{N_P(\gamma, \delta_1)} - P(s_{h+1}  | s_h, a_h) | ]
        \\
        &\leq
        2^{|S_{h+1}|+1} \exp{(-\frac{N_P(\gamma, \delta_1) \gamma ^2}{2})}
        \\
        &\leq 2^{|S|+1}\exp{(-\frac{N_P(\gamma, \delta_1) \gamma ^2}{2})}
        \\
        &\leq \delta_1
        \iff
        N_P(\gamma, \delta_1) \geq  \frac{2}{\gamma^2}\Big(\ln \frac{1}{\delta_1} + (|S| + 1)\ln{2}  \Big).
    \end{align*}
\end{proof}

\begin{corollary}[Sample complexity for approximating the dynamics]\label{corl: Sample complexity to approx dynamics UCFD}
    Let $\gamma = \frac{\epsilon}{48|S|H^2}$ and $\delta_1 = \frac{\delta}{6 |S||A| H} $. Then
    \[
        N_P(\gamma, \delta_1) = 
        O\Big(
        \frac{ H^4 |S|^2}{\epsilon^2}\Big(\ln \frac{|S||A|H}{\delta} + |S|  \Big)
        \Big).
    \]
\end{corollary}

\begin{lemma}[Dynamics distance bound implies occupancy measure distance bound]\label{lemma: occ mesure bound - basic}
    Let $h \in [H]$ and fix a policy $\pi: S \to \Delta(A)$. 
    Assume that for all $k < h$ and $(s_k, a_k) \in S_k \times A$ we have $$\| P(\cdot | s_k, a_k) - \widetilde{P}(\cdot | s_k, a_k) \|_1 \leq \gamma .$$
    Then, $$ \| q_h(\cdot| \pi ,P) -  q_h(\cdot| \pi ,\widetilde{P})\|_1 \leq \gamma h .$$
\end{lemma}

\begin{proof}
    We prove the lemma using induction on the horizon $h$.
    The base case is $h = 0$. As there exists unique start state $s_0$ the claim holds trivially.
    
    We assume correctness for all $i < h$ and show for $i = h$.
    By definition we have
    \begingroup
    \allowdisplaybreaks
    \begin{align*}
        &\| q_h(\cdot| \pi ,P) -  q_h(\cdot| \pi ,\widetilde{P})\|_1
        =
        \\
        = &
        \sum_{s_h \in S_h} |q_h(s_h| \pi ,P) -  q_h(s_h| \pi ,\widetilde{P})|
        \\
        = &
        \sum_{s_{h-1} \in S_{h-1}} \sum_{a_{h-1} \in A} \sum_{s_h \in S_h}
        \pi(a_{h-1}| s_{h-1})
        |q_{h-1}(s_{h-1}| \pi ,P) P(s_h |s_{h-1}, a_{h-1} )
        -  
        q_{h-1}(s_{h-1}| \pi ,\widetilde{P})\widetilde{P}(s_h |s_{h-1}, a_{h-1} )|
        \\
        \leq &
        \sum_{s_{h-1} \in S_{h-1}} \sum_{a_{h-1} \in A} \pi(a_{h-1}| s_{h-1})P(s_h |s_{h-1}, a_{h-1} )
        \sum_{s_h \in S_h}
        |q_{h-1}(s_{h-1}| \pi ,P) - q_{h-1}(s_{h-1}| \pi ,\widetilde{P})|
        \\
        & +
        \sum_{s_{h-1} \in S_{h-1}}
        q_{h-1}(s_{h-1}| \pi ,\widetilde{P})
        \sum_{a_{h-1} \in A} 
        \pi(a_{h-1}| s_{h-1})
        \sum_{s_h \in S_h}
        | P(s_h |s_{h-1}, a_{h-1} ) - \widetilde{P}(s_h |s_{h-1}, a_{h-1} )|
        \\
        \leq &
        \|q_{h-1}(\cdot| \pi ,P) - q_{h-1}(\cdot | \pi ,\widetilde{P})\|_1
        \\
        & +
        \sum_{s_{h-1} \in S_{h-1}}
        q_{h-1}(s_{h-1}| \pi ,\widetilde{P})
        \sum_{a_{h-1} \in A} 
        \pi(a_{h-1}| s_{h-1})
        \| P(\cdot |s_{h-1}, a_{h-1} ) - 
        \widetilde{P}(\cdot |s_{h-1}, a_{h-1} )\|_1
        \\
        \leq &
        \gamma (h-1) + \gamma = \gamma h
    \end{align*}
    \endgroup
\end{proof}

\subsection{Algorithm}

We start by an overview of our algorithms.

Algorithm EXPLORE-UCFD (Algorithm~\ref{alg: EXPLORE-UCFD}) works in phases, where in phase $h \in [H]$ we approximate layer $h$ dynamics. We first collect samples for each (non-negligible) state in layer $h$ and then use them to approximate the dynamics, using simple tabular estimation. Using the same sample we also estimate the rewards using ERM oracle. The required accuracy for each state-action pair $(s_h,a_h)$ is determined by the accuracy-per-state function $\epsilon^i_\star(\cdot)$ (for $i \in \{1,2\}$, with accordance to the used loss function) using $\widehat{p}_{s_h}:= q_h(s_h| \widehat{p}_{s_h}, \widehat{P})$.  
We use $\epsilon^1_\star(p_s)$ for the $\ell_1$ loss, which defines as
\[
    \epsilon^1_\star(p_s) := \epsilon_\star(p_s) = 
    \begin{cases}
        1 &, \text{ if } p_{s} < \frac{\epsilon}{24|S|}\\
        \frac{\epsilon}{24 H |S| |A|} &, \text{ if } p_{s}  > \frac{1}{|S| }\\
        \frac{\epsilon}{24 p_{s}  |S| |A| }&, \text{ if } p_{s}  \in [ \frac{\epsilon}{24 |S|} , \frac{1}{|S| } ]\\
    \end{cases}
\]
while for the $\ell_2$ loss we use $\epsilon^2_\star(p_s) = (\epsilon_\star(p_s))^2$.

After collecting sufficient number of samples for every non-negligible state in layers up to $ h - 1$, we have a good approximation of the dynamics up to layer $h-1$. This yields a good approximation of the occupancy measure of layer $h$ for any policy $\pi$ (regardless of it being context-dependent or not) .

Hence, given a state $s_h\in S_h$ and approximate dynamics $\widehat{P}$ we compute $\widehat{\pi}_{s_h} = \arg \max_{\pi : S \to A} q_h(s_h| \pi,\widehat{P})$. We run $\widehat{\pi}_{s_h}$ to generate the sample of $s_h$. Since the dynamics is context-free, the policy is the same for all of the contexts.

In order to control the number episodes sampled we define non-negligible states as $\beta$-reachable w.r.t $\widehat{P}$, i.e., they have $q_h(s_h| \widehat{\pi}_{s_h},\widehat{P})\geq \beta$.

At the end of the sampling we have for each $\beta$-reachable state $s_h\in S_h$  with respect to $\widehat{P}$, and every action $a_h$ a data set contains tuples of the form $(s_h, a_h, s_{h+1})$ to approximate the transition probability matrix via tabular mean estimation. (Recall that here the dynamics do no depend on the context, this will change in the context-dependent dynamics case.)
For the rewards, we use tuples of the form $((c,s_h, a_h), r_h)$ and run the ERM to approximate the rewards.

Algorithm EXPLOIT-UCFD (Algorithm~\ref{alg: EXPLOIT-UCFD}) get as inputs the MDP parameters, the dynamics approximation $\widehat{P}$ and the functions approximate the rewards (that computed using EXPLORE-UCFD algorithm). Given a context $c$ it computed the approximated MDP $\widehat{\mathcal{M}}(c)$ and use it to compute a near optimal policy $\pi^\star_c$. Then, it run $\pi^\star_c$ to generate trajectory.
Recall that $\widehat{\mathcal{M}}(c)=(S,A,\widehat{P},s_0,\widehat{r}^{c},H)$ where we define $\forall s\in S, a\in A: \widehat{r}^{c}(s,a) = f_{s,a}(c)$, and $\widehat{P}$ is computed using tabular approximation.


\begin{algorithm}
    \caption{Explore Unknown and Context-Free Dynamics CMDP(EXPLORE-UCFD)}
    \label{alg: EXPLORE-UCFD}
    \begin{algorithmic}[1]
        \State \textbf{inputs: }
        \begin{itemize}
            \item MDP parameters: $S = \{S_0, S_1, \ldots, S_H\} $ - a layered states space, $A$ - a finite actions space, $s_0$ - a unique start state, $H$ - the horizon length.
            \item Accuracy and confidence parameters: $\epsilon$,$\delta$.
            \item $\forall s \in S , a \in A : \;\; \mathcal{F}^R_{s,a}$ - the function classes use to approximate the rewards function.
            \item $N_R(\mathcal{F}, \epsilon, \delta)$ - sample complexity function for the ERM oracle.
            \item $\gamma$ - the required approximation error of the dynamics , $\beta$ - the reachability parameter. (We have $\frac{\epsilon}{24 |S|} \geq \beta \geq 2\gamma H$.)
            \item $N_P(\gamma, \delta_1)$ - sample complexity function for approximating the dynamics using tabular approximation.
            \item $\ell$ - loss function (assumed to be one of $\ell_1$ or $\ell_2$) and the appropriate accuracy-per-state function $\epsilon^i_\star$ ( for $i \in \{1,2\}$ in accordance to $\ell$).
        \end{itemize}
        
        \State {set $\delta_1= \frac{\delta}{6|S||A|H}$}
        
        \State{set for all $(s,a) \in S \times A$: $n(s,a) = 0 $ and for all $(s,a,s') \in S \times A \times S$: $n(s'| s,a) = 0 $.}
        \For{$h \in [H-1] $}
            \State
            {
                let $s_{sink} \notin S$ be a new state which denotes a sink.\\
                define the approximated dynamics for all $(s,a,s') \in S \times A \times S$:
                $
                    \widehat{P}(s' | s,a) = \frac{n(s' | s,a)}{n(s,a)} \mathbb{I}[n(s,a) \geq N_P(\gamma, \delta_1)]
                $
                and 
                $\widehat{P}(s_{sink} | s,a) =  \mathbb{I}[n(s,a) < N_P(\gamma, \delta_1)]$.
            }
            \For{$s_h\in S_h$}
                \State{
                $(\widehat{\pi}_{s_h}, \widehat{p}_{s_h}) \gets 
                \texttt{FFP}(S\cup \{s_{sink}\}, A, \widehat{P}, s_0, H, s_h)$.}
                \Comment{$\widehat{p}_{s_h}$ is the highest probability to visit $s_h$ under $\widehat{P}$ and $\widehat{\pi}_{s_h}$ is the policy that reach that probability.}
                
                \If{$\widehat{p}_{s_h}  \geq \beta$}
                    \For{ $a_h \in A$}
                        \State{compute 
                        $
                            T_{s_h,a_h} =
                            \lceil 
                            \frac{2}{\widehat{p}_{s_h} - \gamma h}(\ln(\frac{1}
                            {\delta_1})
                            +
                            \max{\{
                            N_R (\mathcal{F}^R_{s_h,a_h} ,\epsilon^i_\star(\widehat{p}_{s_h}), \delta_1), N_P(\gamma, \delta_1)}\}) 
                            \rceil
                        $} 
                        \State{initialize $Sample(s_h,a_h) = \emptyset$}
                        \State{set $\widehat{\pi}_{s_h}(s_h) \gets a_h$}
                        \For{$t = 1, 2, \ldots, T_{s_h, a_h}$}
                            \State{observe context $c$.}
                            
                            \State{run $\widehat{\pi}_{s_h}$ to generate trajectory $\tau_t$} 
                        
                            \If{$(s_h, a_h, r_h, s_{h+1}) \in \tau_t$ for some $r_h \in [0,1]$ and $s_{h+1} \in S_{h+1}$}
                                
                                \State{update sample: ${Sample(s_h,a_h) = Sample(s_h,a_h) + \{((c,s_h,a_h), r_h)\}}$}
                                
                                \State{update counters: ${n(s_h, a_h) \gets n(s_h, a_h) + 1,\;\; n(s_{h+1} | s_h, a_h) \gets n(s_{h+1}| s_h, a_h) + 1}$}
                                
                            \EndIf{}
                        \EndFor{}
                        
                        \If{$|Sample(s_h,a_h)|\geq \max{\{
                            N_R (\mathcal{F}^R_{s_h,a_h} ,\epsilon_\star(\widehat{p}_{s_h}), \delta_1), N_P(\gamma, \delta_1)}\}$ }
                        
                         \State{call Oracle: 
                        $
                            f_{s_h,a_h} = \texttt{ERM}(\mathcal{F}^R_{s_h,a_h}, Sample(s_h,a_h), \ell) 
                        $}
                        \Else
                        \State{\textbf{return } \texttt{FAIL}}
                        \EndIf{}
                \EndFor{}
            \Else
            \State{set: $\forall a\in A:  f_{s_h,a} = 0$}
            \EndIf{}
        \EndFor{}
    \EndFor{}
    
    \State{ \textbf{return }
        $
            F =
            \{f_{s,a} : \; \forall s\in S, a\in A\}, \widehat{P}
        $}
    \end{algorithmic}
\end{algorithm}

\begin{algorithm}
    \caption{Exploit CMDP for Unknown Context-Free-Dynamics (EXPLOIT-UCFD)}
    \label{alg: EXPLOIT-UCFD}
    \begin{algorithmic}[1]
        \State \textbf{inputs: }
        \begin{itemize}
            \item The MDP parameters: $S = \{S_0, S_1, \ldots, S_H\} $,$A$,$s_0$,$H$.
            \item $\widehat{P}$ approximation of the context-free dynamics.
            \item $\{f_{s,a} | \forall (s,a) \in S \times A\}$ - the function use to approximate the reward for each state-action pair (as function of the context).
        \end{itemize}
        \For{$t=1,2,...$}
            \State{observe context $c_t$}
            \State {define 
            $\forall s\in S, a\in A: \widehat{r}^{c_t}(s,a) = f_{s,a}(c_t)$, $\forall a\in A: \widehat{r}^{c_t}(s_{sink},a) = 0$}
            \State{define the approximated MDP associated with $c_t$: $\widehat{\mathcal{M}}(c_t)=(S\cup\{s_{sink}\}, A,\widehat{P},s_0,\widehat{r}^{c_t},H)$}
            \State{$(\pi_t , V_t) \gets \texttt{Planning}(\widehat{\mathcal{M}}(c_t))$}
            \State{run $\pi_t$ in episode $t$.}
        \EndFor{}
\end{algorithmic}
\end{algorithm}


\subsection{Analysis}

\subsubsection{Analysis Outline}
We provide analysis for both $\ell_1$ and $\ell_2$ loss functions, in the agnostic case.

In the analysis, we bound the error caused by the dynamics approximation (see Sub-subsection~\ref{subsubsec:dynamics-error-UCFD})
and the error caused by the rewards approximation for both the $\ell_2$ loss (see Sub-subsection~\ref{subsubsec:rewards-error-l-2-UCFD}) and the $\ell_1$ loss (see Sub-subsection~\ref{subsubsec:rewards-error-l-1-UCFD}).

We combine both errors to bound the expected value difference between the true model $\mathcal{M}(c)$ and $\widehat{\mathcal{M}}(c)$, for any context-dependent policy $\pi = (\pi_c)_{c \in \mathcal{C}}$, with high probability. 
(See Lemma~\ref{lemma: l_2 val-dif UCFD } for the $\ell_2$ loss and Lemma~\ref{lemma: l_1 val-dif UCFD } for the $\ell_1$).

Using that bound, we derive a bound on the expected value difference between the optimal context-dependent policy $\pi^\star = (\pi^\star_c)_{c \in \mathcal{C}}$ and our approximated optimal policy $\widehat{\pi}^\star = (\widehat{\pi}^\star_c)_{c \in \mathcal{C}}$, which holds with high probability. (See Theorem~\ref{thm: PAC optimal policy for Known CFD with l_2} and~\ref{thm: PAC optimal policy for Known CFD with l_1}).
.

Lastly, we derive sample complexity bound using known uniform convergence sample complexity bounds for the Pseudo dimension (See Theorem~\ref{thm: pseudo dim}) and the fat-shattering dimension (See Theorem~\ref{thm: fat dim}). For the sample complexity analysis, see Sub-subsection~\ref{subsubsec:smaple-complexity-l-2-UCFD} for the $\ell_2$ loss, and~\ref{subsubsec:smaple-complexity-l-1-UCFD} for the $\ell_1$ loss.

In the following analysis,
we assume that $\frac{\epsilon}{24|S|}\geq \beta \geq 2\gamma H$, and later choose $\beta$ and $\gamma$ that satisfies that.
We also choose $\delta_1 = \frac{\delta}{6 |S||A|H}$.
In addition, we use the following notation.
\begin{definition}
    For every state $s \in S$ we denote by $\widehat{p}_s$ the maximal (over the policies) probability to visit $s$ under the approximated dynamics $\widehat{P}$.
\end{definition}

\subsubsection{Good Events}

We analyze algorithms EXPLOR-UCFD (Algorithm~\ref{alg: EXPLORE-UCFD}) and EXPLOIT-UCFD (Algorithm~\ref{alg: EXPLOIT-UCFD}) under the following good events, which we show that hold with high probability.

\paragraph{Event $G_1$.}
For every $h \in [H-1]$ let $G^h_1$ denote the good event in which we have for every state and action pair $(s_h ,a_h ) \in S_h \times A$ such that $s_h$ is $\beta$-reachable for $\widehat{P}$ that $n(s_h, a_h) \geq \max\{N_R(\mathcal{F}^R_{s,a} ,\epsilon^2_\star(\widehat{p}_{s_h}), \delta_1), N_P(\gamma, \delta_1)\}$ for the $\ell_2$ loss 
( $n(s_h, a_h) \geq \max\{N_R(\mathcal{F}^R_{s,a} ,\epsilon^1_\star(\widehat{p}_{s_h}), \delta_1), N_P(\gamma, \delta_1)\}$ for the $\ell_1$ loss) 
samples of were collected. 
We define ${G_1 = \cap_{h \in [H-1]}G^h_1}$.

\paragraph{Event $G_2$.}
For every $h \in [H-1]$ let $G^h_2$ denote the good event in which we have for every state and action pair $(s_h ,a_h) \in S_h \times A$ such that $s_h$ is $\beta$-reachable for $\widehat{P}$ that $\| P(\cdot | s_h, a_h) - \widehat{P}(\cdot | s_h, a_h) \|_1 \leq \gamma$ . (Here, we omit the entry $\widehat{P}(s_{sink}|s,a)$ of $\widehat{P}$).
We define ${G_2 = \cap_{h \in [H-1]}G^h_2}$.

\paragraph{How are we about to use $G_1$ and $G_2$?}
Intuitively, given that $G_1$ holds, we have collected sufficient number of samples for each $\beta$-reachable state $s$ and every action $a$. Hence, by $\widehat{P}$ definition we have that $\widehat{P}(s'|s,a)= n(s,a,s')/n(s,a) \;\; \forall s' \neq s_{sink}$ and $\widehat{P}(s_{sink}|s,a)=0$. Thus, we can ignore the entry related with the sink (which does not exist in the true dynamics $P$), and have that $\|P(\cdot | s,a) - \widehat{P}(\cdot | s,a)\|_1 \leq \gamma$ with high probability, by Lemma~\ref{lemma: UCFD-sample complexity-by-Huver-Carol}.

\paragraph{Event $G_3$.}
For every $h \in [H-1]$ let $G^h_3$ denote the good event in which we have for every state and action pair $(s_h ,a_h ) \in S_h \times A$ such that $s_h$ is $\beta$-reachable for $\widehat{P}$ that 
$$
    \mathbb{E}_{c \sim \mathcal{D}}[(f_{s_h, a_h}(c) - r^c(s_h, a_h))^2] \leq \epsilon^2_\star(\widehat{p}_{s_h}) + \alpha^2_2(\mathcal{F}^R_{s_h, a_h})
$$ 
for the $\ell_2$ loss 
(${\mathbb{E}_{c \sim \mathcal{D}}[|f_{s_h, a_h}(c) - r^c(s_h, a_h)|] \leq \epsilon_\star(\widehat{p}_{s_h}) + \alpha_1(\mathcal{F}^R_{s_h, a_h}})$ for the $\ell_1$ loss).
We define $G_3 = \cap_{h \in [H-1]}G^h_3$.

\subsubsection{Proving The Good Events Hold With High Probability}

\begin{lemma}\label{lemma: G_3 given G_1 UCFD}
    $\mathbb{P}[G_3 | G_1] \geq 1 - \frac{\delta}{6}$.
\end{lemma}

\begin{proof}
    By the ERM guarantees for each state-action pair when combined using union bound.
\end{proof}

\begin{lemma}[occupancy measure lower bound]\label{lemma: true occ mesure lower bound}
    Let $h \in [H]$ and a policy $\pi: S \to \Delta(A)$. 
    Under the good events $G^k_1$ and $G^k_2$ for every $k < h$,  for every state $s_h \in S_h$ it holds that
    $$ q_h(s_h| \pi ,P) \geq  q_h(s_h| \pi ,\widehat{P}) - \gamma h .$$
\end{lemma}

\begin{proof} 
    Define the dynamics $\widetilde{P}$ for all $k < h$ and $(s,a,s') \in S_k \times A \times S_{k+1}$  as follows: 
    \[
        \widetilde{P}(s' | s, a ) = P(s' | s, a ) \cdot \mathbb{I}[n(s,a) \geq N_P(\gamma, \delta_1)],
    \]
    and
    \[
        \widetilde{P}(s_{sink} | s, a ) = \mathbb{I}[ n(s,a) < N_P(\gamma, \delta_1)].
    \]
    So, under the good events $G^k_1$ and $G^k_2$ for every $k < h$ and every $(s_k, a_k) \in S_k \times A_k$ we have that
    \[
        \| \widetilde{P}(\cdot | s_k, a_k) - \widehat{P}(\cdot | s_k, a_k) \|_1 \leq \gamma.
    \]    
    (For states $s_k$ which are $\beta$-reachable, it follows since $G^k_1$ and $G^k_2$ hold.
    For states $s_k$ which are not $\beta$-reachable we have that $\widetilde{P}$ and $\widehat{P}$ are identical, i.e., they both transition to the sink with probability $1$).
    
    Hence, by Lemma~\ref{lemma: occ mesure bound - basic} we have
    that $ \| q_h(\cdot| \pi ,\widetilde{P}) -  q_h(\cdot| \pi ,\widehat{P})\|_1 \leq \gamma h $, which implies
    that for all $s_h \in S_h$ we have
    \[
        q_h(s_h| \pi ,\widetilde{P})
        \geq 
        q_h(s_h| \pi ,\widehat{P}) -\gamma h.
    \]
    By $\widetilde{P}$ definition, we trivially have for all $h \in [H]$ and $s_h \in S_h$ that 
    $        
        q_h(s_h| \pi ,P)
        \geq
        q_h(s_h| \pi ,\widetilde{P})
    $.
    Hence, we obtained
    \[
        q_h(s_h| \pi ,P)
        \geq        
        q_h(s_h| \pi ,\widetilde{P})
        \geq 
        q_h(s_h| \pi ,\widehat{P}) -\gamma h.
    \]
\end{proof}

\begin{lemma}\label{lemma: prob to G^h_2 given G^h_1}
For every layer $h\in [H-1]$ it holds that 
$
    \mathbb{P}[G^h_2 | G^h_1]\geq 1 - \frac{\delta}{6 H}
$.
\end{lemma}
    
\begin{proof}
    Fix a layer $h\in [H]$.
    Since $G^h_1$ holds, we have for every state-action pair $(s_h, a_h) \in S_h \times A$ such that $s_h$ is $\beta$-reachable for $\widehat{P}$ that $n(s_h, a_h) \geq N_P(\gamma, \delta_1)$.
    Hence, by Lemma~\ref{lemma: UCFD-sample complexity-by-Huver-Carol}, for $N_P(\gamma,\delta_1) = O \Big( \frac{1}{\gamma^2}\Big(\ln \frac{1}{\delta_1} + |S|  \Big) \Big)$ we have with probability at least $1-\delta_1$ that $\|P(\cdot| s_h, a_h) - \widehat{P}(\cdot| s_h, a_h)\|_1 \leq \gamma$. Since $\delta_1 = \frac{\delta}{6|S||A|H}$, using a union bound over all the pairs $(s_h, a_h) \in S_h \times A$ such that $s_h$ is $\beta$-reachable for $\widehat{P}$ we obtain the lemma.
\end{proof}

\begin{lemma}\label{lemma: prob to G^h_1 given previous events}
    For every layer $h\in [H-1]$ it holds that 
    $
        \mathbb{P}[G_1^h | G^k_1 ,G^k_2 \;\forall k\in[h-1] ]\geq 1 - \frac{\delta}{6H}
    $.
\end{lemma}
    
\begin{proof}
    We prove the lemma using induction over the horizon $h$.
        
    Base case: $h=0$.
    As for state $s_0$ we collect at least $\max\{N_R(\mathcal{F}^R_{s_h, a_h}, \epsilon^i_\star(\widehat{p}_s), \delta_1, N_P(\gamma, \delta_1\}$ samples in a deterministic manner, therefore we have $\mathbb{P}[G^0_1] = 1$.
        
    Induction step: Assume the lemma holds for all $k \leq h$ and we show it holds for $h+1$.
    Given $G^k_1 ,G^k_2 \;\forall k\in[h]$ hold, by Lemma~\ref{lemma: true occ mesure lower bound} for every state $s_{h+1} \in S_{h+1} $ it holds that
    \[
        q_{h+1}(s_{h+1} | \widehat{\pi}_{s_{h+1}}, P) 
        \geq
        q_{h+1}(s_{h+1} | \widehat{\pi}_{s_{h+1}}, \widehat{P})
        - \gamma (h+1)
        =
        \widehat{p}_{s_{h+1}} - \gamma (h+1)
        .
    \]
    Recall that for every action $a_{h+1} \in A$, the agent runs $\widehat{\pi}_{s_{h+1}}$ for $T_{s_{h+1}, a_{h+1}}$ episodes, in which, when visiting $s_{h+1}$ the agent plays action $a_{h+1}$ 
    (for $T_{s_{h+1}, a_{h+1}}$ which defined in Algorithm~\ref{alg: EXPLORE-UCFD}).
    
    Hence, by by Lemma~\ref{lemma: chernoff case UCFD}, the agent collects at least $\max\{N_R(\mathcal{F}^R_{s_h, a_h}, \epsilon^i_\star(\widehat{p}_s), \frac{\delta}{6 |S||A|H}), N_P(\gamma, \frac{\delta}{6 |S||A|H})\}$ examples of $(s_{h+1}, a_{h+1})$, with probability at lest $1 - \delta_1 = 1- \frac{\delta}{6|S||A|H}$.
    
    Using union bound over each pair $(s_{h+1}, a_{h+1}) \in S_{h+1} \times A$ such that $s_{h+1}$ is $\beta$-reachable for $\widehat{P}$, we obtain that 
    $\mathbb{P}[G_1^{h+1} | G^k_1 ,G^k_2 \;\forall i\in[h] ]\geq 1 - \delta_1|S||A| = 1 - \frac{\delta}{6H}$, which proves the induction step.
\end{proof}

\begin{lemma}\label{lemma: G_1 and G_2}
    $\mathbb{P}[G_1 \cap G_2] \geq 1 - 
    \delta/3$.
\end{lemma}

\begin{proof}
    Recall that $G_1 = \cap_{h \in [H-1]}G^h_1$ and $G_2 = \cap_{h \in [H-1]}G^h_2$.
        
    Let $X$ be a random variable with support $[H-1]$ such that  
    \begin{align*}
        X = \min_{k \in [H-1]}\{\overline{G}^k_1 \cup \overline{G}^k_2 \text{ holds }\}.
    \end{align*}
    Meaning, $X$ is the layer with the lowest index in which at least one of the good events $G^h_1$ or $G^h_2$ does not hold.
    If  $G^h_1$ and $G^h_2$ hold for every layer $h \in [H-1]$ then $X = \bot$.  
    By definition of $X$ and Bayes rule (i.e., for two events $A, B$: $\mathbb{P}[A \cap B] = \mathbb{P}[A|B]\cdot \mathbb{P}[B]$) the following holds.
    \begingroup
    \allowdisplaybreaks
    \begin{align*}
        \forall h \in [H-1].\;\;\; 
            \mathbb{P}[X = h]
            &=
            \mathbb{P}[(\overline{G}^h_1 \cup \overline{G}^h_2)\cap (\cap_{
            k \in [h-1]}G^k_1 \cap G^k_2) ]
            \\
            \tag{By Bayes rule}
            & =
            \mathbb{P}[(\overline{G}^h_1 \cup \overline{G}^h_2) | \cap_{
            k \in [h-1]}(G^k_1 \cap G^k_2)]
            \cdot
            \underbrace{\mathbb{P}[ \cap_{ k \in [h-1]} (G^k_1 \cap G^k_2)]}_{\leq 1}
            \\
            &\leq
            \mathbb{P}[(\overline{G}^h_1 \cup \overline{G}^h_2) | \cap_{
            k \in [h-1]}(G^k_1 \cap G^k_2)]
            \\
            \tag{ By union of disjoint events }
            & =
            \underbrace{\mathbb{P}[\overline{G}^h_1 | 
            \cap_{k \in [h-1]}(G^k_1 \cap G^k_2)]}_{\leq \frac{\delta}{6 H} \text{ by Lemma~\ref{lemma: prob to G^h_1 given previous events}}}
            +
            \mathbb{P}[\overline{G}^h_2 \cap G^h_1 |
            \cap_{k \in [h-1]}(G^k_1 \cap G^k_2)]
            \\
            &\leq
            \frac{\delta}{6H}
            +
            \mathbb{P}[\overline{G}^h_2 \cap G^h_1 |
            \cap_{k \in [h-1]}(G^k_1 \cap G^k_2)]
            \\
            \tag{By Bayes rule}
            & =
            \frac{\delta}{6 H}
            +
            \mathbb{P}[\overline{G}^h_2 | G^h_1 ,
            \cap_{k \in [h-1]}(G^k_1 \cap G^k_2)]
            \cdot
            \underbrace{\mathbb{P}[ G^h_1  |
            \cap_{k \in [h-1]}(G^k_1 \cap G^k_2)]]}_{\leq 1}
            \\
            &\leq
            \frac{\delta}{6 H}
            +
            \mathbb{P}[\overline{G}^h_2 | G^h_1 ,
            \cap_{k \in [h-1]}(G^k_1 \cap G^k_2)]      
            \\
            \tag{$\overline{G}^h_2$ depended only on $G^h_1$}
            &=
            \frac{\delta}{6 H}
            +
            \underbrace{\mathbb{P}[\overline{G}^h_2 | G^h_1]}_{\leq \frac{\delta}{6H} \text{ by Lemma~\ref{lemma: prob to G^h_2 given G^h_1}}}
            \\
            &\leq
            2 \frac{\delta}{6 H} = \frac{\delta}{3 H}.
        \end{align*}
        \endgroup
        
        Lastly, by $G_1$ and $G_2$ definition we have
        \begingroup
        \allowdisplaybreaks
        \begin{align*}
            \mathbb{P}[G_1 \cap G_2]
            &=
            1 - \mathbb{P}[\overline{G}_1 \cup \overline{G}_2]
            \\
            &=
            1 
            -
            \mathbb{P}[\cup_{h \in [H-1]}(\overline{G}^h_1 \cup \overline{G}^h_2)]
            \\
            &=
            1 
            -
            \mathbb{P}[\exists h \in [H-1].(\overline{G}^h_1 \cup \overline{G}^h_2)]
            \\
            &=
            1 
            -
            \mathbb{P}[\exists h \in [H-1].X = h]
            \\
            &=
            1 
            -
            \mathbb{P}[\cup_{ h \in [H-1]}\{X = h\}]
            \\
            &\underbrace{\geq}_{\text{Union Bound}}
            1 
            -
            \sum_{h \in [H-1]}\mathbb{P}[X = h]
            \\
            &\geq
            1 - \frac{\delta}{3} .
        \end{align*}
        \endgroup
    \end{proof}
    
\begin{lemma}\label{lemma: final probs case 2}
      It holds that $\mathbb{P}[G_1 \cap G_2 \cap G_3] \geq 1 - {\delta}/{2}$.
\end{lemma}

\begin{proof}
    By lemmas~\ref{lemma: G_1 and G_2} and~\ref{lemma: G_3 given G_1 UCFD} when combined with a union bound.
\end{proof}

\subsubsection{Bounding the Error Caused by the Dynamics Approximation}\label{subsubsec:dynamics-error-UCFD}

In the following, we consider an intermediate model $\widetilde{M}$, which defined as follows.

For any context $c \in \mathcal{C}$, we define 
$\widetilde{M}(c) = (S\cup \{s_{sink}\}, A, \widehat{P}, r^c, s_0, H)$, where we extend the true rewards function $r^c$ for the sink by defining
$
    r^c(s_{sink}, a):= 0,\;\; \forall c \in \mathcal{C},\; a \in A
$, and $\widehat{P}$ is the approximated dynamics.
    
Recall the true MDP associated with the context $c$ is $\mathcal{M}(c) = (S, A, P, r^c, s_0, H)$.

In the following lemma we bound the occupancy-measures differences under $P$ and $\widehat{P}$, for every policy $\pi$ under the good events.
\begin{lemma}\label{lemma: occ-measure-dist-UCFD}
    Assume the good events $G_1$ and $G_2$ hold. Then, for every policy 
    $\pi: s \to \Delta(A)$ and layer $h \in [H]$ it holds that
    \[
        \|q_h(\cdot | \pi, P) 
        - q_h(\cdot | \pi, \widehat{P}) \|_1 
        \leq \gamma h + \beta \sum_{k = 0}^{h-1}|S_k|,
    \]
    where 
    \[
        \forall h \in [H].\;\; \|q_h(\cdot | \pi, P) - q_h(\cdot | \pi, \widehat{P})\|_1 := 
        \sum_{s_{h} \in S_h}|q_h(s_h | \pi, P) - q_h(s_h | \pi, \widehat{P})|
    \]
    (i.e., $q_h(s_{sink}|\pi,\widehat{P})$ is omitted, for all $h \in [H]$).
\end{lemma}

\begin{remark}
    Since $s_{sink} \notin S$, $q_h(s_{sink}|\pi, P)$ is not defined for the true dynamics $P$. In addition, by $\widehat{P}$ definition, from the sink there are no transitions to any other state and has zero reward. Hence, we can simply ignore it in the following analysis.
\end{remark}

We now prove Lemma~\ref{lemma: occ-measure-dist-UCFD}.
\begin{proof}
    We show the lemma by induction on $h$.
    
    The base case $h = 0$ holds trivially since there is a unique start state $s_0$. Hence $q_0(s_0 | \pi, P) = q_0(s_0 | \pi, \widehat{P})=1$.
    
    For the induction step, we assume correctness for all $k < h$ and show for $h$.
    
    For every $k \leq h$ we define $S^\beta_k = \{s_k \in S_k: s_k \text{ is }\beta\text{-reachable for }\widehat{P}\}$.
    
    Since the good events $G_1$ and $G_2$ hold, we have for every $(s_k, a_k) \in S^\beta_k \times A$ that 
    $$\|P(\cdot| s_k, a_k) -  \widehat{P}(\cdot| s_k, a_k) \|_1 \leq \gamma .$$
    
    We remark that by definition $\widehat{P}(s_{sink}|s,a)= \mathbb{I}[n(s,a)< N_P(\gamma, \delta_1)] = 0$, under the good events and $P$ is not defined for $s_{sink}$, hence we can ignore it when analysing the dynamics total variation distance for the $\beta$-reachable states.
    
    Using the induction hypothesis we obtain,
    \begingroup
    \allowdisplaybreaks
    \begin{align*}
        &\| q_h(\cdot| \pi ,P) -  q_h(\cdot| \pi ,\widehat{P})\|_1
        \\
        =&\sum_{s_h \in S_h} |q_h(s_h| \pi ,P) -  q_h(s_h| \pi ,\widehat{P})|
        \\
        =&\sum_{s_{h-1} \in S_{h-1}} \sum_{a_{h-1} \in A} \sum_{s_h \in S_h}
        \pi(a_{h-1}| s_{h-1})
        |q_{h-1}(s_{h-1}| \pi ,P) P(s_h |s_{h-1}, a_{h-1} )
        -  
        q_{h-1}(s_{h-1}| \pi ,\widehat{P})\widehat{P}(s_h |s_{h-1}, a_{h-1} )|
        \\
        \leq &\sum_{s_{h-1} \in S_{h-1}} \sum_{a_{h-1} \in A} \pi(a_{h-1}| s_{h-1})P(s_h |s_{h-1}, a_{h-1} )
        \sum_{s_h \in S_h}
        |q_{h-1}(s_{h-1}| \pi ,P) - q_{h-1}(s_{h-1}| \pi ,\widehat{P})|
        \\
        &+ \sum_{s_{h-1} \in S_{h-1}}
        q_{h-1}(s_{h-1}| \pi ,\widehat{P})
        \sum_{a_{h-1} \in A} 
        \pi(a_{h-1}| s_{h-1})
        \sum_{s_h \in S_h}
        | P(s_h |s_{h-1}, a_{h-1} ) - \widehat{P}(s_h |s_{h-1}, a_{h-1} )|
        \\
        \leq &\|q_{h-1}(\cdot| \pi ,P) - q_{h-1}(\cdot | \pi ,\widehat{P})\|_1
        \\
        & +
        \sum_{s_{h-1} \in S_{h-1}}
        q_{h-1}(s_{h-1}| \pi ,\widehat{P})
        \sum_{a_{h-1} \in A} 
        \pi(a_{h-1}| s_{h-1})
        \sum_{s_h \in S_h}
        | P(s_h |s_{h-1}, a_{h-1} ) - \widehat{P}(s_h |s_{h-1}, a_{h-1} )|
        \\
        \leq &\gamma (h-1) + \beta \sum_{k = 0}^{h-2}|S_k|
        +         
        \sum_{s_{h-1} \in S^\beta_{h-1}}
        q_{h-1}(s_{h-1}| \pi ,\widehat{P})
        \sum_{a_{h-1} \in A} 
        \pi(a_{h-1}| s_{h-1})
        \underbrace{\| P(\cdot |s_{h-1}, a_{h-1} ) - 
        \widehat{P}(\cdot |s_{h-1}, a_{h-1} )\|_1}_{\leq \gamma}
        \\
        &+         
        \sum_{s_{h-1} \not\in  S^\beta_{h-1}}
        \underbrace{q_{h-1}(s_{h-1}| \pi ,\widehat{P})}_{\leq \beta}
        \underbrace{\sum_{a_{h-1} \in A} 
        \pi(a_{h-1}| s_{h-1})}_{=1}
        \underbrace{\sum_{s_h \in S_h}
        | P(s_h |s_{h-1}, a_{h-1} ) - \widehat{P}(s_h |s_{h-1}, a_{h-1} )|
        }_{\leq 1}
        \\
        \leq & \gamma h + \beta \sum_{k = 0}^{h-1}|S_k|.
    \end{align*}
    \endgroup
\end{proof}

\begin{remark}
    For $\beta = \frac{\epsilon}{24|S|H}$ and $\gamma = \frac{\epsilon}{48|S|H^2}$ we have $\beta - \gamma H \geq \frac{\epsilon}{48 |S|H}$.
\end{remark}

\begin{corollary}\label{corl: general occ measure bound}
    Under the good events $G_1$ and $G_2$, for $\beta = \frac{\epsilon}{24|S|H}$ and $\gamma = \frac{\epsilon}{48|S|H^2}$ we have for all $h \in [H]$ that
    $
        \|q_h(\cdot | \pi, P) 
        - q_h(\cdot | \pi, \widehat{P})\|_1 \leq \frac{ 3 \epsilon}{48 H} = \frac{\epsilon}{16 H}
    $. 
\end{corollary}

\begin{lemma}\label{lemma: UCFD - dynamics error.}
    Assume the good events $G_1$ and $G_2$ hold.
    
    
    
    Then, for the parameters choice of $\beta = \frac{\epsilon}{24|S|H}$ and $\gamma = \frac{\epsilon}{48|S|H^2}$,for every context-dependent policy $\pi=(\pi_c)_{c \in \mathcal{C}}$ it holds that, 
    \[
        |V^{\pi_c}_{\mathcal{M}(c)}(s_0) - V^{\pi_c}_{\widetilde{M}(c)}(s_0)|
        \leq
       \frac{\epsilon}{16}.
    \]
\end{lemma}

\begin{proof}
     Recall that the true rewards function is not defined for $s_{sink}$, since $s_{sink} \notin S$. A natural extension of $r^c$ to $s_{sink}$ is by defining 
    $\forall c \in \mathcal{C},\; \forall a\in A.\;\; r^c(s_{sink},a)=0$. 
    Since $P$ is also not defined for $s_{sink}$, we can simply ignore $s_{sink}$, as the second equality in the following calculation shows.
    
    Fix a context $c \in \mathcal{C}$ and a context-dependent policy $\pi = (\pi_c)_{c \in \mathcal{C}}$. Consider the following derivation.
    
    \begingroup
    \allowdisplaybreaks
    \begin{align*}
        &|V^{\pi_c}_{\mathcal{M}(c)}(s_0) - V^{\pi_c}_{\widetilde{M}(c)}(s_0)|
        \\
        = &
        |
            \sum_{h=0}^{H-1}
            \sum_{s_h \in S_h}
            \sum_{a_h \in A}
            q_h(s_h, a_h|\pi_c, P)\cdot r^c(s_h, a_h)
            - 
            \sum_{h=0}^{H-1}
            \sum_{s_h \in S_h \cup \{s_{sink}\}}
            \sum_{a_h \in A}
            q_h(s_h, a_h|\pi_c, \widehat{P}) \cdot
            r^c(s_h, a_h)
        |
        \\
        \tag{ Since we defined $r^c(s_{sink},a):= 0,\; \forall c,a$}
        &=
        |
            \sum_{h=0}^{H-1}
            \sum_{s_h \in S_h}
            \sum_{a_h \in A}
            q_h(s_h, a_h|\pi_c, P)\cdot r^c(s_h, a_h)
            -
            \sum_{h=0}^{H-1}
            \sum_{s_h \in S_h}
            \sum_{a_h \in A}
            q_h(s_h, a_h|\pi_c, \widehat{P}) \cdot
            r^c(s_h, a_h)
        |
        \\
        &=
        |
            \sum_{h=0}^{H-1}
            \sum_{s_h \in S_h}
            \sum_{a_h \in A}
            (q_h(s_h, a_h|\pi_c, P) - q_h(s_h, a_h|\pi_c, \widehat{P}))
            r^c(s_h, a_h)
        |
        \\
        &\leq
        \sum_{h = 0}^{H} \sum_{s_h \in S_h} \sum_{a_h \in A} \pi(a_h | s_h)
        |r^c(s_h, a_h)|
        |q_h(s_h | \pi_c, P) - q_h(s_h | \pi_c, \widehat{P})|
        \\
        \tag{ $r^c$ is
        bounded in $[0,1]$, and $\sum_{a \in A}\pi_c(a|s)=1$}
        &\leq
        \sum_{h = 0}^{H} \sum_{s_h \in S_h}
        |q_h(s_h | \pi_c, P) - q_h(s_h | \pi_c, \widehat{P})|
        \\
        \tag{By corollary~\ref{corl: general occ measure bound}}
        &\leq 
        \frac{3\epsilon}{48 H} H = \frac{\epsilon}{16},
    \end{align*}
    \endgroup
    which yields the lemma.
\end{proof}

\begin{corollary}\label{corl: UCFD - dynamics error}
    Assume the good events $G_1$ and $G_2$ hold.
    %
    
    %
    %
    Then, for the parameters choice  $\beta = \frac{\epsilon}{24|S|H}$ and $\gamma = \frac{\epsilon}{48|S|H^2}$ for every context-dependent policy $\pi= (\pi_c)_{c \in \mathrm{C}}$ it holds that
    \[
        \mathbb{E}_{c \sim \mathcal{D}}[|V^{\pi_c}_{\mathcal{M}(c)}(s_0) - V^{\pi_c}_{\widetilde{M}(c)}(s_0)|]
        \leq
        \frac{ \epsilon}{16}.
    \]
\end{corollary}

\begin{proof}
    Implied by taking expectation over both sided of the inequality stated in Lemma~\ref{lemma: UCFD - dynamics error.}.
\end{proof}

\subsubsection{Bounding the Error Caused by the Rewards Approximation for the \texorpdfstring{$\ell_2$}{Lg} loss.}\label{subsubsec:rewards-error-l-2-UCFD}

In this sub-subsection, we bound the error caused by the rewards approximation, by bounding the expected value difference between the intermediate model $\widetilde{\mathcal{M}}$ and the approximated model $\widehat{\mathcal{M}}$.
Here, we analyse the error for the $\ell_2$ loss.

Recall the definition of $\widetilde{\mathcal{M}}$.     
For every context $c \in \mathcal{C}$ we define
${\widetilde{\mathcal{M}}(c) = (S\cup \{s_{sink}\}, A, \widehat{P}, r^c, s_0, H)}$, where we extend the true rewards function $r^c$ for the sink by defining
$
    r^c(s_{sink}, a):= 0,\;\; \forall c \in \mathcal{C},\; a \in A
$, and $\widehat{P}$ is the approximated dynamics.
    
Also, recall the approximated MDP for the context $c$, ${\widehat{\mathcal{M}}(c) = (S\cup\{s_{sink}\}, A, \widehat{P}, \widehat{r}^c, s_0, H)}$ which defined in Algorithm~\ref{alg: EXPLOIT-UCFD}. 

In the following analysis,
Let $S^\beta(\widehat{P})$ be the set of $\beta$-reachable states for the dynamics $\widehat{P}$,
and $\alpha^2_2 = \max_{(s,a \in S^\beta(\widehat{P}) \times A)}\alpha^2_2(\mathcal{F}^R_{s,a})$ be the maximal agnostic approximation error.

\begin{lemma}\label{lemma: UCFD - rewards error l_2}
    Assume the good event $G_3$ holds.
    
    
    
    
    Then, for any context-dependent policy $\pi = (\pi_c)_{c \in \mathcal{C}}$ it holds that 
    \[
        \mathbb{E}_{c \sim \mathcal{D}}
        \left[
            \left|
                V^{\pi_c}_{\widetilde{\mathcal{M}}(c)}(s_0) 
                - 
                V^{\pi_c}_{\widehat{\mathcal{M}}(c)}(s_0)
            \right|
        \right]
        \leq
       \frac{\epsilon}{8} +  \alpha_2 H .
    \]
\end{lemma}

\begin{proof}
%
By construction of $\widehat{\mathcal{M}}(c)$ , for any $s\not\in S^\beta(\widehat{P})$, i.e., states which are not $\beta$-reachable for $\widehat{P}$, we set $f_{s,a}(c)=0$ for any action $a$. Hence,
\begin{equation*}
    |r^c(s,a) - f_{s,a}(c)| \leq 1.
\end{equation*}

Since the good event $G_3$ holds, we have for every $(s,a) \in S \times A$ such that $s$ is $\beta$-reachable for $\widehat{P}$ that
    \begin{equation*}
        \epsilon^2_\star(\widehat{p}_s) 
        +
        \alpha^2_2(\mathcal{F}^R_{s,a})
        \underbrace{\geq}_{G_3}
        \mathbb{E}_{c \sim \mathcal{D}}
        [(f_{s,a}(c) - r^c(s,a))^2]
        \underbrace{\geq}_{\text{Jensen's inequality}}
        \mathbb{E}^2_{c \sim \mathcal{D}}
        [|f_{s,a}(c) - r^c(s,a)|].
    \end{equation*}
    
Using that for all $a, b \in [0, \infty)$ it holds that $\sqrt{a} + \sqrt{b} \geq \sqrt{a + b}$, we obtain
\begin{equation}
        \sqrt{\epsilon^2_\star(\widehat{p}_s)} 
        +
        \alpha_2   
        \geq  
        \sqrt{\epsilon^2_\star(\widehat{p}_{s})}
        +
        \alpha_2(\mathcal{F}^R_{s,a})
        \geq 
        \sqrt{\epsilon^2_\star(\widehat{p}_{s}) +
        \alpha^2_2(\mathcal{F}^R_{s,a})}
        \geq
        \mathbb{E}_{c \sim \mathcal{D}}
        [|f_{s,a}(c) - r^c(s,a)|].
\end{equation}
The above implies that
\begin{equation}\label{ineq:l-2-ucfd-G3}
    \sqrt{\epsilon^2_\star(\widehat{p}_s)} 
    \geq  
    \mathbb{E}_{c \sim \mathcal{D}}[|f_{s,a}(c) - r^c(s,a)|-\alpha_2 ].
\end{equation}


Fix a context-dependent policy $\pi= (\pi_c)_{c \in \mathcal{C}}$.

By definition we have:
    \begingroup
    \allowdisplaybreaks
    \begin{align*}
        V^{\pi_c}_{\widetilde{\mathcal{M}}(c)}(s_0) 
        &=
        \sum_{h=0}^{H-1}\sum_{s\in S \cup \{s_{sink}\}} q_h(s| \pi_c, \widehat{P}) 
        \sum_{a\in A}\pi_c(a|s)r^c(s,a)
        \\
        \tag{$r^c(s_{sink},a):= 0,\;\; \forall c \in \mathcal{C}, a \in A$.}
        &=
        \sum_{h=0}^{H-1}\sum_{s\in S} q_h(s| \pi_c, \widehat{P}) 
        \sum_{a\in A}\pi_c(a|s)r^c(s,a)
        \\
        \tag{Since the MDP is layered and loop-free}
        &=
        \sum_{h=0}^{H-1}\sum_{s\in S_h} q_h(s| \pi_c, \widehat{P}) 
        \sum_{a\in A}\pi_c(a|s)r^c(s,a).
    \end{align*}
    \endgroup
    
    Similarly,
    \begingroup
    \allowdisplaybreaks
    \begin{align*}
        V^{\pi_c}_{\widehat{\mathcal{M}}(c)}(s_0)
        &= 
        \sum_{h=0}^{H-1}\sum_{s\in S \cup \{s_{sink}\}} q_h(s| \pi_c, \widehat{P}) 
        \sum_{a\in A}\pi_c(a|s)\widehat{r}^c(s,a)
        \\
        \tag{$\widehat{r}^c(s_{sink},a):= 0,\;\; \forall c \in \mathcal{C}, a \in A$.}
        &=
        \sum_{h=0}^{H-1}\sum_{s\in S} q_h(s| \pi_c, \widehat{P}) 
        \sum_{a\in A}\pi_c(a|s)\widehat{r}^c(s,a)
        \\
        \tag{Since the MDP is layered and loop-free}
        &=
        \sum_{h=0}^{H-1}\sum_{s\in S_h} q_h(s| \pi_c, \widehat{P}) 
        \sum_{a\in A}\pi_c(a|s)\widehat{r}^c(s,a).
        \\
        \tag{By definition, $\widehat{r}^c(s,a) := f_{s,a}(c)$}
        &=
        \sum_{h=0}^{H-1}\sum_{s\in S_h} q_h(s| \pi_c, \widehat{P}) 
        \sum_{a\in A}\pi_c(a|s)f_{s,a}(c).
    \end{align*}
    \endgroup


Recall that $\beta \leq \frac{\epsilon}{24 |S|}$.
Thus, if $s$ is not $\beta$-reachable for $\widehat{P}$, then $\widehat{p}_s < \beta \leq \frac{\epsilon}{24 |S|}$.
Moreover, if $\widehat{p}_s \geq \frac{\epsilon}{24 |S|}$ then $s$ is $\beta$-reachable for $\widehat{P}$, and the good event $G_3$ guarantee is hold for $s$. 
Thus, when combining all the above,
by linearity of expectation and triangle inequality we obtain,
 \begingroup
  \allowdisplaybreaks
  \begin{align*}
        &\mathbb{E}_{c \sim \mathcal{D}}
        \left[
            \left|
                V^{\pi_c}_{\widetilde{\mathcal{M}}(c)}(s_0) 
                - 
                V^{\pi_c}_{\widehat{\mathcal{M}}(c)}(s_0)
            \right|
        \right]
        \\
        =  & 
        \mathbb{E}_{c \sim \mathcal{D}}
        \left[
            \left|
                \sum_{h=0}^{H-1}\sum_{s\in S_h} q_h(s| \pi_c, \widehat{P}) 
                \sum_{a\in A}\pi_c(a|s) r^c(s,a)
                - 
                \sum_{h=0}^{H-1} \sum_{s\in S_h} q_h(s| \pi_c, \widehat{P}) 
                \sum_{a\in A}\pi_c(a|s) f_{s,a}(c)
            \right|
        \right] \\
        \leq &
        \mathbb{E}_{c \sim \mathcal{D}}
        \left[   
            \sum_{h=0}^{H-1}
            \sum_{s\in S_h} q_h(s| \pi_c, \widehat{P})
            \sum_{a\in A} \pi_c(a|s) |r^c(s,a)-f_{s,a}(c)|
        \right] \\
        = &
        \mathbb{E}_{c \sim \mathcal{D}}
        \left[   
            \sum_{h=0}^{H-1}
            \sum_{s\in S_h} q_h(s| \pi_c, \widehat{P})
            \sum_{a\in A} \pi_c(a|s) (|r^c(s,a)-f_{s,a}(c)|- \alpha_2 + \alpha_2)
        \right] \\
        \\
        = &
        \alpha_2 H
        +
        \mathbb{E}_{c \sim \mathcal{D}}
        \left[   
            \sum_{h=0}^{H-1}
            \sum_{s\in S_h} q_h(s| \pi_c, \widehat{P})
            \sum_{a\in A} \pi_c(a|s) (|r^c(s,a)-f_{s,a}(c)|- \alpha_2)
        \right] \\
        = &
        \alpha_2 H
        +
        \mathbb{E}_{c \sim \mathcal{D}}
        \left[
            \sum_{h=0}^{H-1}\sum_{a\in A}
            \sum_{s \in S_h : \widehat{p}_{s} < \frac{\epsilon}{24|S|}}
            q_h(s| \pi_c, \widehat{P}) \pi_c(a|s) 
            (|r^c(s,a)-f_{s,a}(c)|- \alpha_2)
        \right]\\
        & + 
        \mathbb{E}_{c \sim \mathcal{D}}
        \left[
            \sum_{h=0}^{H-1} \sum_{a\in A} 
            \sum_{s \in S_h : \widehat{p}_{s} > \frac{1}{|S|}}
            q_h(s| \pi_c, \widehat{P}) \pi_c(a|s) (|r^c(s,a)-f_{s,a}(c)|- \alpha_2)
        \right]\\
        & +
        \mathbb{E}_{c \sim \mathcal{D}}
        \left[
            \sum_{h=0}^{H-1} \sum_{a\in A} 
            \sum_{s \in S_h : \widehat{p}_{s} \in [\frac{\epsilon}{24|S|}, \frac{1}{|S|}]}
            q_h(s|\pi_c, \widehat{P}) \pi_c(a|s)
            (|r^c(s,a)-f_{s,a}(c)|- \alpha_2)
        \right]\\
        \leq &
        \alpha_2 H
        +
        \mathbb{E}_{c \sim \mathcal{D}}
        \left[
            \sum_{h=0}^{H-1}\sum_{a\in A}
            \pi_c(a|s) 
            \sum_{s \in S_h : \widehat{p}_{s} < \frac{\epsilon}{24|S|}}
            q_h(s| \pi_c, \widehat{P}) 
        \cdot 1
        \right]\\
        & + 
        \mathbb{E}_{c \sim \mathcal{D}}
        \left[
            \sum_{h=0}^{H-1} \sum_{a\in A} 
            \pi_c(a|s)
            \sum_{s \in S_h : \widehat{p}_{s} > \frac{1}{|S|}}
            \widehat{p}_s
            (|r^c(s,a)-f_{s,a}(c)|- \alpha_2)
        \right]\\
        & +
        \mathbb{E}_{c \sim \mathcal{D}}
        \left[
            \sum_{h=0}^{H-1} \sum_{a\in A} 
            \pi_c(a|s)
            \sum_{s \in S_h : \widehat{p}_{s} \in [\frac{\epsilon}{24|S|}, \frac{1}{|S|}]}
            \widehat{p}_s
            (|r^c(s,a)-f_{s,a}(c)|- \alpha_2)
        \right]\\
        & \leq 
        \alpha_2 H
        +        
        \frac{\epsilon}{24}
        + 
        \mathbb{E}_{c \sim \mathcal{D}}
        \left[
            \sum_{h=0}^{H-1} 
            \sum_{a\in A} 
            \sum_{s \in S_h : \widehat{p}_{s} > \frac{1}{|S|}}
            \widehat{p}_s
            (|r^c(s,a)-f_{s,a}(c)|- \alpha_2)
        \right]\\
        & +
        \mathbb{E}_{c \sim \mathcal{D}}
        \left[
            \sum_{h=0}^{H-1} \sum_{a\in A} 
            \sum_{s \in S_h : \widehat{p}_{s} \in [\frac{\epsilon}{24|S|}, \frac{1}{|S|}]}
            \widehat{p}_s
            (|r^c(s,a)-f_{s,a}(c)|- \alpha_2)
        \right]\\
        & =
        \alpha_2 H
        +        
        \frac{\epsilon}{24}
        + 
            \sum_{h=0}^{H-1} 
            \sum_{a\in A} 
            \sum_{s \in S_h : \widehat{p}_{s} > \frac{1}{|S|}}
            \widehat{p}_s
            \mathbb{E}_{c \sim \mathcal{D}}
            \left[    
                (|r^c(s,a)-f_{s,a}(c)|- \alpha_2)
            \right]\\
        & +
            \sum_{h=0}^{H-1} \sum_{a\in A} 
            \sum_{s \in S_h : \widehat{p}_{s} \in [\frac{\epsilon}{24|S|}, \frac{1}{|S|}]}
            \widehat{p}_s
            \mathbb{E}_{c \sim \mathcal{D}}
            \left[    
                (|r^c(s,a)-f_{s,a}(c)|- \alpha_2)
            \right]
        \\
        \tag{By inequality~(\ref{ineq:l-2-ucfd-G3})}
        &\leq
        \alpha_2 H
        +
        \frac{\epsilon}{24}
        + 
            \sum_{h=0}^{H-1} 
            \sum_{a\in A} 
            \sum_{s \in S_h : \widehat{p}_{s} > \frac{1}{|S|}}
            \widehat{p}_s
            \sqrt{\epsilon^2_\star(\widehat{p}_s)}
        +
            \sum_{h=0}^{H-1} \sum_{a\in A} 
            \sum_{s \in S_h : \widehat{p}_{s} \in [\frac{\epsilon}{24|S|}, \frac{1}{|S|}]}
            \widehat{p}_s
            \sqrt{\epsilon^2_\star(\widehat{p}_s)}
        \\  
        & =
        \alpha_2 H
        +
        \frac{\epsilon}{24}
        + 
            \sum_{h=0}^{H-1} 
            \sum_{a\in A} 
            \sum_{s \in S_h : \widehat{p}_{s} > \frac{1}{|S|}}
            \widehat{p}_s
            \frac{\epsilon}{24H|S||A|}
        +
            \sum_{h=0}^{H-1} \sum_{a\in A} 
            \sum_{s \in S_h : \widehat{p}_{s} \in [\frac{\epsilon}{24|S|}, \frac{1}{|S|}]}
            \widehat{p}_s
            \frac{\epsilon}{24\widehat{p}_s|S||A|} 
        \\
        & =        
        \alpha_2 H
        +
        \frac{\epsilon}{8},
    \end{align*}
\endgroup    
    as stated.    
\end{proof}

\paragraph{Combining both errors.}
In the following lemma, we combine the errors of both the dynamics and rewards approximation, to obtain an expected value-difference bound for the approximated and true models, which holds for every context-dependent policy.
Using it, we drive our main result in Theorem~\ref{thm: exploit UCFD l_2}.
\begin{lemma}\label{lemma: l_2 val-dif UCFD }
Assume the good events $G_1$ ,$G_2$ and $G_3$ hold. Then, for every context-dependent policy $\pi = (\pi_c)_{c \in \mathcal{C}}$ it holds  that
\begin{equation*}
    \mathbb{E}_{c}[|V^{\pi_c}_{\mathcal{M}(c)}(s_0) - V^{\pi_c}_{\widehat{\mathcal{M}}(c)}(s_0)|] \leq
    \frac{3}{16}\epsilon + \alpha_2 H.
\end{equation*}
\end{lemma}

\begin{proof}
Fix a context-dependent policy $\pi = (\pi_c)_{c \in \mathcal{C}}$.
By triangle inequality and linearity of expectation we have,
\begin{equation*}
   \mathbb{E}_{c \sim \mathcal{D}} [|V^{\pi_c}_{\mathcal{M}(c)}(s_0) - V^{\pi_c}_{\widehat{M}(c)}(s_0)|]
   \leq
    \mathbb{E}_{c \sim \mathcal{D}}[|V^{\pi_c}_{\mathcal{M}(c)}(s_0) - V^{\pi_c}_{\widetilde{M}(c)}(s_0)|] 
     +
     \mathbb{E}_{c \sim \mathcal{D}} [|V^{\pi_c}_{\widetilde{\mathcal{M}}(c)}(s_0) - V^\pi_{\widehat{M}(c)}(s_0)|] 
\end{equation*}

By Corollary~\ref{corl: UCFD - dynamics error} we have
\[
    \mathbb{E}_{c \sim \mathcal{D}}[|V^{\pi_c}_{\mathcal{M}(c)}(s_0) - V^{\pi_c}_{\widetilde{M}(c)}(s_0)|]
    \leq \frac{\epsilon}{16}.
\]
By Lemma~\ref{lemma: UCFD - rewards error l_2} we have
\[
    \mathbb{E}_{c \sim \mathcal{D}} [|V^{\pi_c}_{\widetilde{\mathcal{M}}(c)}(s_0) - V^{\pi_c}_{\widehat{M}(c)}(s_0)|]
    \leq
    \frac{\epsilon}{8}+ \alpha_2  H .
\]
Hence,
\begin{equation*}
    \mathbb{E}_{c}[|V^{\pi_c}_{\mathcal{M}(c)}(s_0) - V^{\pi_c}_{\widehat{\mathcal{M}}(c)}(s_0)|] 
    \leq \frac{3 \epsilon}{16} + \alpha_2 H .
\end{equation*}
\end{proof}

We have established the following theorem,
\begin{theorem}\label{thm: exploit UCFD l_2}
With probability at least $1-\delta$ it holds that
    \[
        \mathbb{E}_{c \sim \mathcal{D}}[V^{\pi^\star_c}_{\mathcal{M}(c)}(s_0) - V^{\widehat{\pi}^\star_c}_{\mathcal{M}(c)}(s_0)] \leq 
        \frac{3}{8}\epsilon + 2\alpha_2 H ,
    \]
Where $\pi^\star=(\pi^\star_c)_{c\in \mathcal{C}}$ is the optimal context-dependent policy for $\mathcal{M}$ and $\widehat{\pi}^\star= (\widehat{\pi}^\star_c)_{c\in \mathcal{C}}$ is the optimal context-dependent policy for $\widehat{\mathcal{M}}$.
\end{theorem}

\begin{proof}
Assume the good events $G_1$, $G_2$ and $G_3$ hold.
By Lemma~\ref{lemma: l_2 val-dif UCFD } we have for $\pi^\star$ that,
\begin{equation*}
    \left|\mathbb{E}_{c \sim \mathcal{D}}[V^{\pi^\star_c}_{\mathcal{M}(c)}(s_0) -  V^{\pi^\star_c}_{\widehat{\mathcal{M}}(c)}(s_0)] \right|
    \leq
    \mathbb{E}_{c \sim \mathcal{D}}[|V^{\pi^\star_c}_{\mathcal{M}(c)}(s_0) -  V^{\pi^\star_c}_{\widehat{\mathcal{M}}(c)}(s_0)|]
    \leq \frac{3 \epsilon}{16}+\alpha_2 H ,
\end{equation*}
yielding,
\begin{equation*}
    \mathbb{E}_{c \sim \mathcal{D}}[V^{\pi^\star_c}_{\mathcal{M}(c)}(s_0)]
    -  
    \mathbb{E}_{c \sim \mathcal{D}}[
    V^{\pi^\star_c}_{\widehat{\mathcal{M}}(c)}(s_0)]
    \leq \frac{3 \epsilon}{16}+\alpha_2 H .
\end{equation*}

Similarly, we have for $\widehat{\pi}^\star$ that
\begin{equation*}
    \mathbb{E}_{c \sim \mathcal{D}}[V^{\widehat{\pi}^\star_c}_{\widehat{\mathcal{M}}(c)}(s_0)]
    -
    \mathbb{E}_{c \sim \mathcal{D}}[
    V^{\widehat{\pi}^\star_c}_{\mathcal{M}(c)}(s_0)] \leq \frac{3 \epsilon}{16}+\alpha_2 H .
\end{equation*}
Since for all $c \in \mathcal{C}$, $\widehat{\pi}^\star_c$ is the optimal policy for $\widehat{\mathcal{M}}(c)$ we have $V^{\widehat{\pi}^\star_c}_{\widehat{\mathcal{M}}(c)}(s_0) \geq V^{\pi^\star_c}_{\widehat{\mathcal{M}}(c)}(s_0)$, which implies that
\begin{equation*}
    \mathbb{E}_{c \sim \mathcal{D}}[V^{\pi^\star_c}_{\widehat{\mathcal{M}}(c)}(s_0)] 
    -
    \mathbb{E}_{c \sim \mathcal{D}}[V^{\widehat{\pi}^\star_c}_{\widehat{\mathcal{M}}(c)}(s_0)] 
    \leq 0 .
\end{equation*}
By Lemma~\ref{lemma: final probs case 2} we have that $\mathbb{P}[G_1 \cap G_2 \cap G_3] \geq 1- {\delta}/{2}$. Hence the theorem follows by summing the above inequalities.
\end{proof}

For the realizable case, i.e., where $\alpha_2=0$, we obtain the following corollary.
\begin{corollary}
    For $\alpha_2 = 0$, with probability at least $1-\delta$ we have
    \[
        \mathbb{E}_{c \sim \mathcal{D}}[V^{\pi^\star_c}_{\mathcal{M}(c)}(s_0) - V^{\widehat{\pi}^\star_c}_{\mathcal{M}(c)}(s_0)] \leq 
        \epsilon.
    \]
\end{corollary}

\subsubsection{Bounding the Error Caused by the Rewards Approximation for the \texorpdfstring{$\ell_1$}{Lg} loss.}\label{subsubsec:rewards-error-l-1-UCFD}

In this sub-subsection, we bound the error caused by the rewards approximation, by bounding the expected value difference between the intermediate model $\widetilde{\mathcal{M}}$ and the approximated model $\widehat{\mathcal{M}}$.
Here e analyse the error for the $\ell_1$ loss.

Recall the definition of $\widetilde{\mathcal{M}}$.     
For every context $c \in \mathcal{C}$ we define
${\widetilde{\mathcal{M}}(c) = (S\cup \{s_{sink}\}, A, \widehat{P}, r^c, s_0, H)}$, where we extend the true rewards function $r^c$ for the sink by defining
$
    r^c(s_{sink}, a):= 0,\;\; \forall c \in \mathcal{C},\; a \in A
$, and $\widehat{P}$ is the approximated dynamics.
    
Also, recall the approximated MDP for the context $c$. ${\widehat{\mathcal{M}}(c) = (S\cup\{s_{sink}\}, A, \widehat{P}, \widehat{r}^c, s_0, H)}$ which defined in Algorithm~\ref{alg: EXPLOIT-UCFD}.

In the following analysis, 
let $S^\beta(\widehat{P})$ be the set of $\beta$-reachable states for the dynamics $\widehat{P}$,
and $\alpha_1 = \max_{(s,a \in S^\beta(\widehat{P}) \times A)}\alpha_1(\mathcal{F}^R_{s,a})$ be the maximal agnostic approximation error.

\begin{lemma}\label{lemma: UCFD - rewards error l_1.}
    Assume the good event $G_3$ holds.
    %
    %
    %
    %
    Then, for every context-dependent policy $\pi=(\pi_c)_{c \in \mathcal{C}}$ it holds that
    \[
        \mathbb{E}_{c \sim \mathcal{D}}
        \left[
            \left|
                V^{\pi_c}_{\widetilde{\mathcal{M}}(c)}(s_0) 
                - 
                V^{\pi_c}_{\widehat{\mathcal{M}}(c)}(s_0)
            \right|
        \right]
        \leq
        \frac{\epsilon}{8} + \alpha_1 H .
    \]
\end{lemma}

\begin{proof}
By construction of $\widehat{\mathcal{M}}$ , for every state $s\not\in S^\beta(\widehat{P})$, i.e., states which are not $\beta$-reachable for $\widehat{P}$, we set $f_{s,a}(c)=0$ for any action $a$. Hence,
\begin{equation*}
    |r^c(s,a) - f_{s,a}(c)| \leq 1.
\end{equation*}

Since the good event $G_3$ holds, we have for every $(s,a) \in S \times A$ such that $s$ is $\beta$-reachable for $\widehat{P}$ that
    \begin{equation}
        \epsilon_\star(\widehat{p}_s) 
        +
        \alpha_1
        \geq
        \epsilon_\star(\widehat{p}_s) 
        +
        \alpha_1(\mathcal{F}^R_{s,a})
        \underbrace{\geq}_{G_3}
        \mathbb{E}_{c \sim \mathcal{D}}
        [|f_{s,a}(c) - r^c(s,a)|].
    \end{equation}
The above implies that    
\begin{equation}\label{ineq:-l_1-G3-UCFD}
        \epsilon_\star(\widehat{p}_s) 
        \geq
        \mathbb{E}_{c \sim \mathcal{D}}
        [|f_{s,a}(c) - r^c(s,a)|-
        \alpha_1].
    \end{equation}



Fix a context-dependent policy $\pi= (\pi_c)_{c \in \mathcal{C}}$.

By definition we have:
    \begingroup
    \allowdisplaybreaks
    \begin{align*}
        V^{\pi_c}_{\widetilde{\mathcal{M}}(c)}(s_0) 
        &=
        \sum_{h=0}^{H-1}\sum_{s\in S \cup \{s_{sink}\}} q_h(s| \pi_c, \widehat{P}) 
        \sum_{a\in A}\pi_c(a|s)r^c(s,a)
        \\
        \tag{$r^c(s_{sink},a):= 0,\;\; \forall c \in \mathcal{C}, a \in A$.}
        &=
        \sum_{h=0}^{H-1}\sum_{s\in S} q_h(s| \pi_c, \widehat{P}) 
        \sum_{a\in A}\pi_c(a|s)r^c(s,a)
        \\
        \tag{Since the MDP is layered}
        &=
        \sum_{h=0}^{H-1}\sum_{s\in S_h} q_h(s| \pi_c, \widehat{P}) 
        \sum_{a\in A}\pi_c(a|s)r^c(s,a).
    \end{align*}
    \endgroup
    
    Similarly,
    \begingroup
    \allowdisplaybreaks
    \begin{align*}
        V^{\pi_c}_{\widehat{\mathcal{M}}(c)}(s_0)
        &= 
        \sum_{h=0}^{H-1}\sum_{s\in S \cup \{s_{sink}\}} q_h(s| \pi_c, \widehat{P}) 
        \sum_{a\in A}\pi_c(a|s)\widehat{r}^c(s,a)
        \\
        \tag{$\widehat{r}^c(s_{sink},a):= 0,\;\; \forall c \in \mathcal{C}, a \in A$.}
        &=
        \sum_{h=0}^{H-1}\sum_{s\in S} q_h(s| \pi_c, \widehat{P}) 
        \sum_{a\in A}\pi_c(a|s)\widehat{r}^c(s,a)
        \\
        \tag{Since the MDP is layered}
        &=
        \sum_{h=0}^{H-1}\sum_{s\in S_h} q_h(s| \pi_c, \widehat{P}) 
        \sum_{a\in A}\pi_c(a|s)\widehat{r}^c(s,a).
        \\
        \tag{By definition, $\widehat{r}^c(s,a) := f_{s,a}(c)$}
        &=
        \sum_{h=0}^{H-1}\sum_{s\in S_h} q_h(s| \pi_c, \widehat{P}) 
        \sum_{a\in A}\pi_c(a|s)f_{s,a}(c).
    \end{align*}
    \endgroup

Recall that $\beta \leq \frac{\epsilon}{24 |S|}$.
Thus, if $s$ is not $\beta$-reachable for $\widehat{P}$, then $\widehat{p}_s < \beta \leq \frac{\epsilon}{24 |S|}$.
Moreover, if $\widehat{p}_s \geq \frac{\epsilon}{24 |S|}$ then $s$ is $\beta$-reachable for $\widehat{P}$, and the good event $G_3$ guarantee is hold for $s$. 
Thus, when combining all the above,
by linearity of expectation and triangle inequality we obtain,
  \begingroup
  \allowdisplaybreaks
 \begin{align*}
        &\mathbb{E}_{c \sim \mathcal{D}}
        \left[
            \left|
                V^{\pi_c}_{\widetilde{\mathcal{M}}(c)}(s_0) 
                - 
                V^{\pi_c}_{\widehat{\mathcal{M}}(c)}(s_0)
            \right|
        \right]
        \\
        = & 
        \mathbb{E}_{c \sim \mathcal{D}}
        \left[
            \left|
                \sum_{h=0}^{H-1}\sum_{s\in S_h} q_h(s| \pi_c, \widehat{P}) 
                \sum_{a\in A}\pi_c(a|s) r^c(s,a)
                - 
                \sum_{h=0}^{H-1} \sum_{s\in S_h} q_h(s| \pi_c, \widehat{P}) 
                \sum_{a\in A}\pi_c(a|s) f_{s,a}(c)
            \right|
        \right] \\
        \leq &
        \mathbb{E}_{c \sim \mathcal{D}}
        \left[   
            \sum_{h=0}^{H-1}
            \sum_{s\in S_h} q_h(s| \pi_c, \widehat{P})
            \sum_{a\in A} \pi_c(a|s) |r^c(s,a)-f_{s,a}(c)|
        \right] \\
        = &
        \mathbb{E}_{c \sim \mathcal{D}}
        \left[   
            \sum_{h=0}^{H-1}
            \sum_{s\in S_h} q_h(s| \pi_c, \widehat{P})
            \sum_{a\in A} \pi_c(a|s) (|r^c(s,a)-f_{s,a}(c)|- \alpha_1 + \alpha_1)
        \right] \\
        = &
        \alpha_1 H
        +
        \mathbb{E}_{c \sim \mathcal{D}}
        \left[   
            \sum_{h=0}^{H-1}
            \sum_{s\in S_h} q_h(s| \pi_c, \widehat{P})
            \sum_{a\in A} \pi_c(a|s) (|r^c(s,a)-f_{s,a}(c)|- \alpha_1)
        \right] \\
        = &
        \alpha_1 H
        +
        \mathbb{E}_{c \sim \mathcal{D}}
        \left[
            \sum_{h=0}^{H-1}\sum_{a\in A}
            \sum_{s \in S_h : \widehat{p}_{s} < \frac{\epsilon}{24|S|}}
            q_h(s| \pi_c, \widehat{P}) \pi_c(a|s) 
            (|r^c(s,a)-f_{s,a}(c)|- \alpha_1)
        \right]\\
        & + 
        \mathbb{E}_{c \sim \mathcal{D}}
        \left[
            \sum_{h=0}^{H-1} \sum_{a\in A} 
            \sum_{s \in S_h : \widehat{p}_{s} > \frac{1}{|S|}}
            q_h(s| \pi_c, \widehat{P}) \pi_c(a|s) (|r^c(s,a)-f_{s,a}(c)|- \alpha_1)
        \right]\\
        & +
        \mathbb{E}_{c \sim \mathcal{D}}
        \left[
            \sum_{h=0}^{H-1} \sum_{a\in A} 
            \sum_{s \in S_h : \widehat{p}_{s} \in [\frac{\epsilon}{24|S|}, \frac{1}{|S|}]}
            q_h(s|\pi_c, \widehat{P}) \pi_c(a|s)
            (|r^c(s,a)-f_{s,a}(c)|- \alpha_1)
        \right]\\
        \leq &
        \alpha_1 H
        +
        \mathbb{E}_{c \sim \mathcal{D}}
        \left[
            \sum_{h=0}^{H-1}\sum_{a\in A}
            \pi_c(a|s) 
            \sum_{s \in S_h : \widehat{p}_{s} < \frac{\epsilon}{24|S|}}
            q_h(s| \pi_c, \widehat{P}) 
        \cdot 1
        \right]\\
        & + 
        \mathbb{E}_{c \sim \mathcal{D}}
        \left[
            \sum_{h=0}^{H-1} \sum_{a\in A} 
            \pi_c(a|s)
            \sum_{s \in S_h : \widehat{p}_{s} > \frac{1}{|S|}}
            \widehat{p}_s
            (|r^c(s,a)-f_{s,a}(c)|- \alpha_1)
        \right]\\
        & +
        \mathbb{E}_{c \sim \mathcal{D}}
        \left[
            \sum_{h=0}^{H-1} \sum_{a\in A} 
            \pi_c(a|s)
            \sum_{s \in S_h : \widehat{p}_{s} \in [\frac{\epsilon}{24|S|}, \frac{1}{|S|}]}
            \widehat{p}_s
            (|r^c(s,a)-f_{s,a}(c)|- \alpha_1)
        \right]\\
        & \leq 
        \alpha_1 H
        +        
        \frac{\epsilon}{24}
        + 
        \mathbb{E}_{c \sim \mathcal{D}}
        \left[
            \sum_{h=0}^{H-1} 
            \sum_{a\in A} 
            \sum_{s \in S_h : \widehat{p}_{s} > \frac{1}{|S|}}
            \widehat{p}_s
            (|r^c(s,a)-f_{s,a}(c)|- \alpha_1)
        \right]\\
        & +
        \mathbb{E}_{c \sim \mathcal{D}}
        \left[
            \sum_{h=0}^{H-1} \sum_{a\in A} 
            \sum_{s \in S_h : \widehat{p}_{s} \in [\frac{\epsilon}{24|S|}, \frac{1}{|S|}]}
            \widehat{p}_s
            (|r^c(s,a)-f_{s,a}(c)|- \alpha_1)
        \right]\\
        & =
        \alpha_1 H
        +        
        \frac{\epsilon}{24}
        + 
            \sum_{h=0}^{H-1} 
            \sum_{a\in A} 
            \sum_{s \in S_h : \widehat{p}_{s} > \frac{1}{|S|}}
            \widehat{p}_s
            \mathbb{E}_{c \sim \mathcal{D}}
            \left[    
                (|r^c(s,a)-f_{s,a}(c)|- \alpha_1)
            \right]\\
        & +
            \sum_{h=0}^{H-1} \sum_{a\in A} 
            \sum_{s \in S_h : \widehat{p}_{s} \in [\frac{\epsilon}{24|S|}, \frac{1}{|S|}]}
            \widehat{p}_s
            \mathbb{E}_{c \sim \mathcal{D}}
            \left[    
                (|r^c(s,a)-f_{s,a}(c)|- \alpha_1)
            \right]
            \\
            \tag{By inequality~(\ref{ineq:-l_1-G3-UCFD})}
            &\leq
        \alpha_1 H
        +
        \frac{\epsilon}{24}
        + 
            \sum_{h=0}^{H-1} 
            \sum_{a\in A} 
            \sum_{s \in S_h : \widehat{p}_{s} > \frac{1}{|S|}}
            \widehat{p}_s
            \epsilon^1_\star(\widehat{p}_s)
        +
            \sum_{h=0}^{H-1} \sum_{a\in A} 
            \sum_{s \in S_h : \widehat{p}_{s} \in [\frac{\epsilon}{24|S|}, \frac{1}{|S|}]}
            \widehat{p}_s
            \epsilon^1_\star(\widehat{p}_s)
        \\  
        & =
        \alpha_1 H
        +
        \frac{\epsilon}{24}
        + 
            \sum_{h=0}^{H-1} 
            \sum_{a\in A} 
            \sum_{s \in S_h : \widehat{p}_{s} > \frac{1}{|S|}}
            \widehat{p}_s
            \frac{\epsilon}{24H|S||A|}
        +
            \sum_{h=0}^{H-1} \sum_{a\in A} 
            \sum_{s \in S_h : \widehat{p}_{s} \in [\frac{\epsilon}{24|S|}, \frac{1}{|S|}]}
            \widehat{p}_s
            \frac{\epsilon}{24\widehat{p}_s|S||A|} 
        \\
        & =        
        \alpha_1 H
        +
        \frac{\epsilon}{8},
    \end{align*}
    \endgroup
    as stated.
\end{proof}

\paragraph{Combining both errors.}
In the following lemma, we combine the errors of both the dynamics and rewards approximation, to obtain an expected value-difference bound for the approximated and true models, which holds for every context-dependent policy.
Using it, we drive our main result in Theorem~\ref{thm: exploit UCFD l_1}.
\begin{lemma}\label{lemma: l_1 val-dif UCFD }
Assume the good events $G_1$ ,$G_2$ and $G_3$ hold. Then, for every policy context-dependent policy $\pi = (\pi_c)_{c \in \mathcal{C}}$ it holds that
\begin{equation*}
    \mathbb{E}_{c}[|V^{\pi_c}_{\mathcal{M}(c)}(s_0) - V^{\pi_c}_{\widehat{\mathcal{M}}(c)}(s_0)|] \leq \frac{3}{16}\epsilon + \alpha_1 H.
\end{equation*}
\end{lemma}

\begin{proof}
Fix a context-dependent policy $\pi =(\pi_c)_{c \in \mathcal{C}}$. Then,
\begin{equation*}
   \mathbb{E}_{c \sim \mathcal{D}} [|V^{\pi_c}_{\mathcal{M}(c)}(s_0) - V^{\pi_c}_{\widehat{M}(c)}(s_0)|]
   \leq
    \mathbb{E}_{c \sim \mathcal{D}}[|V^{\pi_c}_{\mathcal{M}(c)}(s_0) - V^{\pi_c}_{\widetilde{M}(c)}(s_0)|] 
     +
     \mathbb{E}_{c \sim \mathcal{D}} [|V^{\pi_c}_{\widetilde{\mathcal{M}}(c)}(s_0) - V^{\pi_c}_{\widehat{M}(c)}(s_0)|] 
\end{equation*}

By Corollary~\ref{corl: UCFD - dynamics error} we have
\[
    \mathbb{E}_{c \sim \mathcal{D}}[|V^{\pi_c}_{\mathcal{M}(c)}(s_0) - V^{\pi_c}_{\widetilde{M}(c)}(s_0)|]
    \leq \frac{\epsilon}{16}.
\]
By Lemma~\ref{lemma: UCFD - rewards error l_1.} we have
\[
    \mathbb{E}_{c \sim \mathcal{D}} [|V^{\pi_c}_{\widetilde{\mathcal{M}}(c)}(s_0) - V^{\pi_c}_{\widehat{M}(c)}(s_0)|]
    \leq
    \alpha_1  H + \frac{\epsilon}{8}.
\]
Hence,
\begin{equation*}
    \mathbb{E}_{c}[|V^{\pi_c}_{\mathcal{M}(c)}(s_0) - V^{\pi_c}_{\widehat{\mathcal{M}}(c)}(s_0)|] 
    \leq
    \frac{3 \epsilon}{16} +\alpha_1 H .
\end{equation*}
\end{proof}

We have established the following theorem,
\begin{theorem}\label{thm: exploit UCFD l_1}
With probability at least $1-\delta$ it holds that
    \[
        \mathbb{E}_{c \sim \mathcal{D}}[V^{\pi^\star_c}_{\mathcal{M}(c)}(s_0) - V^{\widehat{\pi}^\star_c}_{\mathcal{M}(c)}(s_0)]
        \leq 
         \frac{3}{8}\epsilon + 2\alpha_1 H,
    \]
where $\pi^\star=(\pi^\star_c)_{c \in \mathcal{C}}$ is the optimal context-dependent policy for $\mathcal{M}$ and $\widehat{\pi}^\star=(\widehat{\pi}^\star_c)_{c \in \mathcal{C}}$ is the optimal context-dependent policy for $\widehat{\mathcal{M}}$.
\end{theorem}

\begin{proof}
Assume the good events $G_1$, $G_2$ and $G_3$ hold.
By Lemma~\ref{lemma: l_1 val-dif UCFD } we have for $\pi^\star$ that,

\begin{equation*}
    \left|\mathbb{E}_{c \sim \mathcal{D}}[V^{\pi^\star_c}_{\mathcal{M}(c)}(s_0) -  V^{\pi^\star_c}_{\widehat{\mathcal{M}}(c)}(s_0)] \right|
    \leq
    \mathbb{E}_{c \sim \mathcal{D}}[|V^{\pi^\star_c}_{\mathcal{M}(c)}(s_0) -  V^{\pi^\star_c}_{\widehat{\mathcal{M}}(c)}(s_0)|]
    \leq \frac{3 \epsilon}{16}+\alpha_1 H ,
\end{equation*}
yielding,
\begin{equation*}
    \mathbb{E}_{c \sim \mathcal{D}}[V^{\pi^\star_c}_{\mathcal{M}(c)}(s_0)]
    -  
    \mathbb{E}_{c \sim \mathcal{D}}[
    V^{\pi^\star_c}_{\widehat{\mathcal{M}}(c)}(s_0)]
    \leq \frac{3 \epsilon}{16}+\alpha_1 H .
\end{equation*}

Similarly, we have for $\widehat{\pi}^\star$ that,
\begin{equation*}
    \mathbb{E}_{c \sim \mathcal{D}}[V^{\widehat{\pi}^\star_c}_{\widehat{\mathcal{M}}(c)}(s_0)]
    -
    \mathbb{E}_{c \sim \mathcal{D}}[
    V^{\widehat{\pi}^\star_c}_{\mathcal{M}(c)}(s_0)] \leq \frac{3 \epsilon}{16}+\alpha_1 H .
\end{equation*}
Since for all $c \in \mathcal{C}$, $\widehat{\pi}^\star_c$ is the optimal policy for $\widehat{\mathcal{M}}(c)$ we have $V^{\widehat{\pi}^\star_c}_{\widehat{\mathcal{M}}(c)}(s_0) \geq V^{\pi^\star_c}_{\widehat{\mathcal{M}}(c)}(s_0)$, which implies that
\begin{equation*}
    \mathbb{E}_{c \sim \mathcal{D}}[V^{\pi^\star_c}_{\widehat{\mathcal{M}}(c)}(s_0)] 
    -
    \mathbb{E}_{c \sim \mathcal{D}}[V^{\widehat{\pi}^\star_c}_{\widehat{\mathcal{M}}(c)}(s_0)] 
    \leq 0 .
\end{equation*}

By Lemma~\ref{lemma: final probs case 2} we have that $\mathbb{P}[G_1 \cap G_2 \cap G_3] \geq 1- {\delta}/{2}$. Hence the theorem follows by summing the above three inequalities.
\end{proof}

For the realizable case, i.e., $\alpha_1=0$ we have the following corollary.
\begin{corollary}
    For $\alpha_1 = 0$, with probability at least $1-\delta$ we have
    \[
        \mathbb{E}_{c \sim \mathcal{D}}[V^{\pi^\star_c}_{\mathcal{M}(c)}(s_0) - V^{\widehat{\pi}^\star_c}_{\mathcal{M}(c)}(s_0)] \leq 
        \epsilon.
    \]
\end{corollary}

\subsection{Sample complexity bounds}\label{sec: sampel complexity UCFD}
We present sample complexity bounds based on dimension analysis.
Recall Theorems~\ref{thm: pseudo dim} and~\ref{thm: fat dim},

\begin{theorem}[Adaption of Theorem 19.2 in~\cite{Bartlett1999NeuralNetsBook}]
    Let $\mathcal{F}$ be a hypothesis space of real valued functions with a finite pseudo dimension, denoted $Pdim(\mathcal{F}) < \infty$. Then, $\mathcal{F}$ has a uniform convergence with 
    \[
        m(\epsilon, \delta) = O \Big( \frac{1}{\epsilon^2}( Pdim(\mathcal{F}) \ln \frac{1}{\epsilon} + \ln \frac{1}{\delta})\Big).
    \]    
\end{theorem}

\begin{theorem}[Adaption of Theorem 19.1 in~\cite{Bartlett1999NeuralNetsBook}]
    Let $\mathcal{F}$ be a hypothesis space of real valued functions with a finite fat-shattering dimension, denoted $fat_{\mathcal{F}}$. Then, $\mathcal{F}$ has a uniform convergence with 
    \[
        m(\epsilon, \delta) = O \Big( \frac{1}{\epsilon^2}( fat_{\mathcal{F}}(\epsilon/256)  \ln^2 \frac{1}{\epsilon} + \ln \frac{1}{\delta})\Big).
    \]    
\end{theorem}

\begin{remark}
    The calculations bellow hold for any set of weights $\{\widehat{p}_s \in [0,1]\}_{s \in S}$. Hence, although $\widehat{p}_s$ is a random variable that depends on the tabular approximation of the dynamics (which affected by the observations), we can use it to compute a general bound on the sample complexity of the algorithm. 
\end{remark}

\subsubsection{Sample complexity bounds for the \texorpdfstring{$\ell_2$}{Lg} loss}\label{subsubsec:smaple-complexity-l-2-UCFD}

We present sample complexity for function classes with finite Pseudo dimension with $\ell_2$ loss.
\begin{corollary}\label{corl: sample complexity UCFD l_2 pseudo}
    Assume that for every $(s,a) \in S \times A$ we have that $Pdim(\mathcal{F}^R_{s,a}) < \infty$.
    Let $Pdim = \max_{(s,a) \in S \times A} Pdim(\mathcal{F}^R_{s,a})$.
    Then, after collecting 
    \[
        O \Big( \frac{|S|^2 |A|H}{\epsilon} \ln{\frac{|S||A|H}{\delta}} +  
        \frac{H^4 |S|^6 |A|^5  (\;Pdim \ln{\frac{H^2 |S|^2 |A|^2}{\epsilon^2}} + \ln {\frac{|S||A|H}{\delta}})}{\epsilon^4}  \Big)
    \] 
    trajectories, with probability at least $1-\delta$ it holds that
    \[
        \mathbb{E}_{c \sim \mathcal{D}}[V^{\pi^\star_c}_{\mathcal{M}(c)}(s_0) - V^{\widehat{\pi}^\star_c}_{\mathcal{M}(c)}(s_0)] \leq 
        \epsilon + 2\alpha_2 H .
    \]
\end{corollary}

\begin{proof}
    Recall that for each state-action pair $(s,a)$ such that $s$ is
    $\beta$-reachable for $\widehat{P}$ , Algorithm EXPLORE-UCFD runs for $T_{s,a} =   \lceil \frac{2}{\widehat{p}_{s} - \gamma h}(\ln(\frac{1}{\delta_1})+
    \max{\{ N_R (\mathcal{F}^R_{s,a} ,\epsilon^2_\star(\widehat{p}_{s}), \delta_1), N_P(\gamma, \delta_1)}\}) \rceil$ episodes.
    By Theorem~\ref{thm: exploit UCFD l_2}, for\\
    $
        \sum_{h =0}^{H-1} \sum_{s \in S_h : \widehat{p}_s \geq \epsilon/ 24|S|h} \sum_{a \in A} T_{s,a}
    $    
    samples we have with probability at least $1-\delta$ that
    \[
        \mathbb{E}_{c \sim \mathcal{D}}[V^{\pi^\star_c}_{\mathcal{M}(c)}(s_0) - V^{\widehat{\pi}^\star_c}_{\mathcal{M}(c)}(s_0)] \leq 
       \epsilon +  2\alpha_2 H.
    \]
    
    To simplify the analysis, assume that we first lean the dynamics (for each $\beta$-reachable state and every action) and then use it to approximate the rewards using an i.i.d sample of contexts and rewards for each non-negligible state and action.
    Note that in algorithm EXPLORE-UCFD we do not separate between the learning phases. 
    By corollary~\ref{corl: Sample complexity to approx dynamics UCFD}, for
    $\gamma = \frac{\epsilon}{48|S|H^2}$ and $\delta_1 = \frac{\delta}{6 |S||A| H} $ we have that
    \[
        N_P(\gamma, \delta_1) = 
        O\Big(
        \frac{ H^4 |S|^2}{\epsilon^2}\Big(\ln \Big(\frac{|S||A|H}{\delta} + |S| \Big) \Big)
        \Big).
    \]
    Hence, to learn the dynamics for each $\beta$-reachable state $s$ and action $a$ for the approximate dynamics $\widehat{P}$, we have to collect
    \begin{align*}
        O\Big( 
             |A||S|\frac{|S|H}{\epsilon} \frac{ H^4 |S|^2}{\epsilon^2}\Big(\ln \Big(\frac{|S||A|H}{\delta} + |S| \Big) \Big)
        \Big)
        =
        O\Big( 
           \frac{ H^5 |S|^4 |A|}{\epsilon^3}\Big(\ln \Big(\frac{|S||A|H}{\delta} + |S| \Big) \Big)
        \Big).        
    \end{align*}
    trajectories.
    (Since for every $\beta$-reachable state $s \in S_h$ and action $a \in A$ we have $\widehat{p}_s - \gamma h \geq \beta - \gamma h \geq \beta - \gamma H \geq \gamma H = O(\epsilon/ |S|H)$).
    
    To approximate the rewards,
    since for every $(s,a) \in S \times A$ we have that $Pdim(\mathcal{F}^R_{s,a}) < \infty$, and $Pdim = \max_{(s,a) \in S \times A} Pdim(\mathcal{F}^R_{s,a})$, by Theorem~\ref{thm: pseudo dim}, for every $(s,a) \in S \times A$  we have 
    \[
        N_R(\mathcal{F}^R_{s,a}, \epsilon^2_\star(\widehat{p}_s), \delta_1)
        =
        O \left( \frac{ Pdim \ln \frac{1}{\epsilon^2_\star(\widehat{p}_s)} + \ln \frac{1}{\delta_1}}{\epsilon_\star^4(\widehat{p}_s)}\right)
    \]
    
    By the accuracy-per-state function, for states $s$ that satisfies $\widehat{p}_s < \frac{\epsilon}{24 |S|}$ we have $\epsilon^2_\star(\widehat{p}_s) =1$, hence for every action $a$, we have that  $ N_R(\mathcal{F}^R_{s,a}, \epsilon^2_\star(\widehat{p}_s), \delta_1) = O(\ln(1/\delta))$. Thus they are negligible.
    
    Overall, the sample complexity for learning the rewards is as follows.
    \begingroup
    \allowdisplaybreaks
    \begin{align*}
        & \sum_{h=0}^{H-1} \sum_{s \in S_h: \widehat{p}_s \geq 24|S|}
        \sum_{a \in A}
        T_{s,a}
        =
        \sum_{h=0}^{H-1} \sum_{s \in S_h: \widehat{p}_s \geq 24|S|}
        \sum_{a \in A}
        O\Big( \frac{1}{\widehat{p}_s - \gamma h}\left(\ln(\frac{1}{\delta_1})+ N_R(\mathcal{F}^R_{s,a} ,\epsilon_\star^2(\widehat{p}_s), \delta_1)\right)  \Big)
        \\
        =&
        \sum_{h=0}^{H-1} \sum_{s \in S_h: \widehat{p}_s \geq 24|S|}
        \sum_{a \in A}
        O\Big( \frac{1}{\widehat{p}_s - \gamma h}\Big(\ln\frac{1}{\delta_1}
        +
        \frac{ Pdim \ln \frac{1}{\epsilon^2_\star (\widehat{p}_s)} + \ln \frac{1}{\delta_1}}{\epsilon_\star^4(\widehat{p}_s)} \Big)  \Big)
        \\ 
        =&
        \sum_{h=0}^{H-1} \sum_{s \in S_h: \widehat{p}_s \in [\epsilon/24|S|,1/|S|]}
        \sum_{a \in A}
        O\Big( \frac{1}{\widehat{p}_s - \gamma h}\Big(\ln\frac{1}{\delta_1}
        +
        \frac{ Pdim \ln \frac{1}{\epsilon^2/576 \widehat{p}_s^2 |S|^2 |A|^2} 
        + \ln \frac{1}{\delta_1}}{\epsilon^4/576^2 \widehat{p}_s^4 |S|^4} |A|^4 \Big)  \Big)
        \\ 
        &+
        \sum_{h=0}^{H-1} \sum_{s \in S_h: \widehat{p}_s > 1/|S|}
        \sum_{a \in A}
        O\Big( \frac{1}{\widehat{p}_s - \gamma h}\Big( \ln\frac{1}{\delta_1}
        +
        \frac{ Pdim \ln \frac{1}{\epsilon^2 / 576 H^2 |S|^2 |A|^2} + \ln \frac{1}{\delta_1}}{\epsilon^4 / 576^2 H^4 |S|^4 |A|^4} \Big)  \Big)
        \\
        =&
        \sum_{h=0}^{H-1} \sum_{s \in S_h: \widehat{p}_s \in [\epsilon/24|S|,1/|S|]}
        \sum_{a \in A}
        O\Big( \frac{1}{\widehat{p}_s - \gamma h}\Big(\ln\frac{1}{\delta_1}
        +
       \frac{ \widehat{p}_s^4 |S|^4 |A|^4 (Pdim \ln \frac{\widehat{p}_s^2 |S|^2 |A|^2}{\epsilon^2} 
        + \ln \frac{1}{\delta_1})}{\epsilon^4} \Big)  \Big)
        \\ 
        &+
        \sum_{h=0}^{H-1} \sum_{s \in S_h: \widehat{p}_s > 1/|S|}
        \sum_{a \in A}
        O\Big( \frac{1}{\widehat{p}_s - \gamma h} \Big( \ln(\frac{1}{\delta_1}
        +
        \frac{H^4 |S|^4 |A|^4 (Pdim \ln \frac{H^2 |S|^2 |A|^2}{\epsilon^2 } + \ln \frac{1}{\delta_1})}{\epsilon^4} \Big) \Big)
        \\
        \underbrace{\leq}_{(\star)}&
        \sum_{h=0}^{H-1} \sum_{s \in S_h: \widehat{p}_s \in [\epsilon/24|S|,1/|S|]}
        \sum_{a \in A}
        O\Big( \frac{|S|H}{\epsilon}\ln{\frac{|S||A|H}{\delta}} +
        \frac{\widehat{p_s}}{\widehat{p}_s - \gamma h}\frac{ \widehat{p}_s^3 |S|^4 |A|^4 (Pdim \ln{\frac{|S|^2 |A|^2}{\epsilon^2}} 
        + \ln {\frac{|S||A|H}{\delta}})}{\epsilon^4}  \Big)
        \\ 
        &+
        \sum_{h=0}^{H-1} \sum_{s \in S_h: \widehat{p}_s > 1/|S|}
        \sum_{a \in A}
        O\Big( |S| \ln{\frac{|S||A|H}{\delta }}
        +
        \frac{ H^4 |S|^5 |A|^4 (Pdim \ln \frac{H^2 |S|^2 |A|^2}{\epsilon^2 } + \ln \frac{|S||A|H}{\delta})}{\epsilon^4}  \Big)
        \\
        \underbrace{\leq}_{(\star \star)}& 
        \sum_{h=0}^{H-1} \sum_{s \in S_h: \widehat{p}_s \in [\epsilon/24|S|,1/|S|]}
        \sum_{a \in A}
        O\Big( \frac{|S|H}{\epsilon}\ln{\frac{|S||A|H}{\delta}}
        +
        \frac{|S| |A|^4 (Pdim \ln{\frac{|S|^2 |A|^2}{\epsilon^2}} 
        + \ln {\frac{|S||A|H}{\delta}})}{\epsilon^4}  \Big)
        \\ 
        &+
        \sum_{h=0}^{H-1} \sum_{s \in S_h: \widehat{p}_s > 1/|S|}
        \sum_{a \in A}
        O\Big( |S| \ln{\frac{|S||A|H}{\delta }}
        +
        \frac{ H^4 |S|^5 |A|^4 (Pdim \ln \frac{H^2 |S|^2 |A|^2}{\epsilon^2 } + \ln \frac{|S||A|H}{\delta})}{\epsilon^4}  \Big)
        \\ 
        =&
        O \Big( \frac{|S|^2 |A|H}{\epsilon} \ln{\frac{|S||A|H}{\delta}} +  
        \frac{H^4 |S|^6 |A|^5  (\;Pdim \ln{\frac{H^2 |S|^2 |A|^2}{\epsilon^2}} + \ln {\frac{|S||A|H}{\delta}})}{\epsilon^4}  \Big)
    \end{align*}
    \endgroup

Where $(\star)$ is since for any $h \in [H]$ we  have $\widehat{p}_s - \gamma h \geq \beta - \gamma h \geq \beta - \gamma H \geq \gamma H = O(\epsilon / |S|H)$, and if $\widehat{p}_s > \frac{1}{|S|}$ we have $\widehat{p}_s -\gamma h \geq \frac{1}{|S|} -\gamma h \geq \frac{1}{|S|} - \gamma H \geq \frac{1}{|S|} - \frac{1}{|S|H} = \frac{H -1}{|S|H} = O(1 /|S|) $.

$(\star \star)$ is since $\widehat{p}^3_s \leq 1/|S|^3$. In addition, $\widehat{p}_s - \gamma h \geq  \gamma H$ which implies that $\frac{\gamma H}{\widehat{p}_s - \gamma h} \leq 1$. Hence,
\[
    \frac{\widehat{p}_s}{\widehat{p}_s - \gamma h} = \frac{\widehat{p}_s - \gamma h + \gamma h}{\widehat{p}_s - \gamma h}
    = 
    1 + \frac{\gamma h}{\widehat{p}_s - \gamma h} 
    \leq
    1 + \frac{\gamma H}{\widehat{p}_s - \gamma h}
    \leq 
    2 .
\]   

Since the MDP is layered $|S| \geq H$, hence, the overall sample complexity is  
\[
    O \Big( \frac{|S|^2 |A|H}{\epsilon} \ln{\frac{|S||A|H}{\delta}} +  
    \frac{H^4 |S|^6 |A|^5  (\;Pdim \ln{\frac{H^2 |S|^2 |A|^2}{\epsilon^2}} + \ln {\frac{|S||A|H}{\delta}})}{\epsilon^4}  \Big)
    .
\]
\end{proof}

We also show similar sample complexity for function classes with finite fat-shattering dimension when using $\ell_2$ loss.
\begin{remark}
    The sample complexity for function classes with finite fat-shattering dimension with $\ell_2$ loss, where in $Fdim$ below we also maximizes over $\epsilon^2_\star(\widehat{p}_s) $ and the maximum is bounded and independent of $\widehat{p}_s$.
\end{remark}

\begin{corollary}\label{corl: sample complexity UCFD l_2 fat}
    Assume that for every $(s,a) \in S \times A$ we have that $\mathcal{F}^R_{s,a}$ has a finite fat-shattering dimension.
    Let $Fdim = \max_{(s,a) \in S \times A} fat_{\mathcal{F}^R_{s,a}}(\epsilon^2_\star(\widehat{p}_s)/256)$.
    Then, after collecting 
    \[
        O \Big( \frac{|S|^2 |A|H}{\epsilon} \ln{\frac{|S||A|H}{\delta}} +  
        \frac{H^4 |S|^6 |A|^5  (\;Fdim \ln^2{\frac{H^2 |S|^2 |A|^2}{\epsilon^2}} + \ln {\frac{|S||A|H}{\delta}})}{\epsilon^4}  \Big)
    \] 
    trajectories, with probability at least $1-\delta$ it holds that
    \[
        \mathbb{E}_{c \sim \mathcal{D}}[V^{\pi^\star_c}_{\mathcal{M}(c)}(s_0) - V^{\widehat{\pi}^\star_c}_{\mathcal{M}(c)}(s_0)] \leq 
        \epsilon +2\alpha_2 H .
    \]
\end{corollary}

\begin{proof}
    Recall that for each state-action pair $(s,a)$ such that $s$ is $\frac{\epsilon}{24 |S|}$-reachable for $\widehat{P}$, we run for $T_{s,a} =
    \lceil \frac{2}{\widehat{p}_{s} - \gamma h}(\ln(\frac{1}{\delta_1})+
    \max{\{ N_R (\mathcal{F}^R_{s,a} ,\epsilon^2_\star(\widehat{p}_{s}), \delta_1), N_P(\gamma, \delta_1)}\}) \rceil$ episodes.
    By Theorem~\ref{thm: exploit UCFD l_2}, for
    $
        \sum_{h =0}^{H-1} \sum_{s \in S_h : \widehat{p}_s \geq \epsilon/ 24|S|h} \sum_{a \in A} T_{s,a}
    $    
    samples we have with probability at least $1-\delta$ that
    \[
        \mathbb{E}_{c \sim \mathcal{D}}[V^{\pi^\star_c}_{\mathcal{M}(c)}(s_0) - V^{\widehat{\pi}^\star_c}_{\mathcal{M}(c)}(s_0)] \leq 
        \epsilon +2\alpha_2 H .
    \]
    
    To simplify the analysis, assume that we first lean the dynamics (for each $\beta$-reachable state and every action) and then use it to approximate the rewards using an i.i.d sample of contexts an rewards for each non-negligible state and action.
    Recall that in algorithm EXPLORE-UCFD we do not separate between the learning phases. 
    By corollary~\ref{corl: Sample complexity to approx dynamics UCFD}, for
    $\gamma = \frac{\epsilon}{48|S|H^2}$ and $\delta_1 = \frac{\delta}{6 |S||A| H} $ we have that
    \[
        N_P(\gamma, \delta_1) = 
        O\Big(
        \frac{ H^4 |S|^2}{\epsilon^2}\Big(\ln \Big(\frac{|S||A|H}{\delta} + |S| \Big) \Big)
        \Big).
    \]
    Hence, to learn the dynamics for each $\beta$-reachable state $a$ and action $a$ for the approximate dynamics $\widehat{P}$, we have to collect
    \begin{align*}
        O\Big( 
             |A||S|\frac{|S|H}{\epsilon} \frac{ H^4 |S|^2}{\epsilon^2}\Big(\ln \Big(\frac{|S||A|H}{\delta} + |S| \Big) \Big)
        \Big)
        =
        O\Big( 
           \frac{ H^5 |S|^4 |A|}{\epsilon^3}\Big(\ln \Big(\frac{|S||A|H}{\delta} + |S| \Big) \Big)
        \Big).        
    \end{align*}
    trajectories.
    (Since for every $\beta$-reachable state $s \in S_h$ and action $a \in A$ we have $\widehat{p}_s - \gamma h \geq \beta - \gamma h \geq \beta - \gamma H \geq \gamma H = O(\epsilon/ |S|H)$).
    
    To approximate the rewards,
    since for every $(s,a) \in S \times A$ we have that each state-action pair has a function class $\mathcal{F}_{s,a}$ with finite fat-shattering dimension.  $Fdim = \max_{(s,a) \in S \times A} fat_{\mathcal{F}^R_{s,a}}(\epsilon^2_\star(\widehat{p}_s)/256)$, by Theorem~\ref{thm: fat dim}, for every $(s,a) \in S \times A$  we have 
    \[
        N_R(\mathcal{F}^R_{s,a}, \epsilon^2_\star(\widehat{p}_s), \delta_1)
        =
        O \Big( \frac{ Fdim \ln^2 \frac{1}{\epsilon^2_\star(\widehat{p}_s)} + \ln \frac{1}{\delta_1}}{\epsilon_\star^4(\widehat{p}_s)}\Big).
    \]
    
    By the accuracy-per-state function, for states $s$ that satisfies $\widehat{p}_s < \frac{\epsilon}{24 |S|}$ we have $\epsilon^2_\star(\widehat{p}_s) =1$, hence for every action $a$, we have that  $ N_R(\mathcal{F}^R_{s,a}, \epsilon^2_\star(\widehat{p}_s), \delta_1) = O(\ln(1/\delta))$. Thus they are negligible.
    
    Overall, the sample complexity for learning the rewards is as follows.
    \begingroup
    \allowdisplaybreaks
    \begin{align*}
        & \sum_{h=0}^{H-1} \sum_{s \in S_h: \widehat{p}_s \geq 24|S|}
        \sum_{a \in A}
        T_{s,a}
        =
        \sum_{h=0}^{H-1} \sum_{s \in S_h: \widehat{p}_s \geq 24|S|}
        \sum_{a \in A}
        O\left( \frac{1}{\widehat{p}_s - \gamma h} \left(\ln \frac{1}{\delta_1}+ N_R(\mathcal{F}^R_{s,a} ,\epsilon_\star^2(\widehat{p}_s), \delta_1) \right)  \right)
        \\
        =&
        \sum_{h=0}^{H-1} \sum_{s \in S_h: \widehat{p}_s \geq 24|S|}
        \sum_{a \in A}
        O\left( \frac{1}{\widehat{p}_s - \gamma h} \left(\ln \frac{1}{\delta_1}
        +
        \frac{ Fdim \ln^2 \frac{1}{\epsilon^2_\star (\widehat{p}_s)} + \ln \frac{1}{\delta_1}}{\epsilon_\star^4(\widehat{p}_s)} \right)  \right)
        \\ 
        =&
        \sum_{h=0}^{H-1} \sum_{s \in S_h: \widehat{p}_s \in [\epsilon/24|S|,1/|S|]}
        \sum_{a \in A}
        O\Big( \frac{1}{\widehat{p}_s - \gamma h} \Big(\ln \frac{1}{\delta_1}
        +
        \frac{ Fdim \ln^2 \frac{1}{\epsilon^2/576 \widehat{p}_s^2 |S|^2 |A|^2} 
        + \ln \frac{1}{\delta_1}}{\epsilon^4/576^2 \widehat{p}_s^4 |S|^4 |A|^4} \Big)  \Big)
        \\ 
        &+
        \sum_{h=0}^{H-1} \sum_{s \in S_h: \widehat{p}_s > 1/|S|}
        \sum_{a \in A}
        O\Big( \frac{1}{\widehat{p}_s - \gamma h} \Big(\ln \frac{1}{\delta_1}
        +
        \frac{ Fdim \ln^2 \frac{1}{\epsilon^2 / 576 H^2 |S|^2 |A|^2} + \ln \frac{1}{\delta_1}}{\epsilon^4 / 576^2 H^4 |S|^4 |A|^4} \Big)  \Big)
        \\
        =&
        \sum_{h=0}^{H-1} \sum_{s \in S_h: \widehat{p}_s \in [\epsilon/24|S|,1/|S|]}
        \sum_{a \in A}
        O\Big( \frac{1}{\widehat{p}_s - \gamma h}\Big(\ln \frac{1}{\delta_1}
        +
       \frac{ \widehat{p}_s^4 |S|^4 |A|^4 (Fdim \ln^2 \frac{\widehat{p}_s^2 |S|^2 |A|^2}{\epsilon^2} 
        + \ln \frac{1}{\delta_1})}{\epsilon^4} \Big)  \Big)
        \\ 
        &+
        \sum_{h=0}^{H-1} \sum_{s \in S_h: \widehat{p}_s > 1/|S|}
        \sum_{a \in A}
        O\Big( \frac{1}{\widehat{p}_s - \gamma h}\Big(\ln  \frac{1}{\delta_1}
        +
        \frac{H^4 |S|^4 |A|^4 (Fdim \ln^2 \frac{H^2 |S|^2 |A|^2}{\epsilon^2 } + \ln \frac{1}{\delta_1})}{\epsilon^4} \Big) \Big)
        \\
        \underbrace{\leq}_{(\star)}&
        \sum_{h=0}^{H-1} \sum_{s \in S_h: \widehat{p}_s \in [\epsilon/24|S|,1/|S|]}
        \sum_{a \in A}
        O\Big( \frac{|S|H}{\epsilon}\ln{\frac{|S||A|H}{\delta}} +
        \frac{\widehat{p_s}}{\widehat{p}_s - \gamma h}\frac{ \widehat{p}_s^3 |S|^4 |A|^4 (Fdim \ln^2{\frac{|S|^2 |A|^2}{\epsilon^2}} 
        + \ln {\frac{|S||A|H}{\delta}})}{\epsilon^4}  \Big)
        \\ 
        &+
        \sum_{h=0}^{H-1} \sum_{s \in S_h: \widehat{p}_s > 1/|S|}
        \sum_{a \in A}
        O\Big( |S| \ln{\frac{|S||A|H}{\delta }}
        +
        \frac{ H^4 |S|^5 |A|^4 (Fdim \ln^2 \frac{H^2 |S|^2 |A|^2}{\epsilon^2 } + \ln \frac{|S||A|H}{\delta})}{\epsilon^4}  \Big)
        \\
        \underbrace{\leq}_{(\star \star)}& 
        \sum_{h=0}^{H-1} \sum_{s \in S_h: \widehat{p}_s \in [\epsilon/24|S|,1/|S|]}
        \sum_{a \in A}
        O\Big( \frac{|S|H}{\epsilon}\ln{\frac{|S||A|H}{\delta}}
        +
        \frac{|S| |A|^4 (Fdim \ln^2{\frac{|S|^2 |A|^2}{\epsilon^2}} 
        + \ln {\frac{|S||A|H}{\delta}})}{\epsilon^4}  \Big)
        \\ 
        &+
        \sum_{h=0}^{H-1} \sum_{s \in S_h: \widehat{p}_s > 1/|S|}
        \sum_{a \in A}
        O\Big( |S| \ln{\frac{|S||A|H}{\delta }}
        +
        \frac{ H^4 |S|^5 |A|^4 (Fdim \ln^2 \frac{H^2 |S|^2 |A|^2}{\epsilon^2 } + \ln \frac{|S||A|H}{\delta})}{\epsilon^4}  \Big)
        \\ 
        =&
        O \Big( \frac{|S|^2 |A|H}{\epsilon} \ln{\frac{|S||A|H}{\delta}} +  
        \frac{H^4 |S|^6 |A|^5  (\;Fdim \ln^2{\frac{H^2 |S|^2 |A|^2}{\epsilon^2}} + \ln {\frac{|S||A|H}{\delta}})}{\epsilon^4}  \Big)
    \end{align*}
     \endgroup
Where $(\star)$ is since for any $h \in [H]$ we  have $\widehat{p}_s - \gamma h \geq \beta - \gamma h \geq \beta - \gamma H \geq \gamma H = O(\epsilon / |S|H)$, and if $\widehat{p}_s > \frac{1}{|S|}$ we have $\widehat{p}_s -\gamma h \geq \frac{1}{|S|} -\gamma h \geq \frac{1}{|S|} - \gamma H \geq \frac{1}{|S|} - \frac{1}{|S|H} = \frac{H -1}{|S|H} = O(1 /|S|) $.

$(\star \star)$ is since $\widehat{p}^3_s \leq 1 / |S|^3$ in the appropriate regime. In addition, $\widehat{p}_s - \gamma h \geq  \gamma H$ which implies that $\gamma H / \widehat{p}_s - \gamma h \leq 1$. Hence
\[
    \frac{\widehat{p}_s}{\widehat{p}_s - \gamma h} = \frac{\widehat{p}_s - \gamma h + \gamma h}{\widehat{p}_s - \gamma h}
    = 
    1 + \frac{\gamma h}{\widehat{p}_s - \gamma h} 
    \leq
    1 + \frac{\gamma H}{\widehat{p}_s - \gamma h}
    \leq 
    2 .
\]  

Hence, the overall sample complexity is 
\[
    O \Big( \frac{|S|^2 |A|H}{\epsilon} \ln{\frac{|S||A|H}{\delta}} +  
    \frac{H^4 |S|^6 |A|^5  (\;Fdim \ln^2{\frac{H^2 |S|^2 |A|^2}{\epsilon^2}} + \ln {\frac{|S||A|H}{\delta}})}{\epsilon^4}  \Big).
\]
\end{proof}

\subsubsection{Sample complexity bounds for the \texorpdfstring{$\ell_1$}{Lg} loss}\label{subsubsec:smaple-complexity-l-1-UCFD}

We present sample complexity bound for function classes with finite Pseudo dimension with $\ell_1$ loss.
\begin{corollary}\label{corl: sample complexity UCFD l_1 pseudo}
    Assume that for every $(s,a) \in S \times A$ we have that $Pdim(\mathcal{F}^R_{s,a}) < \infty$.
    Let $Pdim = \max_{(s,a) \in S \times A} Pdim(\mathcal{F}^R_{s,a})$.
    Then, after collecting 
    \[
        O \Big( \frac{|S|^2 |A|H}{\epsilon} \ln{\frac{|S||A|H}{\delta}} +  
        \frac{ H^5 |S|^5 |A|^3 (\;Pdim \ln{\frac{ H |S| |A|}{\epsilon}} + \ln {\frac{|S||A|H}{\delta}})}{\epsilon^3}  \Big)
        .
    \]
    trajectories, with probability at least $1-\delta$ it holds that
    \[
        \mathbb{E}_{c \sim \mathcal{D}}[V^{\pi^\star_c}_{\mathcal{M}(c)}(s_0) - V^{\widehat{\pi}^\star_c}_{\mathcal{M}(c)}(s_0)] \leq 
        \epsilon +2\alpha_1 H .
    \]
\end{corollary}

\begin{proof}
    Recall that for each state-action pair $(s,a)$ such that $s$ is $\frac{\epsilon}{24 |S|}$-reachable for $\widehat{P}$, we run for\\
    $T_{s,a} =   \lceil \frac{2}{\widehat{p}_{s} - \gamma h}(\ln(\frac{1}{\delta_1})+
    \max{\{ N_R (\mathcal{F}^R_{s,a} ,\epsilon_\star(\widehat{p}_{s}), \delta_1), N_P(\gamma, \delta_1)}\}) \rceil$ episodes.
    By Theorem~\ref{thm: exploit UCFD l_1}, for\\
    $
        \sum_{h =0}^{H-1} \sum_{s \in S_h : \widehat{p}_s \geq \epsilon/ 24|S|h} \sum_{a \in A} T_{s,a}
    $    
    samples we have with probability at least $1-\delta$ that
    \[
        \mathbb{E}_{c \sim \mathcal{D}}[V^{\pi^\star_c}_{\mathcal{M}(c)}(s_0) - V^{\widehat{\pi}^\star_c}_{\mathcal{M}(c)}(s_0)] \leq 
        \epsilon +2\alpha_1 H .
    \]
    
    To simplify the analysis, assume that we first lean the dynamics (for each $\beta$ reachable state and every action) and then use it to approximate the rewards using an i.i.d sample of contexts an rewards for each non-negligible state and action.
    Recall that in algorithm EXPLORE-UCFD we do not separate between the learning phases. 
    By corollary~\ref{corl: Sample complexity to approx dynamics UCFD}, for
    $\gamma = \frac{\epsilon}{48|S|H^2}$ and $\delta_1 = \frac{\delta}{6 |S||A| H} $ we have that
    \[
        N_P(\gamma, \delta_1) = 
        O\Big(
        \frac{ H^4 |S|^2}{\epsilon^2}\Big(\ln \Big(\frac{|S||A|H}{\delta} + |S| \Big) \Big)
        \Big).
    \]
    Hence, to learn the dynamics for each $\beta$-reachable state $a$ and action $a$ for the approximate dynamics $\widehat{P}$, we have to collect
    \begin{align*}
        O\Big( 
             |A||S|\frac{|S|H}{\epsilon} \frac{ H^4 |S|^2}{\epsilon^2}\Big(\ln \Big(\frac{|S||A|H}{\delta} + |S| \Big) \Big)
        \Big)
        =
        O\Big( 
           \frac{ H^5 |S|^4 |A|}{\epsilon^3}\Big(\ln \Big(\frac{|S||A|H}{\delta} + |S| \Big) \Big)
        \Big).        
    \end{align*}
    trajectories. 
    (Since for every $\beta$-reachable state $s \in S_h$ and action $a \in A$ we have $\widehat{p}_s - \gamma h \geq \beta - \gamma h \geq \beta - \gamma H \geq \gamma H = O(\epsilon/|S|H)$).
    
    To approximate the rewards,
    Since for every $(s,a) \in S \times A$ we have that $Pdim(\mathcal{F}^R_{s,a}) < \infty$, and $Pdim = \max_{(s,a) \in S \times A} Pdim(\mathcal{F}^R_{s,a})$, by Theorem~\ref{thm: pseudo dim}, for every $(s,a) \in S \times A$  we have 
    \[
        N_R(\mathcal{F}^R_{s,a}, \epsilon_\star(\widehat{p}_s), \delta_1)
        =
        O \Big( \frac{ Pdim \ln \frac{1}{\epsilon_\star(\widehat{p}_s)} + \ln \frac{1}{\delta_1}}{\epsilon_\star^2(\widehat{p}_s)}\Big)
    \]
    By the accuracy-per-state function, for states $s$ that satisfies $\widehat{p}_s < \frac{\epsilon}{24 |S|}$ we have $\epsilon_\star(\widehat{p}_s) =1$, hence for every action $a$, we have that  $ N_R(\mathcal{F}^R_{s,a}, \epsilon_\star(\widehat{p}_s), \delta_1) = O(\ln(1/\delta))$. Thus they are negligible.
    
    Overall, the sample complexity for learning the rewards is as follows.
    \begingroup
    \allowdisplaybreaks
    \begin{align*}
        &\sum_{h=0}^{H-1} \sum_{s \in S_h: \widehat{p}_s \geq 24|S|}
        \sum_{a \in A}
        T_{s,a}
        =
        \sum_{h=0}^{H-1} \sum_{s \in S_h: \widehat{p}_s \geq 24|S|}
        \sum_{a \in A}
        O\left( \frac{1}{\widehat{p}_s - \gamma h}\left(\ln\frac{1}{\delta_1}+ N_R(\mathcal{F}^R_{s,a} ,\epsilon_\star(\widehat{p}_s), \delta_1) \right)  \right)
        \\
        =&
        \sum_{h=0}^{H-1} \sum_{s \in S_h: \widehat{p}_s \geq 24|S|}
        \sum_{a \in A}
        O\left( \frac{1}{\widehat{p}_s - \gamma h} \left(\ln \frac{1}{\delta_1}
        +
        \frac{ Pdim \ln \frac{1}{\epsilon_\star (\widehat{p}_s)} + \ln \frac{1}{\delta_1}}{\epsilon_\star^2(\widehat{p}_s)} \right)  \right)
        \\ 
        =&
        \sum_{h=0}^{H-1} \sum_{s \in S_h: \widehat{p}_s \in [\epsilon/24|S|,1/|S|]}
        \sum_{a \in A}
        O\Big( \frac{1}{\widehat{p}_s - \gamma h} \Big(\ln \frac{1}{\delta_1}
        +
        \frac{ Pdim \ln \frac{1}{\epsilon/24 \widehat{p}_s |S| |A|} 
        + \ln \frac{1}{\delta_1}}{\epsilon^2/576 \widehat{p}_s^2 |S|^2 |A|^2})  \Big)
        \\ 
        &+
        \sum_{h=0}^{H-1} \sum_{s \in S_h: \widehat{p}_s > 1/|S|}
        \sum_{a \in A}
        O\Big( \frac{1}{\widehat{p}_s - \gamma h} \Big(\ln \frac{1}{\delta_1}
        +
        \frac{ Pdim \ln \frac{1}{\epsilon / 24 H |S| |A|} + \ln \frac{1}{\delta_1}}{\epsilon^2 / 576 H^2 |S|^2 |A|^2} \Big)  \Big)
        \\
        =&
        \sum_{h=0}^{H-1} \sum_{s \in S_h: \widehat{p}_s \in [\epsilon/24|S|,1/|S|]}
        \sum_{a \in A}
        O\Big( \frac{1}{\widehat{p}_s - \gamma h} \Big(\ln \frac{1}{\delta_1}
        +
       \frac{ \widehat{p}_s^2 |S|^2 |A|^2(Pdim \ln \frac{\widehat{p}_s |S| |A|}{\epsilon} 
        + \ln \frac{1}{\delta_1})}{\epsilon^2} \Big)  \Big)
        \\ 
        &+
        \sum_{h=0}^{H-1} \sum_{s \in S_h: \widehat{p}_s > 1/|S|}
        \sum_{a \in A}
        O\Big( \frac{1}{\widehat{p}_s - \gamma h} \Big(\ln \frac{1}{\delta_1} 
        +
        \frac{H^2 |S|^2 |A|^2(Pdim \ln \frac{H |S| |A|}{\epsilon } + \ln \frac{1}{\delta_1})}{\epsilon^2} \Big) \Big)
        \\
        \underbrace{\leq}_{(\star)}&
        \sum_{h=0}^{H-1} \sum_{s \in S_h: \widehat{p}_s \in [\epsilon/24|S|,1/|S|]}
        \sum_{a \in A}
        O\Big( \frac{|S|H}{\epsilon}\ln{\frac{|S||A|H}{\delta}}
        +
        \frac{\widehat{p_s}}{\widehat{p}_s - \gamma h}\frac{ \widehat{p}_s |S|^2 |A|^2 (Pdim \ln{\frac{|S| |A|}{\epsilon}} 
        + \ln {\frac{|S||A|H}{\delta}})}{\epsilon^2}  \Big)
        \\ 
        &+
        \sum_{h=0}^{H-1} \sum_{s \in S_h: \widehat{p}_s > 1/|S|}
        \sum_{a \in A}
        O\Big( |S| \ln{\frac{|S||A|H}{\delta }}
        +
        \frac{ H^2 |S|^3 |A|^2(Pdim \ln \frac{H |S| |A|}{\epsilon} + \ln \frac{|S||A|H}{\delta})}{\epsilon^2}  \Big)
        \\
        \underbrace{\leq}_{(\star \star)}& 
        \sum_{h=0}^{H-1} \sum_{s \in S_h: \widehat{p}_s \in [\epsilon/24|S|,1/|S|]}
        \sum_{a \in A}
        O\Big( \frac{|S|H}{\epsilon}\ln{\frac{|S||A|H}{\delta}}
        +
        \frac{|S| |A|^2 (Pdim \ln{\frac{|S| |A|}{\epsilon}} 
        + \ln {\frac{|S||A|H}{\delta}})}{\epsilon^2}  \Big)
        \\ 
        &+
        \sum_{h=0}^{H-1} \sum_{s \in S_h: \widehat{p}_s > 1/|S|}
        \sum_{a \in A}
        O\Big( |S| \ln{\frac{|S||A|H}{\delta }}
        +
        \frac{H^2 |S|^3 |A|^2 (Pdim \ln \frac{H |S| |A|}{\epsilon} + \ln \frac{|S||A|H}{\delta})}{\epsilon^2}  \Big)
        \\ 
        =&
        O \Big( \frac{|S|^2 |A|H}{\epsilon} \ln{\frac{|S||A|H}{\delta}} +  
        \frac{ H^2 |S|^4 |A|^3 (\;Pdim \ln{\frac{ H |S| |A|}{\epsilon}} + \ln {\frac{|S||A|H}{\delta}})}{\epsilon^2}  \Big)
    \end{align*}
    \endgroup
Where $(\star)$ is since for any $h \in [H]$ we  have $\widehat{p}_s - \gamma h \geq \beta - \gamma h \geq \beta - \gamma H \geq \gamma H = O(\epsilon / |S|H)$, and if $\widehat{p}_s > \frac{1}{|S|}$ we have $\widehat{p}_s -\gamma h \geq \frac{1}{|S|} -\gamma h \geq \frac{1}{|S|} - \gamma H \geq \frac{1}{|S|} - \frac{1}{|S|H} = \frac{H -1}{|S|H} = O(1 /|S|) $.

$(\star \star)$ is since $\widehat{p}_s \leq 1/|S|$ in the appropriate regime. In addition, $\widehat{p}_s - \gamma h \geq  \gamma H$ which implies that $\gamma H / \widehat{p}_s - \gamma h \leq 1$. Hence
\[
    \frac{\widehat{p}_s}{\widehat{p}_s - \gamma h} = \frac{\widehat{p}_s - \gamma h + \gamma h}{\widehat{p}_s - \gamma h}
    = 
    1 + \frac{\gamma h}{\widehat{p}_s - \gamma h} 
    \leq
    1 + \frac{\gamma H}{\widehat{p}_s - \gamma h}
    \leq 
    2 .
\] 


Hence, the overall sample complexity is 
\[
        O \Big( \frac{|S|^2 |A|H}{\epsilon} \ln{\frac{|S||A|H}{\delta}} +  
        \frac{ H^5 |S|^5 |A|^3 (\;Pdim \ln{\frac{ H |S| |A|}{\epsilon}} + \ln {\frac{|S||A|H}{\delta}})}{\epsilon^3}  \Big)
        .
\]
\end{proof}

We also show similar sample complexity for function classes with finite fat-shattering dimension when using $\ell_1$ loss.
\begin{remark}
    The sample complexity for function classes with finite fat-shattering dimension with $\ell_1$ loss, where in $Fdim$ below we also maximizes over $\epsilon_\star(\widehat{p}_s) $ and the maximum is bounded and independent of $\widehat{p}_s$.
\end{remark}
\begin{corollary}\label{corl: sample complexity UCFD l_1 fat}
    Assume that for every $(s,a) \in S \times A$ we have that $\mathcal{F}^R_{s,a}$ has finite fat-shattering dimension.
    Let $Fdim = \max_{(s,a) \in S \times A} fat_{\mathcal{F}^R_{s,a}}(\epsilon_\star(\widehat{p}_s)/256)$.
    Then, after collecting 
    \[
        O \Big( \frac{|S|^2 |A|H}{\epsilon} \ln{\frac{|S||A|H}{\delta}} +  
        \frac{ H^5 |S|^5 |A|^3 (\; Fdim \ln^2{\frac{ H |S| |A|}{\epsilon}} + \ln {\frac{|S||A|H}{\delta}})}{\epsilon^3}  \Big)
        .
    \]
    trajectories, with probability at least $1-\delta$ it holds that
    \[
        \mathbb{E}_{c \sim \mathcal{D}}[V^{\pi^\star_c}_{\mathcal{M}(c)}(s_0) - V^{\widehat{\pi}^\star_c}_{\mathcal{M}(c)}(s_0)] \leq 
        \epsilon +  2\alpha_1 H.
    \]
\end{corollary}

\begin{proof}
    Recall that for each state-action pair $(s,a)$ such that $s$ is $\frac{\epsilon}{24 |S|}$ reachable for $\widehat{P}$, we run for $T_{s,a} =   \lceil \frac{2}{\widehat{p}_{s} - \gamma h}(\ln(\frac{1}{\delta_1})+
    \max{\{ N_R (\mathcal{F}^R_{s,a} ,\epsilon_\star(\widehat{p}_{s}), \delta_1), N_P(\gamma, \delta_1)}\}) \rceil$ episodes.
    By Theorem~\ref{thm: exploit UCFD l_1}, for
    $
        \sum_{h =0}^{H-1} \sum_{s \in S_h : \widehat{p}_s \geq \epsilon/ 24|S|h} \sum_{a \in A} T_{s,a}
    $    
    samples we have with probability at least $1-\delta$ that
    \[
        \mathbb{E}_{c \sim \mathcal{D}}[V^{\pi^\star_c}_{\mathcal{M}(c)}(s_0) - V^{\widehat{\pi}^\star_c}_{\mathcal{M}(c)}(s_0)] \leq 
        \epsilon + 2\alpha_1 H .
    \]

    To simplify the analysis, assume that we first lean the dynamics (for each $\beta$-reachable state and every action) and then use it to approximate the rewards using an i.i.d sample of contexts an rewards for each non-negligible state and action.
    Recall that in algorithm EXPLORE-UCFD we do not separate between the learning phases. 
    By corollary~\ref{corl: Sample complexity to approx dynamics UCFD}, for
    $\gamma = \frac{\epsilon}{48|S|H^2}$ and $\delta_1 = \frac{\delta}{6 |S||A| H} $ we have that
    \[
        N_P(\gamma, \delta_1) = 
        O\Big(
        \frac{ H^4 |S|^2}{\epsilon^2}\Big(\ln \Big(\frac{|S||A|H}{\delta} + |S| \Big) \Big)
        \Big).
    \]
    Hence, to learn the dynamics for each $\beta$-reachable state $a$ and action $a$ for the approximate dynamics $\widehat{P}$, we have to collect
    \begin{align*}
        O\Big( 
             |A||S|\frac{|S|H}{\epsilon} \frac{ H^4 |S|^2}{\epsilon^2}\Big(\ln \Big(\frac{|S||A|H}{\delta} + |S| \Big) \Big)
        \Big)
        =
        O\Big( 
           \frac{ H^5 |S|^4 |A|}{\epsilon^3}\Big(\ln \Big(\frac{|S||A|H}{\delta} + |S| \Big) \Big)
        \Big)       
    \end{align*}
    trajectories.
    (Since for every $\beta$-reachable state $s \in S_h$ and action $a \in A$ we have $\widehat{p}_s - \gamma h \geq \beta - \gamma h \geq \beta - \gamma H \geq \gamma H = O(\epsilon/ |S|H)$).
    
    To approximate the rewards,
    since or every $(s,a) \in S \times A$ we have that $\mathcal{F}^R_{s,a}$ has finite fat-shattering dimension, and $Fdim = \max_{(s,a) \in S \times A} fat_{\mathcal{F}^R_{s,a}}(\epsilon_\star(\widehat{p}_s)/256)$, by Theorem~\ref{thm: fat dim}, for every $(s,a) \in S \times A$  we have 
    \[
        N_R(\mathcal{F}^R_{s,a}, \epsilon_\star(\widehat{p}_s), \delta_1)
        =
        O \Big( \frac{ Fdim \ln^2 \frac{1}{\epsilon_\star(\widehat{p}_s)} + \ln \frac{1}{\delta_1}}{\epsilon_\star^2(\widehat{p}_s)}\Big)
    \]
    
    By the accuracy-per-state function, for states $s$ that satisfies $\widehat{p}_s < \frac{\epsilon}{24 |S|}$ we have $\epsilon_\star(\widehat{p}_s) =1$, hence for every action $a$, we have that  $ N_R(\mathcal{F}^R_{s,a}, \epsilon_\star(\widehat{p}_s), \delta_1) = O(\ln(1/\delta))$. Thus they are negligible.
    
    Overall, the sample complexity for learning the rewards is as follows.
    \begingroup
    \allowdisplaybreaks
    \begin{align*}
        &\sum_{h=0}^{H-1} \sum_{s \in S_h: \widehat{p}_s \geq 24|S|}
        \sum_{a \in A}
        T_{s,a}
        =
        \sum_{h=0}^{H-1} \sum_{s \in S_h: \widehat{p}_s \geq 24|S|}
        \sum_{a \in A}
        O\Big( \frac{1}{\widehat{p}_s - \gamma h}(\ln(\frac{1}{\delta_1})+ N_R(\mathcal{F}^R_{s,a} ,\epsilon_\star(\widehat{p}_s), \delta_1))  \Big)
        \\
        =&
        \sum_{h=0}^{H-1} \sum_{s \in S_h: \widehat{p}_s \geq 24|S|}
        \sum_{a \in A}
        O\Big( \frac{1}{\widehat{p}_s - \gamma h} \Big(\ln \frac{1}{\delta_1}
        +
        \frac{ Fdim \ln^2 \frac{1}{\epsilon_\star (\widehat{p}_s)} + \ln \frac{1}{\delta_1}}{\epsilon_\star^2(\widehat{p}_s)} \Big)  \Big)
        \\ 
        =&
        \sum_{h=0}^{H-1} \sum_{s \in S_h: \widehat{p}_s \in [\epsilon/24|S|,1/|S|]}
        \sum_{a \in A}
        O\Big( \frac{1}{\widehat{p}_s - \gamma h} \Big(\ln \frac{1}{\delta_1}
        +
        \frac{ Fdim \ln^2 \frac{1}{\epsilon/24 \widehat{p}_s |S| |A|} 
        + \ln \frac{1}{\delta_1}}{\epsilon^2/576 \widehat{p}_s^2 |S|^2 |A|^2} \Big) \Big)
        \\ 
        &+
        \sum_{h=0}^{H-1} \sum_{s \in S_h: \widehat{p}_s > 1/|S|}
        \sum_{a \in A}
        O\Big( \frac{1}{\widehat{p}_s - \gamma h} \Big(\ln \frac{1}{\delta_1}
        +
        \frac{ Fdim \ln^2 \frac{1}{\epsilon / 24 H |S| |A|} + \ln \frac{1}{\delta_1}}{\epsilon^2 / 576 H^2 |S|^2 |A|^2} \Big) \Big)
        \\
        =&
        \sum_{h=0}^{H-1} \sum_{s \in S_h: \widehat{p}_s \in [\epsilon/24|S|,1/|S|]}
        \sum_{a \in A}
        O\Big( \frac{1}{\widehat{p}_s - \gamma h} \Big(\ln \frac{1}{\delta_1} 
        +
       \frac{ \widehat{p}_s^2 |S|^2 |A|^2( Fdim \ln^2 \frac{\widehat{p}_s |S| |A|}{\epsilon} 
        + \ln \frac{1}{\delta_1})}{\epsilon^2} \Big)  \Big)
        \\ 
        &+
        \sum_{h=0}^{H-1} \sum_{s \in S_h: \widehat{p}_s > 1/|S|}
        \sum_{a \in A}
        O\Big( \frac{1}{\widehat{p}_s - \gamma h} \Big(\ln \frac{1}{\delta_1} 
        +
        \frac{H^2 |S|^2 |A|^2( Fdim \ln^2 \frac{H |S| |A|}{\epsilon } + \ln \frac{1}{\delta_1})}{\epsilon^2} \Big)  \Big)
        \\
        \underbrace{\leq}_{(\star)}&
        \sum_{h=0}^{H-1} \sum_{s \in S_h: \widehat{p}_s \in [\epsilon/24|S|,1/|S|]}
        \sum_{a \in A}
        O\Big( \frac{|S|H}{\epsilon}\ln{\frac{|S||A|H}{\delta}}
        +
        \frac{\widehat{p_s}}{\widehat{p}_s - \gamma h}\frac{ \widehat{p}_s |S|^2 |A|^2 ( Fdim \ln^2{\frac{|S| |A|}{\epsilon}} 
        + \ln {\frac{|S||A|H}{\delta}})}{\epsilon^2}  \Big)
        \\ 
        &+
        \sum_{h=0}^{H-1} \sum_{s \in S_h: \widehat{p}_s > 1/|S|}
        \sum_{a \in A}
        O\Big( |S| \ln{\frac{|S||A|H}{\delta }}
        +
        \frac{ H^2 |S|^3 |A|^2( Fdim \ln^2 \frac{H |S| |A|}{\epsilon} + \ln \frac{|S||A|H}{\delta})}{\epsilon^2}  \Big)
        \\
        \underbrace{\leq}_{(\star \star)}& 
        \sum_{h=0}^{H-1} \sum_{s \in S_h: \widehat{p}_s \in [\epsilon/24|S|,1/|S|]}
        \sum_{a \in A}
        O\Big( \frac{|S|H}{\epsilon}\ln{\frac{|S||A|H}{\delta}}
        +
        \frac{|S| |A|^2 ( Fdim \ln^2{\frac{|S| |A|}{\epsilon}} 
        + \ln {\frac{|S||A|H}{\delta}})}{\epsilon^2}  \Big)
        \\ 
        &+
        \sum_{h=0}^{H-1} \sum_{s \in S_h: \widehat{p}_s > 1/|S|}
        \sum_{a \in A}
        O\Big( |S| \ln{\frac{|S||A|H}{\delta }}
        +
        \frac{H^2 |S|^3 |A|^2 ( Fdim \ln^2 \frac{H |S| |A|}{\epsilon} + \ln \frac{|S||A|H}{\delta})}{\epsilon^2}  \Big)
        \\ 
        =&
        O \Big( \frac{|S|^2 |A|H}{\epsilon} \ln{\frac{|S||A|H}{\delta}} +  
        \frac{ H^2 |S|^4 |A|^3 (\; Fdim \ln^2{\frac{ H |S| |A|}{\epsilon}} + \ln {\frac{|S||A|H}{\delta}})}{\epsilon^2}  \Big)
    \end{align*}
    \endgroup

Where $(\star)$ is since for any $h \in [H]$ we  have $\widehat{p}_s - \gamma h \geq \beta - \gamma h \geq \beta - \gamma H \geq \gamma H = O(\epsilon / |S|H)$, and if $\widehat{p}_s > \frac{1}{|S|}$ we have $\widehat{p}_s -\gamma h \geq \frac{1}{|S|} -\gamma h \geq \frac{1}{|S|} - \gamma H \geq \frac{1}{|S|} - \frac{1}{|S|H} = \frac{H -1}{|S|H} = O(1 /|S|) $.

$(\star \star)$ is since $\widehat{p}_s \leq 1/|S|$ in the appropriate regime. In addition, $\widehat{p}_s - \gamma h \geq  \gamma H$ which implies that $\gamma H / \widehat{p}_s - \gamma h \leq 1$. Hence
\[
    \frac{\widehat{p}_s}{\widehat{p}_s - \gamma h} = \frac{\widehat{p}_s - \gamma h + \gamma h}{\widehat{p}_s - \gamma h}
    = 
    1 + \frac{\gamma h}{\widehat{p}_s - \gamma h} 
    \leq
    1 + \frac{\gamma H}{\widehat{p}_s - \gamma h}
    \leq 1 + \frac{\gamma H }{\gamma H }= 2 .
\]   
Hence, the overall sample complexity is 
\[
        O \Big( \frac{|S|^2 |A|H}{\epsilon} \ln{\frac{|S||A|H}{\delta}} +  
        \frac{ H^5 |S|^5 |A|^3 (\; Fdim \ln^2{\frac{ H |S| |A|}{\epsilon}} + \ln {\frac{|S||A|H}{\delta}})}{\epsilon^3}  \Big)
        .
\]
\end{proof}

\section{Known and Context Dependent Dynamics}\label{Appendix:KCDD}

In this section we address the challenging model of context dependent dynamics. Meaning, that each context induces a potentially different dynamics. Clearly, this implies that for any policy $\pi$ (which can be either context-dependent or context-independent), the occupancy measure is determined by the context (due to the context-dependent dynamics). 
Hence, a state $s\in S$ that is highly-reachable for some context $c_1 \in \mathcal{C}$ might be poorly-reachable for a different context $c_2 \in \mathcal{C}$. (Something which is impossible in the context-free dynamics setting.)

For the known context-dependent dynamics setting we stay with a similar strategy as in the context-free dynamics, and do the approximation per state-action pair. In order to overcome the reachability issue, we define for each state $s$ a subset of good contexts $ \mathcal{C}^\beta(s)$ whose induced dynamics reaches $s$ with non-negligible probability, i.e., $\beta$. A state $s$ is $(\gamma,\beta)$-good if the probability of $\mathcal{C}^\beta(s)$ is at least $\gamma$.
For each $(\gamma,\beta)$-good state $s$ we build a sample in which the marginal distribution of the context is $\mathcal{D}$ restricted to $\mathcal{C}^\beta(s)$. We do this by using importance sampling.
We can implement the importance sampling since the context-dependent dynamics are known, hence, the probability of reach state $s$ under a good context $c$ can be computed, say it is $q$. We accept a sample that reaches state $s$ with probability $\beta/q\leq 1$. Given such that a sample we can use the \texttt{ERM} oracle and get a good approximation of rewards. Our approximate optimal policy is similar to the case of known context-free dynamics, with the modification that given a context $c$ we use the dynamics $P^c$ in the approximated MDP $\widehat{\mathcal{M}}(c)$.





\subsection{Algorithm}

We start with an overview of our algorithm EXPLORE-KCDD (Algorithm~\ref{alg: EXPLORE-KCDD}) which works in stages.
Each stage learns a layer. When learning layer $h \in [H-1]$ we sample only the $(\gamma,\beta)$-good states of layer $h$.

Since the distribution over the contexts is unknown, we first need to approximate the probability $\mathbb{P}[c \in \mathcal{C}^{\beta}(s_h)]$ for each state $s_h \in S_h$, to approximate the set of $(\gamma,\beta)$-good states of layer $h$. We do it using mean estimation as described in algorithm AGC (i.e., Algorithm~\ref{alg: AGC}).

For every layer $h \in [H-1]$ we first approximate the set $S^{\gamma, \beta}_h$ of $(\gamma,\beta)$-good states. 
Then, for each $s_h \in S^{\gamma, \beta}_h$ and every action $a_h\in A$, we do the following for ``sufficient'' number of episodes:

(1) We observe the episode context $c$, compute $\pi^c_{s_h} = \arg\max_{\pi: S \to A} q_h(s_h | \pi, P^c)$ and set $\pi^c_{s_h}(s_h) \gets a_h$, which guarantees that we perform action $a_h$ in state $s_h$.
(2) If $q_h(s_h | \pi^c_{s_h}, P^c) \geq \beta$, we run $\pi^c_{s_h}$ to generate a trajectory $\tau$. 

(3) If $(s_h, a_h, r_h) \in \tau$ we add $((c, s_h, a_h),r_h)$ to the sample of $(s_h, a_h)$ with probability $\beta / q_h(s_h | \pi^c_{s_h}, P^c)\leq 1$.
After collecting the samples, we approximate the rewards (as function of the context) using the ERM oracle, $f_{s_h, a_h} = \texttt{ERM}(\mathcal{F}^R_{s_h, a_h}, Sample(s_h, a_h),\ell)$.

(4) For states $s$ which are not $(\gamma, \beta)$-good we set $f_{s,a} = 0$, for every action $a$.

Algorithm EXPLOIT-KCDD (Algorithm~\ref{alg: EXPLOIT-KCDD}) get as inputs the MDP parameters (except for the context-dependent rewards function) and the functions approximate the rewards (that computed using algorithm EXPLORE-KCDD ). Given a context $c$ it computes the approximated MDP $\widehat{\mathcal{M}}(c)$ and use it to compute a near optimal context-dependent policy $\hat{\pi}^\star_c$. Then, it run $\hat{\pi}^\star_c$ to generate trajectory.
Recall that $\widehat{\mathcal{M}}(c)=(S,A,P^c,s_0,\widehat{r}^{c},H)$ where we define $\forall s\in S, a\in A: \widehat{r}^{c}(s,a) = f_{s,a}(c)$.

\begingroup
\allowdisplaybreaks
\begin{algorithm}
    \caption{Approximate Good Contexts (AGC)}
    \label{alg: AGC}
    \begin{algorithmic}[1]
    
        \State
        { 
            \textbf{inputs:}
            \begin{itemize}
                \item MDP parameters: 
                $S = \{S_0, S_1, \ldots , S_H\}$,
                $A$,
                $H$.
                \item $P^c$ - The context-dependent dynamics. 
                \item Reachability parameters: $\gamma$ ,$\beta$ 
                    \item Accuracy and confidence parameters $\epsilon_2$, $\delta_2$, where $\epsilon_2 \leq \gamma$.
                \item Current layer $h$ and state $s_h$.
            \end{itemize}
        }
        \State
        {
        calculate 
            $
                m(\epsilon_2, \delta_2)
                =
                \Big\lceil 
                    \frac{\ln{\frac{2}{\delta_2}}}
                    {2 \epsilon_2^2}
                \Big\rceil
            $
        }
        \State{initialize $counter = 0$}
        \For{ $t = 1, 2, ..., m(\epsilon_2, \delta_2)$}
            \State{observe context $c_t$}
            \If{ $c_t \in \mathcal{C}^{\beta}(s_h)$}
                \State{$Counter = Counter + 1$}
            \EndIf{}
        \EndFor{}
        \State{ $\widehat{p}_{\beta}(s_h) = \frac{Counter}{m(\epsilon_2, \delta_2)}$}\\
        \textbf{return } $\mathbb{I}[\widehat{p}_{\beta}(s_h) 
        \geq \gamma -
        \epsilon_2]$ and  $\widehat{p}_{\beta}(s_h)$
    \end{algorithmic}
\end{algorithm}
\endgroup

\begin{remark}
    The check whether $c \in \mathcal{C}^\beta(s)$ can be done in $poly(|S|, |A|, H)$ time by computing the maximal probability to visit $s$ under the dynamics $P^c$, say it is $p^c_s$, and then check whether $p^c_s \geq \beta$.
\end{remark}

\begingroup
\allowdisplaybreaks
\begin{algorithm}
    \caption{Explore Rewards for Known and Context-Dependent Dynamics (EXPLORE-KCDD)}
    \label{alg: EXPLORE-KCDD}
    \begin{algorithmic}[1]
        \State 
        \textbf{inputs:} 
        \begin{itemize}
            \item CMDP parameters: $S = \{S_0, S_1, \ldots, S_H\} $ - a layered states space, $A$, $P^c$ -a context-dependent transition probabilities matrix, $s_0$ - the unique start state, $H$ - the horizon length.
            \item Accuracy and confidence parameters: $\epsilon$, $\delta$.
            \item $\forall s \in S , a \in A : \;\; \mathcal{F}^R_{s,a}$ - the function classes use to approximate the rewards function.
            \item $N_R(\mathcal{F}, \epsilon, \delta)$ - sample complexity function for the ERM oracle.
            \item The extended readability parameters: $\beta$, $\gamma$.
            \item $\ell$ - a loss function (assumed to be $\ell_1$ or $\ell_2$). 
        \end{itemize} 
        \State set $\delta_1 = \frac{\delta}{6|S||A|}, \delta_2 = \frac{\delta}{6|S|}$, $\epsilon_2 = \gamma/2$
        \State set 
        $\epsilon_1 = 
        \begin{cases}
            \frac{\epsilon^2}{64|S||A|H^2},& \text{if } \ell = \ell_1 \\
            \frac{\epsilon^3}{8^3|S||A|H^3},& \text{if } \ell = \ell_2
        \end{cases}$         
        
        \For{$h \in [H-1]$ }
            \For{ $s_h \in S_h$ }
                \State{$I(s_h), \widehat{p}_\beta(s_h) \gets \texttt{AGC}(S, A, H, P^c, \delta_2, \epsilon_2, \gamma, \beta, h, s_h)$}
                \If{  $I(s_h) == 1$ }    
                    \For{$ a_h\in A$}
                        \State{initialize $Sample(s_h, a_h) = \emptyset$}
                        \State{compute the required number of episodes
                        \[
                            T_{s_h,a_h} =
                            \lceil 
                            \frac{2}{\beta \gamma}(\ln(\frac{1}
                            {\delta_1})
                            +
                            N_R(\mathcal{F}^R_{s_h,a_h} ,\epsilon_1, \delta_1)) 
                            \rceil
                        \]
                        }
                        \For{ $t = 1, 2, \ldots, T_{s_h,a_h}$}
                            \State{observe context $c_t$}
                            \State{ $(\pi^{c_t}_{s_h}, p^{c_t}_{s_h}) \gets \texttt{FFP}(S, A, P^c, s_0, H, s_h)$}
                            \State{$\pi^{c_t}_{s_h}(s_h)\gets a_h$}
                            \If{$ p^{c_t}_{s_h} \geq \beta$}
                                \State{run $\pi^{c_t}_{s_h}$ to generate trajectory $\tau_t$}
                                    \If{$(s_h, a_h ,r_h)$ is in $\tau_t$, for a reward $r_h \in [0,1]$}
                                        \State {with probability $\frac{\beta}{p^{c_t}_{s_h}}$ add $((c_t,s_h, a_h),r_h)$ to  $Sample(s_h,a_h)$}
                                    \EndIf{}
                            \EndIf{}
                        \EndFor{}
                \If {$|Sample(s_h,a_h)| \geq N_R(\mathcal{F}^R_{s_h,a_h} ,\epsilon_1, \delta_1) $}
                   \State{$f_{s_h, a_h} = \texttt{ERM}(\mathcal{F}^R_{s_h,a_h}, Sample(s_h,a_h), \ell)$}
                \Else 
                    \State{ set $f_{s_h, a_h} =0$}
                \EndIf{} 
                \EndFor{}   
            \Else  
            \State {set for all $a \in A$: $f_{s_h,a} =0$}
            \EndIf{}
        \EndFor{}
    \EndFor{}
    \State { \textbf{return } $\{f_{s,a} : \forall (s,a) \in S\times A\}$ }
    \end{algorithmic}
\end{algorithm}
\endgroup

 \begingroup
  \allowdisplaybreaks
\begin{algorithm}
    \caption{Exploit for Known and Context-Dependent Dynamics (EXPLOIT-KCDD)}
    \label{alg: EXPLOIT-KCDD}
    \begin{algorithmic}[1]
        \State{
        \textbf{inputs:} 
        \begin{itemize}
            \item MDP parameters: $S = \{S_0, S_1, \ldots, S_H\} $, $A$, $s_0$, $H$.
            \item $P^c$ -A context-dependent transition probabilities matrix.
            \item Accuracy and confidence parameters: $\epsilon$, $\delta$.
            \item $\forall s \in S , a \in A : \;\; \mathcal{F}^R_{s,a}$ - the function classes use to approximate the rewards function.
            \item $N_R(\mathcal{F}, \epsilon, \delta)$ - sample complexity function for the ERM oracle.
            \item Reachability parameters: $\gamma$,$\beta$ and $S^{\gamma, \beta}_h$ for every $h \in [H-1]$
            \item Functions approximate the rewards for each state-action pair: $\{f_{s,a} : \forall (s,a) \in S \times A\}$.
        \end{itemize} }
        \For{$t = 1, 2, \ldots $}
        
            \State{observe context $c_t$}
            
            \State
            {
                define 
                $
                    \widehat{\mathcal{M}}(c_t) 
                    =
                    (S, A, P^{c_t}, \widehat{r}^{c_t}, s_0, H)
                $
                where $\widehat{r}^{c_t}$ defined as:
                \begin{align*}
                    &\forall h \in [H-1],
                    s_h \in S^{\gamma, \beta}_h,
                    a_h \in A:
                    \widehat{r}^{c_t} (s_h,a_h) 
                    = 
                    f_{s_h, a_h}(c_t) \mathbb{I}[c_t \in \mathcal{C}^{\beta}(s_h)]
                    \\
                    &\forall h \in [H-1],
                    s_h \notin S^{\gamma, \beta}_h,
                    a_h \in A:
                    \widehat{r}^{c_t}(s_h,a_h) 
                    = 
                    0                   
                \end{align*}
            }    
            \State{compute the optimal policy for $\widehat{\mathcal{M}}(c_t)$, $(\widehat{\pi}^{c_t}, \cdot) \gets \texttt{Planning}(\widehat{\mathcal{M}}(c_t))$}
            
            \State{run $\widehat{\pi}^{c_t}$ to generate trajectory.}
        \EndFor{}
    \end{algorithmic}
\end{algorithm}
\endgroup

\subsection{Analysis}

\subsubsection{Analysis Outline}
In the following, we present analysis for both the $\ell_1$ (see Sub-subsection~\ref{subsubsec:analysis-l-1-KCDD}) and $\ell_2$ (see Sub-subsection~\ref{subsubsec:analysis-l-2-KCDD}) loss functions.

For both loss functions, our goal is to bound the expected value difference of the true and approximated models, i.e., $\mathcal{M}(c)$ and $\widehat{\mathcal{M}}(c)$, for any context-dependent policy $\pi = (\pi_c)_{c \in \mathcal{C}}$, with high probability. (See Lemmas~\ref{lemma: val-diff KCDD l_1} and~\ref{lemma: value-diff KCDD l_2}).

Using that bound, we derive a bound on the expected value difference between the optimal context-dependent policy $\pi^\star = (\pi^\star_c)_{c \in \mathcal{C}}$ and our approximated optimal policy $\widehat{\pi}^\star = (\hat{\pi}^\star_c)_{c \in \mathcal{C}}$ on the true model, which holds with high probability. (See Theorems~\ref{thm: opt policy KCDD l_1} and~\ref{thm: opt policy KCDD l_2}).

Lastly, we derive sample complexity bound using known uniform convergence sample complexity bounds for the Pseudo dimension (See Theorem~\ref{thm: pseudo dim}) and the fat-shattering dimension (See Theorem~\ref{thm: fat dim}). For the sample complexity analysis, see Sub-subsection~\ref{subsubsec:sample-complexity-l-1-KCDD} for the $\ell_1$ loss, and~\ref{subsubsec:analysis-l-2-KCDD} for the $\ell_2$ loss.

\subsubsection{Good Events}

\paragraph{Event $G_1$.}
Let $G_1$ denote the good event in which for all $h \in [H-1]$ and $s_h \in S_h$ we have
${|\widehat{p}_{\beta}(s_h) -  \mathbb{P}_{c \sim \mathcal{D}}[c \in \mathcal{C}^\beta(s_h)] | \leq \epsilon_2}$, for $\widehat{p}_{\beta}(s_h)$ that is defined in Algorithm AGC (i.e., Algorithm~\ref{alg: AGC}).

For $\epsilon_2 = \gamma/2$, event $G_1$ guarantees that for every layer $h \in [H-1]$ and state $s_h \in S_h$, if $s_h \in S^{\gamma,\beta}_h$, then Algorithm AGC will identity that $s_h$ is $(\gamma,\beta)$-good. Hence, in Algorithm EXPLORE-KCDD we will collect samples for it.

The following lemma shows that event $G_1$ holds with high probability.
\begin{lemma}\label{lemma: G_1 prob case 3}
    For $\delta_2 = {\delta}/{6|S|}$ it holds that $\mathbb{P}[G_1] \geq 1- \frac{\delta}{6}$.
\end{lemma}

\begin{proof}
    For every state $s \in S$, 
    by Hoeffding's inequality, 
    for $m \geq \frac{\ln{\frac{2}{\delta_2}}}{2 \epsilon_2^2}$ examples, we have with probability at least $1-\delta_2$ that $|\widehat{p}_{\beta}(s) -  \mathbb{P}_{c \sim \mathcal{D}}[c \in \mathcal{C}^\beta(s_h)] | \leq \epsilon_2$. Hence, using union bound over the states, we obtain the lemma.
\end{proof}

\paragraph{Event $G_2$.}
Recall that for every $h \in [H-1]$ we define $S^{\gamma, \beta}_h = \{s_h \in S_h : \mathbb{P}[c \in \mathcal{C}^\beta(s_h)] \geq \gamma\}$ where $\mathcal{C}^\beta(s_h) = \{c \in \mathcal{C} : s_h \text{ is } \beta\text{-reachable for } P^c\}$.

Let $G_2$ denote the good event in which for every layer $h \in [H]$ and state $s_h \in S^{\gamma, \beta}_h$ for every action $a_h \in A$ we have that $|Sample(s_h,a_h)| \geq N_R(\mathcal{F}^R_{s_h, a_h}, \epsilon_1,\delta_1)$.

The following lemma shows that event $G_2$ holds with high probability.
\begin{lemma}\label{lemma: G_2 prob case 3}
    We have  
    $\mathbb{P}_{c \sim \mathcal{D}}[G_2|G_1] \geq 1- {\delta}/{6}$.
\end{lemma}

\begin{proof}
    Fix a layer $h \in [H]$, a state $s_h \in S^{\gamma, \beta}_h$ and an action $a_h \in A$.
    
    Let $p_\beta(s_h) := \mathbb{P}_{c \sim \mathcal{D}}[c \in \mathcal{C}^\beta(s_h)]$. Since $s_h \in S^{\gamma, \beta}_h$ it holds that $p_\beta(s_h) \geq \gamma$.
    Since $G_1$ holds we have that
    $
        p_\beta(s_h) -\epsilon_2 \leq \widehat{p}_\beta(s_h)  \leq p_\beta(s_h)+\epsilon_2
    $
    which yielding that $\widehat{p}_\beta(s_h) \geq p_\beta(s_h) -\epsilon_2 \geq \gamma -\epsilon_2$.
    
    Hence, under $G_1$, the agent will identify that $s_h $ is in $S^{\gamma, \beta}_h$ and try to collect at least $ N_R(\mathcal{F}^R_{s_h, a_h}, \epsilon_1,\frac{\delta}{6|S||A|})$ examples of it, for every the action $a_h$.
    
    For a fixed context $c \in \mathcal{C}$, let $\pi^c_{s_h}$ denote the policy with the highest probability to visit $s_h$, which returned by algorithm \texttt{FFP}.
    
    Since $s_h \in S^{\gamma, \beta}_h$ we have that $\mathbb{P}[c \in \mathcal{C}^\beta(s_h)] \geq \gamma$. Recall that we collect only examples of contexts $c \in \mathcal{C}^\beta(s_h)$. 
    
    Hence, the probability to observe a context $c \in \mathcal{C}^\beta(s_h)$ 
    and then visit $s_h$ when playing according to $\pi^c_{s_h}$ is at least 
    $\gamma \cdot q_h(s_h | \pi^c_{s_h}, P^c)  \geq \gamma \beta $,
    since for $c \in \mathcal{C}^\beta (s_h)$ we have that $s_h$ is $\beta$-reachable for $P^c$, which implies that $q_h(s_h | \pi^c_{s_h}, P^c) \geq \beta$. 
    
    Since we use importance sampling, the probability that an observed example $((c, s_h, a_h),r_h)$ (for $c \in \mathcal{C}^\beta(s_h)$) will be added to $Sample(s_h, a_h)$ is $ \frac{\beta}{q_h(s_h | \pi^c_{s_h}, P^c)}$. 
    Overall, the probability of adding a sample of $(c, s_h, a_h)$ to $Sample(s_h, a_h)$ is at least
    $$
        \frac{\beta}{q_h(s_h | \pi^c_{s_h}, P^c)} \cdot q_h(s_h | \pi^c_{s_h}, P^c) \cdot \gamma 
        = 
        \beta \gamma.
    $$
    
    Hence, in expectation, the agent needs to experience at most $\frac{1}{ \beta \gamma}$ episodes to collect one such example of $(s_h, a_h)$ for $s_h \in S^{\gamma,\beta}_h$.
    
    Using Hoeffding's inequality, we obtain that with probability at least $1-\delta_1$, the agent will collect at least $N_R(\mathcal{F}^R_{s_h, a_h}, \epsilon_1,\delta_1)$ examples after experiencing 
    \[
        T_{s_h,a_h} =
        \lceil 
            \frac{2}{\beta \gamma}(\ln(\frac{1}{\delta_1})
            +
            N_R(\mathcal{F}^R_{s_h,a_h} ,\epsilon_1, \delta_1)) 
            \rceil
    \]
    episodes. 
    For $\delta_1 = \frac{\delta}{6|S||A|}$, we obtain the lemma using union bound over $(s_h, a_h) \in S^{\gamma, \beta}_h \times A$ for every $h \in [H-1]$.
\end{proof}

\paragraph{Event $G_3$.}
Let $G_3$ denote the good event in which for every layer $h \in [H-1]$ and state $s_h \in S^{\gamma, \beta}_h$ we have for every action $a_h \in A$ that
    \[
        \mathbb{E}_{c \sim \mathcal{D}}[(f_{s_h, a_h}(c) - r^c(s_h,a_h))^2 |c \in \mathcal{C}^\beta (s_h)]
        \leq \epsilon_1 + \alpha^2_2(\mathcal{F}^R_{s_h, a_h}).
    \]  
for the $\ell_2$ loss,
( or $ \mathbb{E}_{c \sim \mathcal{D}}[|f_{s_h, a_h}(c) - r^c(s_h,a_h)| |c \in \mathcal{C}^\beta (s_h)]
\leq \epsilon_1 + \alpha_1(\mathcal{F}^R_{s_h, a_h})$ for the $\ell_1$ loss). 

The following lemma shows that given events $G_1$ and $G_2$ hold, event $G_3$ holds with high probability.
\begin{lemma}\label{lemma: G_3 prob case 3}
    We have $\mathbb{P}[G_3|G_1, G_2] \geq 1 - {\delta}/{6}$.
\end{lemma}    

\begin{proof}
    Since $G_1$ and $G_2$ hold, we have for every layer $h \in [H-1]$, state $s_h \in S^{\gamma, \beta}_h$ and action $a_h \in A$ that $|Sample(s_h,a_h)| \geq N_R(\mathcal{F}^R_{s_h, a_h}, \epsilon_1,\delta_1)$. Hence, we compute $f_{s_h,a_h}$ using the ERM oracle, and 
    by the ERM guarantees (see~\ref{par: reward function approx gurantees for all s,a}), for every layer $h \in [H-1]$, state $s_h \in S^{\gamma, \beta}_h$ and an action $a_h \in A$ we have with probability at least $1-\delta_1$ that 
    \[
        \mathbb{E}_{c \sim \mathcal{D}}[(f_{s_h, a_h}(c) - r^c(s_h,a_h))^2 |c \in \mathcal{C}^\beta (s_h)]
        \leq \epsilon_1 + \alpha^2(\mathcal{F}^R_{s_h, a_h}).
    \]
    for the $\ell_2$ loss.
    ($ \mathbb{E}_{c \sim \mathcal{D}}[|f_{s_h, a_h}(c) - r^c(s_h,a_h)| |c \in \mathcal{C}^\beta (s_h)] \leq \epsilon_1 + \alpha_1(\mathcal{F}^R_{s_h, a_h})$ for the $\ell_1$ loss.)
    For $\delta_1 = \frac{\delta}{6|S||A|}$ the lemma follows from union bound over each appropriate state-action pair.
\end{proof}

By combining all the above, we obtain that all of the good events hold with high probability.
\begin{lemma}\label{lemma: final probs case 3}
    It holds that $\mathbb{P}[G_1 \cap G_2 \cap G_3] \geq 1- {\delta}/{2}$.
\end{lemma}

\begin{proof}
 By Lemmas~\ref{lemma: G_1 prob case 3},~\ref{lemma: G_2 prob case 3} and~\ref{lemma: G_3 prob case 3} when combined using an union bound.
\end{proof}

\subsubsection{Analysis for the \texorpdfstring{$\ell_2$}{Lg} loss}\label{subsubsec:analysis-l-2-KCDD}

\begin{lemma}\label{lemma: value-diff KCDD l_2}
    Assume the good events $G_1$, $G_2$ and $G_3$ hold.
    Then for every context-dependent policy $\pi=(\pi_c)_{c \in \mathcal{c}}$ it holds that
    \[
        \mathbb{E}_{c \sim \mathcal{D}}
        [|V^{\pi_c}_{\mathcal{M}(c)}(s_0) - V^{\pi_c}_{\widehat{\mathcal{M}}(c)}(s_0)|]
        \leq
        \frac{\epsilon}{2} + \alpha_2 H,
    \]
    where $\alpha^2_2 = \max_{(s_h, a_h) \in \cup_{h \in [H]}S^{\gamma, \beta}_h \times A} \alpha^2_2(\mathcal{F}^R_{s_h, a_h})$,
    for the following parameters choice: $\gamma = \frac{\epsilon}{8|S|H}$, $\beta = \frac{\epsilon}{8|S|}$ and $\epsilon_1 = \frac{\epsilon^3}{8^3|S||A|H^3}$.
\end{lemma}

\begin{proof}
    For all $h \in [H-1]$ and any context $c \in \mathcal{C}$, let us define the following subsets of $S_h$.
       \begin{enumerate}
       \item $B^{h,c}_1 = 
       \{s_h \in  S_h :
            s_h\in S^{\beta,\gamma}_h,
           c  \in \mathcal{C}^{ \beta}(s_h)\}$.
       \item $B^{h,c}_2 = 
        \{s_h \in  S_h :
             s_h\in S^{\beta,\gamma}_h,
          c \notin \mathcal{C}^{ \beta}(s_h)\}$.
       \item $B^{h,c}_3 = 
       \{s_h \in  S_h :
            s_h\not\in S^{\beta,\gamma}_h,
            c \notin \mathcal{C}^{ \beta}(s_h)\}$.
       \item $B^{h,c}_4 = 
       \{s_h \in  S_h :
            s_h\not\in S^{\beta,\gamma}_h,
            c\in \mathcal{C}^{ \beta}(s_h)\}$.    
   \end{enumerate}
   Clearly, $\cup_{i=1}^4 B^{h,c}_i = S_h$ for every $h \in [H-1]$ and $c \in \mathcal{C}$.
   
For $s_h\not\in S^{\beta,\gamma}_h$ we have that $\mathbb{P}[c \in \mathcal{C}^{\beta}(s_h)]<\gamma$, hence,
    \begingroup
    \allowdisplaybreaks
    \begin{align*}
        \mathbb{P}_c[\exists h \in [H-1] : B^{h,c}_4 \neq \emptyset]
        \;\;=\;\;
        \mathbb{P}_c[\exists h\in [H-1]\; \exists s_h \in S_h: 
        s_h\not\in S^{\beta,\gamma}_h \text{ and } c \in \mathcal{C}^{ \beta}(s_h) ]
        \;\;<\;\;
        \gamma|S|.
    \end{align*}
    \endgroup

    Fix  a context $c \in \mathcal{C}$ and
    a context-dependent policy $\pi$ (we will later take the expectation over the context). Consider the following derivation.
    \begingroup
    \allowdisplaybreaks
    \begin{align*}
        |V^{\pi_c}_{\mathcal{M}(c)}(s_0) - V^{\pi_c}_{\widehat{\mathcal{M}}(c)}(s_0)|
        &=
        |
            \sum_{h=0}^{H-1}
            \sum_{s_h \in S_h}
            \sum_{a_h \in A}
            q_h(s_h,a_h|\pi_c, P^c)
            (r^c(s_h,a_h) - \widehat{r}^c(s_h,a_h))
        |
        \\
        &\leq
        \sum_{h=0}^{H-1}
        \sum_{s_h \in S_h}
        \sum_{a_h \in A}
        q_h(s_h,a_h|\pi_c, P^c) 
        |r^c(s_h,a_h) - \widehat{r}^c(s_h,a_h)|
        \\
        &=
        \underbrace
        {
            \sum_{h=0}^{H-1}
            \sum_{s_h \in B^{h,c}_1}
            \sum_{a_h \in A}
            q_h(s_h,a_h|\pi_c, P^c) 
            |r^c(s_h,a_h) - \widehat{r}^c(s_h,a_h)|
        }_{(1)}
        \\
        &+
        \underbrace
        {
            \sum_{h=0}^{H-1}
            \sum_{s_h \in B^{h,c}_2 \cup B^{h,c}_3}
            \sum_{a_h \in A}
           q_h(s_h,a_h|\pi_c, P^c) 
            |r^c(s_h,a_h) - \widehat{r}^c(s_h,a_h)|
        }_{(2)}
        \\
        &+
        \underbrace
        {
            \sum_{h=0}^{H-1}
            \sum_{s_h \in B^{h,c}_4}
            \sum_{a_h \in A}
            q_h(s_h,a_h|\pi_c, P^c) 
            |r^c(s_h,a_h) - \widehat{r}^c(s_h,a_h)|
        }_{(3)}
    \end{align*}
    \endgroup
    
    We bound $(1)$, $(2)$ and $(3)$ separately.
    For $(1)$,
    under the good event $G_3$ we have for every layer $h \in [H-1]$, state $s_h \in S^{\gamma, \beta}_h$ and action $a_h \in A$ that
    \[
        \mathbb{E}_{c \sim \mathcal{D}}\left[(f_{s_h, a_h}(c) - r^c(s_h,a_h))^2-\alpha^2_2(\mathcal{F}^R_{s_h, a_h})\;\Big|c \in \mathcal{C}^\beta (s_h)\right]
        \leq 
        \epsilon_1 .
    \]
    
    Since $\mathbb{E}_{c \sim \mathcal{D}}[(f_{s_h, a_h}(c) - r^c(s_h,a_h))^2 | c \in \mathcal{C}^\beta(s_h)]\geq  \alpha^2_2(\mathcal{F}^R_{s_h, a_h})$, for every layer $h \in [H-1]$, state $s_h \in S^{\gamma, \beta}_h$ and action $a_h \in A$, for a fixed constant $\rho \in [0,1]$ we obtain using Markov's inequality that 
    \begingroup
    \allowdisplaybreaks
    \begin{align*}
        &\mathbb{P}_c
        [|f_{s_h, a_h}(c) - r^c(s_h,a_h)|
        \geq \sqrt{ \alpha^2_2(\mathcal{F}^R_{s_h, a_h}) + \rho}
        \;\Big|c \in \mathcal{C}^\beta (s_h)] =
        \\
        = &
        \mathbb{P}_c
        [(f_{s_h, a_h}(c) - r^c(s_h,a_h))^2 
        -  \alpha^2_2(\mathcal{F}^R_{s_h, a_h})
        \geq \rho\;\Big|c \in \mathcal{C}^\beta (s_h)]
        \leq 
        \frac{\epsilon_1}{\rho},
    \end{align*}
    \endgroup
    which using the following inequality (that holds since $\alpha_2(\mathcal{F}^R_{s_h, a_h}),\rho \in [0,1]$)
    \begingroup\allowdisplaybreaks
    \[
        \sqrt{ \alpha^2_2(\mathcal{F}^R_{s_h, a_h}) + \rho} \leq 
        \alpha_2(\mathcal{F}^R_{s_h, a_h}) + \sqrt{\rho},
    \]
    \endgroup
    yielding that 
    \begingroup
    \allowdisplaybreaks
    \begin{equation}\label{prob: l_2 KCDD}
        \begin{split}
            &\mathbb{P}_c[|f_{s_h, a_h}(c) - r^c(s_h,a_h)|
            \leq 
            \alpha_2(\mathcal{F}^R_{s_h, a_h}) +\sqrt{\rho}
            |c \in \mathcal{C}^\beta (s_h)]
            \\
        &\geq
            \mathbb{P}_c[|f_{s_h, a_h}(c) - r^c(s_h,a_h)|
            \geq \sqrt{ \alpha^2_2(\mathcal{F}^R_{s_h, a_h}) + \rho}|c \in \mathcal{C}^\beta (s_h)]
        \\
        &=
            \mathbb{P}_c[(f_{s_h, a_h}(c) - r^c(s_h,a_h))^2 -  \alpha^2_2(\mathcal{F}^R_{s_h, a_h})
            \geq \rho|c \in \mathcal{C}^\beta (s_h)]
        \\    
        &\geq
        1 - \frac{\epsilon_1}{\rho}.
        \end{split}
    \end{equation}
    \endgroup
    


    Let $G_4$ denote the following good event,
    \begin{align*}
        \forall h \in [H]
        \;
        \forall s_h \in B^{h,c}_1
        \;
        \forall a \in A:
        |f_{s_h, a_h}(c) - r^c(s_h,a_h) | \leq   \alpha_2(\mathcal{F}^R_{s_h, a_h}) +\sqrt{\rho},
    \end{align*}
    and denote by $\overline{G_4}$ the complementary event.
    
    By $B^{h,c}_1$ definition, we have for all $h \in [H-1]$ and $s \in S_h$ that $c \in \mathcal{C}^\beta(s)$. Hence, when combining that with inequality~\ref{prob: l_2 KCDD} we obtain
    \[
        \mathbb{P}_c[G_4] \geq 1 -  \frac{\epsilon_1 }{\rho}|S||A|
        \quad\mbox{ and }\quad
        \mathbb{P}_c[\overline{G_4}] < \frac{\epsilon_1}{\rho}|S||A|.
    \]
    When $G_4$ holds, then
    \begingroup
    \allowdisplaybreaks
    \begin{align*}
        (1)
        &=
        \sum_{h=0}^{H-1}
        \sum_{s_h \in B^{h,c}_1}
        \sum_{a_h \in A}
        q_h(s_h|\pi_c, P^c)
        \pi_c(a_h| s_h)
        |r^c(s_h,a_h) - \widehat{r}^c(s_h,a_h)|
        \\
        &=
        \sum_{h=0}^{H-1}
        \sum_{s_h \in B^{h,c}_1}
        \sum_{a_h \in A}
        q_h(s_h|\pi_c, P^c)
        \pi_c(a_h| s_h)
        \underbrace{|f_{s_h,a_h}(c) - r^c(s_h,a_h)|}_{\leq  \alpha_2(\mathcal{F}^R_{s_h, a_h})+\sqrt{\rho}}
        \\
        &\leq
        \sum_{h=0}^{H-1}
        \sum_{s_h \in B^{h,c}_1}
        \sum_{a_h \in A}
        q_h(s_h|\pi_c, P^c)
        \pi_c(a_h| s_h)
        \underbrace{\alpha_2(\mathcal{F}^R_{s_h, a_h})}_{\leq \alpha_2 }
        +\sqrt{\rho}H
        \leq
        \alpha_2 H + \sqrt{\rho}H.
    \end{align*}
    \endgroup
    Otherwise, when $G_4$ does not hold, then it is bounded by $H$.
 
    Thus, by total expectation low we have
    \begingroup
    \allowdisplaybreaks
    \begin{align*}
        \mathbb{E}_{c \sim \mathcal{D}}[(1)]
        \leq
        \underbrace{\mathbb{E}_{c \sim \mathcal{D}}[(1)|G_4]}_{\leq \alpha_2 H + \sqrt{\rho}H}
        +
        \underbrace{\mathbb{P}[\overline{G_4}]}_{\leq \frac{\epsilon_1}{\rho}|S||A|} \cdot H
        \leq
        \alpha_2 H + \sqrt{\rho}H
        +
        \frac{\epsilon_1}{\rho}|S| |A| H.
    \end{align*}
    \endgroup

    For $(2)$, we have $c\not\in C^\beta(s)$ for every $s\in \cup_{h \in [H-1]}(B^{h,c}_2 \cup B^{h,c}_3)$, which implies
    \begingroup
    \allowdisplaybreaks
    \begin{align*}
        (2)
        &=
        \sum_{h=0}^{H-1}
        \sum_{s_h \in B^{h,c}_2 \cup B^{h,c}_3}
        \sum_{a_h \in A}
        q_h(s_h|\pi_c, P^c)
        \pi_c(a_h| s_h)
        \underbrace{|r^c(s_h,a_h) - \widehat{r}^c(s_h,a_h)|}_{\leq 1}
        \\
        &\leq
        \sum_{h=0}^{H-1}
        \sum_{s_h \in B^{h,c}_2 \cup B^{h,c}_3}
        \sum_{a_h \in A}
        q_h(s_h|\pi_c, P^c)
        \pi_c(a_h| s_h)
        \\
        &=
        \sum_{h=0}^{H-1}
        \sum_{s_h \in B^{h,c}_2 \cup B^{h,c}_3}
        q_h(s_h|\pi_c, P^c)
        \underbrace{\sum_{a_h \in A}
         \pi(a_h|s_h)}_{=1}
         \\
         &=
        \sum_{h=0}^{H-1}
        \sum_{s_h \in B^{h,c}_2 \cup B^{h,c}_3}
        \underbrace{q_h(s_h|\pi_c, P^c)}_{\leq \beta}
        \\
        &\leq
        \beta|S|.
    \end{align*}
    \endgroup
    Thus,
    \[
        \mathbb{E}_{c \sim \mathcal{D}}[(2)] \leq \beta|S|.
    \]
    
  For $(3)$, when there exists $h\in [H-1]$ such that $B^{h,c}_4\neq\emptyset$, we have
   \begingroup
  \allowdisplaybreaks
  \begin{align*}
    (3)
    &=
    \sum_{h=0}^{H-1}
    \sum_{s_h \in B^{h,c}_4}
    \sum_{a_h \in A}
    q_h(s_h|\pi_c, P^c)
    \pi_c(a_h| s_h)
    \underbrace{|r^c(s_h,a_h) - \widehat{r}^c(s_h,a_h)|}_{\leq 1}
    \\
    &\leq
    \sum_{h=0}^{H-1}
    \underbrace
    {
        \sum_{s_h \in B^{h,c}_4}
        \sum_{a_h \in A}
        q_h(s_h|\pi_c, P^c)
        \pi_c(a_h| s_h)
    }_{\leq 1}
    \leq
    H.
  \end{align*}
  \endgroup
  
  Let $G_5$ denote the good event in which $\forall h \in [H -1 ], B^{h,c}_4 = \emptyset$. Denote by $\overline{G_5}$ the complement event of $G_5$.
  We showed that $\mathbb{P}_c[G_5] \geq 1 - \gamma|S|$ and $\mathbb{P}_c[\overline{G_5}] < \gamma|S|$.
  
  Using total expectation we obtain
   \begingroup
  \allowdisplaybreaks
  \begin{align*}
      \mathbb{E}_{c \sim \mathcal{D}}[(3)]
      &=
      \mathbb{P}_c[G_5]\cdot \mathbb{E}_{c \sim \mathcal{D}}[(3)|G_5]
      +
      \mathbb{P}_c[\overline{G_5}] \cdot \mathbb{E}_{c \sim \mathcal{D}}[(3)|\overline{G_5}]
      \\
      &\leq
      1 \cdot 0 
      + 
      \gamma|S|H
      =
      \gamma|S|H.
  \end{align*}
  \endgroup

    Overall,
    by linearity of expectation and the above we obtain
    \begingroup
    \allowdisplaybreaks
    \begin{align*}
        \mathbb{E}_{c \sim \mathcal{D}}
        [|V^{\pi_c}_{\mathcal{M}(c)}(s_0) - V^{\pi_c}_{\widehat{\mathcal{M}}(c)}(s_0)|]
        &\leq
        \mathbb{E}_{c \sim \mathcal{D}}[(1)] + \mathbb{E}_{c \sim \mathcal{D}}[(2)] + \mathbb{E}_{c \sim \mathcal{D}}[(3)]
        \\
        &\leq
        \alpha_2 H 
        +
        \sqrt{\rho}H
        +
        \frac{\epsilon_1}{\rho}|S||A| H
        +
        \beta|S|
        +
        \gamma |S|H.
    \end{align*}
    \endgroup

Now, for $\gamma = \frac{\epsilon}{8|S|H}$,
$\beta = \frac{\epsilon}{8|S|}$, $\rho = (\epsilon_1 |S||A|)^{2/3}$ and $\epsilon_1 = \frac{\epsilon^3}{8^3|S||A|H^3}$ we have
    \begingroup
    \allowdisplaybreaks
    \begin{align*}
        \mathbb{E}_{c \sim \mathcal{D}}
        [|V^{\pi_c}_{\mathcal{M}(c)}(s_0) - V^{\pi_c}_{\widehat{\mathcal{M}}(c)}(s_0)|]
        &\leq
        \alpha_2 H 
        +
        (\epsilon_1 |S||A|)^{1/3} H
        +
        \frac{\epsilon_1}{(\epsilon_1 |S||A|)^{2/3}}|S||A| H
        +
        \frac{2}{8}\epsilon
       \\
       &=
        \alpha_2 H 
        +
        (\frac{\epsilon^3}{8^3|S||A|H^3} |S||A|)^{1/3} H
        +
        \frac{\epsilon_1^{\frac{1}{3}}}{( |S||A|)^{2/3}}|S||A| H
        +
        \frac{2}{8}\epsilon
        \\
       &=
        \alpha_2 H 
        +
        \frac{\epsilon}{8}
        +
        \frac{(\frac{\epsilon^3}{8^3|S||A|H^3})^{\frac{1}{3}}}{( |S||A|)^{2/3}}|S||A| H
        +
        \frac{2}{8}\epsilon
        \\
        &=
        \frac{\epsilon}{2} + \alpha_2 H.
    \end{align*}
    \endgroup
The above inequality completes the proof of the lemma.
\end{proof}

\begin{theorem}\label{thm: opt policy KCDD l_2}
With probability at least $1-\delta$ it holds that 
    \[
        \mathbb{E}_{c \sim \mathcal{D}}
        [V^{\pi^\star_c}_{\mathcal{M}(c)}(s_0) - V^{\widehat{\pi}^\star_c}_{\mathcal{M}(c)}(s_0)]
        \leq 
        \epsilon + 2\alpha_2 H,
    \]
where $\pi^\star=(\pi^\star_c)_{c \in \mathcal{C}}$ is the optimal context-dependent policy for $\mathcal{M}$ and $\widehat{\pi}^\star=(\widehat{\pi}^\star_c)_{c \in \mathcal{C}}$ is the optimal context-dependent policy for $\widehat{\mathcal{M}}$.
\end{theorem}

\begin{proof}
Assume the good events $G_1$, $G_2$ and $G_3$ hold. Then, by Lemma~\ref{lemma: value-diff KCDD l_2} we have for $\pi^\star$ that
    \begin{align*}
        \left|  \mathbb{E}_{c \sim \mathcal{D}}
        [
            V^{\pi^\star_c}_{\mathcal{M}(c)}(s_0) 
            -  V^{\pi^\star_c}_{\widehat{\mathcal{M}}(c)}(s_0)
        ] \right|
        \leq 
         \mathbb{E}_{c \sim \mathcal{D}}
        [|
            V^{\pi^\star_c}_{\mathcal{M}(c)}(s_0) 
            -  V^{\pi^\star_c}_{\widehat{\mathcal{M}}(c)}(s_0)
        |] 
        \leq \frac{\epsilon}{2} + \alpha_2 H,      
    \end{align*}
    yielding,
    %
    \begin{align*}
        \mathbb{E}_{c \sim \mathcal{D}}
        [V^{\pi^\star_c}_{\mathcal{M}(c)}(s_0)] 
        -
        \mathbb{E}_{c \sim \mathcal{D}}
        [V^{\pi^\star_c}_{\widehat{\mathcal{M}}(c)}(s_0)]
        \leq
        \frac{\epsilon}{2} + \alpha_2 H.
    \end{align*}

    Similarly we have for $\widehat{\pi}^\star$ that, 
    \begin{equation*}
        \mathbb{E}_{c \sim \mathcal{D}}
        [V^{\widehat{\pi}^\star_c}_{\widehat{\mathcal{M}}(c)}(s_0)]
        -
        \mathbb{E}_{c \sim \mathcal{D}}
        [V^{\widehat{\pi}^\star_c}_{\mathcal{M}(c)}(s_0)] 
        \leq
       \frac{\epsilon}{2} + \alpha_2 H.
\end{equation*}

Since for all $c \in \mathcal{C}$, 
$\widehat{\pi}^\star_c$ is an optimal policy for $\widehat{\mathcal{M}}(c)$ we have 
    $
        V^{\widehat{\pi}^\star_c}_{\widehat{\mathcal{M}}(c)}(s_0)
        \geq V^{\pi^\star_c}_{\widehat{\mathcal{M}}(c)}(s_0)
    $ 
which implies that
    \begin{equation*}
        \mathbb{E}_{c \sim \mathcal{D}}
        [V^{\pi^\star_c}_{\widehat{\mathcal{M}}(c)}(s_0)]
        -
        \mathbb{E}_{c \sim \mathcal{D}}
        [V^{\widehat{\pi}^\star_c}_{\widehat{\mathcal{M}}(c)}(s_0)]
        \leq
        0.
    \end{equation*}
By Lemma~\ref{lemma: final probs case 3} we have that $\mathbb{P}[G_1 \cap G_2 \cap G_3] \geq 1- {\delta}/{2}$, hence the theorem implied by
summing the above three inequalities .
\end{proof}

\begin{corollary}
    When $\alpha_2 = 0$,
    with probability at least $1-\delta$ it holds that
    \[
        \mathbb{E}_{c \sim \mathcal{D}}
        [V^{\pi^\star_c}_{\mathcal{M}(c)}(s_0) - V^{\widehat{\pi}^\star_c}_{\mathcal{M}(c)}(s_0)]
        \leq 
        \epsilon.
    \]
    where $\pi^\star_c$ is the optimal policy for  $\mathcal{M}$ and $\widehat{\pi}^\star_c$ is the optimal policy for $\widehat{\mathcal{M}}$.
\end{corollary}

\subsubsection{Analysis for the \texorpdfstring{$\ell_1$}{Lg} loss.}\label{subsubsec:analysis-l-1-KCDD}

\begin{lemma}\label{lemma: val-diff KCDD l_1}
    Assume the good events $G_1$, $G_2$ and $G_3$ hold.
    Then we have for every context-dependent policy $\pi$ that
    \[
        \mathbb{E}_{c \sim \mathcal{D}}
        [|V^{\pi_c}_{\mathcal{M}(c)}(s_0) - V^{\pi_c}_{\widehat{\mathcal{M}}(c)}(s_0)|]
        \leq
        \frac{\epsilon}{2} + \alpha_1 H,
    \]
    where $\alpha_1 = \max_{(s_h, a_h) \in \cup_{h \in [H]}S^{\gamma, \beta}_h \times A} \alpha_1(\mathcal{F}^R_{s_h, a_h})$,
    for the parameters choice $\gamma = \frac{\epsilon}{8|S|H}$, $\beta = \frac{\epsilon}{8|S|}$ and $\epsilon_1 = \frac{\epsilon^2}{64|S||A|H^2}$.
\end{lemma}

\begin{proof}
    
    For all $h \in [H-1]$ and any context $c \in \mathcal{C}$, let us define the following subsets of $S_h$.
       \begin{enumerate}
       \item $B^{h,c}_1 = 
       \{s_h \in  S_h :
            s_h\in S^{\beta,\gamma}_h,
           c  \in \mathcal{C}^{ \beta}(s_h)\}$.
       \item $B^{h,c}_2 = 
        \{s_h \in  S_h :
             s_h\in S^{\beta,\gamma}_h,
          c \notin \mathcal{C}^{ \beta}(s_h)\}$.
       \item $B^{h,c}_3 = 
       \{s_h \in  S_h :
            s_h\not\in S^{\beta,\gamma}_h,
            c \notin \mathcal{C}^{ \beta}(s_h)\}$.
       \item $B^{h,c}_4 = 
       \{s_h \in  S_h :
            s_h\not\in S^{\beta,\gamma}_h,
            c\in \mathcal{C}^{ \beta}(s_h)\}$.    
   \end{enumerate}

   Clearly, $\cup_{i=1}^4 B^{h,c}_i = S_h$ for every $h \in [H-1]$ and $c \in \mathcal{C}$.

    For $s_h\not\in S^{\beta,\gamma}_h$ we have that $\mathbb{P}_c[c \in \mathcal{C}^{\beta}(s_h)]<\gamma$, hence,
    \begin{align*}
        \mathbb{P}_c[\exists h \in [H-1] : B^{h,c}_4 \neq \emptyset]
        \;\;=\;\;
        \mathbb{P}_c[\exists h\in [H-1], s_h \in S_h: 
        s_h\not\in S^{\beta,\gamma}_h, \text{ and } c \in \mathcal{C}^{ \beta}(s_h) ]
        \;\;<\;\;
        \gamma|S|.
    \end{align*}

    Fix a context-dependent policy $\pi= (\pi_c)_{c \in \mathcal{C}}$. We have for any given context $c$ (later we will take the expectation over $c$) the following
    \begingroup
    \allowdisplaybreaks
    \begin{align*}
        |V^{\pi_c}_{\mathcal{M}(c)}(s_0) - V^{\pi_c}_{\widehat{\mathcal{M}}(c)}(s_0)|
        &=
        |
            \sum_{h=0}^{H-1}
            \sum_{s_h \in S_h}
            \sum_{a_h \in A}
            q_h(s_h,a_h|\pi_c, P^c)
            (r^c(s_h,a_h) - \widehat{r}^c(s_h,a_h))
        |
        \\
        &\leq
        \sum_{h=0}^{H-1}
        \sum_{s_h \in S_h}
        \sum_{a_h \in A}
        q_h(s_h,a_h|\pi_c, P^c) 
        |r^c(s_h,a_h) - \widehat{r}^c(s_h,a_h)|
        \\
        &=
        \underbrace
        {
            \sum_{h=0}^{H-1}
            \sum_{s_h \in B^{h,c}_1}
            \sum_{a_h \in A}
            q_h(s_h,a_h|\pi_c, P^c) 
            |r^c(s_h,a_h) - \widehat{r}^c(s_h,a_h)|
        }_{(1)}
        \\
        &+
        \underbrace
        {
            \sum_{h=0}^{H-1}
            \sum_{s_h \in B^{h,c}_2 \cup B^{h,c}_3}
            \sum_{a_h \in A}
           q_h(s_h,a_h|\pi_c, P^c) 
            |r^c(s_h,a_h) - \widehat{r}^c(s_h,a_h)|
        }_{(2)}
        \\
        &+
        \underbrace
        {
            \sum_{h=0}^{H-1}
            \sum_{s_h \in B^{h,c}_4}
            \sum_{a_h \in A}
            q^c_h(s_h,a_h|\pi) 
            |r^c(s_h,a_h) - \widehat{r}^c(s_h,a_h)|
        }_{(3)}
    \end{align*}
    \endgroup
    We bound $(1)$, $(2)$ and $(3)$ separately.
    
    For $(1)$,
    under the good event $G_3$ for every layer $h \in [H-1]$, state $s_h \in S^{\gamma, \beta}_h$ and action $a_h \in A$ it holds that
    \[
        \mathbb{E}_{c \sim \mathcal{D}}\Big[|f_{s_h, a_h}(c) - r^c(s_h,a_h)| -\alpha_1(\mathcal{F}^R_{s_h, a_h}) \Big|\; c \in \mathcal{C}^\beta (s_h)\Big]
        \leq 
        \epsilon_1. 
    \]
    Recall that $\mathbb{E}_{c \sim \mathcal{D}}\Big[|f_{s_h, a_h}(c) - r^c(s_h,a_h)|\; \Big|\; c \in \mathcal{C}^\beta (s_h) \Big]\geq  \alpha_1(\mathcal{F}^R_{s_h, a_h})$.
    Hence, for every layer $h \in [H-1]$, state $s_h \in S^{\gamma, \beta}_h$ and an action $a_h \in A$, for a fixed constant $\rho \in [0,1]$ we obtain using Markov's inequality that 
    \begin{align*}
        &\mathbb{P}_c
        \Big[|f_{s_h, a_h}(c) - r^c(s_h,a_h)|
        \geq \alpha_1(\mathcal{F}^R_{s_h, a_h}) + \rho
        \Big|\; c \in \mathcal{C}^\beta (s_h)\Big]
        \\
        = &
        \mathbb{P}_c
        \Big[|f_{s_h, a_h}(c) - r^c(s_h,a_h)|
        - \alpha_1(\mathcal{F}^R_{s_h, a_h}) \geq \rho
        \Big|\; c \in \mathcal{C}^\beta (s_h)\Big]
        \leq 
        \frac{\epsilon_1}{\rho},
    \end{align*}
    which implies that
    \begin{equation}\label{prob: l_1 KCDD}
        \mathbb{P}_c
        \Big[|f_{s_h, a_h}(c) - r^c(s_h,a_h)|
        \leq \alpha_1(\mathcal{F}^R_{s_h, a_h}) + \rho
        \Big|\; c \in \mathcal{C}^\beta (s_h)\Big]
        \geq
        1- \frac{\epsilon_1}{\rho}.
    \end{equation}
    


    Let $G_4$ denote the following good event,
    \begin{align*}
        \forall h \in [H-1]
        \;
        \forall s_h \in B^{h,c}_1
        \;
        \forall a \in A:
        |f_{s_h, a_h}(c) - r^c(s_h,a_h) | \leq   \alpha_1(\mathcal{F}^R_{s_h, a_h}) +\rho,
    \end{align*}
    and denote by $\overline{G_4}$ the complementary event.
    By $B^{h,c}_1$ definition, we have for all $h \in [H-1]$ and $s \in S_h$ that $c \in \mathcal{C}^\beta(s)$. Hence, when combining that with inequality~\ref{prob: l_1 KCDD} we obtain
    \[
        \mathbb{P}_c[G_4] \geq 1 -  \frac{\epsilon_1 }{\rho}|S||A|
        \quad\mbox{ and }\quad
        \mathbb{P}_c[\overline{G_4}] < \frac{\epsilon_1}{\rho}|S||A|.
    \]
    When $G_4$ holds, then
    \begingroup
    \allowdisplaybreaks
    \begin{align*}
        (1)
        &=
        \sum_{h=0}^{H-1}
        \sum_{s_h \in B^{h,c}_1}
        \sum_{a_h \in A}
        q_h(s_h|\pi_c, P^c)
        \pi_c(a_h| s_h)
        |r^c(s_h,a_h) - \widehat{r}^c(s_h,a_h)|
        \\
        &=
        \sum_{h=0}^{H-1}
        \sum_{s_h \in B^{h,c}_1}
        \sum_{a_h \in A}
        q_h(s_h|\pi_c, P^c)
        \pi_c(a_h| s_h)
        \underbrace{|f_{s_h,a_h}(c) - r^c(s_h,a_h)|}_{\leq  \alpha_1(\mathcal{F}^R_{s_h, a_h})+\rho}
        \\
        &\leq
        \sum_{h=0}^{H-1}
        \sum_{s_h \in B^{h,c}_1}
        \sum_{a_h \in A}
        q_h(s_h|\pi_c, P^c)
        \pi_c(a_h| s_h)
        \underbrace{\alpha_1(\mathcal{F}^R_{s_h, a_h})}_{\leq \alpha_1 }
        +\rho H
        \leq
        \alpha_1 H + \rho H.
    \end{align*}
    \endgroup
    Otherwise, when $G_4$ does not hold, then it is bounded by $H$.
 
    Thus,
    \begingroup
    \allowdisplaybreaks
    \begin{align*}
        \mathbb{E}_{c \sim \mathcal{D}}[(1)]
        &\leq
        \underbrace{\mathbb{E}_{c \sim \mathcal{D}}[(1)|G_4]}_{\leq \alpha_1 H + \rho H}
        +
        \underbrace{\mathbb{P}[\overline{G_4}]}_{\leq \frac{\epsilon_1}{\rho}|S||A| }\cdot H
        \\
        &\leq
        \alpha_1 H + \rho H
        +
        \frac{\epsilon_1}{\rho}|S||A| H.
    \end{align*}
    \endgroup

    For $(2)$, we have $c\not\in C^\beta(s)$ for every $s\in \cup_{h \in [H-1]}(B^{h,c}_2 \cup B^{h,c}_3)$, which implies
    \begingroup
    \allowdisplaybreaks
    \begin{align*}
        (2)
        &=
        \sum_{h=0}^{H-1}
        \sum_{s_h \in B^{h,c}_2 \cup B^{h,c}_3}
        \sum_{a_h \in A}
        q_h(s_h|\pi_c, P^c)
        \pi_c(a_h| s_h)
        \underbrace{|r^c(s_h,a_h) - \widehat{r}^c(s_h,a_h)|}_{\leq 1}
        \\
        &\leq
        \sum_{h=0}^{H-1}
        \sum_{s_h \in B^{h,c}_2 \cup B^{h,c}_3}
        \sum_{a_h \in A}
        q_h(s_h|\pi_c, P^c)
        \pi_c(a_h| s_h)
        \\
        &=
        \sum_{h=0}^{H-1}
        \sum_{s_h \in B^{h,c}_2 \cup B^{h,c}_3}
        q_h(s_h|\pi_c, P^c)
        \underbrace{\sum_{a_h \in A}
         \pi(a_h|s_h)}_{=1}
         \\
         &=
        \sum_{h=0}^{H-1}
        \sum_{s_h \in B^{h,c}_2 \cup B^{h,c}_3}
        \underbrace{q_h(s_h|\pi_c, P^c)}_{\leq \beta}
        \leq
        \beta|S|.
    \end{align*}
    \endgroup
    Thus,
    \[
        \mathbb{E}_{c \sim \mathcal{D}}[(2)] \leq \beta|S|.
    \]
    
  For $(3)$, when there exists $h\in [H-1]$ such that $B^{h,c}_4\neq\emptyset$, we have
  \begingroup
  \allowdisplaybreaks
  \begin{align*}
    (3)
    &=
    \sum_{h=0}^{H-1}
    \sum_{s_h \in B^{h,c}_4}
    \sum_{a_h \in A}
    q_h(s_h|\pi_c, P^c)
    \pi_c(a_h| s_h)
    \underbrace{|r^c(s_h,a_h) - \widehat{r}^c(s_h,a_h)|}_{\leq 1}
    \\
    &\leq
    \sum_{h=0}^{H-1}
    \underbrace
    {
        \sum_{s_h \in B^{h,c}_4}
        \sum_{a_h \in A}
        q_h(s_h|\pi_c, P^c)
        \pi_c(a_h| s_h)
    }_{\leq 1}
    \leq H .
  \end{align*}
  \endgroup
  Let $G_5$ denote the good event in which $\forall h \in [H -1 ], B^{h,c}_4 = \emptyset$. Denote by $\overline{G_5}$ the complement event of $G_5$.
  We showed that $\mathbb{P}_c[G_5] \geq 1 - \gamma|S|$ and $\mathbb{P}_c[\overline{G_5}] < \gamma|S|$.
  
  Using total expectation we obtain
  \begingroup
  \allowdisplaybreaks
  \begin{align*}
      \mathbb{E}_{c \sim \mathcal{D}}[(3)]
      &=
      \mathbb{P}_c[G_5]\cdot \mathbb{E}_{c \sim \mathcal{D}}[(3)|G_5]
      +
      \mathbb{P}_c[\overline{G_5}]\cdot \mathbb{E}_{c \sim \mathcal{D}}[(3)|\overline{G_5}]
      \\
      &\leq
      1 \cdot 0 
      + 
      \gamma|S|H
      =
      \gamma|S|H.
  \end{align*}
  \endgroup
    Overall,
    by linearity of expectation and the above we obtain
    \begingroup
    \allowdisplaybreaks
    \begin{align*}
        \mathbb{E}_{c \sim \mathcal{D}}
        [|V^{\pi_c}_{\mathcal{M}(c)}(s_0) - V^{\pi_c}_{\widehat{\mathcal{M}}(c)}(s_0)|]
        &\leq
        \mathbb{E}_{c \sim \mathcal{D}}[(1)] + \mathbb{E}_{c \sim \mathcal{D}}[(2)] + \mathbb{E}_{c \sim \mathcal{D}}[(3)]
        \\
        &\leq
        \alpha_1 H 
        +
        \rho H
        +
        \frac{\epsilon_1}{\rho}|S||A| H
        +
        \beta|S|
        +
        \gamma |S|H.
    \end{align*}
    \endgroup

Finally, for the parameters choice  $\gamma = \frac{\epsilon}{8|S|H}$,
$\beta = \frac{\epsilon}{8|S|}$, $\rho = (\epsilon_1 |S||A|)^{1/2}$ and $\epsilon_1 = \frac{\epsilon^2}{64|S||A|H^2}$ it holds that
 \begingroup
  \allowdisplaybreaks
    \begin{align*}
        \mathbb{E}_{c \sim \mathcal{D}}
        [|V^{\pi_c}_{\mathcal{M}(c)}(s_0) - V^{\pi_c}_{\widehat{\mathcal{M}}(c)}(s_0)|]
        &\leq
        \alpha_1 H 
        +
        2(\epsilon_1 |S||A|)^{1/2} H
        +
        \frac{2}{8}\epsilon
       \\
       &=
        \alpha_1 H 
        +
        2(\frac{\epsilon^2}{64|S||A|H^2} |S||A|)^{1/2} H
        +
        \frac{2}{8}\epsilon
        \\
       &=
        \alpha_1 H 
        +
        2\frac{\epsilon}{4}
        \\
        &=
        \frac{\epsilon}{2} + \alpha_1 H.
    \end{align*}
    \endgroup
The above inequality completes the proof of the lemma.
\end{proof}

\begin{theorem}\label{thm: opt policy KCDD l_1}
With probability at least $1-\delta$ it holds that
    \[
        \mathbb{E}_{c \sim \mathcal{D}}
        [V^{\pi^\star_c}_{\mathcal{M}(c)}(s_0) - V^{\widehat{\pi}^\star_c}_{\mathcal{M}(c)}(s_0)]
        \leq 
        \epsilon + 2\alpha_1 H,
    \]
where $\pi^\star = (\pi^\star_c)_{c \in \mathcal{C}}$ is an optimal context-dependent policy for  $\mathcal{M}$ and $\widehat{\pi}^\star=(\widehat{\pi}^\star_c)_{c \in \mathcal{C}}$ is an optimal context-dependent policy for $\widehat{\mathcal{M}}$.
\end{theorem}

\begin{proof}
Assume the good events $G_1$, $G_2$ and $G_3$ hold. Then, 
by Lemma~\ref{lemma: val-diff KCDD l_1} we have for $\pi^\star$
    \begin{align*}
        \left|\mathbb{E}_{c \sim \mathcal{D}}
        [
            V^{\pi^\star_c}_{\mathcal{M}(c)}(s_0) 
            -  V^{\pi^\star_c}_{\widehat{\mathcal{M}}(c)}(s_0)
        ]\right|
        \leq
         \mathbb{E}_{c \sim \mathcal{D}}
        [|
            V^{\pi^\star_c}_{\mathcal{M}(c)}(s_0) 
            -  V^{\pi^\star_c}_{\widehat{\mathcal{M}}(c)}(s_0)
        |] 
        \leq 
        \frac{\epsilon}{2} + \alpha_1 H ,     
    \end{align*}
    yielding
    \begin{align*}
        \mathbb{E}_{c \sim \mathcal{D}}
        [V^{\pi^\star_c}_{\mathcal{M}(c)}(s_0)] 
        -
        \mathbb{E}_{c \sim \mathcal{D}}
        [V^{\pi^\star_c}_{\widehat{\mathcal{M}}(c)}(s_0)]
        \leq
        \frac{\epsilon}{2} + \alpha_1 H.
    \end{align*}

    Similarly, we have for $\widehat{\pi}^\star$ that 
    \begin{equation*}
        \mathbb{E}_{c \sim \mathcal{D}}
        [V^{\widehat{\pi}^\star_c}_{\widehat{\mathcal{M}}(c)}(s_0)]
        -
        \mathbb{E}_{c \sim \mathcal{D}}
        [V^{\widehat{\pi}^\star_c}_{\mathcal{M}(c)}(s_0)] 
        \leq
       \frac{\epsilon}{2} + \alpha_1 H.
\end{equation*}

Since for all $c \in \mathcal{C}$, $\widehat{\pi}^\star_c$ is the optimal policy for $\widehat{\mathcal{M}}(c)$ we have 
    $
         V^{\widehat{\pi}^\star_c}_{\widehat{\mathcal{M}}(c)}(s_0)
        \geq V^{\pi^\star_c}_{\widehat{\mathcal{M}}(c)}(s_0)
    $
     which implies that 
    \begin{equation*}
        \mathbb{E}_{c \sim \mathcal{D}}
        [V^{\pi^\star_c}_{\widehat{\mathcal{M}}(c)}(s_0)]
        -
        \mathbb{E}_{c \sim \mathcal{D}}
        [V^{\widehat{\pi}^\star_c}_{\widehat{\mathcal{M}}(c)}]
        \leq
        0.
\end{equation*}
By Lemma~\ref{lemma: final probs case 3} we have that $\mathbb{P}[G_1 \cap G_2 \cap G_3] \geq 1- {\delta}/{2}$, hence
the theorem implied by summing the above three inequalities.
\end{proof}

\begin{corollary}
    When $\alpha_1 = 0$,
    with probability at least $1-\delta$ it holds that
    \[
        \mathbb{E}_{c \sim \mathcal{D}}
        [V^{\pi^\star_c}_{\mathcal{M}(c)}(s_0) - V^{\widehat{\pi}^\star_c}_{\mathcal{M}(c)}(s_0)]
        \leq 
        \epsilon,
    \]
    where $\pi^\star_c$ is the optimal policy for  $\mathcal{M}$ and $\widehat{\pi}^\star_c$ is the optimal policy for $\widehat{\mathcal{M}}$.
\end{corollary}

\subsection{Sample complexity bounds.}\label{sec: sampel complexity KCDD}

We show dimension-based sample complexity bounds for both $\ell_1$ and $\ell_2$ loss functions.

Recall Theorems~\ref{thm: pseudo dim} and~\ref{thm: fat dim},

\begin{theorem}[Adaption of Theorem 19.2 in~\cite{Bartlett1999NeuralNetsBook}]
    Let $\mathcal{F}$ be a hypothesis space of real valued functions with a finite pseudo dimension, denoted $Pdim(\mathcal{F}) < \infty$. Then, $\mathcal{F}$ has a uniform convergence with 
    \[
        m(\epsilon, \delta) = O \Big( \frac{1}{\epsilon^2}( Pdim(\mathcal{F}) \ln \frac{1}{\epsilon} + \ln \frac{1}{\delta})\Big).
    \]    
\end{theorem}

\begin{theorem}[Adaption of Theorem 19.1 in~\cite{Bartlett1999NeuralNetsBook}]
    Let $\mathcal{F}$ be a hypothesis space of real valued functions with a finite fat-shattering dimension, denoted $fat_{\mathcal{F}}$. Then, $\mathcal{F}$ has a uniform convergence with 
    \[
        m(\epsilon, \delta) = O \Big( \frac{1}{\epsilon^2}( fat_{\mathcal{F}}(\epsilon/256)  \ln^2 \frac{1}{\epsilon} + \ln \frac{1}{\delta})\Big).
    \]    
\end{theorem}

\begin{remark}
    In the following analysis we omit the sample complexity needed to approximate the faction of good contexts for every $s \in S$ as it is 
    \[
        O \Big( 
        \frac{|S|^3 H^2 \ln{\frac{|S|}{\delta}}}{\epsilon^2}
        \Big)
    \]
    and is negligible additional term in the following analysis.
\end{remark}

\subsubsection{Sample complexity bounds for the \texorpdfstring{$\ell_2$}{Lg} loss.}\label{subsubsec:sample-complexity-l-2-KCDD}

We show sample complexity bounds for function classes with finite Pseudo dimension with $\ell_2$ loss.

\begin{corollary}\label{corl: sample complexity KCDD l_2 pseudo}
    Assume that for every $(s,a) \in S \times A$ we have that $Pdim(\mathcal{F}^R_{s,a}) < \infty$.
    Let $Pdim = \max_{(s,a) \in S \times A} Pdim(\mathcal{F}^R_{s,a})$.
    Then, after collecting 
    \[
        O\Big(
            \frac{|S|^5|A|^3 H^7 }{\epsilon^8}
            \Big( Pdim \ln \frac{|S||A|H^3}{\epsilon^3} + \ln \frac{|S||A|}{\delta} \Big)
        \Big).
    \]
    trajectories, with probability at least $1-\delta$ it holds that
    \[
        \mathbb{E}_{c \sim \mathcal{D}}[V^{\pi^\star_c}_{\mathcal{M}(c)}(s_0) - V^{\widehat{\pi}^\star_c}_{\mathcal{M}(c)}(s_0)] \leq 
        \epsilon + 2\alpha_2 H .
    \]
\end{corollary}

\begin{proof}
    In the worst-case
    for every layer $h \in [H-1]$ and every state-action pair $(s,a)$ 
    we collect 
    \[
        T_{s,a} =
        \lceil 
            \frac{2}{\beta \gamma}(\ln(\frac{1}{\delta_1})
            +
            N_R(\mathcal{F}^R_{s_h,a_h} ,\epsilon_1, \delta_1)) 
        \rceil
    \]
    trajectories.
    By Theorem~\ref{thm: opt policy KCDD l_2},  for $\gamma = \frac{\epsilon}{8|S|H}$,
    $\beta = \frac{\epsilon}{8|S|}$,
    $\epsilon_1 = \frac{\epsilon^3}{8^3|S||A|H^3}$,
    $\delta_1 = \frac{\delta}{6 |S||A|}$ and
    $
        \sum_{h =0}^{H-1} \sum_{s \in
        S} \sum_{a \in A} T_{s,a}
    $    
    examples with probability at least $1-\delta$ it holds that
    \[
        \mathbb{E}_{c \sim \mathcal{D}}[V^{\pi^\star_c}_{\mathcal{M}(c)}(s_0) - V^{\widehat{\pi}^\star_c}_{\mathcal{M}(c)}(s_0)] \leq 
        \epsilon + 2\alpha_2 H. 
    \]
    Since for every $(s,a) \in S \times A$ we have that $Pdim(\mathcal{F}^R_{s,a}) < \infty$, and $Pdim = \max_{(s,a) \in S \times A} Pdim(\mathcal{F}^R_{s,a})$, by Theorem~\ref{thm: pseudo dim}, for every $(s,a) \in S \times A$  we have 
    \begin{align*}
        N_R(\mathcal{F}^R_{s,a}, \epsilon_1, \delta_1)
        =
        O \Big( \frac{ Pdim \ln \frac{1}{\epsilon_1} + \ln \frac{1}{\delta_1}}{\epsilon^2_1}\Big)
        =  
        O \Big(\frac{|S|^2|A|^2 H^6 }{\epsilon^6}\Big( Pdim \ln \frac{|S||A|H^3}{\epsilon^3} + \ln \frac{|S||A|}{\delta} \Big)\Big)
    \end{align*}
    Hence, for each state-action pair $(s,a)$ we have
    \begin{align*}
        T_{s,a} 
        = &
        O\Big(
            \frac{|S|}{\epsilon}\frac{|S|H}{\epsilon}
            \frac{|S|^2|A|^2 H^6 }{\epsilon^6}
            \Big( Pdim \ln \frac{|S||A|H^3}{\epsilon^3} + \ln \frac{|S||A|}{\delta} \Big)
        \Big) 
        \\
        = &
        O\Big(
            \frac{|S|^4|A|^2 H^7 }{\epsilon^8}
            \Big( Pdim \ln \frac{|S||A|H^3}{\epsilon^3} + \ln \frac{|S||A|}{\delta} \Big)
        \Big) .
    \end{align*}
    When summing the above for every state-action pair we obtain that the overall sample complexity is
    \[
        O\Big(
            \frac{|S|^5|A|^3 H^7 }{\epsilon^8}
            \Big( Pdim \ln \frac{|S||A|H^3}{\epsilon^3} + \ln \frac{|S||A|}{\delta} \Big)
        \Big).
    \]
\end{proof}

We also show sample complexity bounds for function classes with finite fat-shattering dimension when using $\ell_2$ loss.
\begin{corollary}\label{corl: sample complexity KCDD l_2 fat}
    Assume that for every $(s,a) \in S \times A$ we have that $\mathcal{F}^R_{s,a}$ has finite fat-shattering dimension.
    Let $Fdim = \max_{(s,a) \in S \times A} fat_{\mathcal{F}^R_{s,a}}(\epsilon_1/ 256)$ for $\epsilon_1 = \frac{\epsilon^3}{8^3|S||A|H^3}$.
    Then, after collecting 
    \[
        O\Big(
            \frac{|S|^5|A|^3 H^7 }{\epsilon^8}
            \Big( Fdim \ln^2 \frac{|S||A|H^3}{\epsilon^3} + \ln \frac{|S||A|}{\delta} \Big)
        \Big).
    \]
    trajectories, with probability at least $1-\delta$ it holds that
    \[
        \mathbb{E}_{c \sim \mathcal{D}}[V^{\pi^\star_c}_{\mathcal{M}(c)}(s_0) - V^{\widehat{\pi}^\star_c}_{\mathcal{M}(c)}(s_0)] \leq 
        \epsilon + 2\alpha_2 H .
    \]
\end{corollary}

\begin{proof}
    In the worst-case, 
    for every layer $h \in [H-1]$ and every state-action pair $(s,a) \in S_h \times A$ 
    we collect 
    \[
        T_{s,a} =
        \lceil 
            \frac{2}{\beta \gamma}(\ln(\frac{1}{\delta_1})
            +
            N_R(\mathcal{F}^R_{s_h,a_h} ,\epsilon_1, \delta_1)) 
        \rceil
    \]
    trajectories.
    By Theorem~\ref{thm: opt policy KCDD l_2},  for $\gamma = \frac{\epsilon}{8|S|H}$,
    $\beta = \frac{\epsilon}{8|S|}$,
    $\epsilon_1 = \frac{\epsilon^3}{8^3|S||A|H^3}$,
    $\delta_1 = \frac{\delta}{6 |S||A|}$ and
    $
        \sum_{h =0}^{H-1} \sum_{s \in
        S} \sum_{a \in A} T_{s,a}
    $    
    examples we have with probability at least $1-\delta$ that
    \[
        \mathbb{E}_{c \sim \mathcal{D}}[V^{\pi^\star_c}_{\mathcal{M}(c)}(s_0) - V^{\widehat{\pi}^\star_c}_{\mathcal{M}(c)}(s_0)] \leq 
       \epsilon + 2\alpha_2 H .
    \]
    Since for every $(s,a) \in S \times A$ we have that $\mathcal{F}^R_{s,a}$ has finite fat-shattering dimension, 
    and $Fdim = \max_{(s,a) \in S \times A} fat_{\mathcal{F}^R_{s,a}}(\epsilon_1/ 256)$ for $\epsilon_1 = \frac{\epsilon^3}{8^3|S||A|H^3}$, by Theorem~\ref{thm: fat dim}, for every $(s,a) \in S \times A$  we have 
    \begin{align*}
        N_R(\mathcal{F}^R_{s,a}, \epsilon_1, \delta_1)
        =
        O \Big( \frac{ Fdim \ln^2 \frac{1}{\epsilon_1} + \ln \frac{1}{\delta_1}}{\epsilon^2_1}\Big)
        =  
        O \Big(\frac{|S|^2|A|^2 H^6 }{\epsilon^6}\Big( Fdim \ln^2 \frac{|S||A|H^3}{\epsilon^3} + \ln \frac{|S||A|}{\delta} \Big)\Big)
    \end{align*}
    Hence, for each state-action pair $(s,a)$ we have
    \begin{align*}
        T_{s,a} 
        = &
        O\Big(
            \frac{|S|}{\epsilon}\frac{|S|H}{\epsilon}
            \frac{|S|^2|A|^2 H^6 }{\epsilon^6}
            \Big( Fdim \ln^2 \frac{|S||A|H^3}{\epsilon^3} + \ln \frac{|S||A|}{\delta} \Big)
        \Big) 
        \\
        = &
        O\Big(
            \frac{|S|^4|A|^2 H^7 }{\epsilon^8}
            \Big( Fdim \ln^2 \frac{|S||A|H^3}{\epsilon^3} + \ln \frac{|S||A|}{\delta} \Big)
        \Big). 
    \end{align*}
    When summing the above for every state-action pair we obtain that the overall sample complexity is
    \[
        O\Big(
            \frac{|S|^5|A|^3 H^7 }{\epsilon^8}
            \Big( Fdim \ln^2 \frac{|S||A|H^3}{\epsilon^3} + \ln \frac{|S||A|}{\delta} \Big)
        \Big).
    \]
\end{proof}

\subsubsection{Sample complexity bounds for the \texorpdfstring{$\ell_1$}{Lg} loss.}\label{subsubsec:sample-complexity-l-1-KCDD}

We present sample complexity bounds for function classes with finite Pseudo dimension with $\ell_1$ loss.

\begin{corollary}\label{corl: sample complexity KCDD l_1 pseudo}
    Assume that for every $(s,a) \in S \times A$ we have that $Pdim(\mathcal{F}^R_{s,a}) < \infty$.
    Let $Pdim = \max_{(s,a) \in S \times A} Pdim(\mathcal{F}^R_{s,a})$.
    Then,  after collecting 
    \[
        O\Big(
            \frac{|S|^5|A|^3 H^5 }{\epsilon^6}
            \Big( Pdim \ln \frac{|S||A|H^2}{\epsilon^2} + \ln \frac{|S||A|}{\delta} \Big)
        \Big).
    \]
    trajectories, with probability at least $1-\delta$ it holds that
    \[
        \mathbb{E}_{c \sim \mathcal{D}}[V^{\pi^\star_c}_{\mathcal{M}(c)}(s_0) - V^{\widehat{\pi}^\star_c}_{\mathcal{M}(c)}(s_0)] \leq 
        \epsilon + 2\alpha_1 H.
    \]
\end{corollary}

\begin{proof}
    In the worst-case,
    for every layer $h \in [H-1]$ and every state-action pair $(s,a)$ 
    we collect 
    \[
        T_{s,a} =
        \lceil 
            \frac{2}{\beta \gamma}(\ln(\frac{1}{\delta_1})
            +
            N_R(\mathcal{F}^R_{s_h,a_h} ,\epsilon_1, \delta_1)) 
        \rceil
    \]
    trajectories.
    By Theorem~\ref{thm: opt policy KCDD l_1},  
    for $\gamma = \frac{\epsilon}{8|S|H}$,
    $\beta = \frac{\epsilon}{8|S|}$,
    $\epsilon_1 = \frac{\epsilon^2}{64|S||A|H^2}$,
    $\delta_1 = \frac{\delta}{6 |S||A|}$ and
    $
        \sum_{h =0}^{H-1} \sum_{s \in
        S} \sum_{a \in A} T_{s,a}
    $    
    examples we have with probability at least $1-\delta$ that
    \[
        \mathbb{E}_{c \sim \mathcal{D}}[V^{\pi^\star_c}_{\mathcal{M}(c)}(s_0) - V^{\widehat{\pi}^\star_c}_{\mathcal{M}(c)}(s_0)] \leq 
       \epsilon + 2\alpha_1 H.
    \]
    Since for every $(s,a) \in S \times A$ we have that $Pdim(\mathcal{F}^R_{s,a}) < \infty$, and $Pdim = \max_{(s,a) \in S \times A} Pdim(\mathcal{F}^R_{s,a})$, by Theorem~\ref{thm: pseudo dim}, for every $(s,a) \in S \times A$  we have 
    \begin{align*}
        N_R(\mathcal{F}^R_{s,a}, \epsilon_1, \delta_1)
        =
        O \Big( \frac{ Pdim \ln \frac{1}{\epsilon_1} + \ln \frac{1}{\delta_1}}{\epsilon^2_1}\Big)
        =  
        O \Big(\frac{|S|^2|A|^2 H^4 }{\epsilon^4}\Big( Pdim \ln \frac{|S||A|H^2}{\epsilon^2} + \ln \frac{|S||A|}{\delta} \Big)\Big)
    \end{align*}
    Hence, for each state-action pair $(s,a)$ we have
    \begin{align*}
        T_{s,a}
        = &
        O\Big(
            \frac{|S|}{\epsilon}\frac{|S|H}{\epsilon}
            \frac{|S|^2|A|^2 H^4 }{\epsilon^4}
            \Big( Pdim \ln \frac{|S||A|H^2}{\epsilon^2} + \ln \frac{|S||A|}{\delta} \Big)
        \Big)
        \\
        = &
        O\Big(
            \frac{|S|^4|A|^2 H^5 }{\epsilon^6}
            \Big( Pdim \ln \frac{|S||A|H^2}{\epsilon^2} + \ln \frac{|S||A|}{\delta} \Big).
        \Big) 
    \end{align*}
    When summing the above for every state-action pair we obtain that the overall sample complexity is
    \[
        O\Big(
            \frac{|S|^5|A|^3 H^5 }{\epsilon^6}
            \Big( Pdim \ln \frac{|S||A|H^2}{\epsilon^2} + \ln \frac{|S||A|}{\delta} \Big)
        \Big).
    \]
\end{proof}

We also show sample complexity bounds for function classes with finite fat-shattering dimension when using $\ell_1$ loss.
\begin{corollary}\label{corl: sample complexity KCDD l_1 fat}
    Assume that for every $(s,a) \in S \times A$ we have that $\mathcal{F}^R_{s,a}$ has finite fat-shattering dimension.
    Let $Fdim = \max_{(s,a) \in S \times A} fat_{\mathcal{F}^R_{s,a}}(\epsilon_1/ 256)$ for $\epsilon_1 = \frac{\epsilon^2}{64|S||A|H^2}$.
    Then, after collecting 
    \[
        O\Big(
            \frac{|S|^5|A|^3 H^5 }{\epsilon^6}
            \Big( Fdim \ln^2 \frac{|S||A|H^2}{\epsilon^2} + \ln \frac{|S||A|}{\delta} \Big)
        \Big).
    \]
    trajectories, with probability at least $1-\delta$ it holds that
    \[
        \mathbb{E}_{c \sim \mathcal{D}}[V^{\pi^\star_c}_{\mathcal{M}(c)}(s_0) - V^{\widehat{\pi}^\star_c}_{\mathcal{M}(c)}(s_0)] \leq 
         \epsilon + 2\alpha_1 H.
    \]
\end{corollary}

\begin{proof}
    In the worst-case,
    for every layer $h \in [H-1]$ and every state-action pair $(s,a)$ 
    we collect
    \[
        T_{s,a} =
        \lceil 
            \frac{2}{\beta \gamma}(\ln(\frac{1}{\delta_1})
            +
            N_R(\mathcal{F}^R_{s_h,a_h} ,\epsilon_1, \delta_1)) 
        \rceil
    \]
    trajectories.
    By Theorem~\ref{thm: opt policy KCDD l_1},  for $\gamma = \frac{\epsilon}{8|S|H}$,
    $\beta = \frac{\epsilon}{8|S|}$,
    $\epsilon_1 = \frac{\epsilon^2}{64|S||A|H^2}$,
    $\delta_1 = \frac{\delta}{6 |S||A|}$ and
    $
        \sum_{h =0}^{H-1} \sum_{s \in
        S} \sum_{a \in A} T_{s,a}
    $    
    samples we have with probability at least $1-\delta$ that
    \[
        \mathbb{E}_{c \sim \mathcal{D}}[V^{\pi^\star_c}_{\mathcal{M}(c)}(s_0) - V^{\widehat{\pi}^\star_c}_{\mathcal{M}(c)}(s_0)] \leq 
       \epsilon + 2\alpha_1 H.
    \]
    Since for every $(s,a) \in S \times A$ we have that $\mathcal{F}^R_{s,a}$ has finite fat-shattering dimension, 
    and $Fdim = \max_{(s,a) \in S \times A} fat_{\mathcal{F}^R_{s,a}}(\epsilon_1/ 256)$ for $\epsilon_1 = \frac{\epsilon^2}{64|S||A|H^2}$, by Theorem~\ref{thm: fat dim}, for every $(s,a) \in S \times A$  we have 
    \begin{align*}
        N_R(\mathcal{F}^R_{s,a}, \epsilon_1, \delta_1)
        =
        O \Big( \frac{ Fdim \ln^2 \frac{1}{\epsilon_1} + \ln \frac{1}{\delta_1}}{\epsilon^2_1}\Big)
        =  
        O \Big(\frac{|S|^2|A|^2 H^4 }{\epsilon^4}\Big( Fdim \ln^2 \frac{|S||A|H^2}{\epsilon^2} + \ln \frac{|S||A|}{\delta} \Big)\Big)
    \end{align*}
    Hence, for each  state-action pair $(s,a)$ we have
    \begin{align*}
        T_{s,a}
        = &
        O\Big(
            \frac{|S|}{\epsilon}\frac{|S|H}{\epsilon}
            \frac{|S|^2|A|^2 H^4 }{\epsilon^4}
            \Big( Fdim \ln^2 \frac{|S||A|H^2}{\epsilon^2} + \ln \frac{|S||A|}{\delta} \Big)
        \Big)
        \\
        = &
        O\Big(
            \frac{|S|^4|A|^2 H^5 }{\epsilon^6}
            \Big( Fdim \ln^2 \frac{|S||A|H^2}{\epsilon^2} + \ln \frac{|S||A|}{\delta} \Big)
        \Big). 
    \end{align*}
    When summing the above for every state-action pair we obtain that the overall sample complexity is
    \[
        O\Big(
            \frac{|S|^5|A|^3 H^5 }{\epsilon^6}
            \Big( Fdim \ln^2 \frac{|S||A|H^2}{\epsilon^2} + \ln \frac{|S||A|}{\delta} \Big)
        \Big).
    \]
\end{proof}

\section{Unknown and Context Dependent Dynamics}\label{Appendix: UCDD}
In this section, we consider the most challenging case of unknown and context-dependent dynamics.  Our approach requires a slight modification of our assumptions.

\subsection{Modification of Assumptions}

While we assume that the dynamics is context dependent, we will also assume that the partition to layers is the same for all contexts.
%
As before, we are assuming the partition is known to the learner.
(Note that the first layer is $S_0 = \{s_0\}$, namely there exist a single start state $s_0$ which is common for all the contexts.)


We also modify our assumption for the function approximation.
We will have a function approximation per layer (and not per state-action pair).
In addition we will assume that the dynamics are realizable by the function class, i.e., for each layer there is function in our class which models the dynamics correctly.
For rewards we will also have a function approximation per layer, but it can be agnostic, i.e., even the best function in the class has a non-zero error.

In more details,

 \subsubsection{Function Approximation Per Layer}\label{par: ERM per layer appexdix}
We slightly modify our assumption for the function approximation class,
which works per layer and not per state-action. 
%
For each layer $h \in [H-1]$ we have a function class for the dynamics $\mathcal{F}^P_{h}=\{f^P_{h}: \mathcal{C} \times S_h \times A \times S_{h+1} \to [0,1]\}$ and for the rewards $\mathcal{F}^R_{h}=\{f^R_{h}: \mathcal{C} \times S_h \times A \to [0,1]\}$. Intuitively, given that we are in state $s$ perform action $a$ and the context is $c$, functions $f^P_{h}\in \mathcal{F}^P_{h}$ and $f^R_{h}\in \mathcal{F}^R_{h}$, approximates the transition probability to state $s'$, i.e., $P^c(s'|s,a)$, and the expected reward, i.e., $r^c(s,a)$, respectively.

\begin{assumption}[layer dynamics realizability]\label{assumption:dynamics-realizability}
    We assume the for every layer $ h \in [H-1]$ there exist a function $f^\star_h \in \mathcal{F}^P_{h}$ for which,
    \[
        \forall (c,s,a,s') \in \mathcal{C} \times S_h \times A \times S_{h+1}.\;\;
        f^\star_h(c,s,a,s') = P^c(s'|s,a).
    \]
    Namely, the true transition probability function of layer $h$ is contained in $\mathcal{F}^P_{h}$.
\end{assumption}
Assumption~\ref{assumption:dynamics-realizability} in particular implies that for every layer $ h \in [H-1]$ it holds that $\alpha_1(\mathcal{F}^P_{h}) = \alpha^2_2(\mathcal{F}^P_{h}) = 0$ (for any distribution over $(c,s,a,s')$).  

The functions $N_P(\mathcal{F}^P_{h}, \epsilon, \delta)$ and $N_R(\mathcal{F}^R_{h},\epsilon, \delta)$ map a function class, required accuracy $\epsilon $ and confidence $\delta $ to the 
required number of samples for the ERM oracle guaranteed performance. For the dynamics the ERM guarantee is that with probability $1-\delta$, that $\mathbb{E}[\ell(f^P_h(x),y)]\leq \epsilon$. For the rewards the guarantee is that $\mathbb{E}[\ell(f^R_h(x),y)]\leq \epsilon+\alpha$, where $\alpha$ is the approximation error, and $\ell$ is the loss function.

\subsubsection{Reachability and the Domain of the Examples}

We redefine reachability with respect to the approximated context-dependent dynamics $\widehat{P}^c$.

For $\beta \in (0,1]$ and layer $h \in [H-1]$ the \emph{$\beta$-good contexts} of state $s_h\in S_h$ 
with respect to $\widehat{P}$ are
    $$
        {\widehat{\mathcal{C}}^{\beta}(s_h)
        =
        \{c \in \mathcal{C} : s_h \text{ is } \beta-\text{reachable for $\widehat{P}^c$}\}}
        .
    $$
The \emph{$(\gamma, \beta)$-good states} of layer $h\in [H-1]$ with respect to $\widehat{P}$ are
    $$
        \widehat{S}^{\gamma, \beta}_h = 
        \{
            s_h \in S_h :
            \mathbb{P}[c \in \widehat{\mathcal{C}}^{\beta}(s_h)] \geq \gamma
        \}
        .
    $$
The \emph{target domain} we would like to collect sufficient number of examples from for each layer $h \in [H-1]$ is defined as 
    $$
        \mathcal{X}^{\gamma, \beta}_h = 
        \{
            (c, s_h, a_h) :
            s_h \in \widehat{S}^{\gamma, \beta}_h , c \in \widehat{\mathcal{C}}^{\beta}(s_h), 
            a_h \in A
        \}
    .
    $$
    
We remark that in the following algorithm, we approximate the set $\widehat{S}^{\gamma, \beta}_h$ for every layer $h\in [H-1] $. We denote the approximation by $\widetilde{S}^{\gamma, \beta}_h$.
In the following analysis we show that with high probability, the set $\widetilde{S}^{\gamma, \beta}_h$ satisfies that 
\[
    \widehat{S}^{\gamma, \beta}_h 
    \subseteq
    \widetilde{S}^{\gamma, \beta}_h
    \subseteq
    \widehat{S}^{\gamma/2, \beta}_h.
\]

Hence, for every layer $h \in [H-1]$, we have the \emph{empirical domain} we collect examples from in practice, and is defined as 
    $
        {\widetilde{\mathcal{X}}^{\gamma, \beta}_h = 
        \{
            (c, s_h, a_h) :
            s_h \in \widetilde{S}^{\gamma, \beta}_h , c \in \widehat{\mathcal{C}}^{\beta}(s_h), 
            a_h \in A
        \}}
    $.
By the above, with high probability it holds that
\[
    \mathcal{X}^{\gamma, \beta}_h 
    \subseteq
    \widetilde{\mathcal{X}}^{\gamma, \beta}_h
    \subseteq
    \mathcal{X}^{\gamma/2, \beta}_h.
\]

We remark that the empirical domain  $\widetilde{\mathcal{X}}^{\gamma, \beta}_h$ is determine before learning layer $h$, based on the approximation of the dynamics up to layer $h-1$ and the approximation of the set $\widehat{S}^{\gamma, \beta}_h$ (which is done before learning layer $h$).

\subsection{Algorithm}

Algorithm EXPLORE-UCDD (Algorithm~\ref{alg: EXPLORE-UCDD}) runs in $H$ phases, one per layer. In phase $h\in [H-1]$ we maintain an approximate dynamics for all previous layers $k\leq h-1$, which we already learned.
In phase $h$ we run multiple iterations.
In each iteration,\\ 
(1) we approximate the set of $(\gamma,\beta)$-good states for layer $h$. We  denote the approximated set  $\widetilde{S}^{\gamma,\beta}_h$.\\
(2) We select at random an approximately $(\gamma,\beta)$-good state $s_h\in \widetilde{S}^{\gamma,\beta}_h$ and an action $a_h\in A$.\\ 
(3) Given a context $c$ and a state $s_h$ we compute a policy $\widehat{\pi}^c_{s_h}$ which maximizes the probability of reaching state $s_h$ under the approximated dynamics $\widehat{P}^c$.\\
(4) We run $\widehat{\pi}^c_{s_h}$. If it reaches $s_h$ we play $a_h$, get a reward $r_h$ and transits to $s_{h+1}$, we add: (a) to the dynamics data set $Sample^P(h)$: $((c,s_h,a_h,s'),\mathbb{I}[s_{h+1}=s'])$ for each $s'\in S_{h+1}$, (b) to the reward data set $Sample^R(h)$: $((c,s_h,a_h),r_h)$.\\
(5) After collecting sufficient number of samples, we use the ERM oracle to  (i) approximate the transition probabilities of layer $h$, i.e., $f^P_h = \texttt{ERM}(\mathcal{F}^P_h, Sample^P(h),\ell)$. (ii) approximate the rewards function of layer $h$, i.e., $f^R_h = \texttt{ERM}(\mathcal{F}^R_h, Sample^R(h),\ell)$.

Algorithm EXPLOIT-UCDD (Algorithm~\ref{alg: EXPLOIT-UCDD}) gets as inputs the MDP parameters and the functions which approximate the rewards and the dynamics (that computed using EXPLORE-UCDD). Given a context $c$ it computes the approximated MDP $\widehat{\mathcal{M}}(c)$ and use it to compute a the optimal policy for it, $\widehat{\pi}^\star_c$. Then, it run $\widehat{\pi}^\star_c$ to generate trajectory.
Recall that $\widehat{\mathcal{M}}(c)=(S\cup \{s_{sink}\},A,\widehat{P}^c,s_0,\widehat{r}^{c},H)$, where $\widehat{P}^c$, $\widehat{r}^c$ defined in algorithm EXPLOIT-UCDD.

\begin{algorithm}
    \caption{Approximate Context-Dependent Dynamics (ACDD)}
    \label{alg: ACDD UCDD}
    \begin{algorithmic}[1]
        \State
        { 
            \textbf{inputs:}
            \begin{itemize}
                \item MDP parameters: 
                $S = \{S_0, S_1, \ldots , S_H\}$, $A$, $H$.
                \item Layer $h \in [H-1]$ and approximation of the dynamics for every layer $l < h: f^P_l$.
                \item Reachability parameter $\beta$.
                \item The approximation of the sets of $(\gamma,\beta)$-good states $\widetilde{S}^{\gamma, \beta}_k$ for all $k \in [h-1]$
            \end{itemize}
        }    
            \State
                {
                for a new state $s_{sink} \notin S$, define the approximated context-dependent dynamics as follows.\\
                \begin{align*}
                    &\forall (s,a) \in S \cup \{s_{sink}\} \times A:
                    \widehat{P}^c (s| s_{sink}, a) 
                    = \mathbb{I}[s = s_{sink}]\\
                    &\forall k  \in [h-1],\;\;
                    (s_k, a_k, s_{k+1})
                    \in \widetilde{S}^{\gamma, \beta}_k \times A \times S_{k+1}:
                    \\
                    &\widehat{P}^c(s_{k+1}| s_k, a_k)
                    = \mathbb{I}[c \in \widehat{C}^{\beta}(s_k)]
                    \cdot 
                    \frac{f^P_k(c, s_k, a_k, s_{k+1})}
                    {\sum_{s'_{k+1} \in S_{k+1} }f^P_k(c, s_k, a_k, s'_{k+1})}
                    \\
                    &\widehat{P}^c(s_{sink}| s_k, a_k)= \mathbb{I}[c \notin \widehat{\mathcal{C}}^{\beta}(s_k)]
                    \\
                    &\forall k  \in [h-1], (s_k, a_k, s_{k+1})\in (S_k \setminus \widetilde{S}^{\gamma, \beta}_k) \times A \times S_{k+1}:
                    \\
                    &\widehat{P}^c(s_{k+1}| s_k, a_k)= 0,
                    \widehat{P}^c(s_{sink}| s_k, a_k)= 1 .                      \end{align*}
                }
        \State{ \textbf{return } $\widehat{P}^c$}
        \Comment{Note that $\widehat{P}^c$ is a function of the context $c$.}
    \end{algorithmic}
\end{algorithm}

\begin{algorithm}
    \caption{Approximate Good States (AGS)}
    \label{alg: ACS}
    \begin{algorithmic}[1]
        \State
        { 
            \textbf{inputs:}
            \begin{itemize}
                \item MDP parameters: 
                $S = \{S_0, S_1, \ldots , S_H\}$, $A$, $H$.
                \item layer $h \in [H-1]$ and approximation of the dynamics for every layer $l < h: f^P_l$.
                \item Reachability parameters $\gamma,\beta$ 
                \item $\epsilon_2, \delta_2$ - accuracy and confidence.
            \end{itemize}
        }    
        \State{ set $\widetilde{S}^{\gamma, \beta}_h = \emptyset$}
        \State{$\widehat{P}^c \gets \texttt{ACDD}(S,A,H,h,\beta, \{f^P_k, \widetilde{S}^{\gamma,\beta}_k| k\in [h-1]\})$}
        \For{$s_h \in S_h$}
            \State{$I, p \gets \texttt{AGC}(S,A,H,\widehat{P}^c, \delta_2, \epsilon_2, \gamma, \beta, h, s_h)$}
            \If{$I == 1$}
                \State{$\widetilde{S}^{\gamma, \beta}_h \gets \widetilde{S}^{\gamma, \beta}_h \cup \{s_h\}$}
            \EndIf{}
        \EndFor{}
        \State{ \textbf{return }$\widetilde{S}^{\gamma, \beta}_h$} 
    \end{algorithmic}
\end{algorithm}

\begingroup
\allowdisplaybreaks
\begin{algorithm}
    \caption{Approximate Good Contexts for UCDD (AGC)}
    \label{alg: AGC-UCDD}
    \begin{algorithmic}[1]
    
        \State
        { 
            \textbf{inputs:}
            \begin{itemize}
                \item MDP parameters: 
                $S = \{S_0, S_1, \ldots , S_H\}$,
                $A$,
                $H$.
                \item $P^c$ - The context-dependent dynamics. 
                \item Reachability parameters: $\gamma$ ,$\beta$ 
                    \item Accuracy and confidence parameters $\epsilon_2$, $\delta_2$
                \item Current layer $h$ and state $s_h$.
            \end{itemize}
        }
        \State
        {
        calculate 
            $
                m(\epsilon_2, \delta_2)
                =
                \Big\lceil 
                    \frac{\ln{\frac{2}{\delta_2}}}
                    {2 \epsilon_2^2}
                \Big\rceil
            $
        }
        \State{initialize $counter = 0$}
        \For{ $t = 1, 2, ..., m(\epsilon_2, \delta_2)$}
            \State{observe context $c_t$}
            \If{ $c_t \in \mathcal{C}^{\beta}(s_h)$}
                \State{$Counter = Counter - 1$}
            \EndIf{}
        \EndFor{}
        \State{ $\widehat{p}_{\beta}(s_h) = \frac{Counter}{m(\epsilon_2, \delta_2)}$}\\
        \textbf{return } $\mathbb{I}[\widehat{p}_{\beta}(s_h) 
        \geq \gamma - 
        \epsilon_2]$ and  $\widehat{p}_{\beta}(s_h)$
    \end{algorithmic}
\end{algorithm}
\endgroup


\begin{algorithm}
    \caption{Explore Unknown and Context-Dependent Dynamics CMDP (EXPLORE-UCDD)}
    \label{alg: EXPLORE-UCDD}
    \begin{algorithmic}[1]
        \State{
        \textbf{inputs:} 
        \begin{itemize}
            \item $S = \{S_0, S_1, \ldots, S_H\} $ - a layered states space, $A$ - a finite actions space, $s_0$ - a unique start state, $H$ - the horizon length.
            \item Accuracy and confidence parameters: $\epsilon$, $\delta$.
            \item $\forall h \in [H-1] : \;\; \mathcal{F}^R_h, \mathcal{F}^P_h$ - the function classes use to approximate the expected reward and dynamics in layer $h$, respectively.
            \item $N_R(\mathcal{F}, \epsilon, \delta), N_P(\mathcal{F}, \epsilon, \delta)$ - sample complexity function for the ERM oracle, for the rewards and dynamics respectively.
            \item The reachability parameters $\gamma \in [0,1]$, $\beta \in [0,1]$.
            \item Loss function $\ell$ (assumed to be one of $\ell_1$ or $\ell_2$).
        \end{itemize} }
        \State{set $\delta_1 = \frac{\delta}{8H}, \delta_2 = \frac{\delta}{8|S|}$.}
        \State{
        set $\epsilon_P = \begin{cases}
            \frac{ \epsilon^3}{10\cdot 2^8 \cdot 20^2  |A| |S|^6 H^5},& \text{if } \ell = \ell_2\\
            \frac{ \epsilon^2}{10\cdot 16 \cdot 20  |A| |S|^4 H^3}
            ,& \text{if } \ell = \ell_1\\
        \end{cases}$,     
        $\epsilon_R = \begin{cases}
            \frac{\epsilon^3}{20^3 |S||A| H^3}
            ,&\text{if } \ell = \ell_2\\
            \frac{\epsilon^2}{20^2 |S||A| H^2}
            ,& \text{if } \ell = \ell_1
        \end{cases}$,      
        $\epsilon_2 = \gamma/4$.}    
        \For{$h \in [H-1]$ }
            \State{initialize $Sample^R(h), Sample^P(h) = \emptyset$.}
            \State{compute the required number of episodes
            \[
            T_h =
            \left\lceil 
            \frac{8 |S|}{\gamma \cdot \beta}
            \left(\ln\frac{1}{\delta_1}
            +
            2\max\{N_P(\mathcal{F}^P_h ,\epsilon_P, \delta_1/2), N_R(\mathcal{F}^R_h ,\epsilon_R, \delta_1/2)\}\right)
            \right\rceil
            \]
            }
            \State{$\widetilde{S}^{\gamma, \beta}_h \gets \texttt{AGS}(S,A,H, h, \gamma,\beta, \{f^P_k, \widetilde{S}^{\gamma, \beta}_k : k \in[ h-1]\}, \epsilon_2, \delta_2)$}
            \State{$\widehat{P}^c \gets \texttt{ACDD}(S,A,H, h, \beta ,\{f^P_k, \widetilde{S}^{\gamma,\beta}_k : k\in [h-1]\})$}
            \For{ $t= 1,2, \ldots T_h$ }
                \State{choose $(s_h, a_h) \in \widetilde{S}^{\gamma, \beta}_h \times A$ uniformly at random}
                \State{observe context $c_t$}
                \State{ $(\widehat{\pi}^{c_t}_{s_h}, \widehat{p}^{c_t}_{s_h}) \gets \texttt{FFP}(S, A, \widehat{P}^{c_t}, s_0, H, s_h)$.}
                \State{set $\widehat{\pi}^{c_t}_{s_h}(s_h)\gets a_h$}
                    \If{$ \widehat{p}^{c_t}_{s_h} \geq \beta$}
                        \State{run $\widehat{\pi}^{c_t}_{s_h}$ and generate trajectory $\tau$}
                        \If{$(s_h, a_h, r_h, s_{h+1})$ is in $\tau$}
                            \State{update samples:
                                \begin{align*}
                                    &Sample^R(h) = Sample^R(h) + ((c_t,s_h,a_h), r_h)\\
                                    &Sample^P(h) = Sample^P(h) + \{((c_t,s_h,a_h, s'_{h+1}), \mathbb{I}[s_{h+1} = s'_{h+1}]) : s'_{h+1} \in S_{h+1}\}
                                \end{align*}
                            }
                        \EndIf{}
                    \EndIf{}
            \EndFor{}
            \If{$|Sample^R(h)| \geq 2 \cdot N_R(\mathcal{F}^R_h ,\epsilon_R, \delta_1/2) $}
            \State{$f^R_h = \texttt{ERM}(\mathcal{F}^R_h, Sample^R(h),\ell)$}
            \Else 
            \State{ \textbf{return } \texttt{FAIL}}
            \EndIf{}
            \If{$|Sample^P(h)| \geq 2 \cdot N_P(\mathcal{F}^P_h,\epsilon_P, \delta_1/2) $}
            \State{$f^P_h = \texttt{ERM}(\mathcal{F}^P_h , Sample^P(h),\ell)$}
            \Else
            \State{ \textbf{return }\texttt{FAIL}}
            \EndIf{}
        \EndFor{}
    \State{ \textbf{return }$\{f^R_h , f^P_h, \widetilde{S}^{\gamma, \beta}_h : \forall h \in [H-1]\}$ }
    \end{algorithmic}
\end{algorithm}

\begin{algorithm}
    \caption{Exploit Unknown and Context-Dependent Dynamics CMDP (EXPLOIT-UCDD) }
    \label{alg: EXPLOIT-UCDD}
    \begin{algorithmic}[1]
        \State{
        \textbf{inputs:} 
        \begin{itemize}
            \item MDP parameters: $S = \{S_0, S_1, \ldots, S_H\}$,$A$,$s_0$,$H$ - the horizon length.
            \item Reachability parameter $\beta$. 
            \item Function approximation for the rewards and dynamics for each layer $h \in [H-1]$ and the approximated set of $(\gamma,\beta)$-good contexts : $\{f^R_h, f^P_h, \widetilde{S}^{\gamma, \beta}_h : h \in [H-1]\}$
        \end{itemize} }

        
        \State{$\widehat{P}^c \gets \texttt{ACDD}(S,A,H,H,\beta, \{f^P_k, \widetilde{S}^{\gamma,\beta}_k : k\in [H-1]\}
        )$}
        \Comment{ $\widehat{P}^c$ is a function of the context $c$.}
        \For{$t = 1, 2, \ldots $}
            \State{observe context $c_t$}
           \State{define the reward approximation:}
                \begin{align*}
                    &\forall h \in [H-1],
                    s_h \in \widetilde{S}^{\gamma, \beta}_h,
                    a_h \in A:
                    \widehat{r}^c (s_h, a_h) 
                    = 
                    \mathbb{I}[c_t \in \widehat{\mathcal{C}}^{\beta}(s_h)]
                    \cdot f^R_h(c_t, s_h, a_h)
                    \\
                    &\forall h \in [H-1],
                    s_h \in S_h \setminus \widetilde{S}^{\gamma, \beta}_h,
                    a_h \in A:
                    \widehat{r}^{c_t}(s_h, a_h) 
                    = 0  
                    \\
                    &\forall a\in A:
                    \widehat{r}^{c_t}(s_{sink}, a) = 0         \end{align*}    
            \State
            {
                $
                    \widehat{\mathcal{M}}(c_t) 
                    =
                    (S \cup \{s_{sink}\}, A, \widehat{P}^{c_t}, \widehat{r}^{c_t}, s_0, H)
                $.
            }
            \State{ $\widehat{\pi}^{c_t} \gets \texttt{Planning}(\widehat{\mathcal{M}}(c_t))$.}
            
            \State{run $\widehat{\pi}^{c_t}$.}
        \EndFor{}
    \end{algorithmic}
\end{algorithm}

\begin{remark}
    In the following algorithms, 
    for a given context $c$ and state $s \in S_h$, the check whether $c \in \widehat{\mathcal{C}}^\beta(s)$ can be done in polynomial time in $|S|,|A|,H$ by computing the highest probability to visit $s$ by any policy on the dynamics $\widehat{P}^c$. Since the CMDP is layered, to compute that, we need $\widehat{P}^c$ to be defined only on $(s,a) \in S_l \times A$ for all $\ell < h$.
    In the following algorithm, $\mathbb{I}[c \in \widehat{\mathcal{C}}^\beta(s)]$ is an indicator function that given a context $c$ return $1$ if and only if $c \in \widehat{\mathcal{C}}^\beta(s)$. By the above, the computation time of that function can be done in $poly(|S|,|A|, H)$ time.
\end{remark}

\subsection{Analysis Outline}

We provide analysis for both $\ell_1$ and $\ell_2$ loss functions.
For both of them, we first bound the expected value difference caused by the dynamics approximation for every context-dependent policy, with high probability. See Sub-subsection~\ref{subsubsec:l-1-dynamics-error}, Lemma~\ref{lemma: expected error from dynamics l_1 UCDD} for the $\ell_1$ loss and Sub-subsection~\ref{subsubsec:l-2-dynamics-error}, Lemma~\ref{lemma: expected error from dynamics l_2 UCDD} for the $\ell_2$ loss.

Then, we bound the expected value difference caused by the rewards approximation for every context-dependent policy, with high probability. See Sub-subsection~\ref{subsubsec:l-1-rewards-error}, Lemma~\ref{lemma: expected error from rewards l_1 UCDD} for the $\ell_1$ loss and Sub-subsection~\ref{subsubsec:l-2-rewards-error}, Lemma~\ref{lemma: expected error from rewards l_2 UCDD} for the $\ell_2$ loss.

The next step is to combine both bounds to obtain a bound the expected value difference between the true model $\mathcal{M}(c)$ and $\widehat{\mathcal{M}}(c)$, for any context-dependent policy $\pi = (\pi_c)_{c \in \mathcal{C}}$, with high probability. 
See Lemma~\ref{lemma: expected gap for general pi l_1} for the $\ell_1$ loss and Lemma~\ref{lemma: expected gap for general pi l_2} for the $\ell_2$.

Using the latter bound, we derive a bound on the expected value difference between the optimal context-dependent policy $\pi^\star = (\pi^\star_c)_{c \in \mathcal{C}}$ and our approximated optimal policy $\widehat{\pi}^\star = (\widehat{\pi}^\star_c)_{c \in \mathcal{C}}$ with respect to the true model $\mathcal{M}(c)$, which holds with high probability. This establish our main result. See Theorem~\ref{thm: UCDD opt policy l_1} for the $\ell_1$ loss and~\ref{thm: UCDD opt policy l_2}) for the $\ell_2$ loss.

Lastly, we derive sample complexity bounds using known uniform convergence sample complexity bounds for the Pseudo dimension see Theorem~\ref{thm: pseudo dim}) and the fat-shattering dimension (see Theorem~\ref{thm: fat dim}). For the sample complexity analysis, see Sub-subsection~\ref{subsubsec:SC-L-1-UCDD} for the $\ell_1$ loss, and~\ref{subsubsec:SC-L-2-UCDD} for the $\ell_2$ loss.

\subsection{Analysis for the \texorpdfstring{$\ell_2$}{Lg} Loss}\label{subsec: analysis-UCDD l_2}

\subsubsection{Good Events}
For the analysis of the algorithm, we define the following good events.

\paragraph{Event $G_1$.}
Intuitively, it states
that the approximation of the probability that $c \in \widehat{\mathcal{C}}^{\beta}(s)$  is accurate for every state $s \in S$.

Formally, 
let $\widehat{p}_\beta(s)$ be the output of Algorithm \texttt{AGC} (see Algorithm~\ref{alg: AGC-UCDD}) for the state $s \in S$, and denote
$ p_\beta(s):= \mathbb{P}[c \in \widehat{\mathcal{C}}^{\beta}(s)]$.
For each layer $h\in[H-1]$ we define the event $G_1^h$ as 
$
    {G^h_1=\{|\widehat{p}_\beta({s_h}) - p_\beta({s_h})| \leq \gamma/4 \;\;\;\forall s_h\in S_h\}}
$ 
and define $G_1=\cap_{h\in[H-1]} G^h_1$.

The good event $G_1$ guarantees that for every layer $h \in [H-1]$ and state $s_h \in S_h$, if $\widehat{p}_\beta({s_h}) \geq \frac{3}{4}\gamma$ then $p_\beta({s_h})\geq \gamma/2$, which implies that $s_h \in \widehat{S}^{\gamma/2,\beta}_h$.
This implies that for every layer $h$ we sample only $(\gamma/2,\beta)$-good states for $\widehat{P}^c$.

More impotently, if $p_\beta({s_h}) \geq \gamma$ then $\widehat{p}_\beta({s_h}) \geq \frac{3}{4}\gamma$. 
Hence, we identify every $(\gamma,\beta)$-good state.

Thus, under the good event $G_1$, for every layer $h \in [H-1]$ the approximated set $\widetilde{S}^{\gamma,\beta}_h$ satisfies
\[
    \widehat{S}^{\gamma,\beta}_h
    \subseteq
    \widetilde{S}^{\gamma,\beta}_h
    \subseteq
    \widehat{S}^{\gamma/2,\beta}_h.
\]

The following lemma shows that for our parameters choice, $G_1$ holds with high probability.    
\begin{lemma}\label{lemma: UCDD l_2 G_1}
    For $\epsilon_2 = \gamma/4$ and $\delta_2 = \frac{\delta}{8 |S|}$, we have that $\mathbb{P}[G_1] \geq 1- {\delta}/{8}$.
\end{lemma}

\begin{proof}
    For each $s \in S$ we have that $\widehat{p}_\beta(s)$ is  calculated over
    $m(\epsilon_2, \delta_2) =
    \Big\lceil 
    \frac{\ln{\frac{2}{\delta_2}}}{2 \epsilon_2^2}
    \Big\rceil$ examples. By Hoeffding's inequality combined with union bound, for $\epsilon_2 = \gamma/4$ and $\delta_2 = \frac{\delta}{8 |S|}$, we obtain that $\mathbb{P}[G_1] \geq 1- {\delta}/{8}$.  
\end{proof}

\paragraph{Sampling distributions.}
 Recall that during the algorithm, for every layer $h \in [H-1]$ we collect examples of $(c, s_h, a_h)$ for which $\widehat{p}_\beta(s_h) \geq \frac{3}{4}\gamma$, which under $G_1$ implies that $p_\beta(s_h) \geq \gamma/2$,
 context $c \in \widehat{C}^\beta (s_h)$ and actions $a_h \in A$.
 
 For every layer $h \in [H-1]$ and reachability parameters $\gamma$ and $\beta$ we define the \emph{target domain} we would like to collect examples from as
 \[
    \mathcal{X}^{\gamma, \beta}_h 
    =
    \{
        (c,s_h, a_h) :
        s_h \in \widehat{S}^{\gamma, \beta}_h, c \in \widehat{\mathcal{C}}^\beta (s_h), a_h \in A
    \},
 \]
 recalling that 
 $${\widehat{C}^\beta (s_h) = \{ c \in \mathcal{C} : s_h \text{ is } \beta-\text{reachable for } \widehat{P}^c\}}$$
 and 
 $${\widehat{S}^{\gamma, \beta}_h = \{s_h \in S_h : \mathbb{P}[c \in \widehat{\mathcal{C}}^\beta(s_h)] \geq \gamma\}}.$$
Meaning, we would like to collect sufficient number of examples of $(\gamma,\beta)$-good states, appropriate good context and action for each layer.

In practice,
we collect examples of states $s \in \widetilde{S}^{\gamma,\beta}_h$
which also contains states $s \in \widehat{S}^{\gamma/2,\beta}_h$. 
Under $G_1$ we have the guarantee that 
$\widehat{S}^{\gamma,\beta}_h \subseteq \widetilde{S}^{\gamma,\beta}_h \subseteq \widehat{S}^{\gamma/2,\beta}_h$.

Hence, we define the \emph{empirical domain}
 \[
    \widetilde{\mathcal{X}}^{\gamma, \beta}_h 
    =
    \{
        (c,s_h, a_h) :
        s_h \in \widetilde{S}^{\gamma, \beta}_h, c \in \widehat{\mathcal{C}}^\beta (s_h), a_h \in A
    \},
 \]
We remark that before learning layer $h$, we compute $\widetilde{S}^{\gamma, \beta}_h$ based on the previous layers approximation for the dynamics which are fixed, hence $\widetilde{\mathcal{X}}^{\gamma, \beta}_h$ is  fixed when learning layer $h$.

We also remark that under $G_1$ it holds, since $\widehat{S}^{\gamma,\beta}_h \subseteq \widetilde{S}^{\gamma,\beta}_h \subseteq \widehat{S}^{\gamma/2,\beta}_h$ it also holds that
\[
    \mathcal{X}^{\gamma, \beta}_h
    \subseteq 
    \widetilde{\mathcal{X}}^{\gamma, \beta}_h 
    \subseteq
    \mathcal{X}^{\gamma/2, \beta}_h.
\]


 We consider the marginal distributions of our observations, that sampled from $\widetilde{\mathcal{X}}^{\gamma, \beta}_h$.
 
 For the rewards denote by $\widetilde{\mathcal{D}}^R_h$ the distribution over the collected examples ${((c,s,a),r) \in \widetilde{\mathcal{X}}^{\gamma,\beta}_h \times [0,1]}$, for each layer $h \in [H-1]$. It holds that
    \begin{align*}
      \widetilde{\mathcal{D}}^R_h ((c,s_h,a_h), r_h) &=
      \mathbb{P}[((c, s_h, a_h), r_h) \in Sample^R(h)| 
        (c, s_h, a_h) \in  \widetilde{\mathcal{X}}^{\gamma, \beta}_h]\\
      \propto&
        \mathbb{P}[c | c \in \widehat{\mathcal{C}}^{\beta}(s_h)] 
        \cdot
        q_h(s_h, a_h| \widehat{\pi}^c_{s_h}, P^c) 
        \cdot 
        \mathbb{P}[R^c(s_h, a_h) = r_h |c,s_h,a_h],
    \end{align*}
   where $\propto$ implies that we normalize to sum to $1$.
   
   Since under $G_1$ we have that $\mathcal{X}^{\gamma,\beta}_h \subseteq \widetilde{\mathcal{X}}^{\gamma,\beta}_h$, $\widetilde{\mathcal{D}}^R_h$ induces a marginal distribution over  $\mathcal{X}^{\gamma,\beta}_h \times [0,1]$, which we denote by $\mathcal{D}^R_h$.
   Clearly, it holds that
    \begin{align*}
      \mathcal{D}^R_h ((c,s_h,a_h), r_h) &=
      \mathbb{P}[((c, s_h, a_h), r_h) \in Sample^R(h)| 
        (c, s_h, a_h) \in  \mathcal{X}^{\gamma, \beta}_h]\\
      \propto&
        \mathbb{P}[c | c \in \widehat{\mathcal{C}}^{\beta}(s_h)] 
        \cdot
        q_h(s_h, a_h| \widehat{\pi}^c_{s_h}, P^c) 
        \cdot 
        \mathbb{P}[R^c(s_h, a_h) = r_h |c,s_h,a_h],
    \end{align*}
   which is the desired marginal distribution over our target domain.

  Similarly, for the next state we have,
    \begin{align*}
    \widetilde{\mathcal{D}}^P_h& ((c,s_h,a_h,s') , \mathbb{I}[s_{h+1}=s'])\\ 
    &=
    \mathbb{P}[((c, s_h, a_h,s'),\mathbb{I}[s_{h+1}=s']))\in Sample^P(h)|(c, s_h, a_h, s') \in (  \widetilde{\mathcal{X}}^{\gamma, \beta}_h \times S_{h+1})]\\
     &\propto
        \mathbb{P}[c | c \in \widehat{\mathcal{C}}^{\beta}(s_h)]
        \cdot
        q_h(s_h, a_h| \widehat{\pi}^c_{s_h}, P^c)
        \cdot
        P^c(s'|s_h,a_h),
    \end{align*}
    and we denote $\mathcal{D}^P_h$ the induced  marginal distribution over  $(\mathcal{X}^{\gamma,\beta}_h \times S_{h+1}) \times [0,1]$.

\begin{remark}
    When it is clear from the context, we use $\mathcal{D}^P_h$ and $\widetilde{\mathcal{D}}^P_h$ to also denote the induced distribution over $(c, s_h, a_h, s_{h+1})\in\mathcal{X}^{\gamma,\beta}_h \times S_{h+1}$ and $(c, s_h, a_h, s_{h+1})\in \widetilde{\mathcal{X}}^{\gamma,\beta}_h \times {S_{h+1}}$, respectively, and drop the indicator bit.
    Similarly for $\mathcal{D}^R_h$ and $\widetilde{\mathcal{D}}^R_h$ we have ${(c, s_h, a_h)\in \mathcal{X}^{\gamma,\beta}_h}$ and ${(c, s_h, a_h)\in\widetilde{\mathcal{X}}^{\gamma,\beta}_h}$.
\end{remark}

\paragraph{Event $G_2$.}
Intuitively it states that sufficient number of examples
have been collected for every layer $h \in [H-1]$.

Formally, let $G_2^{h}$ be the event that
\begin{enumerate}
    \item At least $ \max\{N_R(\mathcal{F}^R_h ,\epsilon_R, \delta_1/2), N_P(\mathcal{F}^P_h ,\epsilon_P, \delta_1/2)\}$ examples of context, state and action from the target domain, i.e., $(c,s,a) \in \mathcal{X}^{\gamma,\beta}_h$, have been collected for layer $h \in [H-1]$ .
    \item At least $ 2\max\{N_R(\mathcal{F}^R_h ,\epsilon_R, \delta_1/2), N_P(\mathcal{F}^P_h ,\epsilon_P, \delta_1/2)\}$ examples of context, state and action from the empirical domain, i.e., $(c,s,a) \in \widetilde{\mathcal{X}}^{\gamma,\beta}_h$, have been collected for layer $h \in [H-1]$ .
\end{enumerate}
Let $G_2$ be the event $\cap_{h\in[H-1]}G^h_2$.


\paragraph{Event $G_3$.}
Intuitively states that the ERM guarantees for the approximation of the dynamics hold.
Formally, let $G_3^h$ denote the following event (for the $\ell_2$ loss),
    \begin{align*}
        \mathbb{E}_{(c, s_h, a_h, s_{h+1})\sim \mathcal{D}^P_h}
        [(f^P_h(c, s_h, a_h, s_{h+1}) - P^c(s_{h+1}| s_h, a_h))^2]
        \leq 
        \epsilon_P ,
    \end{align*}    
and
    \begin{align*}
        \mathbb{E}_{(c, s_h, a_h, s_{h+1})\sim\widetilde{ \mathcal{D}}^P_h}
        [(f^P_h(c, s_h, a_h, s_{h+1}) - P^c(s_{h+1}| s_h, a_h))^2]
        \leq 
        \epsilon_P .
    \end{align*}

Recall that we assume realizability for each layer. 
Define $G_3 = \cap_{h \in [H-1]}G^h_3$.

The following lemma shows that if $G_1$ and $G^h_2$ holds, then $G^h_3$ holds with high probability. (We later show that $G_2$ holds with high probability.)
\begin{lemma}\label{lemma: prob to G^h_3 given conditions l_2}
For any $h\in [H-1]$ it holds that $\mathbb{P}[G_3^h | G_1, G^h_2 ]\geq 1 - \delta_1$.
\end{lemma}

\begin{proof}
    Under $G_1$ and $G^h_2$ we have collected sufficient number of examples from the domain $\mathcal{X}^{\gamma,\beta}_h \times S_{h+1}$ 
    to approximate the transition probability function of layer $h$, for the accuracy parameter $\epsilon_P$ and confidence parameter $\delta_1/2$.
    By the ERM guarantees (see~\ref{par: ERM per layer appexdix}), if
    sufficient number of examples have been collected, 
    then 
    the ERM output $f^P_h$ satisfies that
    ${\mathbb{E}_{\mathcal{D}^P_h}
        [(f^P_h(c, s_h, a_h, s_{h+1}) - P^c(s_{h+1}| s_h, a_h))^2]
        \leq 
    \epsilon_P }$, with probability at least $1-\delta_1/2$.
    Similarly for $\widetilde{\mathcal{X}}^{\gamma,\beta}_h \times S_{h+1}$ it holds that 
    ${\mathbb{E}_{\widetilde{\mathcal{D}}^P_h}
        [(f^P_h(c, s_h, a_h, s_{h+1}) - P^c(s_{h+1}| s_h, a_h))^2]
        \leq 
    \epsilon_P }$, with probability at least $1-\delta_1/2$.
    Hence the lemma follows by union bound.
\end{proof}

The following lemma shows, inductively,  that if for all the previous layers $i < h$ the good events $G_1$, $G^i_2$, $G^i_3$ hold, then $G^h_2$ holds with high probability, for the current layer $h$.   
\begin{lemma}\label{lemma: prob to G^h_2 given conditions l_2}
    For every layer $h\in [H-1]$ it holds that          
    $$\mathbb{P}[G_2^h | G_1, G^i_2 ,G_3^i \;\forall i\in[h-1] ]
    \geq 1 - (\delta_1 + \frac{\epsilon_P}{\rho^2}|S|^2|A|).$$
\end{lemma}

\begin{proof}
    We prove the lemma using induction over the horizon $h$.
        
    \textbf{Base case.} $h=0$.
    
    By definition, the start state $s_0$ is $(1,1)$-good, which implies that for $s_0$ we collect samples in a deterministic manner. Thus, it holds that $\mathbb{P}[G^0_2] = 1$.
        
    \textbf{Induction step.} 
    Assume the lemma holds for all $k < h$ and we show it holds for $h$.\\
    Recall we collect examples of states $s_h \in S_h$ for which
    $\widehat{p}_\beta(s_h) \geq \frac{3}{4}\gamma$.
    Under $G_1$, if $\widehat{p}_\beta(s_h) \geq \frac{3}{4}\gamma$ then
    $
        {\mathbb{P}[c \in \widehat{\mathcal{C}}^{\beta}(s_h) 
        ]
        \geq
        \gamma/2}
    $. 
    In addition, if     
    $
        {\mathbb{P}[c \in \widehat{\mathcal{C}}^{\beta}(s_h) 
        ]
        \geq
        \gamma}
    $ then $\widehat{p}_\beta(s_h) \geq \frac{3}{4}\gamma$.
    
    Thus, the set $\widetilde{S}^{\gamma,\beta}_h$ of approximately $(\gamma,\beta)$-good state for $\widehat{P}^c$ satisfies that  $\widehat{S}^{\gamma,\beta}_h \subseteq \widetilde{S}^{\gamma,\beta}_h \subseteq \widehat{S}^{\gamma/2,\beta}_h$.
    
    Given $G_1, G^k_2 ,G_3^k \;\forall k\in[h-1]$ hold, by Lemma~\ref{lemma: for running time l_2}, for $\beta$ and $\rho$ such that $\beta \geq 2H \frac{4 \rho |S|}{ 1 - \rho^2 |S|^2}$ the following holds. 
    \begin{align*}
        \mathbb{P}_c \left[
        \underbrace{\forall k \in [h], s_k \in S_k. \;\;
        q_k(s_k|\pi_c, P^c) 
        \geq
        q_k(s_k|\pi_c, \widehat{P}^c) -
        \frac{4 \rho|S|}{1 - \rho^2|S|^2}k}_{(\star)} 
        \right]
        &\geq 
        1 - |A|\;\sum_{k=0}^{h-1}\frac{\epsilon_P}{\rho^2}|S_k||S_{k+1}|
        \\
        &\geq
        1 - |A||S|^2\frac{\epsilon_P }{\rho^2}.
    \end{align*}
        
    \begin{claim}
        Assume inequality $(\star)$ holds. Then the probability to collect one example of ${(c, s_{h}, a_{h}) \in  \mathcal{X}^{\gamma, \beta}_{h}}$ is at least $\frac{1}{|S|}\gamma (\beta - \frac{4 \rho |S|}{ 1 - \rho^2 |S|^2} h) \geq \frac{1}{|S|}\cdot \gamma \cdot \beta /2$.
    \end{claim}
    \begin{proof}
        Consider the process of collecting a sample, as described in Algorithm~\ref{alg: EXPLOIT-UCDD}:
        \begin{enumerate}
            \item The algorithm/agent chooses uniformly at random $(s,a) \in \widetilde{S}^{\gamma, \beta}_h \times A$.
            Under the good event $G_1$ we have that 
            ${\widehat{S}^{\gamma, \beta}_h} \subseteq \widetilde{S}^{\gamma, \beta}_h$.
            Hence, the probability to choose 
            $(s,a) \in \widehat{S}^{\gamma, \beta}_h \times A$
            is at least $\frac{1}{|S|}$.
            \item A context $c \sim \mathcal{D}$ is sampled.
            By $\widehat{S}^{\gamma, \beta}_h$ definition, the probability that $c \in \widehat{\mathcal{C}}^\beta(s)$ is at least $\gamma$.
            \begin{itemize}
                \item  If $c \in \widehat{\mathcal{C}}^\beta(s)$, the agent plays $\widehat{\pi}^c_s$ to generate a trajectory where the dynamics is $P^c$. By $(\star)$ and $\widehat{\mathcal{C}}^\beta(s)$ definition, the probability to observe $(s, a)$ in a trajectory generated using $\widehat{\pi}^c_s$ where the dynamics is $P^c$ is $q_h(s|\widehat{\pi}^c_s, P^c) \geq \beta - \frac{4 \rho |S|}{ 1 - \rho^2 |S|^2} h \geq \beta/2$.
                \item Otherwise quite iteration.
            \end{itemize}
        \end{enumerate}
        Overall, 
        the probability to collect one example of a triplet $(c, s, a) \in  \mathcal{X}^{\gamma, \beta}_{h}$ is at least ${\frac{1}{|S|}\cdot \gamma \cdot (\beta - \frac{4 \rho |S|}{ 1 - \rho^2 |S|^2} h)} \geq \frac{1}{|S|}\cdot \gamma \cdot \frac{\beta}{2}$ 
        (since $\beta \geq 2H \frac{4 \rho |S|}{ 1 - \rho^2 |S|^2}$).
    \end{proof}
        
        
      \begin{claim}
        Assume inequality $(\star)$ holds. Then the probability to collect one example of ${(c, s_{h}, a_{h}) \in  \widetilde{\mathcal{X}}^{\gamma, \beta}_{h}}$ is at least $\frac{\gamma}{2}(\beta - \frac{4 \rho |S|}{ 1 - \rho^2 |S|^2} h) \geq  \gamma\cdot \beta /4$.
    \end{claim}
    \begin{proof}
        Consider the process of collecting a sample, as described in Algorithm~\ref{alg: EXPLOIT-UCDD}:
        \begin{enumerate}
            \item The algorithm/agent chooses uniformly at random $(s,a) \in \widetilde{S}^{\gamma, \beta}_h \times A$.
            Under the good event $G_1$ we have that 
            $\widetilde{S}^{\gamma, \beta}_h \subseteq {\widehat{S}^{\gamma/2, \beta}_h}$.
            \item A context $c \sim \mathcal{D}$ is sampled by the nature.
            By $\widehat{S}^{\gamma/2, \beta}_h$ definition, the probability to observe a context $c \in \widehat{\mathcal{C}}^\beta(s)$ is at least $\gamma/2$.
            \begin{itemize}
                \item  If $c \in \widehat{\mathcal{C}}^\beta(s)$, the agent plays $\widehat{\pi}^c_s$ to generate a trajectory where the dynamics is $P^c$. By $(\star)$ and $\widehat{\mathcal{C}}^\beta(s)$ definition, the probability to observe $(s, a)$ in a trajectory generated using $\widehat{\pi}^c_s$ where the dynamics is $P^c$ is $q_h(s|\widehat{\pi}^c_s, P^c) \geq \beta - \frac{4 \rho |S|}{ 1 - \rho^2 |S|^2} h \geq \beta/2$.
                \item Otherwise quite iteration.
            \end{itemize}
        \end{enumerate}
        Overall, 
        the probability to collect one sample of some triplet $(c, s, a) \in \widetilde{\mathcal{X}}^{\gamma, \beta}_{h}$ is at least ${\gamma/2 \cdot (\beta - \frac{4 \rho |S|}{ 1 - \rho^2 |S|^2} h)} \geq \gamma \cdot \beta/4$ 
        (since $\beta \geq 2H \frac{4 \rho |S|}{ 1 - \rho^2 |S|^2}$).
    \end{proof}    
        
    The above claims implies that if $(\star)$ holds, in expectation, the agent needs to experience at most 
   $ \frac{2|S|}{\gamma \cdot \beta}$    
    episodes to collect one example from $ \mathcal{X}^{\gamma, \beta}_{h}$.
    In addition, in expectation, the agent needs to experience at most $\frac{4}{\gamma \cdot \beta}$ 
    episodes to collect one example from 
    $ \widetilde{\mathcal{X}}^{\gamma, \beta}_{h}$.
    
    Since under $G_1$ we have that    
    $ 
        \mathcal{X}^{\gamma, \beta}_{h}
        \subseteq 
        \widetilde{\mathcal{X}}^{\gamma, \beta}_{h}
        \subseteq
        \mathcal{X}^{\gamma/2, \beta}_{h}
    $,
    using multiplicative Chernoff bound we obtain that with probability at least $1 - \delta_1$ after experiencing  
    \[
        T_h =
        \left\lceil 
            \frac{8 |S|}{\gamma \cdot \beta}
            \left(\ln\frac{1}{\delta_1}
            +
            2\max\{N_P(\mathcal{F}^P_h ,\epsilon_P, \delta_1/2), N_R(\mathcal{F}^R_h ,\epsilon_R, \delta_1/2)\}\right)
            \right\rceil
    \]
    episodes, the agent will collect at least $\max\{N_P(\mathcal{F}^P_h ,\epsilon_P, \delta_1/2)), N_R(\mathcal{F}^R_h ,\epsilon_R, \delta_1/2))\}$ examples from $\mathcal{X}^{\gamma,\beta}_h$ and 
    $2\max\{N_P(\mathcal{F}^P_h ,\epsilon_P, \delta_1/2)), N_R(\mathcal{F}^R_h ,\epsilon_R, \delta_1/2))\}$ examples from $\widetilde{\mathcal{X}}^{\gamma,\beta}_h$.
    
    Recall that $T_h$ is exactly the number of episodes we run in Algorithm~\ref{alg: EXPLOIT-UCDD} when learning layer $h$.
    Hence, using union bound we obtain that
    \[
        \mathbb{P}[G_2^{h} | G_1, G^i_2 ,G_3^i \;\forall i\in[h-1] ]
        \geq 1 - (\delta_1  +  |A||S|^2\frac{\epsilon_P }{\rho^2}).
    \]    
\end{proof}

 The following lemma shows that given $G_1$ holds, $G_2$ and $G_3$ holds with high probability.   
\begin{lemma}\label{lemma: UCDD  l_2 G_2 ,G_3 given G_1}
    The following holds.
    \[
        \mathbb{P}[G_2 \cap G_3 | G_1] 
        \geq 1 - (2\delta_1 H + \frac{\epsilon_P}{\rho^2}|S|^2|A|H).
    \]
\end{lemma}
    
\begin{proof}
    Assume the good event $G_1$ holds.
    Recall that $G_2 = \cap_{h \in [H-1]}G^h_2$ and $G_3 = \cap_{h \in [H-1]}G^h_3$.
        
    Let $X$ be a random variable with support $[H-1]$ that satisfies 
    \begin{align*}
        X = \min_{k \in [H-1]}\{\overline{G}^k_2 \cup \overline{G}^k_3 \text{ holds }\},
    \end{align*}
    and otherwise $X=\bot$, meaning if $G_2$ and $G_3$ hold.
        
    In words, $X$ is the first layer in which at least one of the good events $G^h_2$ or $G^h_3$ does not hold.
        
    By $X$ definition and Bayes rule (i.e., $\mathbb{P}[A \cap B] = \mathbb{P}[A|B]\cdot \mathbb{P}[B]$) we have
    \begingroup
    \allowdisplaybreaks
    \begin{align*}
        \forall h \in [H].\;\; 
        \mathbb{P}[X = h|G_1]
        &=
        \mathbb{P}[(\overline{G}^h_2 \cup \overline{G}^h_3)\cap (\cap_{
        k \in [h-1]}G^k_2 \cap G^k_3) | G_1]
        \\
        \tag{Base rule}
        &=
        \mathbb{P}[(\overline{G}^h_2 \cup \overline{G}^h_3) | G_1 , (\cap_{
        k \in [h-1]}G^k_2 \cap G^k_3)]
        \cdot
        \underbrace{\mathbb{P}[ (\cap_{ k \in [h-1]}G^k_2 \cap G^k_3)|G_1]}_{\leq 1}
        \\
        &\leq
        \mathbb{P}[(\overline{G}^h_2 \cup \overline{G}^h_3) | G_1 , (\cap_{
        k \in [h-1]}G^k_2 \cap G^k_3)]
        \\
        \tag{Union of disjoint events}
        &=
        \underbrace{\mathbb{P}[\overline{G}^h_2 | G_1 , 
        (\cap_{k \in [h-1]}G^k_2 \cap G^k_3)]}_{\leq \delta_1 + \frac{\epsilon_P}{\rho^2}|S|^2|A| \text{ by Lemma~\ref{lemma: prob to G^h_2 given conditions l_2}}}
        +
        \mathbb{P}[\overline{G}^h_3\cap G_2^h | G_1 , 
        (\cap_{k \in [h-1]}G^k_2 \cap G^k_3)]
        \\
        &\leq
        \delta_1 + \frac{\epsilon_P}{\rho^2}|S||A|
        +
        \mathbb{P}[\overline{G}^h_3 \cap G^h_2| G_1 , 
        (\cap_{k \in [h-1]}G^k_2 \cap G^k_3)]
        \\
        \tag{caused by rule}
        &=
        \delta_1 + \frac{\epsilon_P}{\rho^2}|S||A|
        +
        \mathbb{P}[\overline{G}^h_3 | G_1 , G^h_2 ,
        (\cap_{k \in [h-1]}G^k_2 \cap G^k_3)]
        \cdot
        \underbrace{\mathbb{P}[ G^h_2  |G_1 \cap
        (\cap_{k \in [h-1]}G^k_2 \cap G^k_3)]]}_{\leq 1}
        \\
        &\leq
        \delta_1 + \frac{\epsilon_P}{\rho^2}|S||A|
        +
        \mathbb{P}[\overline{G}^h_3 | G_1 , G^h_2 ,
        (\cap_{k \in [h-1]}G^k_2 \cap G^k_3)]            
        \\
        &=
        \delta_1 + \frac{\epsilon_P}{\rho^2}|S|^2|A|
        +
        \underbrace{\mathbb{P}[\overline{G}^k_3 | G_1 , G^h_2]}_{\leq \delta_1 \text{ by Lemma~\ref{lemma: prob to G^h_3 given conditions l_2}}}
        \\
        &\leq
        2\delta_1 + \frac{\epsilon_P}{\rho^2}|S|^2|A| .
    \end{align*}
    \endgroup
        
    Now, by $G_2$ and $G_3$ definitions we have
    \begingroup
    \allowdisplaybreaks
    \begin{align*}
        \mathbb{P}[G_2 \cap G_3 |G_1]
            &=
            1 - \mathbb{P}[\overline{G}_2 \cup \overline{G}_3|G_1]
            \\
            &=
            1 
            -
            \mathbb{P}[\cup_{h \in [H-1]}(\overline{G}^h_2 \cup \overline{G}^h_3)|G_1]
            \\
            &=
            1 
            -
            \mathbb{P}[\exists h \in [H-1].(\overline{G}^h_2 \cup \overline{G}^h_3)|G_1]
            \\
            &=
            1 
            -
            \mathbb{P}[\exists h \in [H-1].X = h|G_1]
            \\
            &=
            1 
            -
            \mathbb{P}[\cup_{ h \in [H-1]}\{X = h\}|G_1]
            \\
            \tag{Union bound}
            &=
            1 
            -
            \sum_{h =0}^{H-1}\mathbb{P}[X = h|G_1]
            \\
            &\geq
            1 - (2 \delta_1 H + \frac{\epsilon_P}{\rho^2}|S|^2|A|H),
    \end{align*}
    \endgroup
    as stated.
\end{proof}

\paragraph{Event $G_4$.}
Intuitively states that the
ERM guarantees for the approximation of the rewards function hold (for the $\ell_2$ loss).
    Let $G_4$ denote the good event
    \begin{align*}
        &\forall h \in [H - 1]. \;\;\mathbb{E}_{(c, s_h, a_h) \sim \mathcal{D}^R_h}
        [( f^R_h(c, s_h, a_h) - r^c(s_h, a_h) )^2]
        \leq 
        \epsilon_R + \alpha^2_2(\mathcal{F}^R_h).
    \end{align*}
    
The following lemma shows that given $G_1$ and $G_2$ hold, we have that $G_4$ holds with high probability.
\begin{lemma}\label{lemma: UCDD l_2 G_4}
    It holds that 
    $
        \mathbb{P}[G_4 | G_1, G_2] \geq 1 - \delta_1 H
    $. 
\end{lemma}
\begin{proof}
    Since $G_1$ and $G_2$ hold, for every layer $h \in [H-1]$ sufficient number of examples $((c,s_h,a_h),r_h) \in \mathcal{X}^{\gamma,\beta}_h \times [0,1]$ have been collected for the 
    ERM to output a function $f^R_h$ the satisfies
    $$
    {\mathbb{E}_{(c,s_h,a_h) \sim \mathcal{D}^R_h}
        [( f^R_h(c, s_h, a_h) - r^c(s_h, a_h) )^2]
        \leq 
        \epsilon_R + \alpha^2_2(\mathcal{F}^R_h)}
    $$
     with probability at least $1-\delta_1$.
    Hence, the lemma
    follows by the ERM guarantees (see~\ref{par: ERM per layer appexdix}) and an union bound over every layer $h \in [H-1]$.
\end{proof}

Overall, all the good events hols with high probability.
\begin{corollary}\label{corl: good events probs bound l_2}
    The following holds. 
    \[
        \mathbb{P}[G_1, G_2, G_3, G_4] 
        \geq 
        1- \left(\frac{\delta}{8} + 3 \delta_1 H + \frac{\epsilon_P}{\rho^2}|S|^2|A|H\right).
    \]    
\end{corollary}    

\begin{proof}
    Followed from union bound over the results of Lemmas~\ref{lemma: UCDD l_2 G_1},~\ref{lemma: UCDD l_2 G_4} and~\ref{lemma: UCDD  l_2 G_2 ,G_3 given G_1}. 
\end{proof}


\subsubsection{Analysis of the Error Caused by the Dynamics Approximation Under the Good Events}\label{subsubsec:l-2-dynamics-error}


In the following analysis, for any context $c \in \mathcal{C}$ we  consider an intermediate MDP associated with it: 
$\widetilde{\mathcal{M}}(c) = (S \cup \{s_{sink}\}, A, \widehat{P}^c, r^c, s_0, H)$, where $\widehat{P}^c$ is the approximation of the dynamics $P^c$ and $r^c$ is the true rewards function extended to $s_{sink}$ by defining that $r^c(s_{sink},a):=0,\;\; \forall c \in \mathcal{C},\; a \in A$.
Recall the true MDP associated with this context is
    $
        \mathcal{M}(c)
        = 
        (S, A, P^c, r^c, s_0, H)
    $.





\begin{lemma}\label{lemma: UCDD l_2 matrix diff}
    Let $\rho \in [0,\frac{1}{|S|})$ and $h \in [H -1]$.
    Assume the good events $G_1, G_2^k, G_3^k, \; \forall k \in [h]$ hold, then it holds that 
    \begin{align*}
        \mathop{\mathbb{P}}_{(c,s_h,a_h)}
        \left[ 
            \|\widehat{P}^c(\cdot| s_h, a_h) - P^c(\cdot| s_h, a_h)\|_1   
            \leq \frac{4 \rho |S|}{ 1 - \rho^2 |S|^2}
        \Big|
        (c, s_h, a_h) \in  \mathcal{X}^{\gamma, \beta}_h
        \right]
        \geq
        1- \frac{\epsilon_P}{\rho^2}|S_{h+1}|,
    \end{align*}
    where $\widehat{P}^c$ is the dynamics defined in Algorithm~\ref{alg: ACDD UCDD} and 
    \[
        \|\widehat{P}^c(\cdot| s_h, a_h) - P^c(\cdot| s_h, a_h)\|_1:=
        \sum_{s_{h+1} \in S_{h+1}}
        |\widehat{P}^c(s_{h+1}| s_h, a_h) - P^c(s_{h+1}| s_h, a_h)|
    \]
    (i.e., the entry of $s_{sink}$ in $\widehat{P}^c$ is ignored).   
\end{lemma}


\begin{proof}
    Under  $G_1$ it holds that
    $\mathcal{X}^{\gamma,\beta}_h \subseteq \widetilde{\mathcal{X}}^{\gamma,\beta}_h$.
    Recall that for all $(c, s_h, a_h) \in \widetilde{\mathcal{X}}^{\gamma,\beta}_h$ we have that $\widehat{P}^c(s_{sink}|s_h, a_h) = 0$ by $\widehat{P}^c$ definition.
    Hence, for all $(c, s_h, a_h) \in \mathcal{X}^{\gamma,\beta}_h$ we have that $\widehat{P}^c(s_{sink}|s_h, a_h) = 0$.
    
    In addition, the true dynamics $P^c$ is not defined for $s_{sink}$ since $s_{sink} \notin S$.  
    A natural extension of $P^c$ to $s_{sink}$ is by defining that $\forall (s,a) \in S \times A.\;\; P^c(s_{sink}|s,a) :=0$.
    By that extension, we have for all $(c, s_h, a_h) \in \mathcal{X}^{\gamma,\beta}_h$ that
    $P^c(s_{sink}|s_h,a_h) = \widehat{P}^c(s_{sink}|s_h,a_h)=0$. Hence, we can simply ignore $s_{sink}$ in the following analysis.
    
    Under the good event $ G_3^h $,
    we have
    %
    \begin{align*}
        &\mathop{\mathbb{P}}_{(c,s_h,a_h,s_{h+1})}
        \left[ |f^P_h(c, s_h, a_h, s_{h+1}) - P^c(s_{h+1}| s_h, a_h)| \geq \rho \Big| (c,s_h,a_h) \in \mathcal{X}^{\gamma,\beta}_h \right]
        = 
        \\
        & =
        \mathbb{P}_{\mathcal{D}^P_h}[|f^P_h(c, s_h, a_h, s_{h+1}) - P^c(s_{h+1}| s_h, a_h)| \geq \rho 
        ]
        \\
        & = 
        \mathbb{P}_{\mathcal{D}^P_h}[(f^P_h(c, s_h, a_h, s_{h+1}) - P^c(s_{h+1}| s_h, a_h))^2 \geq \rho^2
        ]
        \\
        \tag{By Markov's inequality}
        & \leq  
        \frac{\mathbb{E}_{\mathcal{D}^P_h}
        [(f^P_h(c, s_h, a_h, s_{h+1}) - P^c(s_{h+1}| s_h, a_h))^2]}{\rho^2}
        \\
        \tag{Under $G^3_h$}
        & \leq 
        \frac{\epsilon_P}{\rho^2}.
    \end{align*}
    
    Hence,
    \[
        \mathop{\mathbb{P}}_{(c,s_h,a_h,s_{h+1})}
        \left[|f^P_h(c, s_h, a_h, s_{h+1}) - P^c(s_{h+1}| s_h, a_h)| \leq \rho\Big| (c,s_h,a_h) \in \mathcal{X}^{\gamma,\beta}_h \right]
        \geq 1 - \frac{\epsilon_P }{\rho^2}.
    \]
    
    As $P^c(\cdot |s_h, a_h)$ is a distribution, we have for every context $c$ that $\sum_{s_{h+1} \in S_{h+1}} P^c(s_{h+1}|s_h, a_h) = 1$. 
    
    Thus, by union bound over $s_{h+1}\in S_{h+1}$
    we obtain
    \begin{align*}
        \mathbb{P}_{(c, s_h, a_h) }
        \left[ 
            1 - \rho |S|\leq
            \sum_{s_{h+1}\in S_{h+1}} f^P_h(c, s_h, a_h, s_{h+1}) \leq
            1 + \rho |S|
        \Big|
        (c, s_h, a_h) \in  \mathcal{X}^{\gamma, \beta}_h    
        \right]
        \geq 
        1 - \frac{\epsilon_P }{\rho^2}|S_{h+1}|.
    \end{align*}
   Hence, we further conclude 
    \begin{equation}\label{prob: lemma C.5}
        \begin{split}
            &\mathop{\mathbb{P}}_{(c,s_h,a_h)}
            \Bigg[  
                \forall s_{h+1} \in S_{h+1}.
                \frac{P^c(s_{h+1}|s_h, a_h) - \rho}{1 + \rho |S|}  \leq \underbrace{\frac{f^P_h(c,s_h,a_h,s_{h+1})}{\sum_{s'\in S_{h+1}} f^P_h(c,s_h,a_h,s')}}_{=\widehat{P}^c(s_{h+1}|s_h, a_h)}
                \leq
                \frac{P^c(s_{h+1}|s_h, a_h) + \rho}{1 - \rho |S|}
            \;\Big|
            (c, s_h, a_h) \in  \mathcal{X}^{\gamma, \beta}_h   \Bigg]
            \\
            & \geq 
            1 - \frac{\epsilon_P }{\rho^2}|S_{h+1}|.
        \end{split}
    \end{equation}
        
    Fix a tuple $(c, s_h, a_h) \in \mathcal{X}^{\gamma, \beta}_h$ and assume the event of inequality~(\ref{prob: lemma C.5}) holds.\\
    Denote ${S^+_{h+1} = \{s_{h+1} \in S_{h+1}:
    \widehat{P}^c(s_{h+1}|s_h, a_h) \geq  P^c(s_{h+1}|s_h, a_h )\}}$ and consider the following derivation.
    \begingroup
    \allowdisplaybreaks
    \begin{align*}
        \| 
            \widehat{P}^c(\cdot|s_h, a_h ) 
            -  
            P^c(\cdot|s_h, a_h )
        \|_1 
        =&
        \sum_{s_{h+1} \in S_{h+1}}
        |\widehat{P}^c(s_{h+1}|s_h, a_h) -  P^c(s_{h+1}|s_h, a_h )|
        \\
        =&
        \sum_{s_{h+1} \in S^+_{h+1}}
        (\widehat{P}^c(s_{h+1}|s_h, a_h) -  P^c(s_{h+1}|s_h, a_h ))       
        \\
        &+
        \sum_{s_{h+1} \in  S_{h+1} \setminus S^+_{h+1}}
        \left(P^c(s_{h+1}|s_h, a_h) -  \widehat{P}^c(s_{h+1}|s_h, a_h )\right)
        \\  
        \leq&
        \sum_{s_{h+1} \in S^+_{h+1}}
        \left(\frac{P^c(s_{h+1}|s_h, a_h) + \rho}{1 - \rho |S|} 
        -  
        P^c (s_{h+1}|s_h, a_h )\right)
        \\
        &+
        \sum_{s_{h+1} \in  S_{h+1} \setminus S^+_{h+1}}
        \left( P^c (s_{h+1}|s_h, a_h )
        -
        \frac{P^c(s_{h+1}|s_h, a_h) - \rho}{1 + \rho |S|} \right) 
        \\
        =&
        \sum_{s_{h+1} \in S^+_{h+1}}
        \frac{P^c (s_{h+1}|s_h, a_h) + \rho - (1- \rho |S|)P^c (s_{h+1}|s_h, a_h) }{1 - \rho |S|} 
        \\
        &+
        \sum_{s_{h+1} \in  S_{h+1} \setminus S^+_{h+1}}
        \frac{-P^c (s_{h+1}|s_h, a_h) + \rho + (1 + \rho |S|)P^c (s_{h+1}|s_h, a_h) }{1 + \rho |S|}         
        \\
        =&
        \frac{1}{1- \rho |S|}
        \sum_{s_{h+1} \in S^+_{h+1}}
        (\rho + \rho|S|P^c (s_{h+1}|s_h,a_h))
        \\
        &+
        \frac{1}{1 + \rho |S|}
        \sum_{s_{h+1} \in  S_{h+1} \setminus S^+_{h+1}}
        (\rho + \rho|S|P^c (s_{h+1}|s_h,a_h))
        \\
        \leq&
        \frac{2 \rho |S|}{ 1- \rho |S|}
        +
        \frac{2 \rho |S|}{ 1 + \rho |S|}
        \\
        =&
        \frac{4 \rho |S|}{1 - \rho^2 |S|^2}.
    \end{align*}
    \endgroup
    By inequality~(\ref{prob: lemma C.5}), the above holds with probability at least $1 - \frac{\epsilon_P }{\rho^2}|S_{h+1}|$ over $(c, s_h, a_h) \in \mathcal{X}^{\gamma,\beta}_h$. Hence the lemma follows.

\end{proof}

\begin{lemma}
    For the parameters choice  $\beta =  \frac{\epsilon}{20|S|H} \in (0,1)$, $\rho = \frac{\beta}{16 |S|H} \in (0, \frac{1}{|S|})$,
    and $\epsilon_P = \frac{ \epsilon^3}{10\cdot 2^8 20^2  |A| |S|^6 H^5}$,
    we have $\beta \geq 2H\frac{4 \rho |S|}{1- \rho^2 |S|^2}$.
    In addition, under the good events $G_1$, $G^k_2$ and $G^k_3$ 
    for all $k \in [h]$,
    we have that
    \[
        \mathbb{P}\left[\|\widehat{P}^c(\cdot| s_h, a_h) - P^c(\cdot| s_h, a_h)\|_1\leq \frac{\epsilon}{40 |S| H^2} \Big| (c,s_h,a_h) \in \mathcal{X}^{\gamma,\beta}_h \right]
        \geq 1-\frac{\epsilon}{10 |S||A|H}.
    \]
\end{lemma}

\begin{proof}
    An immediate implication of lemma~\ref{lemma: UCDD l_2 matrix diff}.
\end{proof}


\begin{lemma}[occupancy measures difference]\label{lemma: UCDD l_2 occ measure diff}
    
    Let  $\rho \in [0, \frac{1}{|S|})$ and $\beta \in (0,1]$ for which ${\beta \geq 2H\frac{4 \rho |S|}{1 - \rho^2 |S|^2}} $.
    Under the good events $G_1, G_2, G_3$, for every context-dependent policy $\pi$ it holds that
    \begin{align*}
        \mathbb{P}_c\left[
		\forall h \in [H].\;
		\|
			q_h (\cdot | \pi_c, P^c) - q_h(\cdot| \pi_c, \widehat{P}^c)		
		\|_1
		\leq 
		\frac{4 \rho |S|}{1 - \rho^2 |S|^2} h 
		+ \beta \sum_{k=1}^{h-1} |S_k|
		\right]
		\geq
		 1 
        - 
        \left(
            \frac{\epsilon_P}{\rho^2}
            |A|\sum_{i=0}^{H-1}|S_i||S_{i+1}|
            + 
            \gamma
            \sum_{i=1}^{H-1}|S_{i}|
        \right)
    \end{align*}
    %
    for a fixed $\rho \in [0, \frac{1}{|S|})$ and $\beta \geq 2H\frac{4 \rho |S|}{1 - \rho^2 |S|^2} $, where we define
    \[
        \forall h \in [H].\;\; \|q_h(\cdot | \pi_c, P^c) - q_h(\cdot | \pi_c, \widehat{P}^c)\|_1 := 
        \sum_{s_{h} \in S_h}|q_h(s_h | \pi_c, P^c) - q_h(s_h | \pi_c, \widehat{P}^c)|
    \]
    (i.e., $q_h(s_{sink}|\pi_c,\widehat{P}^c)$ is ignored for all $h \in [H]$).
\end{lemma}

\begin{remark}
    Since $s_{sink} \notin S$, $q_h(s_{sink}|\pi_c, P^c)$ is not defined for the true dynamics $P^c$.
    In addition, by $\widehat{P}^c$ definition, from the sink there are no transitions to any other state, hence, we can simply ignore it in the following analysis.
\end{remark}

\begin{proof}
   We will show the lemma by induction over the horizon, $h$.
   
    For the base case $h =0$ the claim holds trivially (with probability $1$) since the start state $s_0$ is unique.

    For the induction step, assume that it holds for $h$, namely,
    \begin{align*}
        \mathbb{P}_c
        \left[\forall k \in [h]. \;\;\; 		
        \|
			q_k (\cdot | \pi_c, P^c) - q_k(\cdot | \pi_c, \widehat{P}^c)		
		\|_1
		\leq 
		\frac{4 \rho |S|}{1 - \rho^2 |S|^2}k
		+
		\beta \sum_{i=1}^{k-1}|S_i| 
		\right]
        \geq 
        1 -
        \left(
            \frac{\epsilon_P}{\rho^2}
            |A|\sum_{i=0}^{h-1}|S_i||S_{i+1}|
            + 
            \gamma
            \sum_{i=1}^{h-1}|S_{i}| 
        \right)
    \end{align*}

   and prove for $h+1$.
   
   Under the good events $G_1$ , $G_2$ and $G_3$ 
    %
    by Lemma~\ref{lemma: UCDD l_2 matrix diff} it holds that 
    
    \begin{align*}
        \mathop{\mathbb{P}}_{(c,s_h,a_h)}
        \left[ 
            \|\widehat{P}^c(\cdot| s_h, a_h) - P^c(\cdot| s_h, a_h)\|_1   
            \leq \frac{4 \rho |S|}{ 1 - \rho^2 |S|^2}
        \Big|
        (c, s_h, a_h) \in  \mathcal{X}^{\gamma, \beta}_h
        \right]
        \geq
        1- \frac{\epsilon_P}{\rho^2}|S_{h+1}|.
    \end{align*}

    Consider the following derivation for any fixed context $c \in \mathcal{C}$.  (Later we will take the probability over $c\sim \mathcal{D}$.)
    \begingroup
    \allowdisplaybreaks
    \begin{align*}
        &\| q_{h+1}(\cdot | \pi_c, P^c) - q_{h+1}(\cdot | \pi_c, \widehat{P}^c) \|_1
        \\
       = &\sum_{s_{h+1} \in S_{h+1}}
        |q_{h+1}(s_{h+1} | \pi_c, P^c) - q_{h+1}(s_{h+1} | \pi_c, \widehat{P}^c)|
        \\
        =&\sum_{s_{h+1} \in S_{h+1}}
        |
            \sum_{s_{h} \in S_{h}}
            \sum_{a_h \in A}
            (
                q_{h}(s_{h} | \pi_c, P^c)
                \pi_c(a_h | s_h)
                P^c(s_{h+1}|s_h, a_h)
                -
                q_{h}(s_{h} | \pi_c, \widehat{P}^c)
                \pi_c(a_h | s_h)
                \widehat{P}^c(s_{h+1}|s_h, a_h)
            )
       |
       \\
        \leq&\underbrace
        {
            \sum_{s_{h} \in S_{h}}
            \sum_{a_h \in A}
            \sum_{s_{h+1} \in S_{h+1}}
            |
                q_h(s_h | \pi_c, P^c)
                \pi_c(a_h | s_h)
                P^c(s_{h+1}|s_h, a_h)
                -
                q_h(s_h | \pi_c, \widehat{P}^c)
                \pi_c(a_h | s_h)
                P^c(s_{h+1}|s_h, a_h)
            |
        }_{(1)}
        \\    
        &+
        \underbrace
        {
            \sum_{s_{h} \in S_{h}}
            \sum_{a_h \in A}
            \sum_{s_{h+1} \in S_{h+1}}
            |
                q_h(s_h | \pi_c, \widehat{P}^c)
                \pi_c(a_h | s_h)
                P^c(s_{h+1}|s_h, a_h)  
                -
                q_h(s_h | \pi_c, \widehat{P}^c)
                \pi_c(a_h | s_h)
                \widehat{P}^c(s_{h+1}|s_h, a_h)
            |
        }_{(2)}
    \end{align*}
    \endgroup
   
   We bound $(1)$ and $(2)$ separately.
   
   For $(1)$,  
    since $P^c(\cdot| s_h,a_h)$, $\pi_c(\cdot|s_h)$ are distributions,
    the following holds with probability $1$.
   \begin{align*}
       (1)
       &=
       \sum_{s_{h} \in S_{h}}
        |
           q_h(\cdot | \pi_c, P^c)
            -
            q_h(s_h | \pi_c, \widehat{P}^c)
        |
       \sum_{a_h \in A}
       \pi_c(a_h | s_h)
       \sum_{s_{h+1} \in S_{h+1}}
       P^c(s_{h+1}|s_h, a_h)
       \\
       &=
       \|  
            q_h(\cdot | \pi_c, P^c)
            -
            q_h(\cdot | \pi_c, \widehat{P}^c)
        \|_1
       \sum_{a_h \in A}
       \pi_c(a_h | s_h)
       \sum_{s_{h+1} \in S_{h+1}}
       P^c(s_{h+1}|s_h, a_h)
       \\
       &=
       \|  
            q_h(\cdot | \pi_c, P^c)
            -
            q_h(\cdot | \pi_c, \widehat{P}^c)
        \|_1
        .
   \end{align*}

   For $(2)$, we define for any given context $c$ and every layer $h \in [H-1]$ 
   the following subsets of $S_h$ .
   \begin{enumerate}
       \item $B^{h, c}_1 = 
       \{s_h \in  S_h :
           s_h \in \widehat{S}^{\gamma,\beta}_h  \text{ and } c \in \widehat{C}^{\beta}(s_h)\}$.
       \item $B^{h, c}_2 = 
       \{ s_h \in  S_h :
       s_h \in \widehat{S}^{\gamma,\beta}_h \text{ and } c \notin \widehat{C}^{\beta}(s_h)\}$.
       \item $B^{h, c}_3 = 
       \{s_h \in  S_h :
       s_h \notin \widehat{S}^{\gamma,\beta}_h \text{ and } c \notin \widehat{C}^{\beta}(s_h)\}$.
       \item $B^{h, c}_4 = 
       \{ s_h \in  S_h :
       s_h \notin \widehat{S}^{\gamma,\beta}_h \text{ and } c \in \widehat{C}^{\beta}(s_h)\}$.      
   \end{enumerate}
    Clearly, $\cup_{i=1}^4 B^{h,c}_i = S_h$ for all $h \in [H-1]$ and $c \in \mathcal{C}$.
    
    By definition of $B^{h, c}_1$, for every layer $h\in [H-1]$ we have that $s_h \in B^{h, c}_1$  if and only if for every action $ a_h \in A$ it holds that $(c, s_h, a_h) \in  \mathcal{X}^{\gamma, \beta}_h$.
    
   
   By definition of $B^{h, c}_4$, for every  layer $h\in [H-1]$ it holds that
   \begin{align*}
       \mathbb{P}_c[B^{h, c}_4 \neq \emptyset]
       &=
       \mathbb{P}_c[\exists s_h \in S_h :
       s_h \notin \widehat{S}^{\gamma, \beta}_h \text{ and } c \in \widehat{C}^{\beta}(s_h)]
       \leq 
       \gamma|S_h|.
   \end{align*}
   Thus, for every $h \in [H-1]$ we have $\mathbb{P}_c[B^{h, c}_4 =\emptyset]\geq 1 - \gamma|S_h|$.
   
   In the following, we assume that $B^{h, c}_4 =\emptyset$. Since $\mathbb{P}_c[B^{h, c}_4 =\emptyset]\geq 1 - \gamma|S_h|$, it will only add $\gamma |S_h|$ to the probability of the error. 
   
    Consider the following derivation
    \begingroup
    \allowdisplaybreaks
   \begin{align*}
       (2) &= 
        \sum_{s_{h} \in B^{h, c}_1}
        \sum_{a_h \in A}
        \sum_{s_{h+1} \in S_{h+1}}
        |
            q_h(s_h | \pi_c, \widehat{P}^c)
            \pi_c(a_h | s_h)
            P^c(s_{h+1}|s_h, a_h)  
            -
            q_h(s_h | \pi_c, \widehat{P}^c)
            \pi_c(a_h | s_h)
            \widehat{P}^c(s_{h+1}|s_h, a_h)
        |
        \\
        &+
        \sum_{s_{h} \in B^{h, c}_2 \cup B^{h, c}_3}
        \sum_{a_h \in A}
        \sum_{s_{h+1} \in S_{h+1}}
        |
            q_h(s_h | \pi_c, \widehat{P}^c)
            \pi_c(a_h | s_h)
            P^c(s_{h+1}|s_h, a_h)  
            -
            q_h(s_h | \pi_c,\widehat{P}^c)
            \pi_c(a_h | s_h)
            \widehat{P}^c(s_{h+1}|s_h, a_h)
        |
        \\
        &+
        \sum_{s_{h} \in B^{h, c}_4}
        \sum_{a_h \in A}
        \sum_{s_{h+1} \in S_{h+1}}
        |
            q_h(s_h | \pi_c, \widehat{P}^c)
            \pi_c(a_h | s_h)
            P^c(s_{h+1}|s_h, a_h)  
            -
            q_h(s_h | \pi_c,\widehat{P}^c)
            \pi_c(a_h | s_h)
            \widehat{P}^c(s_{h+1}|s_h, a_h)
        |        
        \\
        &=
        \sum_{s_{h} \in B^{h, c}_1}
        q_h(s_h | \pi_c, \widehat{P}^c)
        \sum_{a_h \in A}
        \pi_c(a_h | s_h)
        \sum_{s_{h+1} \in S_{h+1}}
        |
            P^c(s_{h+1}|s_h, a_h)  
            -
            \widehat{P}^c(s_{h+1}|s_h, a_h)
        |
        \\
        \tag{$B^{h, c}_4 =\emptyset$ w.p. at least $1-\gamma |S_h|$}
        &+
        \sum_{s_{h} \in B^{h, c}_2 \cup B^{h, c}_3}
        \sum_{a_h \in A}
        q_h(s_h | \pi_c, \widehat{P}^c)
        \pi_c(a_h | s_h) 
        \underbrace{
        \sum_{s_{h+1} \in S_{h+1}}
        |
            P^c(s_{h+1}|s_h, a_h)  
            -
            \widehat{P}^c(s_{h+1}|s_h, a_h)
        |}_{\leq 1}
        \\
        &\leq
        \sum_{s_{h} \in B^{h, c}_1}
        q_h(s_h | \pi_c, \widehat{P}^c)
        \sum_{a_h \in A}
        \pi_c(a_h | s_h)
        \|
            P^c(\cdot|s_h, a_h)  
            -
            \widehat{P}^c(\cdot|s_h, a_h)
        \|_1
        +
        \sum_{s_{h} \in B^{h, c}_2 \cup B^{h, c}_3}
        q_h(s_h | \pi_c, \widehat{P}^c)
        \underbrace{
        \sum_{a_h \in A}
        \pi_c(a_h | s_h)
        }_{=1}
        \\
        \tag{By Lemma~\ref{lemma: UCDD l_2 matrix diff} and union bound over $(s_h,a_h) \in B^{h, c}_1 \times A$, holds w.p. at least $1-|A||S_h|\frac{\epsilon_P}{\rho^2}|S_{h+1}|$}
        &\leq
        \underbrace
        {
            \sum_{s_{h} \in B^{h, c}_1}
            q_h(s_h | \pi_c, \widehat{P}^c)
            \sum_{a_h \in A}
            \pi_c(a_h | s_h)
        }_{\leq 1}
        \frac{4 \rho |S|}
        {1 - \rho^2 |S|^2}
        +
        \sum_{s_{h} \in B^{h, c}_2 \cup B^{h, c}_3}
        \underbrace{q_h(s_h | \pi_c, \widehat{P}^c)}_{\leq q_{h}(s_{h} | \widehat{\pi}^c_{s_h}, \widehat{P}^c) < \beta}
        \underbrace{\sum_{a_h \in A}
        \pi_c(a_h | s_h)}_{=1}
        \\
        &=
        \frac{4 \rho |S|}
        {1 - \rho^2 |S|^2}
        + \beta |S_h| .
   \end{align*}
   \endgroup
       Hence, 
   \begin{equation}\label{prob:2-bounded-l-2}
       \mathbb{P}_c\left[ 
            (2) \leq \frac{4 \rho |S|}{1 - \rho^2 |S|^2}+ \beta |S_h|
       \right]
       \geq 
       1- \left(|A||S_h|\frac{\epsilon_P}{\rho^2}|S_{h+1}| + \gamma |S_h| \right).
   \end{equation}
   In addition, we showed above that
    \begin{align*}
       \mathbb{P}_c\left[ 
            \|q_{h+1}(\cdot |\pi_c, P^c) - q_{h+1}(\cdot |\pi_c, \widehat{P}^c)\|_1 \leq (1) + (2)
       \right] = 1,
   \end{align*}
   and
   \begin{align*}
       \mathbb{P}_c\left[ 
            (1) = \|q_h(\cdot |\pi_c, P^c) - q_h(\cdot |\pi_c, \widehat{P}^c)\|_1
       \right] = 1.
   \end{align*}
   Thus, by combining all the above inequalities with the induction hypothesis we obtain 
   \begingroup
   \allowdisplaybreaks
    \begin{align*}
        &\mathbb{P}_c
        \left[\forall k \in [h+1]. \;\;\; 		
        \|
			q_k (\cdot | \pi_c, P^c) - q_k(\cdot | \pi_c, \widehat{P}^c)		
		\|_1
		\leq 
		\frac{4 \rho |S|}{1 - \rho^2 |S|^2}k
		+
		\beta \sum_{i=1}^{k-1}|S_i| \right]
		\\
		&=
		 \mathbb{P}_c
        \Bigg[\forall k \in [h]. \;\;\; 		
        \|
			q_k (\cdot | \pi_c, P^c) - q_k(\cdot | \pi_c, \widehat{P}^c)		
		\|_1
		\leq 
		\frac{4 \rho |S|}{1 - \rho^2 |S|^2}k
		+
		\beta \sum_{i=1}^{k-1}|S_i| 
	    \;
		\text{ and }
		\\
		 & \|
			q_{h+1} (\cdot | \pi_c, P^c) - q_{h+1}(\cdot | \pi_c, \widehat{P}^c)		
		\|_1
		\leq
		\frac{4 \rho |S|}{1 - \rho^2 |S|^2}(h+1)
		+
		\beta \sum_{i=1}^{h}|S_i| 
		\Bigg]
		\\
		&\geq
		\mathbb{P}_c
        \Bigg[\forall k \in [h]. \;\;\; 		
        \|
			q_k (\cdot | \pi_c, P^c) - q_k(\cdot | \pi_c, \widehat{P}^c)		
		\|_1
		\leq 
		\frac{4 \rho |S|}{1 - \rho^2 |S|^2}k
		+
		\beta \sum_{i=1}^{k-1}|S_i|\; \text{ and }
		\\
		\tag{Since $ \mathbb{P}_c\left[ 
            \|q_{h+1}(\cdot |\pi_c, P^c) - q_{h+1}(\cdot |\pi_c, \widehat{P}^c)\|_1 \leq (1) + (2)
       \right] = 1 $.}
		&  
        (1) + (2)
		\leq
		\frac{4 \rho |S|}{1 - \rho^2 |S|^2}(h+1)
		+
		\beta \sum_{i=1}^{h}|S_i|
		\Bigg]
		\\
		&\geq
		\tag{Since $ \mathbb{P}_c\left[ 
        (1) = \|q_h(\cdot |\pi_c, P^c) - q_h(\cdot |\pi_c, \widehat{P}^c)\|_1
       \right] = 1$.}
		\mathbb{P}_c
        \left[\forall k \in [h]. \;\;\; 		
        \|
			q_k (\cdot | \pi_c, P^c) - q_k(\cdot | \pi_c, \widehat{P}^c)		
		\|_1
		\leq 
		\frac{4 \rho |S|}{1 - \rho^2 |S|^2}k
		+
		\beta \sum_{i=1}^{k-1}|S_i|\; 
		\text{ and } 
        (2) \leq \frac{4 \rho |S|}{1 - \rho^2 |S|^2} +\beta |S_h|
		\right]
		\\
		\tag{By the induction hypothesis and equation~(\ref{prob:2-bounded-l-2})}
		&\geq
		1 - 
        \left(
            \frac{\epsilon_P}{\rho^2}
            |A|\sum_{i=0}^{h-1}|S_i||S_{i+1}|
            + 
            \gamma
            \sum_{i=1}^{h-1}|S_{i}|
            +
            \frac{\epsilon_P}{\rho^2}
            |A||S_h||S_{h+1}|
            + \gamma |S_h|
        \right)
        =
		1 - 
        \left(
            \frac{\epsilon_P}{\rho^2}
            |A|\sum_{i=0}^{h}|S_i||S_{i+1}|
            + 
            \gamma
            \sum_{i=1}^{h}|S_{i}|
        \right),
    \end{align*}
    as stated.
   \endgroup
\end{proof}

\begin{lemma}[expected value difference caused by dynamics approximation]\label{lemma: expected error from dynamics l_2 UCDD}
    
    Under the good events $G_1$, $G_2$ and $G_3$, for every context-dependent policy $\pi=(\pi_c)_{c \in \mathcal{C}}$ it holds that
    \begin{align*}
        \mathbb{E}_{c \sim \mathcal{D}}
        [|V^{\pi_c}_{\mathcal{M}(c)}(s_0) - V^{\pi_c}_{\widetilde{\mathcal{M}}(c)}(s_0)|]
        \leq 
	   \frac{4 \rho |S|}{1 - \rho^2 |S|^2}H^2
	   +
	   \beta|S|H
	   +
	   H|S|^2|A|\frac{\epsilon_P}{\rho^2}
	   +
	   \gamma H|S|,
    \end{align*}
    for $\rho \in [0, \frac{1}{|S|})$ and $\beta \geq 2H\frac{4 \rho |S|}{1 - \rho^2 |S|^2}$. 
\end{lemma}

\begin{proof}
    Recall that the true rewards function is not defined for $s_{sink}$, since $s_{sink} \notin S$. For the intermediate MDP $\widetilde{\mathcal{M}}(c)$ we extended  $r^c$ to $s_{sink}$ by defining 
    $ \forall a\in A.\;\; r^c(s_{sink},a)=0$ for every context $c \in \mathcal{C}$. 
    Since $P^c$ is also not defined for $s_{sink}$, we can simply omit $s_{sink}$, as the second equality in the following derivation shows.
    
    Consider the following derivation for any fixed $c \in \mathcal{C}$. (Later we will take the expectation over $c$.)
    \begingroup
    \allowdisplaybreaks
    \begin{align*}
        &|V^{\pi_c}_{\mathcal{M}(c)}(s_0) 
        - 
        V^{\pi_c}_{\widetilde{\mathcal{M}}(c)}(s_0)|
        \\
        = &
        \left|
            \sum_{h=0}^{H-1}
            \sum_{s_h \in S_h}
            \sum_{a_h \in A}
            q_h(s_h, a_h|\pi_c, P^c)\cdot r^c(s_h, a_h)
            - 
            \sum_{h=0}^{H-1}
            \sum_{s_h \in S_h \cup \{s_{sink}\}}
            \sum_{a_h \in A}
            q_h(s_h, a_h|\pi_c, \widehat{P}^c) \cdot
            r^c(s_h, a_h)
        \right|
        \\
        \tag{ $r^c(s_{sink},a):= 0,\; \forall c,a$}
        = &
        \left|
            \sum_{h=0}^{H-1}
            \sum_{s_h \in S_h}
            \sum_{a_h \in A}
            q_h(s_h, a_h|\pi_c, P^c)\cdot r^c(s_h, a_h)
            -
            \sum_{h=0}^{H-1}
            \sum_{s_h \in S_h}
            \sum_{a_h \in A}
            q_h(s_h, a_h|\pi_c, \widehat{P}^c) \cdot
            r^c(s_h, a_h)
        \right|
        \\
        = &
        \left|
            \sum_{h=0}^{H-1}
            \sum_{s_h \in S_h}
            \sum_{a_h \in A}
            (q_h(s_h, a_h|\pi_c, P^c) - q_h(s_h, a_h|\pi_c, \widehat{P}^c))
            r^c(s_h, a_h)
        \right|
        \\
        \leq &
        \sum_{h=0}^{H-1}
        \sum_{s_h \in S_h}
        \sum_{a_h \in A}
        \left|q_h(s_h, a_h|\pi_c, P^c) - q_h(s_h, a_h|\pi_c, \widehat{P}^c)\right|
        \underbrace{\left|r^c(s_h, a_h)\right|}_{\leq 1}
        \\
        \leq &
        \sum_{h=0}^{H-1}
        \sum_{s_h \in S_h}
        \sum_{a_h \in A}
        \left|q_h(s_h, a_h|\pi_c, P^c) - q_h(s_h, a_h|\pi_c, \widehat{P}^c)\right|
        \\
        = &
        \sum_{h=0}^{H-1}
        \sum_{s_h \in S_h}
        \sum_{a_h \in A}
        \pi_c(a_h | s_h)
        |q_h(s_h |\pi_c, P^c) - q_h(s_h|\pi_c, \widehat{P}^c)|
        \\
        = &
        \sum_{h=0}^{H-1}
        \sum_{s_h \in S_h}
        |q_h(s_h |\pi_c, P^c) - q_h(s_h|\pi_c, \widehat{P}^c)|         
        \underbrace{\sum_{a_h \in A}
        \pi_c(a_h | s_h)}_{=1}
       \\
       = &
        \sum_{h=0}^{H-1}
        \sum_{s_h \in S_h}
        |q_h(s_h |\pi_c, P^c) - q_h(s_h|\pi_c, \widehat{P}^c)|
        \\
        = &
        \sum_{h=0}^{H-1}
        \|q_h(\cdot |\pi_c, P^c) - q_h(\cdot|\pi_c, \widehat{P}^c)\|_1 .  
    \end{align*}
    \endgroup
    
    Denote by $G_8$ the good event
	\[
        \forall h \in [H]:
        \|q_h(\cdot |\pi_c, P^c) - q_h(\cdot|\pi_c, \widehat{P}^c)\|_1 
        \leq
        \frac{4 \rho |S|}{1 - \rho^2 |S|^2} h
		+
		\beta \sum_{i=1}^{h-1}|S_i|,
    \]
    and denote by $\overline{G_8}$ its complementary event.
    
    By Lemma~\ref{lemma: UCDD l_2 occ measure diff} we have
    \[
        \mathbb{P}[G_8] 
        \geq
        1 
        - 
        \left(
            \frac{\epsilon_P}{\rho^2}
            |A|\sum_{i=0}^{H-1}|S_i||S_{i+1}|
            + 
            \gamma
            \sum_{i=1}^{H-1}|S_{i}|
        \right)
        \geq
        1 - \left(|S|^2|A|\frac{\epsilon_P}{\rho^2} +|S|\gamma \right).
     \]

    If $G_8$ holds, then
    \begingroup
    \allowdisplaybreaks
	\begin{align*}
	   &\sum_{h=0}^{H-1}
	   \|q_h(\cdot |\pi_c, P^c) - q_h(\cdot|\pi_c, \widehat{P}^c)\|_1
	   \leq
	   \sum_{h=0}^{H-1}
	   \left(\frac{4 \rho |S|}{1 - \rho^2 |S|^2} h
	   +
	   \beta \sum_{i=1}^{h-1}|S_i|\right)
	   \\
	   &\leq
	   \sum_{h=0}^{H-1}
	   \left( \frac{4 \rho |S|}{1 - \rho^2 |S|^2}H
	   +
	   \beta|S|\right)
	   \leq
	   \frac{4 \rho |S|}{1 - \rho^2 |S|^2}H^2
	   +
	   \beta|S|H
	   .
	\end{align*}
	\endgroup
    
    Otherwise,
    \begingroup
    \allowdisplaybreaks
    \begin{align*}
        \sum_{h=0}^{H-1}
        \|q_h(\cdot |\pi_c, P^c) - q_h(\cdot|\pi_c, \widehat{P}^c)\|_1 
	   \leq
	   \sum_{h=0}^{H-1}
	   1
	   \leq
	   H.
    \end{align*}
    \endgroup
    Using total expectation low we obtain
    \begingroup
    \allowdisplaybreaks
    \begin{align*}
        &\mathbb{E}_{c \sim \mathcal{D}}
        \left[|V^{\pi_c}_{\mathcal{M}(c)}(s_0) - V^{\pi_c}_{\widetilde{\mathcal{M}}(c)}(s_0)|\right]
        \\
        \leq &
        \mathbb{P}\left[G_8\right]
        \mathbb{E}_{c \sim \mathcal{D}}
        \left[|V^{\pi_c}_{\mathcal{M}(c)}(s_0) - V^{\pi_c}_{\widetilde{\mathcal{M}}(c)}(s_0)| \;|G_8\;\right]
        +
        \mathbb{P}\left[\overline{G_8}\right]
        \mathbb{E}_{c \sim \mathcal{D}}
        \left[|V^{\pi_c}_{\mathcal{M}(c)}(s_0) - V^{\pi_c}_{\widetilde{\mathcal{M}}(c)}(s_0)|\;|\;\overline{G_8}\right]
        \\
        \leq &
        \mathbb{E}_{c \sim \mathcal{D}}
        \left[|V^{\pi_c}_{\mathcal{M}(c)}(s_0) - V^{\pi_c}_{\widetilde{\mathcal{M}}(c)}(s_0)|\;|\;G_8\right]
        +
        \mathbb{P}\left[\overline{G_8}\right]
        \mathbb{E}_{c \sim \mathcal{D}}
        \left[|V^{\pi_c}_{\mathcal{M}(c)}(s_0) - V^{\pi_c}_{\widetilde{\mathcal{M}}(c)}(s_0)|\;|\;\overline{G_8}\right]
        \\
        \leq &
	   \frac{4 \rho |S|}{1 - \rho^2 |S|^2}H^2
        +
	   \beta|S|H
	   +
	   \mathbb{P}\left[\overline{G_8}\right] H
	   \\
	   \leq &
	   \frac{4 \rho |S|}{1 - \rho^2 |S|^2}H^2
	   +
	   \beta|S|H
	   +
	   H|S|^2|A|\frac{\epsilon_P}{\rho^2}
	   +
	   \gamma H|S|
    \end{align*}
    \endgroup
which proves the lemma.	
\end{proof}

For out parameters choice, we obtain,
\begin{corollary}\label{corl: expected error of dynamics for the chosen parameters UCDD L_2}
    Under the good events $G_1$, $G_2$ and $G_3$,
    For $\epsilon_P =\frac{ \epsilon^3}{10\cdot 2^8 \cdot 20^2  |A| |S|^6 H^5}$,
    $\gamma = \frac{\epsilon}{20 |S| H}$
    $\beta = \frac{\epsilon}{20 |S|H}$,
    $\rho = \frac{\beta}{16 |S|H} $,
    we have that $\rho \in [0, \frac{1}{|S|})$,  $\beta \geq 2H\frac{4 \rho |S|}{1 - \rho^2 |S|^2}$ and
    \begin{align*}
        \mathbb{E}_{c \sim \mathcal{D}}
        [|V^{\pi_c}_{\mathcal{M}(c)}(s_0) - V^{\pi_c}_{\widetilde{\mathcal{M}}(c)}(s_0)|]
        \leq 
	   \frac{\epsilon}{40 |S|}
	   +
	   \frac{2\epsilon}{10}.
    \end{align*}     
\end{corollary}

\begin{proof}
    Implied by assigning the detailed parameters to the results of Lemma~\ref{lemma: expected error from dynamics l_2 UCDD}.
\end{proof}


\subsubsection{Analysis of the Error Caused by the Rewards Approximation Under the Good Events}\label{subsubsec:l-2-rewards-error}

Recall that for every context $c \in \mathcal{C}$, define the following two MDPs.
The intermediate MDP 
    $
        \widetilde{M}(c) 
        = 
        (S\cup \{s_{sink}\}, A, \widehat{P}^c, r^c)
    $ 
and  the approximated MDP
    $
        \widehat{M}(c)
        = 
        (S\cup \{s_{sink}\}, A, \widehat{P}^c,\widehat{r}^c)
    $,
where $r^c$ is the true rewards function extended to $s_{sink}$ by defining $r^c(s_{sink},a):= 0 ,\;\; \forall c \in \mathcal{C}, a \in A$. $\widehat{P}^c, \widehat{r}^c$ are the approximation of the dynamics and the rewards as defined in Algorithm~\ref{alg: EXPLOIT-UCDD}.

\begin{lemma}[expected value difference caused by rewards approximation]\label{lemma: expected error from rewards l_2 UCDD}
    
    Under the good events $G_1, G_2$ and $G_4$, for any context-dependent policy $\pi=(\pi_c)_{c \in \mathcal{C}}$ it holds that
    \begin{align*}
        \mathbb{E}_{c \sim \mathcal{D}}
        [|V^{\pi_c}_{\widetilde{M}(c)}(s_0) - V^{\pi_c}_{\widehat{M}(c)}(s_0)|]
        =
        \alpha_2 H 
        +
        2 (\epsilon_R |S||A|)^{\frac{1}{3}}H
        +
        \beta |S|
        +
        \gamma
        |S|H,
    \end{align*}
    where $\alpha^2_2 := \max_{h \in [H-1]}\alpha^2_2(\mathcal{F}^R_h)$.
\end{lemma}

\begin{proof}

    Recall that $\widehat{r}^c(s_{sink},a) : =0 , \forall c \in \mathcal{C}, a \in A$ by definition.
    In addition, since $r^c$ is the true rewards function and $s_{sink}\notin S$, $r^c$ is not defined for $s_{sink}$. We naturally extended it to $s_{sink}$ by defining $r^c(s_{sink},a) :=0 , \forall c \in \mathcal{C}, a \in A$. Hence, we can simply ignore $s_{sink}$ as the following computation shows.

    
    Let us recall the definition of the following subsets of $S_h$ for every $h \in [H-1]$ given any context $c \in \mathcal{C}$.
       \begin{enumerate}
       \item $B^{h, c}_1 = 
       \{s_h \in  S_h :
        s_h \in \widehat{S}^{\gamma,\beta}_h
        \text{ and } c \in \widehat{C}^{\beta}(s_h)\}$
       \item $B^{h, c}_2 = 
       \{ s_h \in  S_h :
       s_h \in \widehat{S}^{\gamma,\beta}_h
       \text{ and } c \notin \widehat{C}^{\beta}(s_h)\}$
       \item $B^{h, c}_3 = 
       \{s_h \in  S_h :
       s_h \notin \widehat{S}^{\gamma,\beta}_h
       \text{ and } c \notin \widehat{C}^{\beta}(s_h)\}$
       \item $B^{h, c}_4 = 
       \{ s_h \in  S_h :
       s_h \notin \widehat{S}^{\gamma,\beta}_h
       \text{ and } c \in \widehat{C}^{\beta}(s_h)\}$      
   \end{enumerate}
   Clearly, $\cup_{i=1}^4 B^h_i = S_h$.
   
    By definition $s_h \in B^{h, c}_1$ if and only if for every action $a_h \in A$ we have that $(c, s_h, a_h) \in  \mathcal{X}^{\gamma, \beta}_h $.

    For $s_h\not\in \widehat{S}^{\gamma,\beta}_h$ we have that $\mathbb{P}[c \in \widehat{\mathcal{C}}^{ \beta}(s_h)]<\gamma$, hence,
    \begin{align*}
        \mathbb{P}_c[\exists h \in [H-1] : B^{h, c}_4 \neq \emptyset]
        \;\;=\;\;
        \mathbb{P}_c[\exists h\in [H-1], s_h \in S_h: 
        s_h\not\in \widehat{S}^{\gamma,\beta}_h, \text{ and } c \in \widehat{\mathcal{C}}^{\beta}(s_h) ]
        \;\;<\;\;
        \gamma|S|.
    \end{align*}


    Fix a context-dependent policy $\pi=(\pi_c)_{c \in \mathcal{C}}$. The following holds for any given context $c$. (Later we will take the expectation over $c$).
    \begingroup
    \allowdisplaybreaks
    \begin{align*}
        |V^{\pi_c}_{\widetilde{M}(c)}(s_0) - V^{\pi_c}_{\widehat{M}(c)}(s_0)|
        &=
        \left|
            \sum_{h=0}^{H-1}
            \sum_{s_h \in S_h \cup \{s_{sink}\}}
            \sum_{a_h \in A}
            q_h(s_h,a_h|\pi_c, \widehat{P}^c)
            (r^c(s_h,a_h) - \widehat{r}^c(s_h,a_h))
        \right|
        \\
        \tag{By definition, $r^c(s_{sink},a) = \widehat{r}^c(s_{sink},a) = 0,\;\; \forall c \in \mathcal{C}, a \in A$.}
        &=
        \left|
            \sum_{h=0}^{H-1}
            \sum_{s_h \in S_h}
            \sum_{a_h \in A}
            q_h(s_h,a_h|\pi_c, \widehat{P}^c)
            (r^c(s_h,a_h) - \widehat{r}^c(s_h,a_h))
        \right|
        \\
        &\leq
        \sum_{h=0}^{H-1}
        \sum_{s_h \in S_h}
        \sum_{a_h \in A}
        q_h(s_h,a_h|\pi_c, \widehat{P}^c) 
        |r^c(s_h,a_h) - \widehat{r}^c(s_h,a_h)|
        \\
        &=
        \underbrace
        {
            \sum_{h=0}^{H-1}
            \sum_{s_h \in B^{h, c}_1}
            \sum_{a_h \in A}
            q_h(s_h,a_h|\pi_c,\widehat{P}^c) 
            |r^c(s_h,a_h) - \widehat{r}^c(s_h,a_h)|
        }_{(1)}
        \\
        &+
        \underbrace
        {
            \sum_{h=0}^{H-1}
            \sum_{s_h \in B^{h, c}_2 \cup B^{h, c}_3}
            \sum_{a_h \in A}
           q_h(s_h,a_h|\pi_c, \widehat{P}^c) 
            |r^c(s_h,a_h) - \widehat{r}^c(s_h,a_h)|
        }_{(2)}
        \\
        &+
        \underbrace
        {
            \sum_{h=0}^{H-1}
            \sum_{s_h \in B^{h, c}_4}
            \sum_{a_h \in A}
             q_h(s_h,a_h|\pi_c, \widehat{P}^c) 
            |r^c(s_h,a_h) - \widehat{r}^c(s_h,a_h)|
        }_{(3)}.
    \end{align*}
    \endgroup
    We bound $(1)$, $(2)$ and $(3)$ separately.
    
    For $(1)$,
    under the good events $G_1$, $G_2$ , $G_3$ and $G_4$, we have for all $h \in [H-1]$ that
    \[
        \mathbb{E}_{(c,s_h,a_h)\sim \mathcal{D}^R_h}[(f^R_h(c, s_h, a_h) - r^c(s_h,a_h))^2 -\alpha^2_2(\mathcal{F}^R_h)]
        \leq 
        \epsilon_R. 
    \]
    
    Since $\mathbb{E}_{ \mathcal{D}^R_h}[(f^R_h(c,s_h, a_h) - r^c(s_h,a_h))^2 ]\geq  \alpha^2_2(\mathcal{F}^R_h)$, for all $h \in [H-1]$  and $\xi \in (0,1]$ we obtain using Markov's inequality that \begingroup
    \allowdisplaybreaks
    \begin{align*}
        &\mathop{\mathbb{P}}_{(c,s_h,a_h)}
        [|f^R_h(c,s_h, a_h) - r^c(s_h,a_h)|
        \geq \sqrt{ \alpha^2_2(\mathcal{F}^R_{s_h, a_h}) + \xi^2}\;
        \Big|(c,s_h, a_h) \in \mathcal{X}^{\gamma, \beta}_h]
        =
        \\
        &=
        \mathop{\mathbb{P}}_{(c,s_h,a_h)}
        [(f^R_h(c,s_h, a_h) - r^c(s_h,a_h))^2 
        -  \alpha^2_2(\mathcal{F}^R_{s_h, a_h})
        \geq \xi^2\;\Big|(c,s_h, a_h) \in \mathcal{X}^{\gamma, \beta}_h]
        \\
        &\leq
        \frac{\mathbb{E}_{ \mathcal{D}^R_h}[(f^R_h(c,s_h, a_h) - r^c(s_h,a_h))^2 - \alpha^2_2(\mathcal{F}^R_h) ]}{\xi^2}
        \leq
        \frac{\epsilon_R}{\xi^2}.
    \end{align*}
    \endgroup
    Since $\sqrt{a+b}\leq \sqrt{a}+\sqrt{b}$, for $a,b \in [0,1]$, it holds that
    \begingroup
    \allowdisplaybreaks
    \begin{align*}
        &\mathbb{P}
        \left[\left|f^R_h(c,s_h, a_h) - r^c(s_h,a_h)\;\;\right|\;\;
        \leq \alpha_2(\mathcal{F}^R_{s_h, a_h}) + \xi
        \Big|(c,s_h, a_h) \in \mathcal{X}^{\gamma, \beta}_h\right]
        \\
        &\geq
        \mathbb{P}
        \left[\left|f^R_h(c,s_h, a_h) - r^c(s_h,a_h)\right|\;\;
        \leq \sqrt{ \alpha^2_2(\mathcal{F}^R_{s_h, a_h}) + \xi^2}
        \;\;\Big|\;\;(c,s_h, a_h) \in \mathcal{X}^{\gamma, \beta}_h\right]
        \geq
        1 - \frac{\epsilon_R}{\xi^2}.
    \end{align*}
    \endgroup

    Let $G_5$ denote the following good event.
    \begingroup
    \allowdisplaybreaks
    \begin{align*}
        \forall h \in [H-1]
        \;
        \forall s_h \in B^{h, c}_1
        \;
        \forall a \in A:
        |f^R_h(c, s_h, a_h) - r^c(s_h,a_h) | \leq  \alpha_2(\mathcal{F}^R_{s_h, a_h}) + \xi
    \end{align*}
    \endgroup
    and denote by $\overline{G_5}$ the complementary event.
    By the above and union bound over $h \in [H-1]$ and $(s_h, a_h) \in B^{h,c}_1 \times A$ it holds that
    $
        \mathbb{P}_c[G_5] \geq 1 -  \frac{\epsilon_R}{\xi^2}|S||A|
    $
    and
    $
        \mathbb{P}_c[\overline{G_5}] \leq \frac{\epsilon_R}{\xi^2}|S||A|
    $.

    If $G_5$ holds then,
    \begingroup
    \allowdisplaybreaks
    \begin{align*}
        (1)
        &=
        \sum_{h=0}^{H-1}
        \sum_{s_h \in B^{h, c}_1}
        \sum_{a_h \in A}
        q_h(s_h,a_h|\pi_c, \widehat{P}^c) 
        \left|r^c(s_h,a_h) - \widehat{r}^c(s_h,a_h)\right|
        \\
        &=
        \sum_{h=0}^{H-1}
        \sum_{s_h \in B^{h, c}_1}
        \sum_{a_h \in A}
        \pi_c(a_h | s_h)
        q_h(s_h |\pi_c, \widehat{P}^c) 
        \left|f^R_h(c,s_h,a_h)- r^c(s_h,a_h)\right|  
        \\
        &\leq
        \sum_{h=0}^{H-1}
        \sum_{s_h \in B^{h, c}_1}
        \sum_{a_h \in A}
        \pi_c(a_h | s_h)
        q_h(s_h |\pi_c, \widehat{P}^c) 
        (\alpha_2(\mathcal{F}^R_h) + \xi)
        \\
        &\leq
        \sum_{h=0}^{H-1}
        \sum_{s_h \in B^{h, c}_1}
        \sum_{a_h \in A}
        \pi_c(a_h | s_h)
        q_h(s_h |\pi_c, \widehat{P}^c) 
        (\alpha_2  + \xi)
        \leq
        \alpha_2 H + \xi H .
    \end{align*}
    \endgroup
    Otherwise,
    \begin{align*}
        (1)
        &=
        \sum_{h=0}^{H-1}
        \sum_{s_h \in B^{h, c}_1}
        \sum_{a_h \in A}
        q_h(s_h,a_h|\pi_c, \widehat{P}^c) 
        \underbrace{\left|r^c(s_h,a_h) - \widehat{r}^c(s_h,a_h)\right|}_{\leq 1}
        \leq H .
    \end{align*}
 
    Thus,
    \begin{align*}
        \mathbb{E}_{c \sim \mathcal{D}}[(1)]
        &\leq
        \alpha_2 H + \xi H
        +
        \frac{\epsilon_R}{\xi^2} |S||A|H.
    \end{align*}

    For $(2)$, consider the following derivation:
    \begingroup
    \allowdisplaybreaks
    \begin{align*}
        (2)
        &=
        \sum_{h=0}^{H-1}
        \sum_{s_h \in B^{h, c}_2 \cup B^{h, c}_3}
        \sum_{a_h \in A}
        q_h(s_h,a_h|\pi_c, \widehat{P}^c) 
        \underbrace{\left|r^c(s_h,a_h) - \widehat{r}^c(s_h,a_h)\right|}_{\leq 1}
        \\
        &\leq
        \sum_{h=0}^{H-1}
        \sum_{s_h \in B^{h, c}_2 \cup B^{h, c}_3}
        \sum_{a_h \in A}
        q_h(s_h,a_h|\pi_c, \widehat{P}^c) 
        \\
        &=
        \sum_{h=0}^{H-1}
        \sum_{s_h \in B^{h, c}_2 \cup B^{h, c}_3}
        \sum_{a_h \in A}
        \pi_c(a_h | s_h)
        q_h(s_h|\pi_c, \widehat{P}^c) 
        \\
        &=
        \sum_{h=0}^{H-1}
        \sum_{s_h \in B^{h, c}_2 \cup B^{h, c}_3}
        q_h(s_h|\pi_c, \widehat{P}^c) 
        \underbrace{\sum_{a_h \in A}
         \pi_c(a_h | s_h)}_{=1}
         \\
         &=
        \sum_{h=0}^{H-1}
        \sum_{s_h \in B^{h, c}_2 \cup B^{h, c}_3}
        \underbrace{q_h(s_h|\pi_c, \widehat{P}^c) )}_{\leq q_h(s_h|\widehat{\pi}^c_{s_h}, \widehat{P}^c) <\beta}
        \leq
        \beta|S|.
    \end{align*}
    \endgroup
    Thus,
    \[
        \mathbb{E}_{c \sim \mathcal{D}}[(2)] \leq \beta|S|.
    \]
    
  For $(3)$, let $G_6$ denote the good event in which $\forall h \in [H -1 ], B^{h, c}_4 = \emptyset$. Denote by $\overline{G_6}$ the complement event of $G_6$.
  
  We showed that $\mathbb{P}_c[G_6] \geq 1 - \gamma|S|$ thus $\mathbb{P}_c[\overline{G_6}] \leq \gamma|S|$.
  
  If $G_6$ holds, then $(3) =0$. Otherwise,
  \begingroup
  \allowdisplaybreaks
  \begin{align*}
    (3)
    &=
    \sum_{h=0}^{H-1}
    \sum_{s_h \in B^{h, c}_4}
    \sum_{a_h \in A}
    q_h(s_h,a_h|\pi_c, \widehat{P}^c) 
    \underbrace{\left|r^c(s_h,a_h) - \widehat{r}^c(s_h,a_h)\right|}_{\leq 1}
    \\
    &\leq
    \sum_{h=0}^{H-1}
    \underbrace
    {
        \sum_{s_h \in B^{h, c}_4}
        \sum_{a_h \in A}
       q_h(s_h,a_h|\pi_c, \widehat{P}^c)
    }_{\leq 1}
    \leq H.
  \end{align*}
  \endgroup

  Using total expectation we obtain
  \begin{align*}
      \mathbb{E}_{c \sim \mathcal{D}}[(3)]
      &=
      \mathbb{P}[G_6]\mathbb{E}_{c \sim \mathcal{D}}[(3)|G_6]
      +
      \mathbb{P}[\overline{G_6}]\mathbb{E}_{c \sim \mathcal{D}}[(3)|\overline{G_5}]
      \\
      &\leq
      1 \cdot 0 
      + 
      \gamma|S|\cdot H
      \\
      &=
      \gamma|S|H.
  \end{align*}

    Overall,
    by linearity of expectation and the above, we obtain for $\xi = (\epsilon_R |S|\;|A|)^\frac{1}{3}$ that
    \begingroup\allowdisplaybreaks
    \begin{align*}
        &\mathbb{E}_{c \sim \mathcal{D}}
        [|V^{\pi_c}_{\widetilde{M}(c)}(s_0) - V^{\pi_c}_{\widehat{M}(c)}(s_0)|]
        \\
        &\leq
        \mathbb{E}_{c \sim \mathcal{D}}[(1)] + \mathbb{E}_{c \sim \mathcal{D}}[(2)] + \mathbb{E}_{c \sim \mathcal{D}}[(3)]
        \\
        &\leq
        \alpha_2 H 
        + 
        \xi H
        +
        \frac{\epsilon_R}{\xi^2} |S||A|H
        +
        \beta |S|
        +
        \gamma |S|H\\
        &=
        \alpha_2 H 
        +
        2 (\epsilon_R |S||A|)^{\frac{1}{3}}H
        +
        \beta |S|
        +
        \gamma
        |S|H,
    \end{align*}
    \endgroup
as stated.
\end{proof}

\begin{corollary}\label{corl: expected error of rewards for the chosen parameters UCDD L_2}
    Under the good events $G_1, G_2$ and $G_4$, 
    for $\gamma = \frac{\epsilon}{20 |S|H}$, $\beta =  \frac{\epsilon}{20|S|H}$ 
    and $\epsilon_R = \frac{\epsilon^3}{20^3 |S||A| H^3}$, we have for any context-dependent policy $\pi=(\pi_c)_{c \in \mathcal{C}}$ that
    \begin{align*}
        \mathbb{E}_{c \sim \mathcal{D}}
        [|V^{\pi_c}_{\widetilde{M}(c)}(s_0) - V^{\pi_c}_{\widehat{M}(c)}(s_0)|]
        \leq
        \alpha_2 H 
        + 
        \frac{3\epsilon}{20}
        +
        \frac{\epsilon}{20 H}
    \end{align*}
\end{corollary}

\begin{proof}
    Implied by assigning the detailed parameters to the results of Lemma~\ref{lemma: expected error from rewards l_2 UCDD}.
\end{proof}


    

\subsubsection{Combining Value Differences Caused By Dynamics and Rewards Approximation to a Sub-optimality Bound}\label{subsubsec:l-2-suboptimality-bound}

Let SP2 denote the following parameters set.
\begin{itemize}
    \item $\gamma = \frac{\epsilon}{20 |S|H} \in (0,1)$.
    \item $\beta =  \frac{\epsilon}{20|S|H} \in (0,1)$.
    \item $\rho = \frac{\beta}{16 |S|H} \in (0, \frac{1}{|S|})$.
    \item $\epsilon_P = \frac{ \epsilon^3}{10\cdot 2^8 20^2  |A| |S|^6 H^5}$.
    \item $\epsilon_R = \frac{\epsilon_3}{20^3 H^4}$.
\end{itemize}
We remark that for our choice of $\rho$ and $\beta$ it holds that $\rho \in [0,\frac{1}{|S|})$ and $\beta \geq 2H\frac{4 \rho |S|}{1 - \rho^2 |S|^2}$.

\begin{lemma}[expected value difference]\label{lemma: expected gap for general pi l_2}
    Under the good events $G_1$, $G_2$,$G_3$ and $G_4$, for every context-dependent policy $\pi=(\pi_c)_{c \in \mathcal{C}}$ it holds that,
    \[
        \mathbb{E}_{c \sim \mathcal{D}}
        [| V^{\pi_c}_{\mathcal{M}(c)}(s_0) -  V^{\pi_c}_{\widehat{\mathcal{M}}(c)}(s_0)|]
        \leq
        \alpha_2 H + \frac{1}{2} \epsilon,
    \]
    where $\mathcal{M}(c)$ is the true MDP associated with the context $c$ and $\widehat{\mathcal{M}}(c)$ is it's the approximated model,
    for the parameters set SP2.
\end{lemma}

\begin{proof}
    For any context $c \in \mathcal{C}$, consider the intermediate MDP $\widetilde{\mathcal{M}}(c) = (S, A, \widehat{P}^c, r^c, H, s_0)$.
    Using triangle inequality and linearity of expectation we obtain
    \begin{align*}
        \mathbb{E}_{c \sim \mathcal{D}}
        [
            | V^{\pi_c}_{\mathcal{M}(c)}(s_0) 
            -  
            V^{\pi_c}_{\widehat{\mathcal{M}}(c)}(s_0)|
        ]
        &=
        \mathbb{E}_{c \sim \mathcal{D}}
        [| 
            V^{\pi_c}_{\mathcal{M}(c)}(s_0) 
            - 
            V^{\pi_c}_{\widetilde{\mathcal{M}}(c)}(s_0)
            +
            V^{\pi_c}_{\widetilde{\mathcal{M}}(c)}(s_0)
            -
            V^{\pi_c}_{\widehat{\mathcal{M}}(c)}(s_0)
        |]
        \\
        &\leq
        \mathbb{E}_{c \sim \mathcal{D}}
        [
            \underbrace
            {| V^{\pi_c}_{\mathcal{M}(c)}(s_0) 
            - 
            V^{\pi_c}_{\widetilde{\mathcal{M}}(c)}(s_0)|}_{(1)}
        ]
        +
        \mathbb{E}_{c \sim \mathcal{D}}
        [
            \underbrace
            {| V^{\pi_c}_{\widetilde{\mathcal{M}}(c)}(s_0)
            -
            V^{\pi_c}_{\widehat{\mathcal{M}}(c)}(s_0)|}_{(2)}
        ]
    \end{align*}
    
    By Lemma~\ref{lemma: expected error from dynamics l_2 UCDD} we have
    \begin{align*}
        \mathbb{E}_{c \sim \mathcal{D}}
        [|V^{\pi_c}_{\mathcal{M}(c)}(s_0) - V^{\pi_c}_{\widetilde{\mathcal{M}}(c)}(s_0)|]
        \leq 
        \frac{4 \rho |S|}{1 - \rho^2 |S|^2}H^2
        +
        |S|^2|A|H\frac{\epsilon_P}{\rho^2}
        + 
        \gamma
        |S|H
        +
	   \beta|S|H.
    \end{align*}
    
    By Lemma~\ref{lemma: expected error from rewards l_2 UCDD} we have
    \begin{align*}
        \mathbb{E}_{c \sim \mathcal{D}}
        [|V^{\pi_c}_{\widetilde{M}(c)}(s_0) - V^{\pi_c}_{\widehat{M}(c)}(s_0)|]
        \leq
        \alpha_2 H 
        +
        2 (\epsilon_R |S||A|)^{\frac{1}{3}}H
        +
        \beta |S|
        +
        \gamma
        |S|H.
    \end{align*}
    
    Overall,
    \begin{align*}
        \mathbb{E}_{c \sim \mathcal{D}}
        [| 
            V^{\pi_c}_{\mathcal{M}(c)}(s_0) 
            -  
            V^{\pi_c}_{\widehat{\mathcal{M}}(c)}(s_0)
        |]
        &=
	    \frac{4 \rho |S|}{1 - \rho^2 |S|^2}H^2
        +
        |S|^2|A|H\frac{\epsilon_P}{\rho^2}
        + 
        2\gamma|S|H
        +
        2\beta|S|H
        +
        \alpha_2 H 
        +
        2 (\epsilon_R |S||A|)^{\frac{1}{3}}H
    \end{align*}
For the parameters set SP2 we have, 
$\gamma = \frac{\epsilon}{20 |S|H} \in (0,1)$, 
$\beta =  \frac{\epsilon}{20|S|H} \in (0,1)$, 
$\rho = \frac{\beta}{16 |S|H} \in (0, \frac{1}{|S|})$, 
$\epsilon_P = \frac{ \epsilon^3}{10\cdot 2^8 20^2  |A| |S|^6 H^5}$, 
$\epsilon_R = \frac{\epsilon^3}{20^3 |S||A| H^3}$.

In addition it holds that $\beta < \frac{1}{2 |S|}$, which implies that $0 < \rho < \frac{1}{|S|}$.

We also have that
\begin{align*}
    2H\frac{4 \rho |S|}{1- \rho^2 |S|^2} 
    =
    \frac{8H |S|\frac{\beta}{16 |S|H}}{1- \frac{\beta^2 |S|^2}{2^8 |S|^2 H^2}}
    =
    \frac{\frac{\beta}{2}}{1 - \underbrace{\frac{\beta^2}{2^8 H^2}}_{\leq 1/2}}
    \leq 2\frac{\beta}{2} = \beta.
\end{align*}
Hence, the constrains on $\rho$ and $\beta$ are satisfied. 

Finally, 
    \begingroup
    \allowdisplaybreaks
    \begin{align*}
        \mathbb{E}_{c \sim \mathcal{D}}
        [| 
            V^{\pi_c}_{\mathcal{M}(c)}(s_0) 
            -  
            V^{\pi_c}_{\widehat{\mathcal{M}}(c)}(s_0)
        |]
        &=
	    \frac{4 \rho |S|}{1 - \rho^2 |S|^2}H^2
        +
        |S|^2|A|H\frac{\epsilon_P}{\rho^2}
        + 
        2\gamma|S|H
        +
        2\beta|S|H
        +
        \alpha_2 H 
        +
        2 (\epsilon_R |S||A|)^{\frac{1}{3}}H
        \\
        &\leq
        \frac{\frac{1}{4} \beta H}{1 - \underbrace{\frac{\beta^2}{2^8 H^2}}_{\leq 1/2}}
        +
        2^8|S|^4|A|H^3
        \frac{\epsilon_P}
        {\beta^2}  
        +
        2 \frac{\epsilon}{10}
        + 
        \alpha_2 H
        +
        \frac{\epsilon}{10}
        \\
        &\leq
        \frac{1}{2}\beta H
        +
        2^8 20^2 |S|^6|A|H^5
        \frac{\epsilon_P}
        {\epsilon^2} 
        +
        3 \frac{\epsilon}{10}
        + 
        \alpha_2 H
        \\
        &=
        \frac{1}{2}\frac{\epsilon}{20|S|}
        +
        4 \frac{\epsilon}{10}
        + 
        \alpha_2 H
        \\
        &\leq \frac{1}{2}\epsilon + \alpha_2 H,
    \end{align*}
    \endgroup
as stated.    
\end{proof}

The following corollary shows that for our choice of parameters, all good events holds with high probability. 
\begin{corollary}\label{corl: final probs l_2}
    For the parameters set SP2 it holds that ${\mathbb{P}[
    G_1, G_2,G_3, G_4
    ] \geq 1- (\frac{\delta}{2}+ \frac{\epsilon}{10})}$.
\end{corollary}

\begin{proof}
    By Corollary~\ref{corl: good events probs bound l_2} it holds that $
    \mathbb{P}[
    G_1, G_2,G_3, G_4] \geq 1- (\frac{\delta}{4} + 3 \delta_1 H + \frac{\epsilon_P}{\rho^2}|S|^2|A|H)$. Hence by $\rho$, $\beta$, $\epsilon_P$ and $\delta_1$ choice we obtain
    \begin{align*}
        \mathbb{P}[\cap_{i \in [4]}G_i] 
        &\geq 
        1- \left(\frac{\delta}{8} + 3 \delta_1 H + \frac{\epsilon_P}{\rho^2}|S|^2|A|H\right)
        \\
        &=
        1- \frac{\delta}{2} - \frac{\epsilon_P}{\beta^2}2^8 |S|^4|A|H^3
        \\
        &=
        1- \frac{\delta}{2} - 2^8 20^2 |S|^6|A|H^5\frac{\epsilon_P}{\epsilon^2} 
        \\
        &=
        1- \frac{\delta}{2} -\frac{\epsilon}{10}.
    \end{align*}
\end{proof}

Finally, the following theorem bound the expected sub-optimality of our approximated optimal policy $\widehat{\pi}^\star$. 
\begin{theorem}[expected sub-optimality bound]\label{thm: UCDD opt policy l_2}
With probability at least $1- (\delta + \frac{\epsilon}{5})$ it holds that
    \[
        \mathbb{E}_{c \sim \mathcal{D}}
        [V^{\pi^\star_c}_{\mathcal{M}(c)}(s_0) - V^{\widehat{\pi}^\star_c}_{\mathcal{M}(c)}(s_0)]
        \leq 
        \epsilon  + 2\alpha_2 H,
    \]
where $\pi^\star = (\pi^\star_c)_{c \in \mathcal{C}}$ is the optimal context-dependent policy for  $\mathcal{M}$ and $\widehat{\pi}^\star=(\widehat{\pi}^\star_c)_{c \in \mathcal{C}}$ is the optimal context-dependent policy for $\widehat{\mathcal{M}}$.
\end{theorem}

\begin{proof}
Assume the good events $G_1$, $G_2$, $G_3$ and $G_4$ hold.

Then, by Lemma~\ref{lemma: expected gap for general pi l_2}, we have for $\pi^\star$
    \begin{align*}
        \left|
        \mathbb{E}_{c \sim \mathcal{D}}
        [
            V^{\pi^\star_c}_{\mathcal{M}(c)}(s_0) 
            -  V^{\pi^\star_c}_{\widehat{\mathcal{M}}(c)}(s_0)
        ] 
        \right|
        \leq
         \mathbb{E}_{c \sim \mathcal{D}}
        [|
            V^{\pi^\star_c}_{\mathcal{M}(c)}(s_0) 
            -  V^{\pi^\star_c}_{\widehat{\mathcal{M}}(c)}(s_0)
        |] 
        \leq \frac{1}{2}\epsilon + \alpha_2 H ,      
    \end{align*}
    yielding,
    \begin{align*}
        \mathbb{E}_{c \sim \mathcal{D}}
        [V^{\pi^\star_c}_{\mathcal{M}(c)}(s_0)] 
        -
        \mathbb{E}_{c \sim \mathcal{D}}
        [V^{\pi^\star_c}_{\widehat{\mathcal{M}}(c)}(s_0)]
        \leq 
        \frac{1}{2}\epsilon + \alpha_2 H.
    \end{align*}

    

    Similarly, we obtain for $\widehat{\pi}^\star_c$ that
    \begin{equation*}
        \mathbb{E}_{c \sim \mathcal{D}}
        [V^{\widehat{\pi}^\star_c}_{\widehat{\mathcal{M}}(c)}(s_0)]
        -
        \mathbb{E}_{c \sim \mathcal{D}}
        [V^{\widehat{\pi}^\star_c}_{\mathcal{M}(c)}(s_0)] 
        \leq
        \frac{1}{2}\epsilon + \alpha_2 H.
\end{equation*}

Since for all $c \in \mathcal{C}$, $\widehat{\pi}^\star_c$ is the optimal policy for $\widehat{\mathcal{M}}(c)$ we have 
$ 
    V^{\widehat{\pi}^\star_c}_{\widehat{\mathcal{M}}(c)}(s_0)
    \geq V^{\pi^\star_c}_{\widehat{\mathcal{M}}(c)}(s_0)
$ 
which implies that
    \begin{equation*}
        \mathbb{E}_{c \sim \mathcal{D}}
        [V^{\pi^\star_c}_{\widehat{\mathcal{M}}(c)}(s_0)]
        -
        \mathbb{E}_{c \sim \mathcal{D}}
        [V^{\widehat{\pi}^\star_c}_{\widehat{\mathcal{M}}(c)}(s_0)]
        \leq
        0.
\end{equation*}
Since by Corollary~\ref{corl: final probs l_2} we have that $\mathbb{P}[
G_1, G_2,G_3, G_4] \geq 1- (\frac{\delta}{2}+ \frac{\epsilon}{10})$, 
the theorem implied by summing the above three inequalities.
\end{proof}

\subsubsection{Additional Lemmas for bounding the sample complexity for the \texorpdfstring{$\ell_2$}{Lg} loss}

\begin{lemma}\label{lemma:UCDD-l_2-matrix-diff-ver-2}
    Let $\rho \in [0,\frac{1}{|S|})$ and $h \in [H -1]$.
    Assume the good events $G_1, G_2^k, G_3^k, \; \forall k \in [h]$ hold, then we have 
    \begin{align*}
        \mathbb{P}
        \left[ 
            \|\widehat{P}^c(\cdot| s_h, a_h) - P^c(\cdot| s_h, a_h)\|_1   
            \leq \frac{4 \rho |S|}{ 1 - \rho^2 |S|^2}
        \Big|
        (c, s_h, a_h) \in  \widetilde{\mathcal{X}}^{\gamma, \beta}_h
        \right]
        \geq
        1- \frac{\epsilon_P}{\rho^2}|S_{h+1}|,
    \end{align*}
    where $\widehat{P}^c$ is the dynamics defined in Algorithm~\ref{alg: ACDD UCDD} and 
    \[
        \|\widehat{P}^c(\cdot| s_h, a_h) - P^c(\cdot| s_h, a_h)\|_1 :=
        \sum_{s_{h+1} \in S_{h+1}} 
        |\widehat{P}^c(s_{h+2}| s_h, a_h)
        -
        P^c(s_{h+1}| s_h, a_h)
        |
    \]
    (i.e., the entry of $s_{sink}$ is in $\widehat{P}$ is ignored).   
\end{lemma}

\begin{proof}
    We prove similarly to shown for Lemma~\ref{lemma: UCDD l_2 matrix diff}, when using the good events $G_3^k$ for all $k \in [h]$ guarantees for the distribution $\widetilde{D}^{\gamma,\beta}_h$ over $\widetilde{\mathcal{X}}^{\gamma,\beta}_h \times S_{h+1}$.

    Recall that for $(c, s_h, a_h) \in \widetilde{\mathcal{X}}^{\gamma,\beta}_h$ we have that $\widehat{P}^c(s_{sink}|s_h, a_h) = 0$ by $\widehat{P}^c$ definition. 
    In addition, the true dynamics $P^c$ is not defined for $s_{sink}$ since $s_{sink} \notin S$.  
    A natural extension of $P^c$ to $s_{sink}$ is by defining that $\forall (s,a) \in S \times A.\;\; P^c(s_{sink}|s,a) :=0$.
    By that extension, we have for all $(c, s_h, a_h) \in \widetilde{\mathcal{X}}^{\gamma,\beta}_h$ that
    $P^c(s_{sink}|s_h,a_h) = \widehat{P}^c(s_{sink}|s_h,a_h)=0$. Hence, we can simply ignore $s_{sink}$ in the following analysis.
    
    Under the good event $ G_3^h $,
    it holds that
    %
    \begin{align*}
        &\mathop{\mathbb{P}}_{(c,s_h,a_h,s_{h+1})}
        \left[ |f^P_h(c, s_h, a_h, s_{h+1}) - P^c(s_{h+1}| s_h, a_h)| \geq \rho \Big| (c,s_h,a_h) \in \widetilde{\mathcal{X}}^{\gamma,\beta}_h \right]
        = 
        \\
        & =
        \mathbb{P}_{\widetilde{\mathcal{D}}^P_h}[|f^P_h(c, s_h, a_h, s_{h+1}) - P^c(s_{h+1}| s_h, a_h)| \geq \rho 
        ]
        \\
        & = 
        \mathbb{P}_{\widetilde{\mathcal{D}}^P_h}[(f^P_h(c, s_h, a_h, s_{h+1}) - P^c(s_{h+1}| s_h, a_h))^2 \geq \rho^2
        ]
        \\
        \tag{By Markov's inequality}
        & \leq  
        \frac{\mathbb{E}_{\widetilde{\mathcal{D}}^P_h}
        [(f^P_h(c, s_h, a_h, s_{h+1}) - P^c(s_{h+1}| s_h, a_h))^2]}{\rho^2}
        \\
        \tag{Under $G^3_h$}
        & \leq 
        \frac{\epsilon_P}{\rho^2}.
    \end{align*}
    
    Hence,
    \[
        \mathop{\mathbb{P}}_{(c,s_h,a_h,s_{h+1})}
        [|f^P_h(c, s_h, a_h, s_{h+1}) - P^c(s_{h+1}| s_h, a_h)| \leq \rho]
        \geq 1 - \frac{\epsilon_P }{\rho^2}.
    \]
    
    As $P^c(\cdot |s_h, a_h)$ is a distribution, we have for every context $c$ that $\sum_{s_{h+1} \in S_{h+1}} P^c(s_{h+1}|s_h, a_h) = 1$. 
    
    Thus, by union bound over $s_{h+1}\in S_{h+1}$
    we obtain
    \begin{align*}
        \mathbb{P}_{(c, s_h, a_h) }
        \left[ 
            1 - \rho |S|\leq
            \sum_{s_{h+1}\in S_{h+1}} f^P_h(c, s_h, a_h, s_{h+1}) \leq
            1 + \rho |S|
        \Big|
        (c, s_h, a_h) \in  \widetilde{\mathcal{X}}^{\gamma, \beta}_h    
        \right]
        \geq 
        1 - \frac{\epsilon_P }{\rho^2}|S_{h+1}|.
    \end{align*}
   Hence, we further conclude that  
    \begin{equation}\label{prob: lemma C.5-2}
        \begin{split}
            &\mathop{\mathbb{P}}_{(c,s_h,a_h)}
            \left[  
                \forall s_{h+1} \in S_{h+1}.\;
                \frac{P^c(s_{h+1}|s_h, a_h) - \rho}{1 + \rho |S|}  \leq
                \underbrace{\frac{f^P_h(c,s_h,a_h,s_{h+1})}{\sum_{s'\in S_{h+1}} f^P_h(c,s_h,a_h,s')}}_{=\widehat{P}^c(s_{h+1}|s_h, a_h)}
                \leq
                \frac{P^c(s_{h+1}|s_h, a_h) + \rho}{1 - \rho |S|}
            \;\Big|
            (c, s_h, a_h) \in \widetilde{\mathcal{X}}^{\gamma, \beta}_h   \right]
            \\
            & \geq 
            1 - \frac{\epsilon_P }{\rho^2}|S_{h+1}|.
        \end{split}
    \end{equation}
    
    For any fixed tuple $(c, s_h, a_h) \in \widetilde{\mathcal{X}}^{\gamma, \beta}_h$ denote $S^+_{h+1} = \{s_{h+1} \in S_{h+1}:
    \widehat{P}^c(s_{h+1}|s_h, a_h) \geq  P^c(s_{h+1}|s_h, a_h )\}$ and consider the following derivation:
    \begingroup
    \allowdisplaybreaks
    \begin{align*}
        \| 
            \widehat{P}^c(\cdot|s_h, a_h ) 
            -  
            P^c(\cdot|s_h, a_h )
        \|_1 
        =&
        \sum_{s_{h+1} \in S_{h+1}}
        |\widehat{P}^c(s_{h+1}|s_h, a_h) -  P^c(s_{h+1}|s_h, a_h )|
        \\
        =&
        \sum_{s_{h+1} \in S^+_{h+1}}
        (\widehat{P}^c(s_{h+1}|s_h, a_h) -  P^c(s_{h+1}|s_h, a_h ))       
        \\
        &+
        \sum_{s_{h+1} \in  S_{h+1} \setminus S^+_{h+1}}
        (P^c(s_{h+1}|s_h, a_h) -  \widehat{P}^c(s_{h+1}|s_h, a_h ))
        \\  
        \leq&
        \sum_{s_{h+1} \in S^+_{h+1}}
        \left(\frac{P^c(s_{h+1}|s_h, a_h) + \rho}{1 - \rho |S|} 
        -  
        P^c (s_{h+1}|s_h, a_h )\right)
        \\
        &+
        \sum_{s_{h+1} \in  S_{h+1} \setminus S^+_{h+1}}
        \left( P^c (s_{h+1}|s_h, a_h )
        -
        \frac{P^c(s_{h+1}|s_h, a_h) - \rho}{1 + \rho |S|} \right) 
        \\
        =&
        \sum_{s_{h+1} \in S^+_{h+1}}
        \frac{P^c (s_{h+1}|s_h, a_h) + \rho - (1- \rho |S|)P^c (s_{h+1}|s_h, a_h) }{1 - \rho |S|} 
        \\
        &+
        \sum_{s_{h+1} \in  S_{h+1} \setminus S^+_{h+1}}
        \frac{-P^c (s_{h+1}|s_h, a_h) + \rho + (1 + \rho |S|)P^c (s_{h+1}|s_h, a_h) }{1 + \rho |S|}         
        \\
        =&
        \frac{1}{1- \rho |S|}
        \sum_{s_{h+1} \in S^+_{h+1}}
        (\rho + \rho|S|P^c (s_{h+1}|s_h,a_h))
        \\
        &+
        \frac{1}{1 + \rho |S|}
        \sum_{s_{h+1} \in  S_{h+1} \setminus S^+_{h+1}}
        (\rho + \rho|S|P^c (s_{h+1}|s_h,a_h))
        \\
        \leq&
        \frac{2 \rho |S|}{ 1- \rho |S|}
        +
        \frac{2 \rho |S|}{ 1 + \rho |S|}
        \\
        =&
        \frac{4 \rho |S|}{1 - \rho^2 |S|^2}.
    \end{align*}
    \endgroup
    By inequality~\ref{prob: lemma C.5-2}, the above holds with probability at least $1 - \frac{\epsilon_P }{\rho^2}|S_{h+1}|$ over $(c, s_h, a_h) \in \widetilde{\mathcal{X}}^{\gamma,\beta}_h$. Hence the lemma follows.

\end{proof}

\begin{lemma}\label{lemma: for running time l_2}
    Fix $\beta \in (0,1]$ and $\rho \in [0,\frac{1}{|S|})$ such that $\beta \geq 2H \frac{4 \rho |S|}{ 1 - \rho^2 |S|^2}$.
    
    Then, for every (context-dependent) policy $\pi= (\pi_c)_{c \in \mathcal{C}}$ and a layer $h \in [H-1]$,
     under the good events $G_1, G_2^i, G_3^i, \forall i \in [h-1]$ 
     the following holds. 
    \[
        \mathbb{P}_c \left[
        \forall k \in [h], s_k \in S_k. \;\;
        q_k(s_k|\pi_c, P^c) 
        \geq
        q_k(s_k|\pi_c, \widehat{P}^c) -
        \frac{4 \rho|S|}{1 - \rho^2|S|^2}k 
        \right]
        \geq 
        1 - |A|\;\sum_{k=0}^{h-1}\frac{\epsilon_P}{\rho^2}|S_k||S_{k+1}|
        .
    \]
\end{lemma}

\begin{proof}
    For every context $c \in \mathcal{C}$ define the dynamics $\widetilde{P}^c$ over
    $S \cup \{s_{sink}\} \times A$:
    \[
        \forall (s,a) 
        \in 
        S\cup \{s_{sink}\} \times A:
        \widetilde{P}^c (s| s_{sink}, a) 
        = 
        \begin{cases}
        1 &, \text{if } s = s_{sink}\\
        0 &, \text{otherwise}
        \end{cases}
        .
    \]
    In addition we define                
    \begin{align*}
        &\forall k  \in [h-1],\;\;
        \forall (s_k, a_k, s_{k+1})
        \in \widetilde{S}^{\gamma, \beta}_k \times A \times S_{k+1}:
        \\
        &
         \widetilde{P}^c(s_{k+1}| s_k, a_k)
        =
        \begin{cases}
        P^c(s_{k+1}|s_k,a_k)
        &, \text{if } c \in \widehat{C}^{\beta}({s_k})\\
        0 &, \text{otherwise}
        \end{cases}
        \\
        & \widetilde{P}^c(s_{sink}| s_k, a_k)
        =
        \begin{cases}
            0 &, \text{if } c \in \widehat{\mathcal{C}}^{\beta}({s_k})\\
            1 &, \text{otherwise}
        \end{cases}
        %
        \\
        &\forall k  \in [h-1],\;\;
        \forall (s_k, a_k, s_{k+1})
        \in (S_k \setminus \widetilde{S}^{\gamma, \beta}_k) \times A \times S_{k+1}:
        \\
        & \widetilde{P}^c(s_{k+1}| s_k, a_k)= 0,\;\;\;
        \widetilde{P}^c(s_{sink}| s_k, a_k) = 1.
    \end{align*}
    Clearly, by definition of $\widetilde{P}^c$,
    we have for every (context-dependent) policy $\pi$ that
    \begin{equation}\label{ineq:UCDD-l_2-prob-final}
        \mathbb{P}_c
        \left[ \forall k\in [h] , s_k \in S_k.\;\;
            q_k(s_k | \pi_c, P^c)
            \geq
            q_k(s_k | \pi_c, \widetilde{P}^c)
        \right]= 1.
    \end{equation}

    By Lemma~\ref{lemma:UCDD-l_2-matrix-diff-ver-2} under the good events $G_1,G_2^k,G_3^k \;\; \forall k \in [h-1]$ for any $k \in [h-1]$
    it holds that
    \begin{equation}\label{eq:prob-1-lb-occ-measure}
        \mathop{
        \mathbb{P}
        }_{(c,s_k,a_k)}
        \left[ 
            \|\widehat{P}^c(\cdot| s_k, a_k) - \widetilde{P}^c(\cdot| s_k, a_k)\|_1   
            \leq \frac{4 \rho |S|}{ 1 - \rho^2 |S|^2}
        \Big|
        (c, s_k, a_k) \in  \widetilde{\mathcal{X}}^{\gamma, \beta}_k
        \right]
        \geq
        1- \frac{\epsilon_P}{\rho^2}|S_{k+1}|,
    \end{equation}
    
    We now show that 
    \begin{equation}\label{eq:prob-2-lb-occ-measure}
        \mathop{
        \mathbb{P}
        }_{(c,s_k,a_k)}
        \left[ 
            \|\widehat{P}^c(\cdot| s_k, a_k) - \widetilde{P}^c(\cdot| s_k, a_k)\|_1 =
            0
        \Big|
        (c, s_k, a_k) \notin  \widetilde{\mathcal{X}}^{\gamma, \beta}_k
        \right]
        =
        1.
    \end{equation}
    
    For every layer $ k \in [h-1]$ we have by definition that $(c, s_k, a_k) \in  \widetilde{\mathcal{X}}^{\gamma, \beta}_k$ if and only if $s_k \in \widetilde{S}^{\gamma, \beta}_k$ and $c \in \widehat{\mathcal{C}}^\beta(s_k)$.

    By the definition of $\widetilde{P}^c$ and $\widehat{P}^c$ we have for every layer $k \in [h-1]$ and context $c \in \mathcal{C}$ that
    \begin{equation*}
            \forall (s_k,a_k) \in (S_k \setminus \widetilde{S}^{\gamma, \beta}_k) \times A.\;\;
            \|\widehat{P}^c(\cdot|s_k,a_k) - \widetilde{P}^c(\cdot|s_k,a_k)\|_1 = 0
    \end{equation*}

    In addition, by definition of $\widehat{P}^c$ and $\widetilde{P}^c$, 
    for every layer $k \in [h-1]$
    if $(s_k,a_k) \in \widetilde{S}^{\gamma, \beta}_k \times A$ but $c \notin \widehat{\mathcal{C}}^\beta(s_k)$, then
    $$
        \|\widehat{P}^c(\cdot| s_k, a_k)- \widetilde{P}^c (\cdot| s_k, a_k)\|_1 = 0.
    $$
    
    Thus, equation~(\ref{eq:prob-2-lb-occ-measure}) follows.
    
    Using total probability low, equations~(\ref{eq:prob-1-lb-occ-measure}) and~(\ref{eq:prob-2-lb-occ-measure}) yield that 
    \begin{align*}
        \mathop{
        \mathbb{P}
        }_{(c,s_k,a_k)}
        \left[ 
            \|\widehat{P}^c(\cdot| s_k, a_k) - \widetilde{P}^c(\cdot| s_k, a_k)\|_1   
            \leq \frac{4 \rho |S|}{ 1 - \rho^2 |S|^2}
        \right]
        \geq
        1- \frac{\epsilon_P}{\rho^2}|S_{k+1}|,
    \end{align*}
    which by union bound over $(s_k, a_k) \in S_k \times A$ for every layer $k \in [h-1]$ implies that
    \begin{equation}\label{ineq:l_2-running-time-prob}
        \mathbb{P}_c
        \left[ 
            \forall k \in [h-1], (s_k,a_k) \in S_k \times A.\;\;
            \|\widehat{P}^c(\cdot| s_k, a_k) - \widetilde{P}^c(\cdot| s_k, a_k)\|_1   
            \leq \frac{4 \rho |S|}{ 1 - \rho^2 |S|^2}
        \right]
        \geq
        1 - |A|\;\sum_{ k = 0 }^{h-1}\frac{\epsilon_P }{\rho^2}|S_k||S_{k+1}|.
    \end{equation}

    By Theorem~\ref{thm: bounding occupancy measures}
    the above yields that
    \begin{align*}
        \mathbb{P}_c
        \left[ 
            \forall k \in [h].\;\;
            \| q_k(\cdot|\pi_c, \widetilde{P}^c) - q_k(\cdot| \pi_c, \widehat{P}^c) \|_1 
            \leq
            \frac{4 \rho |S|}{1 - \rho^2 |S|^2} k 
        \right]
        \geq
        1 - |A|\;\sum_{ k = 0 }^{h-1}\frac{\epsilon_P }{\rho^2}|S_k||S_{k+1}|,
    \end{align*}
    which in particularly implies that
    \begin{equation}\label{ineq:l_2-final-pron-to-combine}
        \mathbb{P}_c
        \left[ 
            \forall k \in [h], s_k \in S_k.\;\;
            q_k(s_k |\pi_c, \widetilde{P}^c) \geq q_k(s_k |\pi_c, \widehat{P}^c)- \frac{4 \rho |S|}{1 - \rho^2 |S|^2} k 
        \right]
        \geq
        1 - |A|\;\sum_{ k = 0 }^{h-1}\frac{\epsilon_P }{\rho^2}|S_k||S_{k+1}|.
    \end{equation}
    
    Finally, the lemma follows by combining
    inequalities~\ref{ineq:UCDD-l_2-prob-final} and~\ref{ineq:l_2-final-pron-to-combine}.
\end{proof}

\subsection{Analysis for the \texorpdfstring{$\ell_1$}{Lg} Loss}\label{subsec: analysis-UCDD l_1}

\subsubsection{Good Events}
For the analysis of the algorithm, we define the following good events.





\paragraph{Event $G_1$.}
Intuitively, it states
that the approximation of the probability that $c \in \widehat{\mathcal{C}}^{\beta}(s)$  is accurate for every state $s \in S$.

Formally, 
let $\widehat{p}_\beta(s)$ be the output of Algorithm \texttt{AGC} (see Algorithm~\ref{alg: AGC-UCDD}) for the state $s \in S$, and denote
$ p_\beta(s):= \mathbb{P}[c \in \widehat{\mathcal{C}}^{\beta}(s)]$.
For each layer $h\in[H-1]$ we define the event $G_1^h$ as 
$
    {G^h_1=\{|\widehat{p}_\beta({s_h}) - p_\beta({s_h})| \leq \gamma/4 \;\;\;\forall s_h\in S_h\}}
$ 
and define $G_1=\cap_{h\in[H-1]} G^h_1$.

The good event $G_1$ guarantees that for every layer $h \in [H-1]$ and state $s_h \in S_h$, if $\widehat{p}_\beta({s_h}) \geq \frac{3}{4}\gamma$ then $p_\beta({s_h})\geq \gamma/2$, which implies that $s_h \in \widehat{S}^{\gamma/2,\beta}_h$.
This implies that for every layer $h$ we sample only $(\gamma/2,\beta)$-good states for $\widehat{P}^c$.

More impotently, if $p_\beta({s_h}) \geq \gamma$ then $\widehat{p}_\beta({s_h}) \geq \frac{3}{4}\gamma$. 
Hence, we identify every $(\gamma,\beta)$-good state.

Thus, under the good event $G_1$, for every layer $h \in [H-1]$ the approximated set $\widetilde{S}^{\gamma,\beta}_h$ satisfies
\[
    \widehat{S}^{\gamma,\beta}_h
    \subseteq
    \widetilde{S}^{\gamma,\beta}_h
    \subseteq
    \widehat{S}^{\gamma/2,\beta}_h.
\]

The following lemma shows that for our parameters choice, $G_1$ holds with high probability.    
\begin{lemma}\label{lemma: UCDD G_1 l_1}
    For $\epsilon_2 = \gamma/4$ and $\delta_2 = \frac{\delta}{8 |S|}$, we have that $\mathbb{P}[G_1] \geq 1- {\delta}/{8}$.
\end{lemma}

\begin{proof}
    For each $s \in S$ we have that $\widehat{p}_\beta(s)$ is  calculated over
    $m(\epsilon_2, \delta_2) =
    \Big\lceil 
    \frac{\ln{\frac{2}{\delta_2}}}{2 \epsilon_2^2}
    \Big\rceil$ examples. By Hoeffding's inequality combined with union bound, for $\epsilon_2 = \gamma/4$ and $\delta_2 = \frac{\delta}{8 |S|}$, we obtain that $\mathbb{P}[G_1] \geq 1- {\delta}/{8}$.  
\end{proof}

\paragraph{Sampling distributions.}
 Recall that during the algorithm, for every layer $h \in [H-1]$ we collect examples of $(c, s_h, a_h)$ for which $\widehat{p}_\beta(s_h) \geq \frac{3}{4}\gamma$, which under $G_1$ implies that $p_\beta(s_h) \geq \gamma/2$,
 context $c \in \widehat{C}^\beta (s_h)$ and actions $a_h \in A$.
 
 For every layer $h \in [H-1]$ and reachability parameters $\gamma$ and $\beta$ we define the \emph{target domain} we would like to collect examples from as
 \[
    \mathcal{X}^{\gamma, \beta}_h 
    =
    \{
        (c,s_h, a_h) :
        s_h \in \widehat{S}^{\gamma, \beta}_h, c \in \widehat{\mathcal{C}}^\beta (s_h), a_h \in A
    \},
 \]
 recalling that 
 $${\widehat{C}^\beta (s_h) = \{ c \in \mathcal{C} : s_h \text{ is } \beta-\text{reachable for } \widehat{P}^c\}}$$
 and 
 $${\widehat{S}^{\gamma, \beta}_h = \{s_h \in S_h : \mathbb{P}[c \in \widehat{\mathcal{C}}^\beta(s_h)] \geq \gamma\}}.$$
Meaning, we would like to collect sufficient number of examples of $(\gamma,\beta)$-good states, appropriate good context and action for each layer.

In practice,
we collect examples of states $s \in \widetilde{S}^{\gamma,\beta}_h$
which also contains states $s \in \widehat{S}^{\gamma/2,\beta}_h$. 
Under $G_1$ we have the guarantee that 
$\widehat{S}^{\gamma,\beta}_h \subseteq \widetilde{S}^{\gamma,\beta}_h \subseteq \widehat{S}^{\gamma/2,\beta}_h$.

Hence, we define the \emph{empirical domain}
 \[
    \widetilde{\mathcal{X}}^{\gamma, \beta}_h 
    =
    \{
        (c,s_h, a_h) :
        s_h \in \widetilde{S}^{\gamma, \beta}_h, c \in \widehat{\mathcal{C}}^\beta (s_h), a_h \in A
    \},
 \]
We remark that before learning layer $h$, we compute $\widetilde{S}^{\gamma, \beta}_h$ based on the previous layers approximation for the dynamics which are fixed, hence $\widetilde{\mathcal{X}}^{\gamma, \beta}_h$ is  fixed when learning layer $h$.

We also remark that under $G_1$ it holds, since $\widehat{S}^{\gamma,\beta}_h \subseteq \widetilde{S}^{\gamma,\beta}_h \subseteq \widehat{S}^{\gamma/2,\beta}_h$ it also holds that
\[
    \mathcal{X}^{\gamma, \beta}_h
    \subseteq 
    \widetilde{\mathcal{X}}^{\gamma, \beta}_h 
    \subseteq
    \mathcal{X}^{\gamma/2, \beta}_h.
\]


 We consider the marginal distributions of our observations, that sampled from $\widetilde{\mathcal{X}}^{\gamma, \beta}_h$.
 
 For the rewards denote by $\widetilde{\mathcal{D}}^R_h$ the distribution over the collected examples ${((c,s,a),r) \in \widetilde{\mathcal{X}}^{\gamma,\beta}_h \times [0,1]}$, for each layer $h \in [H-1]$. It holds that
    \begin{align*}
      \widetilde{\mathcal{D}}^R_h ((c,s_h,a_h), r_h) &=
      \mathbb{P}[((c, s_h, a_h), r_h) \in Sample^R(h)| 
        (c, s_h, a_h) \in  \widetilde{\mathcal{X}}^{\gamma, \beta}_h]\\
      \propto&
        \mathbb{P}[c | c \in \widehat{\mathcal{C}}^{\beta}(s_h)] 
        \cdot
        q_h(s_h, a_h| \widehat{\pi}^c_{s_h}, P^c) 
        \cdot 
        \mathbb{P}[R^c(s_h, a_h) = r_h |c,s_h,a_h],
    \end{align*}
   where $\propto$ implies that we normalize to sum to $1$.
   
   Since under $G_1$ we have that $\mathcal{X}^{\gamma,\beta}_h \subseteq \widetilde{\mathcal{X}}^{\gamma,\beta}_h$, $\widetilde{\mathcal{D}}^R_h$ induces a marginal distribution over  $\mathcal{X}^{\gamma,\beta}_h \times [0,1]$, which we denote by $\mathcal{D}^R_h$.
   Clearly, it holds that
    \begin{align*}
      \mathcal{D}^R_h ((c,s_h,a_h), r_h) &=
      \mathbb{P}[((c, s_h, a_h), r_h) \in Sample^R(h)| 
        (c, s_h, a_h) \in  \mathcal{X}^{\gamma, \beta}_h]\\
      \propto&
        \mathbb{P}[c | c \in \widehat{\mathcal{C}}^{\beta}(s_h)] 
        \cdot
        q_h(s_h, a_h| \widehat{\pi}^c_{s_h}, P^c) 
        \cdot 
        \mathbb{P}[R^c(s_h, a_h) = r_h |c,s_h,a_h],
    \end{align*}
   which is the desired marginal distribution over our target domain.

  Similarly, for the next state we have,
    \begin{align*}
    \widetilde{\mathcal{D}}^P_h& ((c,s_h,a_h,s') , \mathbb{I}[s_{h+1}=s'])\\ 
    &=
    \mathbb{P}[((c, s_h, a_h,s'),\mathbb{I}[s_{h+1}=s']))\in Sample^P(h)|(c, s_h, a_h, s') \in (  \widetilde{\mathcal{X}}^{\gamma, \beta}_h \times S_{h+1})]\\
     &\propto
        \mathbb{P}[c | c \in \widehat{\mathcal{C}}^{\beta}(s_h)]
        \cdot
        q_h(s_h, a_h| \widehat{\pi}^c_{s_h}, P^c)
        \cdot
        P^c(s'|s_h,a_h),
    \end{align*}
    and we denote $\mathcal{D}^P_h$ the induced  marginal distribution over  $(\mathcal{X}^{\gamma,\beta}_h \times S_{h+1}) \times [0,1]$.
    

\begin{remark}
    When it is clear from the context, we use $\mathcal{D}^P_h$ and $\widetilde{\mathcal{D}}^P_h$ to also denote the induced distribution over $(c, s_h, a_h, s_{h+1})\in\mathcal{X}^{\gamma,\beta}_h \times S_{h+1}$ and $(c, s_h, a_h, s_{h+1})\in \widetilde{\mathcal{X}}^{\gamma,\beta}_h \times {S_{h+1}}$, respectively, and drop the indicator bit.
    Similarly for $\mathcal{D}^R_h$ and $\widetilde{\mathcal{D}}^R_h$ we have ${(c, s_h, a_h)\in \mathcal{X}^{\gamma,\beta}_h}$ and ${(c, s_h, a_h)\in\widetilde{\mathcal{X}}^{\gamma,\beta}_h}$.
\end{remark}

\paragraph{Event $G_2$.}
Intuitively it states that sufficient number of examples
have been collected for every layer $h \in [H-1]$.

Formally, let $G_2^{h}$ be the event that
\begin{enumerate}
    \item At least $ \max\{N_R(\mathcal{F}^R_h ,\epsilon_R, \delta_1/2), N_P(\mathcal{F}^P_h ,\epsilon_P, \delta_1/2)\}$ examples of context, state and action from the target domain, i.e., $(c,s,a) \in \mathcal{X}^{\gamma,\beta}_h$, have been collected for layer $h \in [H-1]$ .
    \item At least $ 2\max\{N_R(\mathcal{F}^R_h ,\epsilon_R, \delta_1/2), N_P(\mathcal{F}^P_h ,\epsilon_P, \delta_1/2)\}$ examples of context, state and action from the empirical domain, i.e., $(c,s,a) \in \widetilde{\mathcal{X}}^{\gamma,\beta}_h$, have been collected for layer $h \in [H-1]$ .
\end{enumerate}
Let $G_2$ be the event $\cap_{h\in[H-1]}G^h_2$.

\paragraph{Event $G_3$.}
Intuitively states that the ERM guarantees for the approximation of the dynamics hold.
Let $G_3^h$ denote the following event (for the $\ell_1$ loss) that
    \begin{align*}
        &\mathbb{E}_{(c, s_h, a_h, s_{h+1})\sim \mathcal{D}^P_h}
        [|f^P_h(c, s_h, a_h, s_{h+1}) - P^c(s_{h+1}| s_h, a_h)|]
        \leq 
        \epsilon_P 
    \end{align*}
and
    \begin{align*}
        \mathbb{E}_{(c, s_h, a_h, s_{h+1})\sim \widetilde{\mathcal{D}}^P_h}
        [|f^P_h(c, s_h, a_h, s_{h+1}) - P^c(s_{h+1}| s_h, a_h)|]
        \leq 
        \epsilon_P.
    \end{align*}    
Recall that we assume realizability for each layer. 
Define $G_3 = \cap_{h \in [H-1]}G^h_3$.

The following lemma shows that if $G_1$ and $G^h_2$ holds, then $G^h_3$ holds with high probability. (We later show that $G_2$ holds with high probability.)
\begin{lemma}\label{lemma: prob to G^h_3 given conditions l_1}
For any $h\in [H-1]$ we have $\mathbb{P}[G_3^h | G_1, G^h_2 ]\geq 1 - \delta_1$.
\end{lemma}


\begin{proof}
    Under $G_1$ and $G^h_2$ we have collected sufficient number of examples from the domain $\mathcal{X}^{\gamma,\beta}_h \times S_{h+1}$ 
    to approximate the transition probability function of layer $h$, for the accuracy parameter $\epsilon_P$ and confidence parameter $\delta_1/2$.
    By the ERM guarantees (see~\ref{par: ERM per layer appexdix}), if
    sufficient number of examples have been collected, 
    then 
    the ERM output $f^P_h$ satisfies that
    ${\mathbb{E}_{\mathcal{D}^P_h}
        [|f^P_h(c, s_h, a_h, s_{h+1}) - P^c(s_{h+1}| s_h, a_h)|]
        \leq 
    \epsilon_P }$, with probability at least $1-\delta_1/2$.
    Similarly for $\widetilde{\mathcal{X}}^{\gamma,\beta}_h \times S_{h+1}$ it holds that 
    ${\mathbb{E}_{\widetilde{\mathcal{D}}^P_h}
        [|f^P_h(c, s_h, a_h, s_{h+1}) - P^c(s_{h+1}| s_h, a_h)|]
        \leq 
    \epsilon_P }$, with probability at least $1-\delta_1/2$.
    Hence the lemma follows by union bound.
\end{proof}

The following lemma shows, inductively,  that if in all the previous layers $i < h$ we have that $G_1$, $G^i_2$, $G^i_3$ hold, then $G^h_2$ holds with high probability in the current layer $h$.    
\begin{lemma}\label{lemma: prob to G^h_2 given conditions l_1}
    For each layer $h\in [H-1]$ it holds that            
    ${\mathbb{P}[G_2^h | G_1, G^i_2 ,G_3^i \;\forall i\in[h-1] ]
    \geq 1 - (\delta_1 + \frac{\epsilon_P}{\rho}|S|^2|A|)}$.
\end{lemma}

\begin{proof}
    We prove the lemma using induction over the horizon $h$.
        
        
    
    \textbf{Base case.} $h=0$.
    
    By definition, the start state $s_0$ is $(1,1)$-good, which implies that for $s_0$ we collect samples in a deterministic manner. Thus, it holds that $\mathbb{P}[G^0_2] = 1$.
        
    \textbf{Induction step.} 
    Assume the lemma holds for all $k < h$ and we show it holds for $h$.\\
    Recall we collect examples of states $s_h \in S_h$ for which
    $\widehat{p}_\beta(s_h) \geq \frac{3}{4}\gamma$.
    Under $G_1$, if $\widehat{p}_\beta(s_h) \geq \frac{3}{4}\gamma$ then
    $
        {\mathbb{P}[c \in \widehat{\mathcal{C}}^{\beta}(s_h) 
        ]
        \geq
        \gamma/2}
    $. 
    In addition, if     
    $
        {\mathbb{P}[c \in \widehat{\mathcal{C}}^{\beta}(s_h) 
        ]
        \geq
        \gamma}
    $ then $\widehat{p}_\beta(s_h) \geq \frac{3}{4}\gamma$.
    
    Thus, the set $\widetilde{S}^{\gamma,\beta}_h$ of approximately $(\gamma,\beta)$-good state for $\widehat{P}^c$ satisfies that  $\widehat{S}^{\gamma,\beta}_h \subseteq \widetilde{S}^{\gamma,\beta}_h \subseteq \widehat{S}^{\gamma/2,\beta}_h$.
    
    Given $G_1, G^k_2 ,G_3^k \;\forall k\in[h-1]$ hold, by Lemma~\ref{lemma: for running time l_1}, for $\beta$ and $\rho$ such that $\beta \geq 2H \frac{4 \rho |S|}{ 1 - \rho^2 |S|^2}$ it holds that 
    \begin{align*}
        \mathbb{P}_c \left[
        \underbrace{\forall k \in [h], s_k \in S_k. \;\;
        q_k(s_k|\pi_c, P^c) 
        \geq
        q_k(s_k|\pi_c, \widehat{P}^c) -
        \frac{4 \rho|S|}{1 - \rho^2|S|^2}k}_{(\star)} 
        \right]
        &\geq 
        1 - |A|\;\sum_{k=0}^{h-1}\frac{\epsilon_P}{\rho}|S_k||S_{k+1}|
        \\
        &\geq
        1 - |A||S|^2\frac{\epsilon_P }{\rho}.
    \end{align*}
    \begin{claim}
        Assume inequality $(\star)$ holds. Then the probability to collect one example of ${(c, s_{h}, a_{h}) \in  \mathcal{X}^{\gamma, \beta}_{h}}$ is at least $\frac{1}{|S|}\gamma (\beta - \frac{4 \rho |S|}{ 1 - \rho^2 |S|^2} h) \geq \frac{1}{|S|}\cdot \gamma \cdot \beta /2$.
    \end{claim}
    \begin{proof}
        Consider the process of collecting a sample, as described in Algorithm~\ref{alg: EXPLOIT-UCDD}:
        \begin{enumerate}
            \item The algorithm/agent chooses uniformly at random $(s,a) \in \widetilde{S}^{\gamma, \beta}_h \times A$.
            Under the good event $G_1$ we have that 
            ${\widehat{S}^{\gamma, \beta}_h} \subseteq \widetilde{S}^{\gamma, \beta}_h$.
            Hence, the probability to choose 
            $(s,a) \in \widehat{S}^{\gamma, \beta}_h \times A$
            is at least $\frac{1}{|S|}$.
            \item A context $c \sim \mathcal{D}$ is sampled.
            By $\widehat{S}^{\gamma, \beta}_h$ definition, the probability that $c \in \widehat{\mathcal{C}}^\beta(s)$ is at least $\gamma$.
            \begin{itemize}
                \item  If $c \in \widehat{\mathcal{C}}^\beta(s)$, the agent plays $\widehat{\pi}^c_s$ to generate a trajectory where the dynamics is $P^c$. By $(\star)$ and $\widehat{\mathcal{C}}^\beta(s)$ definition, the probability to observe $(s, a)$ in a trajectory generated using $\widehat{\pi}^c_s$ where the dynamics is $P^c$ is $q_h(s|\widehat{\pi}^c_s, P^c) \geq \beta - \frac{4 \rho |S|}{ 1 - \rho^2 |S|^2} h \geq \beta/2$.
                \item Otherwise quite iteration.
            \end{itemize}
        \end{enumerate}
        Overall, 
        the probability to collect one example of a triplet $(c, s, a) \in  \mathcal{X}^{\gamma, \beta}_{h}$ is at least ${\frac{1}{|S|}\cdot \gamma \cdot (\beta - \frac{4 \rho |S|}{ 1 - \rho^2 |S|^2} h)} \geq \frac{1}{|S|}\cdot \gamma \cdot \frac{\beta}{2}$ (since $\beta \geq 2H \frac{4 \rho |S|}{ 1 - \rho^2 |S|^2}$).
    \end{proof}
    
    \begin{claim}
        Assume inequality $(\star)$ holds. Then the probability to collect one example of ${(c, s_{h}, a_{h}) \in  \widetilde{\mathcal{X}}^{\gamma, \beta}_{h}}$ is at least $\frac{\gamma}{2}(\beta - \frac{4 \rho |S|}{ 1 - \rho^2 |S|^2} h) \geq  \gamma\cdot \beta /4$.
    \end{claim}
    \begin{proof}
        Consider the process of collecting a sample, as described in Algorithm~\ref{alg: EXPLOIT-UCDD}:
        \begin{enumerate}
            \item The algorithm/agent chooses uniformly at random $(s,a) \in \widetilde{S}^{\gamma, \beta}_h \times A$.
            Under the good event $G_1$ we have that 
            $\widetilde{S}^{\gamma, \beta}_h \subseteq {\widehat{S}^{\gamma/2, \beta}_h}$.
            \item A context $c \sim \mathcal{D}$ is sampled by the nature.
            By $\widehat{S}^{\gamma/2, \beta}_h$ definition, the probability to observe a context $c \in \widehat{\mathcal{C}}^\beta(s)$ is at least $\gamma/2$.
            \begin{itemize}
                \item  If $c \in \widehat{\mathcal{C}}^\beta(s)$, the agent plays $\widehat{\pi}^c_s$ to generate a trajectory where the dynamics is $P^c$. By $(\star)$ and $\widehat{\mathcal{C}}^\beta(s)$ definition, the probability to observe $(s, a)$ in a trajectory generated using $\widehat{\pi}^c_s$ where the dynamics is $P^c$ is $q_h(s|\widehat{\pi}^c_s, P^c) \geq \beta - \frac{4 \rho |S|}{ 1 - \rho^2 |S|^2} h \geq \beta/2$.
                \item Otherwise quite iteration.
            \end{itemize}
        \end{enumerate}
        Overall, 
        the probability to collect one sample of some triplet $(c, s, a) \in \widetilde{\mathcal{X}}^{\gamma, \beta}_{h}$ is at least ${\gamma/2 \cdot (\beta - \frac{4 \rho |S|}{ 1 - \rho^2 |S|^2} h)} \geq \gamma \cdot \beta/4$ (since $\beta \geq 2H \frac{4 \rho |S|}{ 1 - \rho^2 |S|^2}$).
    \end{proof}    
        
    The above claims implies that if $(\star)$ holds, in expectation, the agent needs to experience at most 
   $ \frac{2|S|}{\gamma \cdot \beta}$    
    episodes to collect one example from $ \mathcal{X}^{\gamma, \beta}_{h}$.
    In addition, in expectation, the agent needs to experience at most $\frac{4}{\gamma \cdot \beta}$ 
    episodes to collect one example from 
    $ \widetilde{\mathcal{X}}^{\gamma, \beta}_{h}$.
    
    Since under $G_1$ we have that    
    $ 
        \mathcal{X}^{\gamma, \beta}_{h}
        \subseteq 
        \widetilde{\mathcal{X}}^{\gamma, \beta}_{h}
        \subseteq
        \mathcal{X}^{\gamma/2, \beta}_{h}
    $,
    using multiplicative Chernoff bound we obtain that with probability at least $1 - \delta_1$ after experiencing  
    \[
        T_h =
        \left\lceil 
            \frac{8 |S|}{\gamma \cdot \beta}
            \left(\ln\frac{1}{\delta_1}
            +
            2\max\{N_P(\mathcal{F}^P_h ,\epsilon_P, \delta_1/2), N_R(\mathcal{F}^R_h ,\epsilon_R, \delta_1/2)\}\right)
            \right\rceil
    \]
    episodes, the agent will collect at least $\max\{N_P(\mathcal{F}^P_h ,\epsilon_P, \delta_1/2)), N_R(\mathcal{F}^R_h ,\epsilon_R, \delta_1/2))\}$ examples from $\mathcal{X}^{\gamma,\beta}_h$ and 
    $2\max\{N_P(\mathcal{F}^P_h ,\epsilon_P, \delta_1/2)), N_R(\mathcal{F}^R_h ,\epsilon_R, \delta_1/2))\}$ examples from $\widetilde{\mathcal{X}}^{\gamma,\beta}_h$.
    Recall that $T_h$ is exactly the number of episodes we run in Algorithm~\ref{alg: EXPLOIT-UCDD} when learning layer $h$.
    Hence, using union bound we obtain that
    \[
        \mathbb{P}[G_2^{h} | G_1, G^i_2 ,G_3^i \;\forall i\in[h-1] ]
        \geq 1 - (\delta_1  +  |A||S|^2\frac{\epsilon_P }{\rho}).
    \]    
 
\end{proof}

The following lemma shows that, given $G_1$, $G_2$ and $G_3$ hold with high probability.
\begin{lemma}\label{lemma: UCDD  l_1 G_2 ,G_3 given G_1}
    The following holds.
    \[
        \mathbb{P}[G_2 \cap G_3 | G_1] 
        \geq 1 - (2\delta_1 H + \frac{\epsilon_P}{\rho}|S|^2|A|H).
    \]
\end{lemma}
    
\begin{proof}
    Assume the good event $G_1$ holds.
    Recall that $G_2 = \cap_{h \in [H-1]}G^h_2$ and $G_3 = \cap_{h \in [H-1]}G^h_3$.
        
    Let $X$ be a random variable with support $[H-1]$ that satisfies 
    \begin{align*}
        X = \min_{k \in [H-1]}\{\overline{G}^k_2 \cup \overline{G}^k_3 \text{ holds }\},
    \end{align*}
    and otherwise $X=\bot$, (i.e., if $G_2$ and $G_3$ hold).
    In words, $X$ is the first layer in which at least one of the good events $G^h_2$ or $G^h_3$ does not hold.
        
    By $X$ definition and Bayes rule (i.e., $\mathbb{P}[A \cap B] = \mathbb{P}[A|B]\cdot \mathbb{P}[B]$) we have
    \begingroup
    \allowdisplaybreaks
    \begin{align*}
        \forall h \in [H].\;\; 
        \mathbb{P}[X = h|G_1]
        &=
        \mathbb{P}[(\overline{G}^h_2 \cup \overline{G}^h_3)\cap (\cap_{
        k \in [h-1]}G^k_2 \cap G^k_3) | G_1]
        \\
        \tag{ Bayes rule}
        &=
        \mathbb{P}[(\overline{G}^h_2 \cup \overline{G}^h_3) | G_1 , (\cap_{
        k \in [h-1]}G^k_2 \cap G^k_3)]
        \cdot
        \underbrace{\mathbb{P}[ (\cap_{ k \in [h-1]}G^k_2 \cap G^k_3)|G_1]}_{\leq 1}
        \\
        &\leq
        \mathbb{P}[(\overline{G}^h_2 \cup \overline{G}^h_3) | G_1 , (\cap_{
        k \in [h-1]}G^k_2 \cap G^k_3)]
        \\
        \tag{Union of disjoint events}
        &=
        \underbrace{\mathbb{P}[\overline{G}^h_2 | G_1 , 
        (\cap_{k \in [h-1]}G^k_2 \cap G^k_3)]}_{\leq \delta_1 + \frac{\epsilon_P}{\rho}|S|^2|A| \text{ by Lemma~\ref{lemma: prob to G^h_2 given conditions l_1}}}
        +
        \mathbb{P}[\overline{G}^h_3\cap G_2^h | G_1 , 
        (\cap_{k \in [h-1]}G^k_2 \cap G^k_3)]
        \\
        &\leq
        \delta_1 + \frac{\epsilon_P}{\rho^2}|S||A|
        +
        \mathbb{P}[\overline{G}^h_3 \cap G^h_2| G_1 , 
        (\cap_{k \in [h-1]}G^k_2 \cap G^k_3)]
        \\
        \tag{caused by rule}
        &=
        \delta_1 + \frac{\epsilon_P}{\rho^2}|S||A|
        +
        \mathbb{P}[\overline{G}^h_3 | G_1 , G^h_2 ,
        (\cap_{k \in [h-1]}G^k_2 \cap G^k_3)]
        \cdot
        \underbrace{\mathbb{P}[ G^h_2  |G_1 \cap
        (\cap_{k \in [h-1]}G^k_2 \cap G^k_3)]]}_{\leq 1}
        \\
        &\leq
        \delta_1 + \frac{\epsilon_P}{\rho}|S||A|
        +
        \mathbb{P}[\overline{G}^h_3 | G_1 , G^h_2 ,
        (\cap_{k \in [h-1]}G^k_2 \cap G^k_3)]                        \\
        &=
        \delta_1 + \frac{\epsilon_P}{\rho^2}|S|^2|A|
        +
        \underbrace{\mathbb{P}[\overline{G}^k_3 | G_1 , G^h_2]}_{\leq \delta_1 \text{ by Lemma~\ref{lemma: prob to G^h_3 given conditions l_1}}}
        \\
        &\leq
        2\delta_1 + \frac{\epsilon_P}{\rho}|S|^2|A| 
    \end{align*}
    \endgroup
        
    Lastly, by $G_2$ and $G_3$ definitions we have
    \begingroup
    \allowdisplaybreaks
    \begin{align*}
        \mathbb{P}[G_2 \cap G_3 |G_1]
        &=
        1 - \mathbb{P}[\overline{G}_2 \cup \overline{G}_3|G_1]
        \\
        &=
        1 
        -
        \mathbb{P}[\cup_{h \in [H-1]}(\overline{G}^h_2 \cup \overline{G}^h_3)|G_1]
       \\
        &=
        1 
        -
        \mathbb{P}[\exists h \in [H-1].(\overline{G}^h_2 \cup \overline{G}^h_3)|G_1]
        \\
        &=
        1 
        -
        \mathbb{P}[\exists h \in [H-1].X = h|G_1]
        \\
        &=
        1 
        -
        \mathbb{P}[\cup_{ h \in [H-1]}\{X = h\}|G_1]
        \\
        \tag{Union bound.}
        &\geq
        1 
        -
        \sum_{h=0}^{H-1}\mathbb{P}[X = h|G_1]
        \\
        &\geq
        1 - (2 \delta_1 H + \frac{\epsilon_P}{\rho}|S|^2|A|H),
    \end{align*}
    \endgroup
    as stated.
\end{proof}

\paragraph{Event $G_4$.}
Intuitively states that the
ERM guarantees for the approximation of the rewards function hold (for the $\ell_1$ loss).
Let $G_4$ denote the good event
\begin{align*}
    &\forall h \in [H - 1] \;\;\mathbb{E}_{(c, s_h, a_h) \sim \mathcal{D}^R_h}
    [| f^R_h(c, s_h, a_h) - r^c(s_h, a_h) |]
    \leq 
    \epsilon_R + \alpha_1(\mathcal{F}^R_h).
\end{align*}

The following lemma shows that if $G_1$ and $G_2$ hold, the $G_4$ hold with high probability.    
\begin{lemma}\label{lemma: UCDD l_1 G_4 given G_2, G_2}
    It holds that 
    $
        \mathbb{P}[G_4 | G_1, G_2] \geq 1 - \delta_1 H
    $. 
\end{lemma}

\begin{proof}
    Since $G_1$ and $G_2$ hold, for every layer $h \in [H-1]$ sufficient number of examples $((c,s_h,a_h),r_h) \in \mathcal{X}^{\gamma,\beta}_h \times [0,1]$  have been collected for the 
    ERM to output a function $f^R_h$ the satisfies
    $${\mathbb{E}_{(c,s_h,a_h) \sim \mathcal{D}^R_h}
        [| f^R_h(c, s_h, a_h) - r^c(s_h, a_h) |]
        \leq 
    \epsilon_R + \alpha^2_2(\mathcal{F}^R_h)}$$
    with probability at least $1-\delta_1$.
    Hence, the lemma
    follows by the ERM guarantees (see~\ref{par: ERM per layer appexdix}) and an union bound over every layer $h \in [H-1]$.
\end{proof}

The following corollary shows that all four good events hold with high probability.
\begin{corollary}\label{corl: good events probs bound l_1}
    It holds that  
    \[
        \mathbb{P}[G_1,G_2,G_3,G_4] 
        \geq 
        1- \left(\frac{\delta}{8} + 3 \delta_1 H + \frac{\epsilon_P}{\rho}|S|^2|A|H\right).
    \]    

\end{corollary}

\begin{proof}
    Followed from union bound over the results of Lemmas~\ref{lemma: UCDD G_1 l_1},~\ref{lemma: UCDD  l_1 G_2 ,G_3 given G_1} and~\ref{lemma: UCDD  l_1 G_4 given G_2, G_2}. 
\end{proof}


\subsubsection{Analysis of the Error Caused by the Dynamics Approximation Under the Good Events}\label{subsubsec:l-1-dynamics-error}


In the following analysis, for any context $c \in \mathcal{C}$ we  consider an intermediate MDP associated with it: 
$\widetilde{\mathcal{M}}(c) = (S \cup \{s_{sink}\}, A, \widehat{P}^c, r^c, s_0, H)$, where $\widehat{P}^c$ is the approximation of the dynamics $P^c$ and $r^c$ is the true rewards function extended to $s_{sink}$ by defining that $r^c(s_{sink},a):=0,\;\; \forall c \in \mathcal{C},\; a \in A$.
Recall the true MDP associated with this context is
    \[
        \mathcal{M}(c)
        = 
        (S, A, P^c, r^c, s_0, H).
    \]

\begin{lemma}\label{lemma: UCDD l_1 matrix diff}
    Let $\rho \in [0,\frac{1}{|S|})$ and $h \in [H -1]$.
    Assume the good events $G_1, G_2^k, G_3^k, \; \forall k \in [h]$ hold, then it holds that
    \[
        \mathbb{P}\left[
        \|\widehat{P}^c(\cdot| s_h, a_h) - P^c(\cdot| s_h, a_h)\|_1
       \leq \frac{4 \rho |S|}{ 1 - \rho^2 |S|^2} \Big| (c,s_h,a_h) \in \mathcal{X}^{\gamma,\beta}_h \right]
       \geq 1 - \frac{\epsilon_P }{\rho}|S_{h+1}|,
    \]
    where $\widehat{P}^c$ is the dynamics defined in Algorithm~\ref{alg: ACDD UCDD} and 
    \[
        \|\widehat{P}^c(\cdot| s_h, a_h) - P^c(\cdot| s_h, a_h)\|_1:=
        \sum_{s_{h+1} \in S_{h+1}}
        |\widehat{P}^c(s_{h+1}| s_h, a_h) - P^c(s_{h+1}| s_h, a_h)|
    \]
    (i.e., the entry of $s_{sink}$ in $\widehat{P}^c$ is ignored).
\end{lemma}
 
\begin{proof}
    
    Under $G_1$ it holds that
    $\mathcal{X}^{\gamma,\beta}_h \subseteq \widetilde{\mathcal{X}}^{\gamma,\beta}_h$.
    Recall that for all $(c, s_h, a_h) \in \widetilde{\mathcal{X}}^{\gamma,\beta}_h$ we have that $\widehat{P}^c(s_{sink}|s_h, a_h) = 0$ by $\widehat{P}^c$ definition.
    Hence, for all $(c, s_h, a_h) \in \mathcal{X}^{\gamma,\beta}_h$ we have that $\widehat{P}^c(s_{sink}|s_h, a_h) = 0$.
    
    In addition, the true dynamics $P^c$ is not defined for $s_{sink}$ since $s_{sink} \notin S$.  
    A natural extension of $P^c$ to $s_{sink}$ is by defining that $\forall (s,a) \in S \times A.\;\; P^c(s_{sink}|s,a) :=0$.
    By that extension, we have for all $(c, s_h, a_h) \in \mathcal{X}^{\gamma,\beta}_h$ that
    $P^c(s_{sink}|s_h,a_h) = \widehat{P}^c(s_{sink}|s_h,a_h)=0$. Hence, we can simply ignore $s_{sink}$ in the following analysis.

    Under the good event $ G_3^h $, by Markov's inequality we have
    \begin{align*}
        &\mathop{\mathbb{P}}_{(c,s_h,a_h,s_{h+1})}
        \left[ |f^P_h(c, s_h, a_h, s_{h+1}) - P^c(s_{h+1}| s_h, a_h)| \geq \rho \Big| (c,s_h,a_h) \in \mathcal{X}^{\gamma,\beta}_h \right]
        = 
        \\
        & =
        \mathbb{P}_{\mathcal{D}^P_h}[|f^P_h(c, s_h, a_h, s_{h+1}) - P^c(s_{h+1}| s_h, a_h)| \geq \rho 
        ]
        \\
        \tag{By Markov's inequality}
        & \leq  
        \frac{\mathbb{E}_{\mathcal{D}^P_h}
        [|f^P_h(c, s_h, a_h, s_{h+1}) - P^c(s_{h+1}| s_h, a_h)|]}{\rho}
        \\
        \tag{Under $G^3_h$}
        & \leq 
        \frac{\epsilon_P}{\rho}.
    \end{align*}
    
    Hence,
    \[
        \mathop{\mathbb{P}}_{(c,s_h,a_h,s_{h+1})}\left[|f^P_h(c, s_h, a_h, s_{h+1}) - P^c(s_{h+1}| s_h, a_h)| \leq \rho \Big| (c,s_h,a_h) \in \mathcal{X}^{\gamma,\beta}_h \right]
        \geq 1 - \frac{\epsilon_P }{\rho}.
    \]
    Since $P^c(\cdot |s_h, a_h)$ is a distribution over $s_{h+1} \in S_{h+1}$, we have that $\sum_{s_{h+1} \in S_{h+1}} P^c(s_{h+1}|s_h, a_h) = 1$. 
    Thus, by union bound applied on $s_{h+1}\in S_{h+1}$, we obtain
    \begin{align*}
        \mathbb{P}_{(c, s_h, a_h) }
        \left[ 
            1 - \rho |S|\leq
            \sum_{s_{h+1}\in S_{h+1}} f^P_h(c, s_h, a_h, s_{h+1}) \leq
            1 + \rho |S|
        \Big|
        (c, s_h, a_h) \in  \mathcal{X}^{\gamma, \beta}_h    
        \right]
        \geq 
        1 - \frac{\epsilon_P }{\rho}|S_{h+1}|.
    \end{align*}
    
     Hence, we further conclude 
     that 
    \begin{equation}\label{prob: lemma C.5 l_1}
        \begin{split}
            &\mathop{\mathbb{P}}_{(c,s_h,a_h)}
            \left[  
                \forall s_{h+1} \in S_{h+1}.
                \frac{P^c(s_{h+1}|s_h, a_h) - \rho}{1 + \rho |S|}  \leq
                \underbrace{\frac{f^P_h(c,s_h,a_h,s_{h+1})}{\sum_{s'\in S_{h+1}} f^P_h(c,s_h,a_h,s')}}_{=\widehat{P}^c(s_{h+1}|s_h, a_h)}
                \leq
                \frac{P^c(s_{h+1}|s_h, a_h) + \rho}{1 - \rho |S|}
            \;\Big|
            (c, s_h, a_h) \in  \mathcal{X}^{\gamma, \beta}_h   \right]
            \\
            & \geq 
            1 - \frac{\epsilon_P }{\rho}|S_{h+1}|.
        \end{split}
    \end{equation}

        
    Fix a tuple $(c, s_h, a_h) \in \mathcal{X}^{\gamma, \beta}_h$ and assume the event of inequality~(\ref{prob: lemma C.5 l_1}) holds.\\
    Denote ${S^+_{h+1} = \{s_{h+1} \in S_{h+1}:
    \widehat{P}^c(s_{h+1}|s_h, a_h) \geq  P^c(s_{h+1}|s_h, a_h )\}}$ and consider the following derivation.
    
    \begingroup
    \allowdisplaybreaks
    \begin{align*}
        \| 
            \widehat{P}^c(\cdot|s_h, a_h ) 
            -  
            P^c(\cdot|s_h, a_h )
        \|_1
        =&
        \sum_{s_{h+1} \in S_{h+1}}
        |\widehat{P}^c(s_{h+1}|s_h, a_h) -  P^c(s_{h+1}|s_h, a_h )|
        \\
        =&
        \sum_{s_{h+1} \in S^+_{h+1}}
        (\widehat{P}^c(s_{h+1}|s_h, a_h) -  P^c(s_{h+1}|s_h, a_h ))       
        \\
        &+
        \sum_{s_{h+1} \in  S_{h+1} \setminus S^+_{h+1}}
        (P^c(s_{h+1}|s_h, a_h) -  \widehat{P}^c(s_{h+1}|s_h, a_h ))
        \\  
        \leq&
        \sum_{s_{h+1} \in S^+_{h+1}}
        \left(\frac{P^c(s_{h+1}|s_h, a_h) + \rho}{1 - \rho |S|} 
        -  
        P^c (s_{h+1}|s_h, a_h )\right)
        \\
        &+
        \sum_{s_{h+1} \in  S_{h+1} \setminus S^+_{h+1}}
        \left( P^c (s_{h+1}|s_h, a_h )
        -
        \frac{P^c(s_{h+1}|s_h, a_h) - \rho}{1 + \rho |S|} \right)        
        \\
        =&
        \sum_{s_{h+1} \in S^+_{h+1}}
        \frac{P^c (s_{h+1}|s_h, a_h) + \rho - (1- \rho |S|)P^c (s_{h+1}|s_h, a_h) }{1 - \rho |S|} 
        \\
        &+
        \sum_{s_{h+1} \in  S_{h+1} \setminus S^+_{h+1}}
        \frac{-P^c (s_{h+1}|s_h, a_h) + \rho + (1 + \rho |S|)P^c (s_{h+1}|s_h, a_h) }{1 + \rho |S|}         
        \\
        =&
        \frac{1}{1- \rho |S|}
        \sum_{s_{h+1} \in S^+_{h+1}}
        (\rho + \rho|S|P^c (s_{h+1}|s_h,a_h))
        \\
        &+
        \frac{1}{1 + \rho |S|}
        \sum_{s_{h+1} \in  S_{h+1} \setminus S^+_{h+1}}
        (\rho + \rho|S|P^c (s_{h+1}|s_h,a_h))
        \\
        \leq&
        \frac{2 \rho |S|}{ 1- \rho |S|}
        +
        \frac{2 \rho |S|}{ 1 + \rho |S|}
        \\
        =&
        \frac{4 \rho |S|}{1 - \rho^2 |S|^2}.
    \end{align*}
    \endgroup
    
    By inequality~(\ref{prob: lemma C.5 l_1}), the above holds with probability at least $1 - \frac{\epsilon_P }{\rho}|S_{h+1}|$ over $(c, s_h, a_h) \in \mathcal{X}^{\gamma,\beta}_h$. Hence the lemma follows.

\end{proof}

\begin{lemma}
    For the parameters choice $\beta =  \frac{\epsilon}{20|S|H} \in (0,1)$, $\rho = \frac{\beta}{16 |S|H} \in (0, \frac{1}{|S|})$,
    and $\epsilon_P = \frac{ \epsilon^2}{10\cdot 16 \cdot 20  |A| |S|^4 H^3}$
    we have $\beta \geq 2H\frac{4 \rho |S|}{1- \rho^2 |S|^2}$.
    In addition, under the good events $G_1$, $G^k_2$ and $G^k_3$ for all $k \in [h]$ it holds that
    \[
        \mathbb{P}\left[\|\widehat{P}^c(\cdot| s_h, a_h) - P^c(\cdot| s_h, a_h)\|_1\leq \frac{\epsilon}{40 |S| H^2} \Big| (c,s_h,a_h) \in \mathcal{X}^{\gamma,\beta}_h \right]
        \geq 1- \frac{\epsilon}{10 |S||A|H}.
    \]
\end{lemma}

\begin{proof}
    An immediate implication of lemma~\ref{lemma: UCDD l_2 matrix diff}.
\end{proof}


\begin{lemma}[occupancy measure difference]\label{lemma: UCDD l_1 occ measure diff}
    
    Under the good events $G_1$, $G_2$ and $G_3$, we have for any (context-dependent) policy $\pi = (\pi_c)_{c \in \mathcal{C}}$ that
    \begin{align*}
        \mathbb{P}_c\left[
		\forall h \in [H].\;\;\;
		\|
			q_h (\cdot | \pi_c, P^c) - q_h(\cdot| \pi_c, \widehat{P}^c)		
		\|_1
		\leq 
		\frac{4 \rho |S|}{1 - \rho^2 |S|^2} h 
		+ \beta \sum_{k=1}^{h-1} |S_k| \right]
		\geq
		 1 
        - 
        \left(
            \frac{\epsilon_P}{\rho}
            |A|\sum_{i=0}^{H-1}|S_i||S_{i+1}|
            + 
            \gamma
            \sum_{i=1}^{H-1}|S_{i}|
        \right)
    \end{align*}

    for a fixed $\rho \in [0, \frac{1}{|S|})$ and $\beta \in (0,1]$ for which $\beta \geq 2H\frac{4 \rho |S|}{1 - \rho^2 |S|^2} $, where
    \[
        \forall h \in [H].\;\; \|q_h(\cdot | \pi_c, P^c) - q_h(\cdot | \pi_c, \widehat{P}^c)\|_1 := 
        \sum_{s_{h} \in S_h}|q_h(s_h | \pi_c, P^c) - q_h(s_h | \pi_c, \widehat{P}^c)|
    \]
    (i.e., $q_h(s_{sink}|\pi_c,\widehat{P}^c)$ is ignored for all $h \in [H]$). 
\end{lemma}

\begin{remark}
    Since $s_{sink} \notin S$, $q_h(s_{sink}|\pi_c, P^c)$ is not defined for the true dynamics $P^c$.
    In addition, by $\widehat{P}^c$ definition, from the sink there are no transitions to any other state, hence, we can simply ignore it in the following analysis.
\end{remark}

\begin{proof}
   We will show the lemma by induction over the horizon, $h$.
   
    For the base case $h =0$ the claim holds trivially (with probability $1$) since the start state $s_0$ is unique.

    For the induction step, assume that it holds up to layer $h$, namely
    \begin{align*}
        \mathbb{P}
        \Big[\forall k \in [h]. \;\;\; 		
        \|
			q_k (\cdot | \pi_c, P^c) - q_k(\cdot | \pi_c, \widehat{P}^c)		
		\|_1
		\leq 
		\frac{4 \rho |S|}{1 - \rho^2 |S|^2}k
		+
		\beta \sum_{i=1}^{k-1}|S_i| \Big]
        \geq 
        1 -
        \left(
            \frac{\epsilon_P}{\rho}
            |A|\sum_{i=0}^{h-1}|S_i||S_{i+1}|
            + 
            \gamma
            \sum_{i=1}^{h-1}|S_{i}| 
        \right)
    \end{align*}

   and prove for layer $h+1$.
   
    %
    by Lemma~\ref{lemma: UCDD l_1 matrix diff} it holds that 
    \begin{align*}
        \mathop{\mathbb{P}}_{(c,s_h,a_h)}
        \left[ \|\widehat{P}^{c'}(\cdot| s_h, a_h) - P^{c'}(\cdot| s_h, a_h)\|_1
       \leq 
       \frac{4 \rho |S|}{ 1 - \rho^2 |S|^2}
       \Big|(c, s_h, a_h) \in  \mathcal{X}^{\gamma, \beta}_h\right]
       \geq 1 - \frac{\epsilon_P}{\rho}|S_{h+1}|.
    \end{align*}
    
    
    Consider the following derivation for any (fixed) context $c$.  (Later we will take the probability over $c\sim \mathcal{D}$.)
   \begingroup
   \allowdisplaybreaks
    \begin{align*}
        &\| q_{h+1}(\cdot | \pi_c, P^c) - q_{h+1}(\cdot | \pi_c, \widehat{P}^c) \|_1
        \\
       = &\sum_{s_{h+1} \in S_{h+1}}
        |q_{h+1}(s_{h+1} | \pi_c, P^c) - q_{h+1}(s_{h+1} | \pi_c, \widehat{P}^c)|
        \\
        =&\sum_{s_{h+1} \in S_{h+1}}
        |
            \sum_{s_{h} \in S_{h}}
            \sum_{a_h \in A}
            (
                q_{h}(s_{h} | \pi_c, P^c)
                \pi_c(a_h | s_h)
                P^c(s_{h+1}|s_h, a_h)
                -
                q_{h}(s_{h} | \pi_c, \widehat{P}^c)
                \pi_c(a_h | s_h)
                \widehat{P}^c(s_{h+1}|s_h, a_h)
            )
       |
       \\
        \leq&\underbrace
        {
            \sum_{s_{h} \in S_{h}}
            \sum_{a_h \in A}
            \sum_{s_{h+1} \in S_{h+1}}
            |
                q_h(s_h | \pi_c, P^c)
                \pi_c(a_h | s_h)
                P^c(s_{h+1}|s_h, a_h)
                -
                q_h(s_h | \pi_c, \widehat{P}^c)
                \pi_c(a_h | s_h)
                P^c(s_{h+1}|s_h, a_h)
            |
        }_{(1)}
        \\    
        &+
        \underbrace
        {
            \sum_{s_{h} \in S_{h}}
            \sum_{a_h \in A}
            \sum_{s_{h+1} \in S_{h+1}}
            |
                q_h(s_h | \pi_c, \widehat{P}^c)
                \pi_c(a_h | s_h)
                P^c(s_{h+1}|s_h, a_h)  
                -
                q_h(s_h | \pi_c, \widehat{P}^c)
                \pi_c(a_h | s_h)
                \widehat{P}^c(s_{h+1}|s_h, a_h)
            |
        }_{(2)}.
    \end{align*}
    \endgroup
   
   We bound $(1)$ and $(2)$ separately.
   
   For $(1)$, since $P^c(\cdot| s_h,a_h)$ and $\pi_c(\cdot |s_h)$  are distributions, we have
   \begin{align*}
       (1)
       &=
       \sum_{s_{h} \in S_{h}}
        |
           q_h(\cdot | \pi_c, P^c)
            -
            q_h(s_h | \pi_c, \widehat{P}^c)
        |
       \sum_{a_h \in A}
       \pi_c(a_h | s_h)
       \sum_{s_{h+1} \in S_{h+1}}
       P^c(s_{h+1}|s_h, a_h)
       \\
       &=
       \|  
            q_h(\cdot | \pi_c, P^c)
            -
            q_h(\cdot | \pi_c, \widehat{P}^c)
        \|_1
       \sum_{a_h \in A}
       \pi_c(a_h | s_h)
       \sum_{s_{h+1} \in S_{h+1}}
       P^c(s_{h+1}|s_h, a_h)
       \\
       &=
       \|  
            q_h(\cdot | \pi_c, P^c)
            -
            q_h(\cdot | \pi_c, \widehat{P}^c)
        \|_1,
   \end{align*}
   which holds with probability $1$.
   
    To bound $(2)$,
    for all $h \in [H-1]$, let us define the following subsets of $S_h$ for any given context $c$.
   \begin{enumerate}
       \item $B^{h, c}_1 = 
       \{s_h \in  S_h :
           s_h \in \widehat{S}^{\gamma,\beta}  \text{ and } c \in \widehat{C}^{\beta}(s_h)\}$.
       \item $B^{h, c}_2 = 
       \{ s_h \in  S_h :
       s_h \in \widehat{S}^{\gamma,\beta} \text{ and } c \notin \widehat{C}^{\beta}(s_h)\}$.
       \item $B^{h, c}_3 = 
       \{s_h \in  S_h :
       s_h \notin \widehat{S}^{\gamma,\beta} \text{ and } c \notin \widehat{C}^{\beta}(s_h)\}$.
       \item $B^{h, c}_4 = 
       \{ s_h \in  S_h :
       s_h \notin \widehat{S}^{\gamma,\beta} \text{ and } c \in \widehat{C}^{\beta}(s_h)\}$.      
   \end{enumerate}

   Clearly, for every layer $h\in [H-1]$ and context $c \in \mathcal{C}$ it holds that $\cup_{i=1}^4 B^{h,c}_i = S_h$.
    
    By definition of $B^{h, c}_1$, for every layer $h\in [H-1]$ we have that $s_h \in B^{h, c}_1$  if and only if   for every action $ a_h \in A$ it holds that $(c, s_h, a_h) \in  \mathcal{X}^{\gamma, \beta}_h$.
   
   
   By definition of $B^{h, c}_4$, for every  layer $h\in [H-1]$ we have that
   \begin{align*}
       \mathbb{P}_c[B^{h, c}_4 \neq \emptyset]
       &=
       \mathbb{P}_c[\exists s_h \in S_h :
       s_h \notin \widehat{S}^{\gamma, \beta}_h \text{ and } c \in \widehat{C}^{\beta}(s_h)]
       \leq 
       \gamma|S_h|.
   \end{align*}
   Thus, for every $h \in [H-1]$ we have $\mathbb{P}_c[B^{h, c}_4 =\emptyset]\geq 1 - \gamma|S_h|$.
   
   In the following, we assume that $B^{h, c}_4 =\emptyset$, since $\mathbb{P}_c[B^{h, c}_4 =\emptyset]\geq 1 - \gamma|S_h|$, it will only add $\gamma |S_h|$ to the probability of the error.
   
    Consider the following derivation
   \begingroup
   \allowdisplaybreaks
   \begin{align*}
       (2) &= 
        \sum_{s_{h} \in B_1}
        \sum_{a_h \in A}
        \sum_{s_{h+1} \in S_{h+1}}
        |
            q_h(s_h | \pi_c, \widehat{P}^c)
            \pi_c(a_h | s_h)
            P^c(s_{h+1}|s_h, a_h)  
            -
            q_h(s_h | \pi_c, \widehat{P}^c)
            \pi_c(a_h | s_h)
            \widehat{P}^c(s_{h+1}|s_h, a_h)
        |
        \\
        &+
        \sum_{s_{h} \in B^{h, c}_2 \cup B^{h, c}_3}
        \sum_{a_h \in A}
        \sum_{s_{h+1} \in S_{h+1}}
        |
            q_h(s_h | \pi_c, \widehat{P}^c)
            \pi_c(a_h | s_h)
            P^c(s_{h+1}|s_h, a_h)  
            -
            q_h(s_h | \pi_c,\widehat{P}^c)
            \pi_c(a_h | s_h)
            \widehat{P}^c(s_{h+1}|s_h, a_h)
        |
        \\
        &+
        \sum_{s_{h} \in B^{h, c}_4}
        \sum_{a_h \in A}
        \sum_{s_{h+1} \in S_{h+1}}
        |
            q_h(s_h | \pi_c, \widehat{P}^c)
            \pi_c(a_h | s_h)
            P^c(s_{h+1}|s_h, a_h)  
            -
            q_h(s_h | \pi_c,\widehat{P}^c)
            \pi_c(a_h | s_h)
            \widehat{P}^c(s_{h+1}|s_h, a_h)
        |        
        \\
        &=
        \sum_{s_{h} \in B^{h, c}_1}
        q_h(s_h | \pi_c, \widehat{P}^c)
        \sum_{a_h \in A}
        \pi_c(a_h | s_h)
        \sum_{s_{h+1} \in S_{h+1}}
        |
            P^c(s_{h+1}|s_h, a_h)  
            -
            \widehat{P}^c(s_{h+1}|s_h, a_h)
        |
        \\
        \tag{$B^{h, c}_4 =\emptyset$ w.p. at least $1- \gamma|S_h|$}
        &+
        \sum_{s_{h} \in B^{h, c}_2 \cup B^{h, c}_3}
        \sum_{a_h \in A}
        q_h(s_h | \pi_c, \widehat{P}^c)
        \pi_c(a_h | s_h) 
        \underbrace{
        \sum_{s_{h+1} \in S_{h+1}}
        |
            P^c(s_{h+1}|s_h, a_h)  
            -
            \widehat{P}^c(s_{h+1}|s_h, a_h)
        |
        }_{\leq 1}
        \\
        &\leq
        \sum_{s_{h} \in B^{h, c}_1}
        q_h(s_h | \pi_c, \widehat{P}^c)
        \sum_{a_h \in A}
        \pi_c(a_h | s_h)
        \|
            P^c(\cdot|s_h, a_h)  
            -
            \widehat{P}^c(\cdot|s_h, a_h)
        \|_1
        +
        \sum_{s_{h} \in B^{h, c}_2 \cup B^{h, c}_3}
        q_h(s_h | \pi_c, \widehat{P}^c)
        \underbrace{
        \sum_{a_h \in A}
        \pi_c(a_h | s_h)}_{=1}        
        \\
        \tag{By Lemma~\ref{lemma: UCDD l_2 matrix diff} and union bound over $(s_h,a_h) \in B^{h, c}_1 \times A$, holds w.p. at least $1-|A||S_h|\frac{\epsilon_P}{\rho}|S_{h+1}|$}
        &\leq
        \underbrace
        {
            \sum_{s_{h} \in B^{h, c}_1}
            q_h(s_h | \pi_c, \widehat{P}^c)
            \sum_{a_h \in A}
            \pi_c(a_h | s_h)
        }_{\leq 1}
        \frac{4 \rho |S|}
        {1 - \rho^2 |S|^2}
        +
        \sum_{s_{h} \in B^{h, c}_2 \cup B^{h, c}_3}
        \underbrace{q_h(s_h | \pi_c, \widehat{P}^c)}_{\leq q_{h}(s_{h} | \widehat{\pi}^c_{s_h}, \widehat{P}^c) < \beta}
        \underbrace{\sum_{a_h \in A}
        \pi_c(a_h | s_h)}_{=1}
        \\
        &=
        \frac{4 \rho |S|}
        {1 - \rho^2 |S|^2}
        + \beta |S_h| .
   \end{align*}
   \endgroup
   Hence, 
   \begin{align*}
       \mathbb{P}_c\left[ 
            (2) \leq \frac{4 \rho |S|}{1 - \rho^2 |S|^2}+ \beta |S_h|
       \right]
       \geq 
       1- \left(|A||S_h|\frac{\epsilon_P}{\rho}|S_{h+1}| + \gamma |S_h| \right).
   \end{align*}
   In addition, we proved above that
    \begin{align*}
       \mathbb{P}_c\left[ 
            \|q_{h+1}(\cdot |\pi_c, P^c) - q_{h+1}(\cdot |\pi_c, \widehat{P}^c)\|_1 \leq (1) + (2)
       \right] = 1,
   \end{align*}
   and
   \begin{align*}
       \mathbb{P}_c\left[ 
            (1) = \|q_h(\cdot |\pi_c, P^c) - q_h(\cdot |\pi_c, \widehat{P}^c)\|_1
       \right] = 1.
   \end{align*}
   Thus, by combining all the above inequalities with the induction hypothesis we obtain 
   \begingroup
   \allowdisplaybreaks
    \begin{align*}
        &\mathbb{P}_c
        \left[\forall k \in [h+1]. \;\;\; 		
        \|
			q_k (\cdot | \pi_c, P^c) - q_k(\cdot | \pi_c, \widehat{P}^c)		
		\|_1
		\leq 
		\frac{4 \rho |S|}{1 - \rho^2 |S|^2}k
		+
		\beta \sum_{i=1}^{k-1}|S_i| \right]
		\\
		&\geq
		\mathbb{P}_c
        \Bigg[\forall k \in [h]. \;\;\; 
        \|
			q_k (\cdot | \pi_c, P^c) - q_k(\cdot | \pi_c, \widehat{P}^c)		
		\|_1
		\leq 
		\frac{4 \rho |S|}{1 - \rho^2 |S|^2}k
		+
		\beta \sum_{i=1}^{k-1}|S_i|\; \text{ and } 
		\\
        & (1) + (2)
		\leq
		\frac{4 \rho |S|}{1 - \rho^2 |S|^2}(h+1)
		+
		\beta \sum_{i=1}^{h}|S_i|
		\Bigg]
		\\
		&\geq
		\mathbb{P}_c
        \left[\forall k \in [h]. \;\;\; 		
        \|
			q_k (\cdot | \pi_c, P^c) - q_k(\cdot | \pi_c, \widehat{P}^c)		
		\|_1
		\leq 
		\frac{4 \rho |S|}{1 - \rho^2 |S|^2}k
		+
		\beta \sum_{i=1}^{k-1}|S_i|\; 
		\text{ and } 
        (2) \leq \frac{4 \rho |S|}{1 - \rho^2 |S|^2} +\beta |S_h|
		\right]
		\\
		&\geq
		1 - 
        \left(
            \frac{\epsilon_P}{\rho}
            |A|\sum_{i=0}^{h-1}|S_i||S_{i+1}|
            + 
            \gamma
            \sum_{i=1}^{h-1}|S_{i}|
            +
            \frac{\epsilon_P}{\rho}
            |A||S_h||S_{h+1}|
            + \gamma |S_h|
        \right)
        =
		1 - 
        \left(
            \frac{\epsilon_P}{\rho}
            |A|\sum_{i=0}^{h}|S_i||S_{i+1}|
            + 
            \gamma
            \sum_{i=1}^{h}|S_{i}|
        \right),
    \end{align*}
    as stated.
   \endgroup

\end{proof}

\begin{lemma}[expected value difference caused by dynamics approximation]\label{lemma: expected error from dynamics l_1 UCDD}
    
    Then, under the good events $G_1$, $G_2$ and $G_3$, for every context-dependent policy $\pi= (\pi_c)_{c \in \mathcal{C}}$ we have
    \begin{align*}
        \mathbb{E}_{c \sim \mathcal{D}}
        [|V^{\pi_c}_{\mathcal{M}(c)}(s_0) - V^{\pi_c}_{\widetilde{\mathcal{M}}(c)}(s_0)|]
        \leq 
	   \frac{4 \rho |S|}{1 - \rho^2 |S|^2}H^2
	   +
	   \beta|S|H
	   +
	   H|S|^2|A|\frac{\epsilon_P}{\rho}
	   +
	   \gamma H|S|,
    \end{align*}
    where $\rho \in [0, \frac{1}{|S|})$ and $\beta \in (0,1]$ for which $\beta \geq 2H\frac{4 \rho |S|}{1 - \rho^2 |S|^2}$. 
\end{lemma}

\begin{proof}
    Recall that the true rewards function is not defined for $s_{sink}$, since $s_{sink} \notin S$. For the intermediate MDP $\widetilde{\mathcal{M}}(c)$ we extended  $r^c$ to $s_{sink}$ by defining 
    $ \forall a\in A.\;\; r^c(s_{sink},a)=0$ for every context $c \in \mathcal{C}$. 
    Since $P^c$ is also not defined for $s_{sink}$, we can simply omit $s_{sink}$, as the second equality in the following derivation shows.
    
    Consider the following derivation for any fixed $c \in \mathcal{C}$. (Later we will take the expectation over $c$.)
    \begingroup
    \allowdisplaybreaks
    \begin{align*}
        &|V^{\pi_c}_{\mathcal{M}(c)}(s_0) 
        - 
        V^{\pi_c}_{\widetilde{\mathcal{M}}(c)}(s_0)|
        \\
        = &
        \left|
            \sum_{h=0}^{H-1}
            \sum_{s_h \in S_h}
            \sum_{a_h \in A}
            q_h(s_h, a_h|\pi_c, P^c)\cdot r^c(s_h, a_h)
            - 
            \sum_{h=0}^{H-1}
            \sum_{s_h \in S_h \cup \{s_{sink}\}}
            \sum_{a_h \in A}
            q_h(s_h, a_h|\pi_c, \widehat{P}^c) \cdot
            r^c(s_h, a_h)
        \right|
        \\
        \tag{ $r^c(s_{sink},a):= 0,\; \forall c,a$}
        = &
        \left|
            \sum_{h=0}^{H-1}
            \sum_{s_h \in S_h}
            \sum_{a_h \in A}
            q_h(s_h, a_h|\pi_c, P^c)\cdot r^c(s_h, a_h)
            -
            \sum_{h=0}^{H-1}
            \sum_{s_h \in S_h}
            \sum_{a_h \in A}
            q_h(s_h, a_h|\pi_c, \widehat{P}^c) \cdot
            r^c(s_h, a_h)
        \right|
        \\
        = &
        \left|
            \sum_{h=0}^{H-1}
            \sum_{s_h \in S_h}
            \sum_{a_h \in A}
            (q_h(s_h, a_h|\pi_c, P^c) - q_h(s_h, a_h|\pi_c, \widehat{P}^c))
            r^c(s_h, a_h)
        \right|
        \\
        \leq &
        \sum_{h=0}^{H-1}
        \sum_{s_h \in S_h}
        \sum_{a_h \in A}
        |q_h(s_h, a_h|\pi_c, P^c) - q_h(s_h, a_h|\pi_c, \widehat{P}^c)|
        \underbrace{|r^c(s_h, a_h)|}_{\leq 1}
        \\
        \leq &
        \sum_{h=0}^{H-1}
        \sum_{s_h \in S_h}
        \sum_{a_h \in A}
        |q_h(s_h, a_h|\pi_c, P^c) - q_h(s_h, a_h|\pi_c, \widehat{P}^c)|
        \\
        = &
        \sum_{h=0}^{H-1}
        \sum_{s_h \in S_h}
        \sum_{a_h \in A}
        \pi_c(a_h | s_h)
        |q_h(s_h |\pi_c, P^c) - q_h(s_h|\pi_c, \widehat{P}^c)|
        \\
        = &
        \sum_{h=0}^{H-1}
        \sum_{s_h \in S_h}
        |q_h(s_h |\pi_c, P^c) - q_h(s_h|\pi_c, \widehat{P}^c)|         
        \underbrace{\sum_{a_h \in A}
        \pi_c(a_h | s_h)}_{=1}
       \\
       = &
        \sum_{h=0}^{H-1}
        \sum_{s_h \in S_h}
        |q_h(s_h |\pi_c, P^c) - q_h(s_h|\pi_c, \widehat{P}^c)|
        \\
        = &
        \sum_{h=0}^{H-1}
        \|q_h(\cdot |\pi_c, P^c) - q_h(\cdot|\pi_c, \widehat{P}^c)\|_1 .
    \end{align*}
    \endgroup
    
    Denote by $G_8$ the good event
	\[
        \forall h \in [H]:
        \|q_h(\cdot |\pi_c, P^c) - q_h(\cdot|\pi_c, \widehat{P}^c)\|_1 
        \leq
        \frac{4 \rho |S|}{1 - \rho^2 |S|^2} h
		+
		\beta \sum_{i=1}^{h-1}|S_i|,
    \]
    and denote by $\overline{G_8}$ its complementary event.
    
    By Lemma~\ref{lemma: UCDD l_1 occ measure diff} we have
    \[
        \mathbb{P}_c[G_8] 
        \geq
        1 
        - 
        \left(
            \frac{\epsilon_P}{\rho}
            |A|\sum_{i=0}^{H-1}|S_i||S_{i+1}|
            + 
            \gamma
            \sum_{i=1}^{H-1}|S_{i}|
        \right)
        \geq
        1 - \left( |S|^2|A|\frac{\epsilon_P}{\rho} +|S|\gamma \right).
     \]

    If $G_8$ holds, then
    \begingroup
    \allowdisplaybreaks
	\begin{align*}
	   \sum_{h=0}^{H-1}
	   \|q_h(\cdot |\pi_c, P^c) - q_h(\cdot|\pi_c, \widehat{P}^c)\|_1 
	   \leq &
	   \sum_{h=0}^{H-1}
	   \frac{4 \rho |S|}{1 - \rho^2 |S|^2} h
	   +
	   \beta \sum_{i=1}^{h-1}|S_i|
	   \\
	   \leq &
	   \sum_{h=0}^{H-1}
	   \frac{4 \rho |S|}{1 - \rho^2 |S|^2}H
	   +
	   \beta|S|
	   \leq 
	   \frac{4 \rho |S|}{1 - \rho^2 |S|^2}H^2
	   +
	   \beta|S|H.
	\end{align*}
	\endgroup
    
    Otherwise,
    \begin{align*}
        \sum_{h=0}^{H-1}
        \|q_h(\cdot |\pi_c, P^c) - q_h(\cdot|\pi_c, \widehat{P}^c)\|_1 
	   \leq
	   \sum_{h=0}^{H-1}
	   1
	   \leq
	   H.
    \end{align*}
    Using total expectation low we obtain
    
    \begingroup
    \allowdisplaybreaks
    \begin{align*}
        &\mathbb{E}_{c \sim \mathcal{D}}
        [|V^{\pi_c}_{\mathcal{M}(c)}(s_0) - V^{\pi_c}_{\widetilde{\mathcal{M}}(c)}(s_0)|]
        \\
        \leq &
        \mathbb{P}[G_8]
        \mathbb{E}_{c \sim \mathcal{D}}
        [|V^{\pi_c}_{\mathcal{M}(c)}(s_0) - V^{\pi_c}_{\widetilde{\mathcal{M}}(c)}(s_0)||G_8]
        +
        \mathbb{P}[\overline{G_8}]
        \mathbb{E}_{c \sim \mathcal{D}}
        [|V^{\pi_c}_{\mathcal{M}(c)}(s_0) - V^{\pi_c}_{\widetilde{\mathcal{M}}(c)}(s_0)||\overline{G_8}]
        \\
        \leq &
        \mathbb{E}_{c \sim \mathcal{D}}
        [|V^{\pi_c}_{\mathcal{M}(c)}(s_0) - V^{\pi_c}_{\widetilde{\mathcal{M}}(c)}(s_0)||G_8]
        +
        \mathbb{P}[\overline{G_8}]
        \mathbb{E}_{c \sim \mathcal{D}}
        [|V^{\pi_c}_{\mathcal{M}(c)}(s_0) - V^{\pi_c}_{\widetilde{\mathcal{M}}(c)}(s_0)||\overline{G_8}]
        \\
        \leq &
	   \frac{4 \rho |S|}{1 - \rho^2 |S|^2}H^2
        +
	   \beta|S|H
	   +
	   \mathbb{P}[\overline{G_8}] H
	   \\
	   \leq &
	   \frac{4 \rho |S|}{1 - \rho^2 |S|^2}H^2
	   +
	   \beta|S|H
	   +
	   H|S|^2|A|\frac{\epsilon_P}{\rho}
	   +
	   \gamma H|S|.
    \end{align*}
    \endgroup
	
\end{proof}

\begin{corollary}\label{corl: expected error of dynamics for the chosen parameters UCDD L_1}
    Under the good events $G_1$, $G_2$ and $G_3$,
    for the parameter choice 
    $\gamma = \frac{\epsilon}{20 |S|H} $,
    $\beta =  \frac{\epsilon}{20|S|H} $, 
    $\rho = \frac{\beta}{16 |S|H} $,  
    $\epsilon_P = \frac{ \epsilon^2}{10\cdot 16 \cdot 20  |A| |S|^4 H^3}$,
    we have for every policy $\pi$ that
    \begin{align*}
        \mathbb{E}_{c \sim \mathcal{D}}
        [|V^{\pi_c}_{\mathcal{M}(c)}(s_0) - V^{\pi_c}_{\widetilde{\mathcal{M}}(c)}(s_0)|]
        \leq 
	   \frac{\epsilon}{40|S|}
	   +
	   \frac{2\epsilon}{10}
	   \leq 0.225 \epsilon.
    \end{align*}
\end{corollary}

\begin{proof}
    Implied by assigning the detailed parameters to the results of Lemma~\ref{lemma: expected error from dynamics l_1 UCDD}.
\end{proof}

\subsubsection{Analysis of the Error Caused by the Rewards Approximation Under the Good Events}\label{subsubsec:l-1-rewards-error}

Recall that for every $c \in \mathcal{C}$, define the following two MDPs. The intermediate MDP
    $
        \widetilde{M}(c) 
        = 
        (S\cup \{s_{sink}\}, A, \widehat{P}^c, r^c)
    $,
and the approximated MDP 
    $
        \widehat{M}(c)
        = 
        (S\cup \{s_{sink}\}, A, \widehat{P}^c,\widehat{r}^c)
    $
where $r^c$ is the true rewards function extended to $s_{sink}$ by defining $r^c(s_{sink},a):= 0 ,\;\; \forall c \in \mathcal{C}, a \in A$.  
$\widehat{P}^c$, $\widehat{r}^c$ are the approximation of the dynamics and the rewards as defined algorithm~\ref{alg: EXPLOIT-UCDD}.

\begin{lemma}[expected value difference caused by rewards approximation]\label{lemma: expected error from rewards l_1 UCDD}
    
    Then, under the good events $G_1, G_2$ and $G_4$, for every context-dependent policy $\pi=(\pi_c)_{c \in \mathcal{C}}$ we have
    \begin{align*}
        \mathbb{E}_{c \sim \mathcal{D}}
        [|V^{\pi_c}_{\widetilde{M}(c)}(s_0) - V^{\pi_c}_{\widehat{M}(c)}(s_0)|]
        \leq
        \alpha_1 H 
        + 
        2(\epsilon_R |S||A|)^\frac{1}{2} H
        +
        \beta |S|
        +
        \gamma |S|H
    \end{align*}
    where $\alpha_1 := \max_{h \in [H-1]}\alpha_1(\mathcal{F}^R_h)$.
\end{lemma}

\begin{proof}
    Recall that $\widehat{r}^c(s_{sink},a) : =0 , \forall c \in \mathcal{C}, a \in A$ by definition.
    In addition, since $r^c$ is the true rewards function and $s_{sink}\notin S$, $r^c$ is not defined for $s_{sink}$. We naturally extended it to $s_{sink}$ by defining $r^c(s_{sink},a) :=0 , \forall c \in \mathcal{C}, a \in A$. Hence, we can simply ignore $s_{sink}$ as the following analysis shows.

    
    Let us recall the definition of the following subsets of $S_h$ for every $h \in [H-1]$ and a given context $c \in \mathcal{C}$.
     \begin{enumerate}
       \item $B^{h, c}_1 = 
       \{s_h \in  S_h :
           s_h \in \widehat{S}^{\gamma,\beta}_h  \text{ and } c \in \widehat{C}^{\beta}(s_h)\}$.
       \item $B^{h, c}_2 = 
       \{ s_h \in  S_h :
       s_h \in \widehat{S}^{\gamma,\beta}_h \text{ and } c \notin \widehat{C}^{\beta}(s_h)\}$.
       \item $B^{h, c}_3 = 
       \{s_h \in  S_h :
       s_h \notin \widehat{S}^{\gamma,\beta}_h \text{ and } c \notin \widehat{C}^{\beta}(s_h)\}$.
       \item $B^{h, c}_4 = 
       \{ s_h \in  S_h :
       s_h \notin \widehat{S}^{\gamma,\beta}_h \text{ and } c \in \widehat{C}^{\beta}(s_h)\}$.      
   \end{enumerate}
   Clearly, $\cup_{i=1}^4 B^h_i = S_h$.
   
    By definition $s_h \in B^{h, c}_1$ if and only if for every action $a_h \in A$ we have that $(c, s_h, a_h) \in  \mathcal{X}^{\gamma, \beta}_h $.

    For $s_h\not\in \widetilde{S}^{\gamma,\beta}_h$ we have that $\mathbb{P}_c[c \in \widehat{\mathcal{C}}^{ \beta}(s_h)]<\gamma$, hence,
    \begin{align*}
        \mathbb{P}_c[\exists h \in [H-1] : B^{h, c}_4 \neq \emptyset]
        \;\;=\;\;
        \mathbb{P}_c[\exists h\in [H-1], s_h \in S_h: 
        s_h\not\in \widehat{S}^{\gamma,\beta}_h, \text{ and } c \in \widehat{\mathcal{C}}^{\beta}(s_h) ]
        \;\;<\;\;
        \gamma|S|
    \end{align*}


    Fix a context-dependent policy $\pi$. The following holds for any given context $c$. (Later we will take the expectation over $c$).
    \begingroup
    \allowdisplaybreaks
    \begin{align*}
        |V^{\pi_c}_{\widetilde{M}(c)}(s_0) - V^{\pi_c}_{\widehat{M}(c)}(s_0)|
        &=
        |
            \sum_{h=0}^{H-1}
            \sum_{s_h \in S_h \cup \{s_{sink}\}}
            \sum_{a_h \in A}
            q_h(s_h,a_h|\pi_c, \widehat{P}^c)
            (r^c(s_h,a_h) - \widehat{r}^c(s_h,a_h))
        |
        \\
        \tag{By definition, $r^c(s_{sink},a) = \widehat{r}^c(s_{sink},a) = 0,\;\; \forall c \in \mathcal{C}, a \in A$.}
        &=
        |
            \sum_{h=0}^{H-1}
            \sum_{s_h \in S_h}
            \sum_{a_h \in A}
            q_h(s_h,a_h|\pi_c, \widehat{P}^c)
            (r^c(s_h,a_h) - \widehat{r}^c(s_h,a_h))
        |
        \\
        &\leq
        \sum_{h=0}^{H-1}
        \sum_{s_h \in S_h}
        \sum_{a_h \in A}
        q_h(s_h,a_h|\pi_c, \widehat{P}^c) 
        |r^c(s_h,a_h) - \widehat{r}^c(s_h,a_h)|
        \\
        &=
        \underbrace
        {
            \sum_{h=0}^{H-1}
            \sum_{s_h \in B^{h, c}_1}
            \sum_{a_h \in A}
            q_h(s_h,a_h|\pi_c,\widehat{P}^c) 
            |r^c(s_h,a_h) - \widehat{r}^c(s_h,a_h)|
        }_{(1)}
        \\
        &+
        \underbrace
        {
            \sum_{h=0}^{H-1}
            \sum_{s_h \in B^{h, c}_2 \cup B^{h, c}_3}
            \sum_{a_h \in A}
           q_h(s_h,a_h|\pi_c, \widehat{P}^c) 
            |r^c(s_h,a_h) - \widehat{r}^c(s_h,a_h)|
        }_{(2)}
        \\
        &+
        \underbrace
        {
            \sum_{h=0}^{H-1}
            \sum_{s_h \in B^{h, c}_4}
            \sum_{a_h \in A}
             q_h(s_h,a_h|\pi_c, \widehat{P}^c) 
            |r^c(s_h,a_h) - \widehat{r}^c(s_h,a_h)|
        }_{(3)}.
    \end{align*}
    \endgroup
    We bound $(1)$, $(2)$ and $(3)$ separately.
    
    For $(1)$,
    under the good events $G_1$, $G_2$ , $G_3$ and $G_4$, we have for all $h \in [H-1]$ that
    \[
        \mathbb{E}_{ \mathcal{D}^R_h}[|f^R_h(c, s_h, a_h) - r^c(s_h,a_h)| -\alpha_1(\mathcal{F}^R_h)]
        \leq 
        \epsilon_R. 
    \]
    
    Since $\mathbb{E}_{ \mathcal{D}^R_h}[|f^R_h(c,s_h, a_h) - r^c(s_h,a_h)| ]\geq  \alpha_1(\mathcal{F}^R_h)$, for all $h \in [H-1]$  and $\xi \in (0,1]$ we obtain using Markov's inequality that 
    \begin{align*}
        &\mathbb{P}
        [|f^R_h(c,s_h, a_h) - r^c(s_h,a_h)|
        \geq \alpha_1(\mathcal{F}^R_{s_h, a_h}) + \xi\;
        \Big|(c,s_h, a_h) \in \mathcal{X}^{\gamma, \beta}_h]
        \leq
        \frac{\epsilon_R}{\xi}.
    \end{align*}
    Hence,
    \begin{align*}
        \mathbb{P}
        [|f^R_h(c,s_h, a_h) - r^c(s_h,a_h)|\;
        \leq \alpha_1(\mathcal{F}^R_{s_h, a_h}) + \xi
        \Big|(c,s_h, a_h) \in \mathcal{X}^{\gamma, \beta}_h]
        \geq
        1 - \frac{\epsilon_R}{\xi}.
    \end{align*}

    Let $G_5$ denote the following good event.
    \begin{align*}
        \forall h \in [H-1]
        \;
        \forall s_h \in B^{h, c}_1
        \;
        \forall a \in A.\;
        |f^R_h(c, s_h, a_h) - r^c(s_h,a_h) | \leq  \alpha_1(\mathcal{F}^R_{s_h, a_h}) + \xi
    \end{align*}
    and denote by $\overline{G_5}$ the complementary event.
    By the above and an union bound over ${(s_h, a_h) \in B^{h,c}_1 \times A}$ for all $h \in [H-1]$ we have,
    $
        \mathbb{P}_c[G_5] \geq 1 -  \frac{\epsilon_R}{\xi}|S||A|
    $
    and
    $
        \mathbb{P}_c[\overline{G_5}] \leq \frac{\epsilon_R}{\xi}|S||A|
    $.

    If $G_5$ holds then,
    \begingroup
    \allowdisplaybreaks
    \begin{align*}
        (1)
        &=
        \sum_{h=0}^{H-1}
        \sum_{s_h \in B^{h, c}_1}
        \sum_{a_h \in A}
        q_h(s_h,a_h|\pi_c, \widehat{P}^c) 
        |r^c(s_h,a_h) - \widehat{r}^c(s_h,a_h)|
        \\
        &=
        \sum_{h=0}^{H-1}
        \sum_{s_h \in B^{h, c}_1}
        \sum_{a_h \in A}
        \pi_c(a_h | s_h)
        q_h(s_h |\pi_c, \widehat{P}^c) 
        |f^R_h(c,s_h,a_h)- r^c(s_h,a_h)|  
        \\
        &\leq
        \sum_{h=0}^{H-1}
        \sum_{s_h \in B^{h, c}_1}
        \sum_{a_h \in A}
        \pi_c(a_h | s_h)
        q_h(s_h |\pi_c, \widehat{P}^c) 
        (\alpha_1(\mathcal{F}^R_h) + \xi)
        \\
        &\leq
        \sum_{h=0}^{H-1}
        \sum_{s_h \in B^{h, c}_1}
        \sum_{a_h \in A}
        \pi_c(a_h | s_h)
        q_h(s_h |\pi_c, \widehat{P}^c) 
        (\alpha_1  + \xi)
        \leq
        \alpha_1 H + \xi H .
    \end{align*}
    \endgroup
    Otherwise,
    \begin{align*}
        (1)
        &=
        \sum_{h=0}^{H-1}
        \sum_{s_h \in B^{h, c}_1}
        \sum_{a_h \in A}
        q_h(s_h,a_h|\pi_c, \widehat{P}^c) 
        \underbrace{|r^c(s_h,a_h) - \widehat{r}^c(s_h,a_h)|}_{\leq 1}
        \leq H .
    \end{align*}
 
    Thus,
    \begin{align*}
        \mathbb{E}_{c \sim \mathcal{D}}[(1)]
        &\leq
        \alpha_1 H 
        + 
        \xi H
        +
        \frac{\epsilon_R}{\xi} |S| |A| H.
    \end{align*}

    For $(2)$, consider the following derivation:
    \begingroup
    \allowdisplaybreaks
    \begin{align*}
        (2)
        &=
        \sum_{h=0}^{H-1}
        \sum_{s_h \in B^{h, c}_2 \cup B^{h, c}_3}
        \sum_{a_h \in A}
        q_h(s_h,a_h|\pi_c, \widehat{P}^c) 
        \underbrace{|r^c(s_h,a_h) - \widehat{r}^c(s_h,a_h)|}_{\leq 1}
        \\
        &\leq
        \sum_{h=0}^{H-1}
        \sum_{s_h \in B^{h, c}_2 \cup B^{h, c}_3}
        \sum_{a_h \in A}
        q_h(s_h,a_h|\pi_c, \widehat{P}^c) 
        \\
        &=
        \sum_{h=0}^{H-1}
        \sum_{s_h \in B^{h, c}_2 \cup B^{h, c}_3}
        \sum_{a_h \in A}
        \pi_c(a_h | s_h)
        q_h(s_h|\pi_c, \widehat{P}^c) 
        \\
        &=
        \sum_{h=0}^{H-1}
        \sum_{s_h \in B^{h, c}_2 \cup B^{h, c}_3}
        q_h(s_h|\pi_c, \widehat{P}^c) 
        \underbrace{\sum_{a_h \in A}
         \pi_c(a_h | s_h)}_{=1}
         \\
         &=
        \sum_{h=0}^{H-1}
        \sum_{s_h \in B^{h, c}_2 \cup B^{h, c}_3}
        \underbrace{q_h(s_h|\pi_c, \widehat{P}^c) }_{\leq q_h(s_h|\widehat{\pi}^c_{s_h}, \widehat{P}^c) <\beta}
        \beta|S|.
    \end{align*}
    \endgroup
    Thus,
    \[
        \mathbb{E}_{c \sim \mathcal{D}}[(2)] \leq \beta|S|.
    \]
    
  For $(3)$, let $G_6$ denote the good event in which $\forall h \in [H -1 ], B^{h, c}_4 = \emptyset$. Denote by $\overline{G_6}$ the complement event of $G_5$.
  
  We showed that $\mathbb{P}_c[G_6] \geq 1 - \gamma|S|$ thus $\mathbb{P}_c[\overline{G_6}] \leq \gamma|S|$.
  
  If $G_6$ holds, then $(3) =0$. Othwewise,
  \begingroup
  \allowdisplaybreaks
  \begin{align*}
    (3)
    &=
    \sum_{h=0}^{H-1}
    \sum_{s_h \in B^{h, c}_4}
    \sum_{a_h \in A}
    q_h(s_h,a_h|\pi_c, \widehat{P}^c) 
    \underbrace{|r^c(s_h,a_h) - \widehat{r}^c(s_h,a_h)|}_{\leq 1}
    \\
    &\leq
    \sum_{h=0}^{H-1}
    \underbrace
    {
        \sum_{s_h \in B^{h, c}_4}
        \sum_{a_h \in A}
       q_h(s_h,a_h|\pi_c, \widehat{P}^c)
    }_{\leq 1}
    \leq H.
  \end{align*}
  \endgroup

  Using total expectation we obtain
  \begin{align*}
      \mathbb{E}_{c \sim \mathcal{D}}[(3)]
      &\leq
      \gamma|S|H.
  \end{align*}

    Overall,
    by linearity of expectation and the above, we obtain for $\xi = (\epsilon_R |S||A|)^\frac{1}{2}$ that
    \begin{align*}
        \mathbb{E}_{c \sim \mathcal{D}}
        [|V^{\pi_c}_{\widetilde{M}(c)}(s_0) - V^{\pi_c}_{\widehat{M}(c)}(s_0)|]
        &\leq
        \mathbb{E}_{c \sim \mathcal{D}}[(1)] + \mathbb{E}_{c \sim \mathcal{D}}[(2)] + \mathbb{E}_{c \sim \mathcal{D}}[(3)]
        \\
        &\leq
        \alpha_1 H 
        + 
        \xi H
        +
        \frac{\epsilon_R}{\xi} |S| |A| H
        +
        \beta |S|
        +
        \gamma |S|H
        \\
        &=
        \alpha_1 H 
        + 
        2(\epsilon_R |S||A|)^\frac{1}{2} H
        +
        \beta |S|
        +
        \gamma |S|H,
    \end{align*}
as stated.
\end{proof}

\begin{corollary}\label{corl: expected error of rewards for the chosen parameters UCDD L_1}
    Under the good events $G_1, G_2$ and $G_4$, 
    for $\gamma = \frac{\epsilon}{20 |S|H}$, $\beta =  \frac{\epsilon}{20|S|H}$.
    and $\epsilon_R = \frac{\epsilon^2}{20^2 |S||A| H^2}$, we have for any context-dependent policy $\pi=(\pi_c)_{c \in \mathcal{C}}$ that
    \begin{align*}
        \mathbb{E}_{c \sim \mathcal{D}}
        [|V^{\pi_c}_{\widetilde{M}(c)}(s_0) - V^{\pi_c}_{\widehat{M}(c)}(s_0)|]
        \leq
        \alpha_1 H 
        +
        \frac{3\epsilon}{20}
        +
        \frac{\epsilon}{20 H}
    \end{align*}
\end{corollary}

\begin{proof}
    Implied by assigning the detailed parameters to the results of Lemma~\ref{lemma: expected error from rewards l_1 UCDD}.
\end{proof}


    

\subsubsection{Combining Value Differences Caused by Dynamics and Rewards Approximation to Sub-optimality Bound}\label{subsubsec:l-1-suboptimality-bound}

Let the following set of selected parameters, called SP1, be
    \begin{itemize}
        \item $\gamma = \frac{\epsilon}{20 |S|H} \in (0,1)$.
        \item $\beta =  \frac{\epsilon}{20|S|H} \in (0,1)$.
        \item $\rho = \frac{\beta}{16 |S|H} \in (0, \frac{1}{|S|})$.
        \item $\epsilon_P = \frac{ \epsilon^2}{10\cdot 16 \cdot 20  |A| |S|^4 H^3}$.
        \item $\epsilon_R = \frac{\epsilon^2}{20^2 |S||A| H^2}$.
    \end{itemize}
We remark that for our choice of $\rho$ and $\beta$ it holds that $\rho \in [0,\frac{1}{|S|})$ and $\beta \geq 2H\frac{4 \rho |S|}{1 - \rho^2 |S|^2}$.    

\begin{lemma}[expected value difference]\label{lemma: expected gap for general pi l_1}
    Under the good events $G_1$, $G_2$,$G_3$ and $G_4$, we have for every policy context-dependent policy $\pi= (\pi_c)_{c \in \mathcal{C}}$ that
    \[
        \mathbb{E}_{c \sim \mathcal{D}}
        [| V^{\pi_c}_{\mathcal{M}(c)}(s_0) -  V^{\pi_c}_{\widehat{\mathcal{M}}(c)}(s_0)|]
        \leq
        \alpha_1 H + \frac{1}{2} \epsilon
    \]
    where $\mathcal{M}(c)$ is the true MDP associated with the context $c$ and $\widehat{\mathcal{M}}(c)$ is it's the approximated model,
    using parameters SP1. 
\end{lemma}

\begin{proof}
    For a fixed $c \in \mathcal{C}$, consider the intermediate MDP $\widetilde{\mathcal{M}}(c) = (S, A, \widehat{P}^c, r^c, H, s_0)$. 
    Using triangle inequality and linearity of expectation we obtain
    \begin{align*}
        \mathbb{E}_{c \sim \mathcal{D}}
        [
            | V^{\pi_c}_{\mathcal{M}(c)}(s_0) 
            -  
            V^{\pi_c}_{\widehat{\mathcal{M}}(c)}(s_0)|
        ]
        &=
        \mathbb{E}_{c \sim \mathcal{D}}
        [| 
            V^{\pi_c}_{\mathcal{M}(c)}(s_0) 
            - 
            V^{\pi_c}_{\widetilde{\mathcal{M}}(c)}(s_0)
            +
            V^{\pi_c}_{\widetilde{\mathcal{M}}(c)}(s_0)
            -
            V^{\pi_c}_{\widehat{\mathcal{M}}(c)}(s_0)
        |]
        \\
        &\leq
        \mathbb{E}_{c \sim \mathcal{D}}
        [
            \underbrace
            {| V^{\pi_c}_{\mathcal{M}(c)}(s_0) 
            - 
            V^{\pi_c}_{\widetilde{\mathcal{M}}(c)}(s_0)|}_{(1)}
        ]
        +
        \mathbb{E}_{c \sim \mathcal{D}}
        [
            \underbrace
            {| V^{\pi_c}_{\widetilde{\mathcal{M}}(c)}(s_0)
            -
            V^{\pi_c}_{\widehat{\mathcal{M}}(c)}(s_0)|}_{(2)}
        ]
    \end{align*}
    
    By Lemma~\ref{lemma: expected error from dynamics l_1 UCDD} we have
    \begin{align*}
        \mathbb{E}_{c \sim \mathcal{D}}
        [|V^{\pi_c}_{\mathcal{M}(c)}(s_0) - V^{\pi_c}_{\widetilde{\mathcal{M}}(c)}(s_0)|]
        \leq 
	   \frac{4 \rho |S|}{1 - \rho^2 |S|^2}H^2
	   +
	   \beta|S|H
	   +
	   H|S|^2|A|\frac{\epsilon_P}{\rho}
	   +
	   \gamma H|S|.
    \end{align*}
    
    By Lemma~\ref{lemma: expected error from rewards l_1 UCDD} we have
    \begin{align*}
        \mathbb{E}_{c \sim \mathcal{D}}
        [|V^{\pi_c}_{\widetilde{M}(c)}(s_0) - V^{\pi_c}_{\widehat{M}(c)}(s_0)|]
        \leq
        \alpha_1 H 
        + 
        2 (\epsilon_R |S||A|)^{\frac{1}{2}}H
        +
        \beta |S|
        +
        \gamma
        |S|H.
    \end{align*}
    
    Overall,
    \begin{align*}
        \mathbb{E}_{c \sim \mathcal{D}}
        [| 
            V^{\pi_c}_{\mathcal{M}(c)}(s_0) 
            -  
            V^{\pi_c}_{\widehat{\mathcal{M}}(c)}(s_0)
        |]
        &=
	    \frac{4 \rho |S|}{1 - \rho^2 |S|^2}H^2
        +
        |S|^2|A|H\frac{\epsilon_P}{\rho}
        + 
        2\gamma|S|H
        +
        2\beta|S|H
        +
        \alpha_1 H 
        +
        2 (\epsilon_R |S||A|)^{\frac{1}{2}}H
    \end{align*}
 For the parameters set SP1   
we have that $\beta < \frac{1}{2 |S|}$, which implies that $0 < \rho < \frac{1}{|S|}$.
We also have that
\begin{align*}
    2H\frac{4 \rho |S|}{1- \rho^2 |S|^2} 
    =
    \frac{8H |S|\frac{\beta}{16 |S|H}}{1- \frac{\beta^2 |S|^2}{2^8 |S|^2 H^2}}
    =
    \frac{\frac{\beta}{2}}{1 - \underbrace{\frac{\beta^2}{2^8 H^2}}_{\leq 1/2}}
    \leq 2\frac{\beta}{2} = \beta.
\end{align*}
Hence, the constrains on $\rho$ and $\beta$ are both hold. 

Finally, 
    \begingroup
    \allowdisplaybreaks
    \begin{align*}
        \mathbb{E}_{c \sim \mathcal{D}}
        [| 
            V^{\pi_c}_{\mathcal{M}(c)}(s_0) 
            -  
            V^{\pi_c}_{\widehat{\mathcal{M}}(c)}(s_0)
        |]
        &=
	    \frac{4 \rho |S|}{1 - \rho^2 |S|^2}H^2
        +
        |S|^2|A|H
        \frac{\epsilon_P}
        {\rho}
        + 
        2\gamma
        |S|H
        +
        2\beta|S|H
        +
        \alpha_1 H 
        +
         2 (\epsilon_R |S||A|)^{\frac{1}{2}}H
        \\
        &\leq
        \frac{\frac{1}{4} \beta H}{1 - \underbrace{\frac{\beta^2}{2^8 H^2}}_{\leq 1/2}}
        +
        16|S|^3|A|H^2
        \frac{\epsilon_P}
        {\beta}  
        +
        2 \frac{\epsilon}{10}
        + 
        \alpha_1 H
        +
        \frac{\epsilon}{10}
        \\
        &\leq
        \frac{1}{2}\beta H
        +
        16|S|^3|A|H^2
        \frac{\epsilon_P}
        {\beta}  
        +
        2 \frac{\epsilon}{10}
        + 
        \alpha_1 H
        +
        \frac{\epsilon}{10}
        \\
        &\leq
        \frac{1}{2}\beta H
        +
        16 \cdot 20 |S|^4|A|H^3
        \frac{\epsilon_P}
        {\epsilon} 
        +
        3 \frac{\epsilon}{10}
        + 
        \alpha_1 H
        \\
        &=
        \frac{1}{2}\frac{\epsilon}{20|S|}
        +
        4 \frac{\epsilon}{10}
        + 
        \alpha_1 H
        \\
        &\leq 
        \frac{1}{2}\epsilon + \alpha_1 H,
    \end{align*}
    \endgroup
as stated.
\end{proof}

The following corollary shows that for our choice of parameters, all good events holds with high probability.
\begin{corollary}\label{corl: final probs l_1}
    Using parameters SP1 we have $\mathbb{P}[G_1, G_2,G_3, G_4] \geq 1- (\frac{\delta}{2}+ \frac{\epsilon}{10})$.
\end{corollary}
\begin{proof}
    By Corollary~\ref{corl: good events probs bound l_1} we have that $\mathbb{P}[G_1, G_2,G_3, G_4] \geq 1- (\frac{\delta}{8} + 3 \delta_1 H + \frac{\epsilon_P}{\rho}|S|^2|A|H)$. Hence by $\rho$, $\beta$, $\epsilon_P$ and $\delta_1$ choice we obtain
    \begingroup
    \allowdisplaybreaks
    \begin{align*}
        \mathbb{P}[G_1, G_2,G_3, G_4] 
        &\geq 
        1- (\frac{\delta}{8} + 3 \delta_1 H + \frac{\epsilon_P}{\rho}|S|^2|A|H)
        \\
        &=
        1- \frac{\delta}{2} - 16 |S|^3|A|H^2 \frac{\epsilon_P}{\beta}
        \\
        &=
        1- \frac{\delta}{2} - 16 \cdot 20 |S|^4|A|H^3\frac{\epsilon_P}{\epsilon} 
        \\
        &=
        1- \frac{\delta}{2} -\frac{\epsilon}{10},
    \end{align*}
    \endgroup
    as stated.
\end{proof}

Finally, the following theorem bound the expected sub-optimality of our approximated optimal policy $\widehat{\pi}^\star$. 
\begin{theorem}[expected suboptimality bound]\label{thm: UCDD opt policy l_1}
With probability at least $1-(\delta + \frac{\epsilon}{5})$ it holds that
    \[
        \mathbb{E}_{c \sim \mathcal{D}}
        [V^{\pi^\star_c}_{\mathcal{M}(c)}(s_0) - V^{\widehat{\pi}^\star_c}_{\mathcal{M}(c)}(s_0)]
        \leq 
        \epsilon  + 2\alpha_1 H,
    \]
where $\pi^\star=(\pi^\star_c)_{c \in \mathcal{C}}$ is the optimal policy for  $\mathcal{M}$ and $\widehat{\pi}^\star=(\widehat{\pi}^\star_c)_{c \in \mathcal{C}}$ is the optimal policy for $\widehat{\mathcal{M}}$.
\end{theorem}

\begin{proof}
Assume the good events $G_1$, $G_2$, $G_3$ and $G_4$ hold.

By Lemma~\ref{lemma: expected gap for general pi l_1},  we have for $\pi^\star$
    \begin{align*}
    \left|
        \mathbb{E}_{c \sim \mathcal{D}}
        [
            V^{\pi^\star_c}_{\mathcal{M}(c)}(s_0) 
            -  V^{\pi^\star_c}_{\widehat{\mathcal{M}}(c)}(s_0)
        ] 
    \right|
        \leq
         \mathbb{E}_{c \sim \mathcal{D}}
        [|
            V^{\pi^\star_c}_{\mathcal{M}(c)}(s_0) 
            -  V^{\pi^\star_c}_{\widehat{\mathcal{M}}(c)}(s_0)
        |] 
        \leq \frac{1}{2}\epsilon + \alpha_1 H,
    \end{align*}
    yielding,
    \begin{align*}
        \mathbb{E}_{c \sim \mathcal{D}}
        [
            V^{\pi^\star_c}_{\mathcal{M}(c)}(s_0)
        ]    
        -
        \mathbb{E}_{c \sim \mathcal{D}}
        [        
        V^{\pi^\star_c}_{\widehat{\mathcal{M}}(c)}(s_0)
        ] 
        \leq \frac{1}{2}\epsilon + \alpha_1 H.
    \end{align*}

    

    Similarly, we obtain for $\widehat{\pi}^\star$, 
    \begin{equation*}
        \mathbb{E}_{c \sim \mathcal{D}}
        [V^{\widehat{\pi}^\star_c}_{\widehat{\mathcal{M}}(c)}(s_0)]
        -
        \mathbb{E}_{c \sim \mathcal{D}}
        [V^{\widehat{\pi}^\star_c}_{\mathcal{M}(c)}(s_0)] 
        \leq
        \frac{1}{2}\epsilon + \alpha_1 H.
\end{equation*}

Since for all $c \in \mathcal{C}$, $\widehat{\pi}^\star_c$ is the optimal policy for $\widehat{\mathcal{M}}(c)$ we have $   
    V^{\widehat{\pi}^\star_c}_{\widehat{\mathcal{M}}(c)}(s_0)
    \geq 
    V^{\pi^\star_c}_{\widehat{\mathcal{M}}(c)}(s_0)
$
which implies that
\begin{equation*}
        \mathbb{E}_{c \sim \mathcal{D}}
        [V^{\pi^\star_c}_{\widehat{\mathcal{M}}(c)}(s_0)]
        -
        \mathbb{E}_{c \sim \mathcal{D}}
        [V^{\widehat{\pi}^\star_c}_{\widehat{\mathcal{M}}(c)}(s_0)]
        \leq
        0.
\end{equation*}
Since by Corollary~\ref{corl: final probs l_1} we have that $\mathbb{P}[G_1,G_2,G_3,G_4] \geq 1- (\frac{\delta}{2}+ \frac{\epsilon}{10})$,
the theorem implied by summing the above three inequalities.
\end{proof}

\subsubsection{Additional Lemmas for bounding the sample complexity for the \texorpdfstring{$\ell_1$}{Lg} loss}

\begin{lemma}\label{lemma: UCDD l_1 matrix diff-ver-2}
    Let $\rho \in [0,\frac{1}{|S|})$ and $h \in [H -1]$.
    Assume the good events $G_1, G_2^k, G_3^k, \; \forall k \in [h]$ hold, then it holds that
    \[
        \mathbb{P}\left[
        \|\widehat{P}^c(\cdot| s_h, a_h) - P^c(\cdot| s_h, a_h)\|_1
       \leq \frac{4 \rho |S|}{ 1 - \rho^2 |S|^2} \Big| (c,s_h,a_h) \in \widetilde{\mathcal{X}}^{\gamma,\beta}_h \right]
       \geq 1 - \frac{\epsilon_P }{\rho}|S_{h+1}|,
    \]
    where $\widehat{P}^c$ is the dynamics defined in Algorithm~\ref{alg: ACDD UCDD} and 
    \[
        \|\widehat{P}^c(\cdot| s_h, a_h) - P^c(\cdot| s_h, a_h)\|_1:=
        \sum_{s_{h+1} \in S_{h+1}}
        |\widehat{P}^c(s_{h+1}| s_h, a_h) - P^c(s_{h+1}| s_h, a_h)|
    \]
    (i.e., the entry of $s_{sink}$ in $\widehat{P}^c$ is ignored).
\end{lemma}
 
\begin{proof}
    
    
    We prove similarly to shown for Lemma~\ref{lemma: UCDD l_1 matrix diff}, when using the good events $G_3^k$ for all $k \in [h]$ guarantees for the distribution $\widetilde{D}^{\gamma,\beta}_h$ over $\widetilde{\mathcal{X}}^{\gamma,\beta}_h \times S_{h+1}$.

    Recall that for $(c, s_h, a_h) \in \widetilde{\mathcal{X}}^{\gamma,\beta}_h$ we have that $\widehat{P}^c(s_{sink}|s_h, a_h) = 0$ by $\widehat{P}^c$ definition. 
    In addition, the true dynamics $P^c$ is not defined for $s_{sink}$ since $s_{sink} \notin S$.  
    A natural extension of $P^c$ to $s_{sink}$ is by defining that $\forall (s,a) \in S \times A.\;\; P^c(s_{sink}|s,a) :=0$.
    By that extension, we have for all $(c, s_h, a_h) \in \widetilde{\mathcal{X}}^{\gamma,\beta}_h$ that
    $P^c(s_{sink}|s_h,a_h) = \widehat{P}^c(s_{sink}|s_h,a_h)=0$. Hence, we can simply ignore $s_{sink}$ in the following analysis.

    Under the good event $ G_3^h $, by Markov's inequality we have
    \begin{align*}
        &\mathop{\mathbb{P}}_{(c,s_h,a_h,s_{h+1})}
        \left[ |f^P_h(c, s_h, a_h, s_{h+1}) - P^c(s_{h+1}| s_h, a_h)| \geq \rho \Big| (c,s_h,a_h) \in \widetilde{\mathcal{X}}^{\gamma,\beta}_h \right]
        = 
        \\
        & =
        \mathbb{P}_{\mathcal{D}^P_h}[|f^P_h(c, s_h, a_h, s_{h+1}) - P^c(s_{h+1}| s_h, a_h)| \geq \rho 
        ]
        \\
        \tag{By Markov's inequality}
        & \leq  
        \frac{\mathbb{E}_{\mathcal{D}^P_h}
        [|f^P_h(c, s_h, a_h, s_{h+1}) - P^c(s_{h+1}| s_h, a_h)|]}{\rho}
        \\
        \tag{Under $G^3_h$}
        & \leq 
        \frac{\epsilon_P}{\rho}.
    \end{align*}
    Hence,
    \[
        \mathop{\mathbb{P}}_{(c,s_h,a_h,s_{h+1})}\left[|f^P_h(c, s_h, a_h, s_{h+1}) - P^c(s_{h+1}| s_h, a_h)| \leq \rho \Big| (c,s_h,a_h) \in \widetilde{\mathcal{X}}^{\gamma,\beta}_h \right]
        \geq 1 - \frac{\epsilon_P }{\rho}.
    \]
    Since $P^c(\cdot |s_h, a_h)$ is a distribution over $s_{h+1} \in S_{h+1}$, we have that $\sum_{s_{h+1} \in S_{h+1}} P^c(s_{h+1}|s_h, a_h) = 1$. 
    Thus, by union bound applied on $s_{h+1}\in S_{h+1}$, we obtain
    \begin{align*}
        \mathbb{P}_{(c, s_h, a_h) }
        \left[ 
            1 - \rho |S|\leq
            \sum_{s_{h+1}\in S_{h+1}} f^P_h(c, s_h, a_h, s_{h+1}) \leq
            1 + \rho |S|
        \Big|
        (c, s_h, a_h) \in  \widetilde{\mathcal{X}}^{\gamma, \beta}_h    
        \right]
        \geq 
        1 - \frac{\epsilon_P }{\rho}|S_{h+1}|.
    \end{align*}
     Hence, we further conclude 
     that 
    \begin{equation}\label{prob: lemma C.5 l_1-v2}
        \begin{split}
            &\mathop{\mathbb{P}}_{(c,s_h,a_h)}
            \left[  
                \forall s_{h+1} \in S_{h+1}.
                \frac{P^c(s_{h+1}|s_h, a_h) - \rho}{1 + \rho |S|}  \leq \underbrace{\frac{f^P_h(c,s_h,a_h,s_{h+1})}{\sum_{s'\in S_{h+1}} f^P_h(c,s_h,a_h,s')}}_{=\widehat{P}^c(s_{h+1}|s_h, a_h)}
                \leq
                \frac{P^c(s_{h+1}|s_h, a_h) + \rho}{1 - \rho |S|}
            \Big|
            (c, s_h, a_h) \in  \widetilde{\mathcal{X}}^{\gamma, \beta}_h   \right]
            \\
            & \geq 
            1 - \frac{\epsilon_P }{\rho}|S_{h+1}|.
        \end{split}
    \end{equation}
    Fix a tuple $(c, s_h, a_h) \in \widetilde{\mathcal{X}}^{\gamma, \beta}_h$ and assume the event of inequality~(\ref{prob: lemma C.5 l_1-v2}) holds.\\
    Denote ${S^+_{h+1} = \{s_{h+1} \in S_{h+1}:
    \widehat{P}^c(s_{h+1}|s_h, a_h) \geq  P^c(s_{h+1}|s_h, a_h )\}}$ and consider the following derivation.

    \begingroup
    \allowdisplaybreaks
    \begin{align*}
        \| 
            \widehat{P}^c(\cdot|s_h, a_h ) 
            -  
            P^c(\cdot|s_h, a_h )
        \|_1
        =&
        \sum_{s_{h+1} \in S_{h+1}}
        |\widehat{P}^c(s_{h+1}|s_h, a_h) -  P^c(s_{h+1}|s_h, a_h )|
        \\
        =&
        \sum_{s_{h+1} \in S^+_{h+1}}
        (\widehat{P}^c(s_{h+1}|s_h, a_h) -  P^c(s_{h+1}|s_h, a_h ))       
        \\
        &+
        \sum_{s_{h+1} \in  S_{h+1} \setminus S^+_{h+1}}
        (P^c(s_{h+1}|s_h, a_h) -  \widehat{P}^c(s_{h+1}|s_h, a_h ))
        \\  
        \leq&
        \sum_{s_{h+1} \in S^+_{h+1}}
        \left(\frac{P^c(s_{h+1}|s_h, a_h) + \rho}{1 - \rho |S|} 
        -  
        P^c (s_{h+1}|s_h, a_h )\right)
        \\
        &+
        \sum_{s_{h+1} \in  S_{h+1} \setminus S^+_{h+1}}
        \left( P^c (s_{h+1}|s_h, a_h )
        -
        \frac{P^c(s_{h+1}|s_h, a_h) - \rho}{1 + \rho |S|} \right)        
        \\
        =&
        \sum_{s_{h+1} \in S^+_{h+1}}
        \frac{P^c (s_{h+1}|s_h, a_h) + \rho - (1- \rho |S|)P^c (s_{h+1}|s_h, a_h) }{1 - \rho |S|} 
        \\
        &+
        \sum_{s_{h+1} \in  S_{h+1} \setminus S^+_{h+1}}
        \frac{-P^c (s_{h+1}|s_h, a_h) + \rho + (1 + \rho |S|)P^c (s_{h+1}|s_h, a_h) }{1 + \rho |S|}         
        \\
        =&
        \frac{1}{1- \rho |S|}
        \sum_{s_{h+1} \in S^+_{h+1}}
        (\rho + \rho|S|P^c (s_{h+1}|s_h,a_h))
        \\
        &+
        \frac{1}{1 + \rho |S|}
        \sum_{s_{h+1} \in  S_{h+1} \setminus S^+_{h+1}}
        (\rho + \rho|S|P^c (s_{h+1}|s_h,a_h))
        \\
        \leq&
        \frac{2 \rho |S|}{ 1- \rho |S|}
        +
        \frac{2 \rho |S|}{ 1 + \rho |S|}
        \\
        =&
        \frac{4 \rho |S|}{1 - \rho^2 |S|^2}.
    \end{align*}
    \endgroup
    
    By inequality~(\ref{prob: lemma C.5 l_1-v2}), the above holds with probability at least $1 - \frac{\epsilon_P }{\rho}|S_{h+1}|$ over $(c, s_h, a_h) \in \mathcal{X}^{\gamma,\beta}_h$. Hence the lemma follows.

\end{proof}

\begin{lemma}\label{lemma: for running time l_1}
    Fix $\beta \in (0,1]$ and $\rho \in [0,\frac{1}{|S|})$ such that $\beta \geq 2H \frac{4 \rho |S|}{ 1 - \rho^2 |S|^2}$.
    
    Then, for every (context-dependent) policy $\pi=(\pi_c)_{c \in \mathcal{C}}$, and a layer $h \in [H-1]$, 
     under the good events $G_1, G_2^i, G_3^i, \forall i \in [h-1]$ the following holds.
    \begin{align*}
        \mathbb{P}_c\left[
            \forall k \in [h], s_k \in S_k.\;\;
            q_k(s_k|\pi_c, P^c) \geq q_k(s_k|\pi_c, \widehat{P}^c) -\frac{4 \rho|S|}{1 - \rho^2|S|^2}k
        \right]
        \geq 
        1 - |A|\;\sum_{ k = 0 }^{h-1}\frac{\epsilon_P }{\rho}|S_k||S_{k+1}|
    \end{align*}
    
\end{lemma}

\begin{proof}
    For every context $c \in \mathcal{C}$ and define the dynamics $\widetilde{P}^c$ over
    $S \cup \{s_{sink}\} \times A$:
    \[
        \forall (s,a) 
        \in 
        S\cup \{s_{sink}\} \times A:
        \widetilde{P}^c (s| s_{sink}, a) 
        = 
        \begin{cases}
        1 &, \text{if } s = s_{sink}\\
        0 &, \text{otherwise}
        \end{cases}
    \]
    In addition we define
    \begin{align*}
        &\forall k  \in [h-1],\;\; \forall (s_k, a_k, s_{k+1})
        \in \widetilde{S}^{\gamma, \beta}_k \times A \times S_{k+1}:
        \\
        & \widetilde{P}^c(s_{k+1}| s_k, a_k)
        =
        \begin{cases}
        P^c(s_{k+1}|s_k,a_k)
        &, \text{if } c \in \widehat{C}^{\beta}({s_k})\\
        0 &, \text{otherwise}
        \end{cases}
        \\
        & \widetilde{P}^c(s_{sink}| s_k, a_k)
        =
        \begin{cases}
            0 &, \text{if } c \in \widehat{\mathcal{C}}^{\beta}({s_k})\\
            1 &, \text{otherwise}
        \end{cases}
        \\
        %
            &\forall k  \in [h-1], \forall (s_k, a_k, s_{k+1})
            \in (S_k \setminus \widetilde{S}^{\gamma, \beta}_k) \times A \times S_{k+1}:
            \\
            & \widetilde{P}^c(s_{k+1}| s_k, a_k)= 0,\;\;\; \widetilde{P}^c(s_{sink}| s_k, a_k) = 1.
        \end{align*}
    Clearly, by definition of $\widetilde{P}^c$,
    we have for every (context-dependent) policy $\pi$ that
    \begin{equation}\label{ineq:UCDD-l_1-prob-final}
        \mathbb{P}_c
        \left[ \forall k\in [h] , s_k \in S_k.\;\;
            q_k(s_k | \pi_c, P^c)
            \geq
            q_k(s_k | \pi_c, \widetilde{P}^c)
        \right]= 1.
    \end{equation}
    

    By Lemma~\ref{lemma: UCDD l_1 matrix diff-ver-2} under the good events $G_1,G_2^k,G_3^k \;\; \forall k \in [h-1]$  we have for any $k \in [h-1]$ that
    \begin{equation}\label{eq:prob-1-lb-occ-measure-l-1}
        \mathop{\mathbb{P}}_{(c,s_k,a_k)}
        \left[ \|\widetilde{P}^c(\cdot|s_k,a_k) - \widehat{P}^c(\cdot|s_k,a_k)\|_1 
        \leq
        \frac{4 \rho |S|}{1 - \rho^2 |S|^2} \Big| (c,s_k,a_k) \in \widetilde{\mathcal{X}}^{\gamma,\beta}_k\right]
        \geq
        1 - \frac{\epsilon_P}{\rho}|S_{k+1}|.
    \end{equation}
    
     We now show that 
    \begin{equation}\label{eq:prob-2-lb-occ-measure-l-1}
        \mathop{
        \mathbb{P}
        }_{(c,s_k,a_k)}
        \left[ 
            \|\widehat{P}^c(\cdot| s_k, a_k) - \widetilde{P}^c(\cdot| s_k, a_k)\|_1 =
            0
        \Big|
        (c, s_k, a_k) \notin  \widetilde{\mathcal{X}}^{\gamma, \beta}_k
        \right]
        =
        1.
    \end{equation}
    
    For every layer $ k \in [h-1]$ we have by definition that $(c, s_k, a_k) \in  \widetilde{\mathcal{X}}^{\gamma, \beta}_k$ if and only if $s_k \in \widetilde{S}^{\gamma, \beta}_k$ and $c \in \widehat{\mathcal{C}}^\beta(s_k)$.

    By the definition of $\widetilde{P}^c$ and $\widehat{P}^c$ we have for every layer $k \in [h-1]$ and context $c \in \mathcal{C}$ that
    \begin{equation*}
            \forall (s_k,a_k) \in (S_k \setminus \widetilde{S}^{\gamma, \beta}_k) \times A.\;\;
            \|\widehat{P}^c(\cdot|s_k,a_k) - \widetilde{P}^c(\cdot|s_k,a_k)\|_1 = 0
    \end{equation*}

    In addition, by definition of $\widehat{P}^c$ and $\widetilde{P}^c$, 
    for every layer $k \in [h-1]$
    if $(s_k,a_k) \in \widetilde{S}^{\gamma, \beta}_k \times A$ but $c \notin \widehat{\mathcal{C}}^\beta(s_k)$, then
    $$
        \|\widehat{P}^c(\cdot| s_k, a_k)- \widetilde{P}^c (\cdot| s_k, a_k)\|_1 = 0.
    $$
    
    Thus, equation~(\ref{eq:prob-2-lb-occ-measure-l-1}) follows.

    
    
    Using total probability low, equations~(\ref{eq:prob-1-lb-occ-measure-l-1}) and~(\ref{eq:prob-2-lb-occ-measure-l-1}) yield that 
    \begin{align*}
        \mathop{
        \mathbb{P}
        }_{(c,s_k,a_k)}
        \left[ 
            \|\widehat{P}^c(\cdot| s_k, a_k) - \widetilde{P}^c(\cdot| s_k, a_k)\|_1   
            \leq \frac{4 \rho |S|}{ 1 - \rho^2 |S|^2}
        \right]
        \geq
        1- \frac{\epsilon_P}{\rho^2}|S_{k+1}|,
    \end{align*}
    which by union bound over $(s_k, a_k) \in S_k \times A$ for every layer $k \in [h-1]$ implies that
    \begin{equation}\label{ineq:l_1-running-time-prob}
       \mathbb{P}_c
        \left[ 
            \forall k \in [h-1], (s_k,a_k) \in S_k \times A.\;\;
            \|\widehat{P}^c(\cdot| s_k, a_k) - \widetilde{P}^c(\cdot| s_k, a_k)\|_1   
            \leq \frac{4 \rho |S|}{ 1 - \rho^2 |S|^2}
        \right] 
        \geq
        1 - |A|\;\sum_{ k = 0 }^{h-1}\frac{\epsilon_P }{\rho}|S_k||S_{k+1}|.
    \end{equation}

    
    
    By Theorem~\ref{thm: bounding occupancy measures}
    the above yields that
     \begin{equation*}
        \mathbb{P}_c
        \left[ 
        \forall k\in [h].\;\;
         \| q_k(\cdot|\pi_c, \widetilde{P}^c) - q_k(\cdot| \pi_c, \widehat{P}^c) \|_1 
       \leq
       \frac{4 \rho |S|}{1 - \rho^2 |S|^2} k
        \right]
        \geq
         1 - |A|\;\sum_{ k = 0 }^{h-1}\frac{\epsilon_P }{\rho}|S_k||S_{k+1}|,
    \end{equation*}
    which in particularly implies that
    \begin{equation}\label{ineq:l_1-final-pron-to-combine}
        \mathbb{P}_c
        \left[ 
        \forall k\in [h],\;s_k \in s_k.\;\;
        q_k(s_k | \pi_c, \widetilde{P}^c) \geq q_k(s_k |\pi_c, \widehat{P}^c) 
        - \frac{4 \rho |S|}{1 - \rho^2 |S|^2}
        \right]
        \geq
         1 - |A|\;\sum_{ k = 0 }^{h-1}\frac{\epsilon_P }{\rho}|S_k||S_{k+1}|.
    \end{equation}
    Finally, the lemma follows by combining
    inequalities~(\ref{ineq:UCDD-l_1-prob-final}) and~(\ref{ineq:l_1-final-pron-to-combine}).

\end{proof}

\subsection{Sample Complexity Bounds}\label{sec: sampel complexity UCDD}
We show sample complexity bounds for both $\ell_1$ and $\ell_2$ loss functions. 
Recall Theorems~\ref{thm: pseudo dim} and~\ref{thm: fat dim},

\begin{theorem}[Adaption of Theorem 19.2 in~\cite{Bartlett1999NeuralNetsBook}]
    Let $\mathcal{F}$ be a hypothesis space of real valued functions with a finite pseudo dimension, denoted $Pdim(\mathcal{F}) < \infty$. Then, $\mathcal{F}$ has a uniform convergence with 
    \[
        m(\epsilon, \delta) = O \Big( \frac{1}{\epsilon^2}( Pdim(\mathcal{F}) \ln \frac{1}{\epsilon} + \ln \frac{1}{\delta})\Big).
    \]    
\end{theorem}

\begin{theorem}[Adaption of Theorem 19.1 in~\cite{Bartlett1999NeuralNetsBook}]
    Let $\mathcal{F}$ be a hypothesis space of real valued functions with a finite fat-shattering dimension, denoted $fat_{\mathcal{F}}$. Then, $\mathcal{F}$ has a uniform convergence with 
    \[
        m(\epsilon, \delta) = O \Big( \frac{1}{\epsilon^2}( fat_{\mathcal{F}}(\epsilon/256)  \ln^2 \frac{1}{\epsilon} + \ln \frac{1}{\delta})\Big).
    \]    
\end{theorem}

\begin{remark}
    In the following analysis we omit the sample complexity needed to approximate the faction of good contexts for every $s \in S$ as it is 
    \[
        O \Big( 
        \frac{|S|^3 H^2 \ln{\frac{|S|}{\delta}}}{\epsilon^2}
        \Big)
    \]
    which is negligible additional term in the following analysis.
\end{remark}

\subsubsection{Sample Complexity Bounds for the \texorpdfstring{$\ell_2$}{Lg} Loss.}\label{subsubsec:SC-L-2-UCDD}

We show sample complexity for function classes with finite Pseudo dimension with $\ell_2$ loss.
\begin{corollary}
    Assume that for every $h \in [H-1]$ we have that $Pdim(\mathcal{F}^R_{h}), Pdim(\mathcal{F}^P_{h}) < \infty$.
    Let $Pdim = \max_{h \in [H -1]} \max \{Pdim(\mathcal{F}^R_{h}), Pdim(\mathcal{F}^P_{h})\}$.
    Then, after collecting 
    \[
        O\Big( 
        \frac{|A|^2 |S|^{15} H^{13} }{\epsilon^8}\Big( Pdim \ln \frac{|A| |S|^6 H^5}{\epsilon^3} + \ln \frac{H}{\delta} \Big)\Big).
    \]
    trajectories, with probability at least $1-(\delta + \frac{\epsilon}{5})$ it holds that
    \[
        \mathbb{E}_{c \sim \mathcal{D}}[V^{\pi^\star_c}_{\mathcal{M}(c)}(s_0) - V^{\widehat{\pi}^\star_c}_{\mathcal{M}(c)}(s_0)] \leq 
       \epsilon +  2\alpha_2 H .
    \]
\end{corollary}

\begin{proof}
    Recall that for every layer $h \in [H-1]$ our algorithm run for 
    \[
        T_h =
            \left\lceil 
            \frac{8 |S|}{\gamma \cdot \beta}
            \left(\ln\frac{1}{\delta_1}
            +
            2\max\{N_P(\mathcal{F}^P_h ,\epsilon_P, \delta_1/2)), N_R(\mathcal{F}^R_h ,\epsilon_R, \delta_1/2))\}\right)
            \right\rceil.
    \]
    episodes.
    Recall our choice of parameters is 
    $\gamma = \frac{\epsilon}{20 |S|H} $, 
    $\beta =  \frac{\epsilon}{20|S|H} $, 
    $\rho = \frac{\beta}{16 |S|H} $, 
    $\epsilon_P = \frac{ \epsilon^3}{10\cdot 2^8\cdot 20^2  |A| |S|^6 H^5}$, 
    $\epsilon_R = \frac{\epsilon^3}{20^3 |S||A| H^3}$,
    $\delta_1 = \frac{\delta}{8 H}$.
    
    By Theorem~\ref{thm: UCDD opt policy l_2},  for this choice of parameters and
    $
        \sum_{h =0}^{H-1}  T_{h}
    $    
    examples we have with probability at least $1-(\delta + \frac{\epsilon}{5})$ that
    \[
        \mathbb{E}_{c \sim \mathcal{D}}[V^{\pi^\star_c}_{\mathcal{M}(c)}(s_0) - V^{\widehat{\pi}^\star_c}_{\mathcal{M}(c)}(s_0)] \leq 
        \epsilon + 2\alpha_2 H .
    \]
    Since for every $h \in [H-1]$ we have that $Pdim(\mathcal{F}^R_{h}), Pdim(\mathcal{F}^P_{h}) < \infty$ 
    and $Pdim = \max_{h \in [H -1]} \max \{Pdim(\mathcal{F}^R_{h}), Pdim(\mathcal{F}^P_{h})\}$, by Theorem~\ref{thm: pseudo dim}, for every $h \in [H-1]$ we have
    \begin{align*}
        N_P(\mathcal{F}^P_{h}, \epsilon_P, \delta_1)
        =
        O \Big( \frac{ Pdim \ln \frac{1}{\epsilon_P} + \ln \frac{1}{\delta_1}}{\epsilon^2_P}\Big)
        =  
        O \Big(\frac{|A|^2 |S|^{12} H^{10} }{\epsilon^6}\Big( Pdim \ln \frac{|A| |S|^6 H^5}{\epsilon^3} + \ln \frac{H}{\delta} \Big)\Big),
    \end{align*}
    and
    \begin{align*}
        N_R(\mathcal{F}^R_{h}, \epsilon_R, \delta_1)
        =
        O \Big( \frac{ Pdim \ln \frac{1}{\epsilon_R} + \ln \frac{1}{\delta_1}}{\epsilon^2_R}\Big)
        =  
        O \Big(\frac{|S|^2|A|^2 H^6}{\epsilon^6}\Big( Pdim \ln \frac{|S||A|H^3}{\epsilon^3} + \ln \frac{H}{\delta} \Big)\Big).
    \end{align*}
    Hence for every $h \in [H -1]$ we have that
    \[
        \max\{N_P(\mathcal{F}^P_{h}, \epsilon_P, \delta_1), N_R(\mathcal{F}^R_{h}, \epsilon_R, \delta_1)\}
        = 
        O \Big(\frac{|A|^2 |S|^{12} H^{10} }{\epsilon^6}\Big( Pdim \ln \frac{|A| |S|^6 H^5}{\epsilon^3} + \ln \frac{H}{\delta} \Big)\Big).
    \]
    
    By our choice of $\beta$ and $\gamma$ it holds that 
    \[
        T_h =
        O\Big( 
        |S|\frac{|S|^2 H^2}{\epsilon^2}
        \frac{|A|^2 |S|^{12} H^{10} }{\epsilon^6}\Big( Pdim \ln \frac{|A| |S|^6 H^5}{\epsilon^3} + \ln \frac{H}{\delta} \Big)\Big)
        =
        O\Big( 
        \frac{|A|^2 |S|^{15} H^{12} }{\epsilon^8}\Big( Pdim \ln \frac{|A| |S|^6 H^5}{\epsilon^3} + \ln \frac{H}{\delta} \Big)\Big).
    \]
    By summing the above for every layer $h \in [H-1]$ we obtain that the  sample complexity is
    \[
        O\Big( 
        \frac{|A|^2 |S|^{15} H^{13} }{\epsilon^8}\Big( Pdim \ln \frac{|A| |S|^6 H^5}{\epsilon^3} + \ln \frac{H}{\delta} \Big)\Big).
    \]
\end{proof}    

We show sample complexity for function classes with finite fat-shattering dimension with $\ell_2$ loss:
\begin{corollary}
    Assume that for every $h \in [H-1]$ we have that $\mathcal{F}^R_{h}$ and $\mathcal{F}^R_{h}$ has finite fat-shattering dimension. 
    Let $Fdim = \max_{h \in [H -1]} \max \{fat_{\mathcal{F}^R_{h}}(\epsilon_R), fat_{\mathcal{F}^P_{h}}(\epsilon_R)\}$.
    Then, after collecting 
    \[
        O\Big( 
        \frac{|A|^2 |S|^{15} H^{13} }{\epsilon^8}\Big( Fdim \ln^2 \frac{|A| |S|^6 H^5}{\epsilon^3} + \ln \frac{H}{\delta} \Big)\Big).
    \]
    trajectories, with probability at least $1-(\delta + \frac{\epsilon}{5})$ it holds that
    \[
        \mathbb{E}_{c \sim \mathcal{D}}[V^{\pi^\star_c}_{\mathcal{M}(c)}(s_0) - V^{\widehat{\pi}^\star_c}_{\mathcal{M}(c)}(s_0)] \leq 
        \epsilon + 2\alpha_2 H .
    \]
\end{corollary}

\begin{proof}
    Recall that for every layer $h \in [H-1]$ our algorithm run for 
    \[
        T_h =
        \left\lceil 
        \frac{8 |S|}{\gamma \cdot \beta}
        \left(\ln\frac{1}{\delta_1}
        +
        2\max\{N_P(\mathcal{F}^P_h ,\epsilon_P, \delta_1/2), N_R(\mathcal{F}^R_h ,\epsilon_R, \delta_1/2)\}\right)
        \right\rceil
    \]
    episodes.
    Recall our choice of parameters is 
    $\gamma = \frac{\epsilon}{20 |S|H} $, 
    $\beta =  \frac{\epsilon}{20|S|H} $, 
    $\rho = \frac{\beta}{16 |S|H} $, 
    $\epsilon_P = \frac{ \epsilon^3}{10\cdot 2^8 20^2  |A| |S|^6 H^5}$, 
    $\epsilon_R = \frac{\epsilon^3}{20^3 |S||A| H^3}$,
    $\delta_1 = \frac{\delta}{8 H}$.
    
    By Theorem~\ref{thm: UCDD opt policy l_2},  for this choice of parameters and
    $
        \sum_{h =0}^{H-1}  T_{h}
    $    
    samples we have with probability at least $1-(\delta + \frac{\epsilon}{5})$ that
    \[
        \mathbb{E}_{c \sim \mathcal{D}}[V^{\pi^\star_c}_{\mathcal{M}(c)}(s_0) - V^{\widehat{\pi}^\star_c}_{\mathcal{M}(c)}(s_0)] \leq 
        \epsilon + 2\alpha_2 H .
    \]
    Since every $h \in [H-1]$ we have that $\mathcal{F}^R_{h}$ and\\ $\mathcal{F}^R_{h}$ has finite fat-shattering dimension, and $Fdim = \max_{h \in [H -1]} \max \{fat_{\mathcal{F}^R_{h}}(\epsilon_R), fat_{\mathcal{F}^P_{h}}(\epsilon_R)\}$, by Theorem~\ref{thm: fat dim}, for every $h \in [H-1]$ we have
    \begin{align*}
        N_P(\mathcal{F}^P_{h}, \epsilon_P, \delta_1)
        =
        O \Big( \frac{ Pdim \ln \frac{1}{\epsilon_P} + \ln \frac{1}{\delta_1}}{\epsilon^2_P}\Big)
        =  
        O \Big(\frac{|A|^2 |S|^{12} H^{10} }{\epsilon^6}\Big( Fdim \ln^2 \frac{|A| |S|^6 H^5}{\epsilon^3} + \ln \frac{H}{\delta} \Big)\Big),
    \end{align*}
    and
    \begin{align*}
        N_R(\mathcal{F}^R_{h}, \epsilon_R, \delta_1)
        =
        O \Big( \frac{ Pdim \ln \frac{1}{\epsilon_R} + \ln \frac{1}{\delta_1}}{\epsilon^2_R}\Big)
        =  
        O \Big(\frac{|S|^2 |A|^2 H^6 }{\epsilon^6}\Big( Fdim \ln^2 \frac{|S||A|H^3}{\epsilon^3} + \ln \frac{H}{\delta} \Big)\Big).
    \end{align*}
    Hence for every $h \in [H -1]$ we have that
    \[
        \max\{N_P(\mathcal{F}^P_{h}, \epsilon_P, \delta_1), N_R(\mathcal{F}^R_{h}, \epsilon_R, \delta_1)\}
        = 
        O \Big(\frac{|A|^2 |S|^{12} H^{10} }{\epsilon^6}\Big( Fdim \ln^2 \frac{|A| |S|^6 H^5}{\epsilon^3} + \ln \frac{H}{\delta} \Big)\Big).
    \]
    
     By our choice of $\beta$ and $\gamma$ it holds that
    \[
        T_h =
        O\Big( 
        |S|\frac{|S|^2 H^2}{\epsilon^2}
        \frac{|A|^2 |S|^{12} H^{10} }{\epsilon^6}\Big( Fdim \ln^2 \frac{|A| |S|^6 H^5}{\epsilon^3} + \ln \frac{H}{\delta} \Big)\Big)
        =
        O\Big( 
        \frac{|A|^2 |S|^{15} H^{12} }{\epsilon^8}\Big( Fdim \ln^2 \frac{|A| |S|^6 H^5}{\epsilon^3} + \ln \frac{H}{\delta} \Big)\Big).
    \]
    By summing the above for every layer $h \in [H-1]$ we obtain that the  sample complexity is
    \[
        O\Big( 
        \frac{|A|^2 |S|^{15} H^{13} }{\epsilon^8}\Big( Fdim \ln^2 \frac{|A| |S|^6 H^5}{\epsilon^3} + \ln \frac{H}{\delta} \Big)\Big).
    \]
\end{proof}

\subsubsection{Sample Complexity Bounds for the \texorpdfstring{$\ell_1$}{Lg} Loss}\label{subsubsec:SC-L-1-UCDD}

We show sample complexity for function classes with finite Pseudo dimension with $\ell_1$ loss using the set of parameters SP1.
\begin{corollary}
    Assume that for every $h \in [H-1]$ we have that $Pdim(\mathcal{F}^R_{h}), Pdim(\mathcal{F}^P_{h}) < \infty$.
    Let $Pdim = \max_{h \in [H -1]} \max \{Pdim(\mathcal{F}^R_{h}), Pdim(\mathcal{F}^P_{h})\}$.
    Then, after collecting 
    \[
        O\Big( 
        \frac{|A|^2 |S|^{11} H^9 }{\epsilon^6}
        \Big( Pdim \ln \frac{|A| |S|^4 H^3}{\epsilon^2} + \ln \frac{H}{\delta} \Big)\Big)
    \]
    trajectories, with probability at least $1-(\delta + \frac{\epsilon}{5})$ it holds that
    \[
        \mathbb{E}_{c \sim \mathcal{D}}[V^{\pi^\star_c}_{\mathcal{M}(c)}(s_0) - V^{\widehat{\pi}^\star_c}_{\mathcal{M}(c)}(s_0)] \leq 
        2\alpha_1 H + \epsilon.
    \]
\end{corollary}

\begin{proof}
    Recall that for every layer $h \in [H-1]$ our algorithm run for 
    \[
        T_h =
        \left\lceil 
        \frac{8 |S|}{\gamma \cdot \beta}
        \left(\ln\frac{1}{\delta_1}
        +
        2\max\{N_P(\mathcal{F}^P_h ,\epsilon_P, \delta_1/2), N_R(\mathcal{F}^R_h ,\epsilon_R, \delta_1/2)\}\right)
        \right\rceil
    \]
    episodes.
    Recall SP1 parameters are 
    $\gamma = \frac{\epsilon}{20 |S|H} $,
    $\beta =  \frac{\epsilon}{20|S|H} $, 
    $\rho = \frac{\beta}{16 |S|H} $,  
    $\epsilon_P = \frac{ \epsilon^2}{10\cdot 16 \cdot 20  |A| |S|^4 H^3}$,
    $\epsilon_R = \frac{\epsilon^2}{20^2 |S||A| H^2}$
    $\delta_1 = \frac{\delta}{8 H}$.
    
    By Theorem~\ref{thm: UCDD opt policy l_1},  for this choice of parameters and
    $
        \sum_{h =0}^{H-1}  T_{h}
    $    
    examples we have with probability at least $1-(\delta + \frac{\epsilon}{5})$ that
    \[
        \mathbb{E}_{c \sim \mathcal{D}}[V^{\pi^\star_c}_{\mathcal{M}(c)}(s_0) - V^{\widehat{\pi}^\star_c}_{\mathcal{M}(c)}(s_0)] \leq 
        2\alpha_1 H + \epsilon.
    \]
    Since every $h \in [H-1]$ we have that $Pdim(\mathcal{F}^R_{h}), Pdim(\mathcal{F}^P_{h}) < \infty$.
    and $Pdim = \max_{h \in [H -1]} \max \{Pdim(\mathcal{F}^R_{h}), Pdim(\mathcal{F}^P_{h})\}$, by Theorem~\ref{thm: pseudo dim}, for every $h \in [H-1]$ we have
    \begin{align*}
        N_P(\mathcal{F}^P_{h}, \epsilon_P, \delta_1)
        =
        O \Big( \frac{ Pdim \ln \frac{1}{\epsilon_P} + \ln \frac{1}{\delta_1}}{\epsilon^2_P}\Big)
        =  
        O \Big(\frac{|A|^2 |S|^8 H^6 }{\epsilon^4}\Big( Pdim \ln \frac{|A| |S|^4 H^3}{\epsilon^2} + \ln \frac{H}{\delta} \Big)\Big),
    \end{align*}
    and
    \begin{align*}
        N_R(\mathcal{F}^R_{h}, \epsilon_R, \delta_1)
        =
        O \Big( \frac{ Pdim \ln \frac{1}{\epsilon_R} + \ln \frac{1}{\delta_1}}{\epsilon^2_R}\Big)
        =  
        O \Big(\frac{|S|^2 |A|^2 H^4 }{\epsilon^4}\Big( Pdim \ln \frac{|S||A|H^2}{\epsilon^2} + \ln \frac{H}{\delta} \Big)\Big).
    \end{align*}
    Hence for every $h \in [H -1]$ we have that
    \[
        \max\{N_P(\mathcal{F}^P_{h}, \epsilon_P, \delta_1), N_R(\mathcal{F}^R_{h}, \epsilon_R, \delta_1)\}
        = 
        O \Big(\frac{|A|^2 |S|^8 H^6 }{\epsilon^4}\Big( Pdim \ln \frac{|A| |S|^4 H^3}{\epsilon^2} + \ln \frac{H}{\delta} \Big)\Big).
    \]
    
    By our choice of $\beta$ and $\gamma$ it holds that
    \[
        T_h =
        O\Big( 
        |S|\frac{|S|^2 H^2}{\epsilon^2}
        \frac{|A|^2 |S|^8 H^6 }{\epsilon^4}\Big( Pdim \ln \frac{|A| |S|^4 H^3}{\epsilon^2} + \ln \frac{H}{\delta} \Big)\Big)
        =
        O\Big( 
        \frac{|A|^2 |S|^{11} H^8 }{\epsilon^6}\Big( Pdim \ln \frac{|A| |S|^4 H^3}{\epsilon^2} + \ln \frac{H}{\delta} \Big)\Big).
    \]
    By summing the above for every layer $h \in [H-1]$ we obtain that the  sample complexity is
    \[
        O\Big( 
        \frac{|A|^2 |S|^{11} H^9 }{\epsilon^6}\Big( Pdim \ln \frac{|A| |S|^4 H^3}{\epsilon^2} + \ln \frac{H}{\delta} \Big)\Big).
    \]
\end{proof}    

We show sample complexity for function classes with finite fat-shattering dimension with $\ell_1$ loss:
\begin{corollary}
    Assume that for every $h \in [H-1]$ we have that $\mathcal{F}^R_{h}$ and $\mathcal{F}^R_{h}$ has finite fat-shattering dimension. 
    Let $Fdim = \max_{h \in [H -1]} \max \{fat_{\mathcal{F}^R_{h}}(\epsilon_R), fat_{\mathcal{F}^P_{h}}(\epsilon_R)\}$.
    Then, after collecting 
    \[
        O\Big( 
        \frac{|A|^2 |S|^{11} H^9 }{\epsilon^6}\Big( Fdim \ln^2 \frac{|A| |S|^4 H^3}{\epsilon^2} + \ln \frac{H}{\delta} \Big)\Big).
    \]
    examples, with probability at least $1-(\delta + \frac{\epsilon}{5})$ it holds that
    \[
        \mathbb{E}_{c \sim \mathcal{D}}[V^{\pi^\star_c}_{\mathcal{M}(c)}(s_0) - V^{\widehat{\pi}^\star_c}_{\mathcal{M}(c)}(s_0)] \leq 
        2\alpha_1 H + \epsilon.
    \]
\end{corollary}

\begin{proof}
    Recall that for every layer $h \in [H-1]$ our algorithm run for 
    \[
        T_h =
        \left\lceil 
        \frac{8 |S|}{\gamma \cdot \beta}
        \left(\ln\frac{1}{\delta_1}
        +
        2\max\{N_P(\mathcal{F}^P_h ,\epsilon_P, \delta_1/2), N_R(\mathcal{F}^R_h ,\epsilon_R, \delta_1/2)\}\right)
        \right\rceil
    \]    
    episodes.
    Recall our choice of parameters is 
    $\gamma = \frac{\epsilon}{20 |S|H} $,
    $\beta =  \frac{\epsilon}{20|S|H} $, 
    $\rho = \frac{\beta}{16 |S|H} $,  
    $\epsilon_P = \frac{ \epsilon^2}{10\cdot 16 \cdot 20  |A| |S|^4 H^3}$,
    $\epsilon_R = \frac{\epsilon^2}{20^2 |S| |A| H^2}$,
    $\delta_1 = \frac{\delta}{8 H}$.
    
    By Theorem~\ref{thm: UCDD opt policy l_1},  for this choice of parameters and
    $
        \sum_{h =0}^{H-1}  T_{h}
    $    
    samples we have with probability at least $1-(\delta + \frac{\epsilon}{5})$ that
    \[
        \mathbb{E}_{c \sim \mathcal{D}}[V^{\pi^\star_c}_{\mathcal{M}(c)}(s_0) - V^{\widehat{\pi}^\star_c}_{\mathcal{M}(c)}(s_0)] \leq 
        2\alpha_1 H + \epsilon.
    \]
    Since every $h \in [H-1]$ we have that $\mathcal{F}^R_{h}$ and $\mathcal{F}^R_{h}$ has finite fat-shattering dimension, and \\
    $Fdim = \max_{h \in [H -1]} \max \{fat_{\mathcal{F}^R_{h}}(\epsilon_R), fat_{\mathcal{F}^P_{h}}(\epsilon_R)\}$, by Theorem~\ref{thm: fat dim}, for every $h \in [H-1]$ we have
       \begin{align*}
        N_P(\mathcal{F}^P_{h}, \epsilon_P, \delta_1)
        =
        O \Big( \frac{ Fdim \ln^2 \frac{1}{\epsilon_P} + \ln \frac{1}{\delta_1}}{\epsilon^2_P}\Big)
        =  
        O \Big(\frac{|A|^2 |S|^8 H^6 }{\epsilon^4}\Big( Fdim \ln^2 \frac{|A| |S|^4 H^3}{\epsilon^2} + \ln \frac{H}{\delta} \Big)\Big),
    \end{align*}
    and
    \begin{align*}
        N_R(\mathcal{F}^R_{h}, \epsilon_R, \delta_1)
        =
        O \Big( \frac{ Fdim \ln^2 \frac{1}{\epsilon_R} + \ln \frac{1}{\delta_1}}{\epsilon^2_R}\Big)
        =  
        O \Big(\frac{|S|^2 |A|^2 H^4 }{\epsilon^4}\Big( Fdim \ln^2 \frac{|S| |A| H^2}{\epsilon^2} + \ln \frac{H}{\delta} \Big)\Big).
    \end{align*}
    Hence for every $h \in [H -1]$ we have that
    \[
        \max\{N_P(\mathcal{F}^P_{h}, \epsilon_P, \delta_1), N_R(\mathcal{F}^R_{h}, \epsilon_R, \delta_1)\}
        = 
        O \Big(\frac{|A|^2 |S|^8 H^6 }{\epsilon^4}\Big( Fdim \ln^2 \frac{|A| |S|^4 H^3}{\epsilon^2} + \ln \frac{H}{\delta} \Big)\Big).
    \]
    
    By our choice of $\beta$ and $\gamma$ it holds that
    \[
        T_h =
        O\Big( 
        |S|\frac{|S|^2 H^2}{\epsilon^2}
        \frac{|A|^2 |S|^8 H^6 }{\epsilon^4}\Big( Fdim \ln^2 \frac{|A| |S|^4 H^3}{\epsilon^2} + \ln \frac{H}{\delta} \Big)\Big)
        =
        O\Big( 
        \frac{|A|^2 |S|^{11} H^8 }{\epsilon^6}\Big( Fdim \ln^2 \frac{|A| |S|^4 H^3}{\epsilon^2} + \ln \frac{H}{\delta} \Big)\Big).
    \]
    By summing the above for every layer $h \in [H-1]$ we obtain that the  sample complexity is
    \[
        O\Big( 
        \frac{|A|^2 |S|^{11} H^9 }{\epsilon^6}\Big( Fdim \ln^2 \frac{|A| |S|^4 H^3}{\epsilon^2} + \ln \frac{H}{\delta} \Big)\Big).
    \]
\end{proof}

\subsection{Useful Lemmas and Theorems}


\begin{definition}
    Let $M$ be $n \times m$ matrix. Let us denote
    \[
        \| M \|_{1, \infty} 
        = 
        \max_{i\in [n]} \sum_{j=0}^{m-1} |M_{ij}|
    \]
\end{definition}

The following is a known theorem from Markov Chains theory.
\begin{theorem}\label{thm: bounding occupancy measures}
    Let 
    $P_1(\cdot|\cdot, \cdot): S \times (S \times A) \to [0,1]$ 
    and  
    $P_2(\cdot|\cdot, \cdot): S \times (S \times A) \to [0,1]$
    denote two transition probabilities functions of two MDPs.
    Let 
    $\pi(\cdot| \cdot): A \times S \to [0,1]$ 
    be a policy such that the induced Markov chains
    $P^\pi_1(s'|s) = \langle \pi(\cdot |s) , P_1(s'|s,\cdot) \rangle $ 
    and  
    $P^\pi_2(s'|s) = \langle \pi(\cdot |s) , P_2(s'|s,\cdot) \rangle$ satisfies for some $\alpha \geq 0$ that 
    \[
        \| P^\pi_1 - P^\pi_2 \|_{1, \infty}
        \leq \alpha
        .
    \]
    Let $q^h_1(\cdot| \pi)$ and $q^h_2(\cdot| \pi)$ be the distribution over states after trajectories of length $h$ of $P^\pi_1$ and $P^\pi_2$ respectively. Then,
    \[
        \| q^h_1 - q^h_2 \|_1 \leq \alpha h.
    \]
\end{theorem}

\end{document}